\RequirePackage{etoolbox}

\newcommand{\type}{report}

\newcommand{\verbosity}{long}

\newcommand{\mode}{final}

\newcommand{\apx}{proofs}

\newcommand{\numberlevel}{chapter}

\newcommand{\counterlevel}{same}

\newcommand{\colorthm}{true}

\newif\ifsamecounter\ifdefstring{\counterlevel}{same}{\samecountertrue}{\samecounterfalse}

\newif\iflong\ifdefstring{\verbosity}{long}{\longtrue}{\longfalse}
\newif\ifshort\ifdefstring{\verbosity}{long}{\shortfalse}{\shorttrue}

\newif\iffinal\ifdefstring{\mode}{final}{\finaltrue}{\finalfalse}

\newif\ifdraft\ifdefstring{\mode}{draft}{\drafttrue}{\draftfalse}

\newif\ifreview\ifdefstring{\mode}{review}{\reviewtrue}{\reviewfalse}

\newif\ifsubmission\ifdefstring{\mode}{submission}{\submissiontrue}{\submissionfalse}

\newcommand{\contact}[1]{}
\newcommand{\affiliation}[1]{\\#1}

\ifdefstring{\type}{conference}{%

    \documentclass{article} 

}{\ifdefstring{\type}{journal}{%

        \documentclass{report} 

    }{

        \newif\ifconf\conffalse

        \documentclass[12pt,a4paper,twoside,openright]{book}

        \setcounter{secnumdepth}{3}

        \usepackage[
        hcentering,
        margin=1in,
        ]{geometry}

        \usepackage{libertine}

        \emergencystretch 3em

}}

\makeatletter
\patchcmd{\@chapter}
{\chaptermark{#1}}
{\chaptermark{#1}%
    \addtocontents{loa}{\protect\addvspace{10\p@}}}
{}{}
\makeatother

\usepackage{amsmath, amsthm, amssymb, thmtools}

\usepackage[
]{tocbibind}
\usepackage{tocbasic}
\DeclareTOCStyleEntry[
beforeskip=0.5em plus 1pt,
]{tocline}{chapter}

\usepackage{xcolor}

\usepackage[
colorlinks = true,
linkcolor = blue!80!black,
urlcolor  = green!50!black,
citecolor = blue!80!black,
anchorcolor = blue!80!black,
]{hyperref}
\usepackage{hyperxmp}

\usepackage{xstring}

\usepackage{subcaption}

\usepackage{blindtext}

\usepackage[shortlabels]{enumitem}

\usepackage[
\ifdraft\else
disable,
\fi
colorinlistoftodos,
prependcaption,
textsize=tiny,
color=red!20,
linecolor=red!20,
shadow,
backgroundcolor=red!25,
bordercolor=red
]{todonotes}
\presetkeys{todonotes}{fancyline}{}

\ifreview
\usepackage{lineno}\linenumbers
\BeforeBeginEnvironment{dmath}{\begin{nolinenumbers}}%
    \AfterEndEnvironment{dmath}{\end{nolinenumbers}}
\fi

\iffinal\else\ifsubmission\else
\usepackage{draftwatermark}
\ifdraft
\SetWatermarkText{Draft Copy - Do Not Distribute}
\SetWatermarkColor[rgb]{0.8,0,0}
\else\ifreview
\SetWatermarkText{Review Copy - Do Not Distribute}
\SetWatermarkColor[rgb]{0,0,0.8}
\fi\fi
\SetWatermarkFontSize{2em}
\SetWatermarkAngle{0.0}
\SetWatermarkHorCenter{\paperwidth/2}
\SetWatermarkVerCenter{2em}
\SetWatermarkScale{1.0}
\fi\fi

\usepackage{tikz}
\usetikzlibrary{automata,positioning,shapes,shapes.geometric,arrows,fit,calc}
\usetikzlibrary{decorations.pathmorphing}

\usepackage{pgfplots}

\usepackage{nameref}

\usepackage{cases}

\usepackage[\numberlevel]{algorithm}
\usepackage[noend]{algpseudocode}
\algnewcommand\Input{\item[{\textbf{Input:}}]}
\algnewcommand\Output{\item[{\textbf{Output:}}]}
\algnewcommand\And{\textbf{and} }

\ifdefempty{\apx}{}{
}

\usepackage{flexisym}
\usepackage{breqn}

\usepackage[autostyle]{csquotes}
\usepackage[
natbib=true,
backend=bibtex,
style=alphabetic,
doi=false,
eprint=false,
isbn=false,
]{biblatex}
\AtEveryBibitem{%
    \clearfield{pages}
    \clearfield{volume}
    \clearfield{number}
    \ifentrytype{online}{}{
        \clearfield{url}
    }
}
\usepackage{cleveref}

\usepackage[many]{tcolorbox}
\tcbset{
    breakable,
    enhanced,
    before skip=\topsep,
    after skip=\topsep,
}

\newif\ifuc\ucfalse

\makeatletter

\ifsamecounter

\newtheorem{thm}{Should Not Be Visible}[\numberlevel]
\ifx\c@theorem\undefined \newtheorem{theorem}[thm]{Theorem}\fi
\ifx\c@lemma\undefined \newtheorem{lemma}[thm]{Lemma} \fi
\ifx\c@proposition\undefined \newtheorem{proposition}[thm]{Proposition} \fi
\ifx\c@definition\undefined \newtheorem{definition}[thm]{Definition} \fi
\ifx\c@remark\undefined \newtheorem{remark}[thm]{Remark} \fi
\ifx\c@corollary\undefined \newtheorem{corollary}[thm]{Corollary} \fi
\ifx\c@example\undefined \newtheorem{example}[thm]{Example}\fi
\ifx\c@assumption\undefined \newtheorem{assumption}[thm]{Assumption}\fi
\ifx\c@conjecture\undefined \newtheorem{conjecture}[thm]{Conjecture}\fi

\else

\ifx\c@theorem\undefined \newtheorem{theorem}{Theorem}[\numberlevel] \fi
\ifx\c@lemma\undefined \newtheorem{lemma}{Lemma}[\numberlevel] \fi
\ifx\c@proposition\undefined \newtheorem{proposition}{Proposition}[\numberlevel] \fi
\ifx\c@definition\undefined \newtheorem{definition}{Definition}[\numberlevel] \fi
\ifx\c@remark\undefined \newtheorem{remark}{Remark}[\numberlevel] \fi
\ifx\c@corollary\undefined \newtheorem{corollary}{Corollary}[\numberlevel] \fi
\ifx\c@example\undefined  \fi
\ifx\c@assumption\undefined \newtheorem{assumption}{Assumption}[\numberlevel] \fi
\ifx\c@conjecture\undefined \newtheorem{conjecture}{Conjecture}[\numberlevel] \fi

\fi

\makeatother

\ifdefstring{\colorthm}{true}{

    \let\thmOrg\theorem
    \let\endthmOrg\endtheorem
    \renewenvironment{theorem}{\begin{tcolorbox}\thmOrg}{\endthmOrg\end{tcolorbox}}

    \let\lemOrg\lemma
    \let\endlemOrg\endlemma
    \renewenvironment{lemma}{\begin{tcolorbox}\lemOrg}{\endlemOrg\end{tcolorbox}}

    \let\proOrg\proposition
    \let\endproOrg\endproposition
    \renewenvironment{proposition}{\begin{tcolorbox}\proOrg}{\endproOrg\end{tcolorbox}}

    \let\defOrg\definition
    \let\enddefOrg\enddefinition
    \renewenvironment{definition}{\begin{tcolorbox}\defOrg}{\enddefOrg\end{tcolorbox}}

    \let\corOrg\corollary
    \let\endcorOrg\endcorollary
    \renewenvironment{corollary}{\begin{tcolorbox}\corOrg}{\endcorOrg\end{tcolorbox}}

    \let\asmOrg\assumption
    \let\endasmOrg\endassumption
    \renewenvironment{assumption}{\begin{tcolorbox}\asmOrg}{\endasmOrg\end{tcolorbox}}

    \let\conOrg\conjecture
    \let\endconOrg\endconjecture
    \renewenvironment{conjecture}{\begin{tcolorbox}\conOrg}{\endconOrg\end{tcolorbox}}

    \let\rmkOrg\remark
    \let\endrmkOrg\endremark
    \renewenvironment{remark}{\begin{tcolorbox}\rmkOrg}{\endrmkOrg\end{tcolorbox}}

}{}

\usepackage{agi-research}

\usepackage{setspace}
\onehalfspace

\usepackage{aligned-overset}

\usepackage[object=am]{pgfornament}

\usepackage[
type={CC},
modifier={by},
version={4.0},
]{doclicense}

\usepackage{titlesec}

\usepackage{fancyhdr}

\renewcommand{\chaptermark}[1]{\markboth{#1}{}}

\usepackage{epigraph}
\usepackage{diagbox}
\usepackage{rotating}

\usepackage{booktabs}

\usepackage{stackengine}
\setstackgap{S}{1pt}

\usetikzlibrary{matrix}
\usetikzlibrary{patterns,shapes.geometric}

\newlist{notation}{itemize}{1}
\setlist[notation,1]{label=,labelwidth=1in,align=parleft,itemsep=0.01\baselineskip,leftmargin=!}

\newcommand{\cyl}[1]{{\rm Cyl}(#1)}

\newcommand{\AxisRotator}[1][rotate=0]{%
    \tikz [x=0.25cm,y=0.60cm,line width=.2ex,-stealth,#1] \draw (0,0) arc (-150:150:1 and 1);%
}

\tikzset{
    >=stealth',
    roundbox/.style={
        rectangle,
        rounded corners,
        draw=black, very thick,
        text width=6.5em,
        minimum height=2em,
        text centered},
    thickarrow/.style={
        ->,
        thick,
        shorten <=2pt,
        shorten >=2pt,}
}

\tikzset{
    node distance = 7mm and -3mm,
    innernode/.style = {draw=black, thick, fill=gray!30,
        minimum width=2cm, minimum height=0.5cm,
        align=center},
    outernode/.style = {draw=black, thick, rounded corners, fill=none,
        minimum width=1cm, minimum height=0.5cm,
        align=center, inner sep=0.5cm},
    endpoint/.style={draw,circle,
        fill=gray, inner sep=0pt, minimum width=4pt},
    arrow/.style={->,thick,rounded corners},
    point/.style={circle,inner sep=0pt,minimum size=2pt,fill=black},
    skip loop/.style={to path={-- ++(#1,0) |- (\tikztotarget)}},
    every path/.style = {draw, -latex}
}

\tikzset{
    >=stealth',
    roundbox/.style={
        rectangle,
        rounded corners,
        draw=black, very thick,
        text width=6.5em,
        minimum height=2em,
        text centered},
    thickarrow/.style={
        ->,
        thick,
        shorten <=2pt,
        shorten >=2pt,},
    graycircle/.style={
        circle,
        draw,
        fill=gray!30,
        inner sep=0,
        thick,
        minimum size=14mm}
}

\newenvironment{outline}{\begin{center}\bfseries\large Outline \end{center}\begin{quote}}{\end{quote}}

\titlespacing{\chapter}{0pt}{0ex}{*4}
\titleformat{\chapter}[display]
{\bfseries\Large}
{\filleft\MakeUppercase{\chaptertitlename} \Huge\thechapter}
{2ex}
{\titlerule[1pt]\vspace{1pt}\titlerule
    \vspace{1ex}%
    \centering
}
[\vspace{1ex}%
{\titlerule}]

\newcommand{\MONTH}{%
    \ifcase\the\month
    \or January
    \or February
    \or March
    \or April
    \or May
    \or June
    \or July
    \or August
    \or September
    \or October
    \or November
    \or December
    \fi}

\newcommand*{\blankpage}{%
    \vspace*{\fill}
    {\centering\footnotesize\itshape This page is intentionally left blank, no more.\par}
    }
\makeatletter
\renewcommand*{\cleardoublepage}{\clearpage\if@twoside \ifodd\c@page\else
    \blankpage
    \thispagestyle{empty}
    \newpage
    \if@twocolumn\hbox{}\newpage\fi\fi\fi}
\makeatother

\newcommand{\iparadot}{\item\paradot}

\def\frs#1#2{{^{#1}\!/\!_{#2}}} 

\def\DPi{\eps_\Pi}
\def\DQ{\eps_Q}
\def\piR{{\u\pi_R}}

\def\mse{{\rm MSE}}

\newcommand{\T}{\mathscr T}

\title{
    Abstractions of General Reinforcement Learning
}
\author{
    Sultan J. Majeed \contact{sultan.pk} \affiliation{Research School of Computer Science, ANU}
}

\begin{document}


\ifdefempty{\mode}{
    \pagestyle{empty}
    \newpage
    \listoftodos[Todo \& Notes]
    \newpage
    \pagestyle{plain}
    \setcounter{page}{1}
}{}



\begin{titlepage}
    \centering
    {\LARGE Abstractions of\\}
    {\LARGE General Reinforcement Learning\\[1ex]}
    {\large\itshape An inquiry into the scalability of generally intelligent agents}
    \vfill
    {\Large\scshape\href{http://sultan.pk}{Sultan J. Majeed}}
    \\[1ex]
    \rule{29ex}{1pt}\\[1ex]
    {\scshape Supervised By\\[1.5ex]}
    {\Large\scshape\href{http://hutter1.net}{Prof. Marcus Hutter}}
    \vfill

    \includegraphics[scale=1.0]{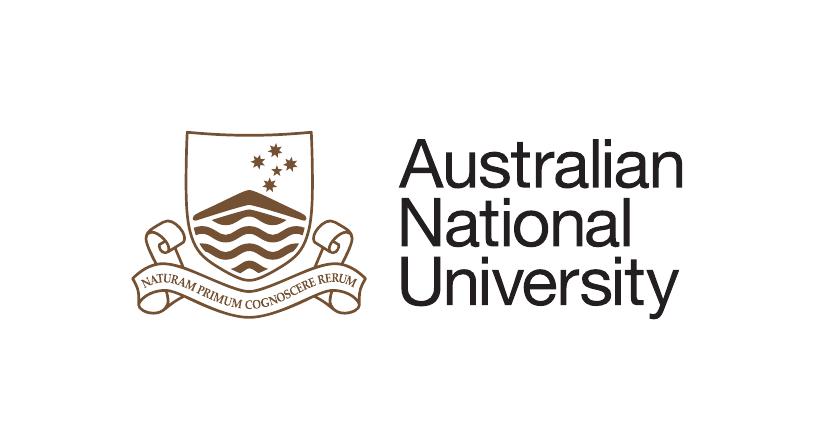}

    {\itshape \large A thesis submitted for the degree of Doctor of Philosophy at the Australian National University\\[1.5ex]}
    {\large August 2021}
\end{titlepage}


\newpage
\thispagestyle{empty}
\vspace*{\fill}

{\centering Copyright \copyright \ 2021 \ Sultan J. Majeed Some Rights Reserved.\par}

\doclicenseThis

\cleardoublepage
\thispagestyle{empty}
\vspace*{\fill}
{\itshape Unless referenced otherwise, this thesis is my own original work.}

{\flushright

{--- Sultan J. Majeed \\ August 23, 2021 \,}

\vspace{\fill}}


\cleardoublepage
\thispagestyle{empty}
\vspace*{\fill}

\begin{center}
   {\itshape
   To Ammi and Abu
%
%
%
%
%
   }
%
%

\end{center}

\vspace{\fill}


\cleardoublepage
\thispagestyle{empty}
\vspace*{\fill}
\begin{center}
    \bfseries\Large I would like to acknowledge my ...
\end{center}
\begin{quote}
\begin{itemize}
    \item[Supervisor] Marcus Hutter for his fathomless wisdom, constant support, and guidance which not only shaped my understanding of the subject matter, but he also made me think critically about the science as a whole. He is, and will always be, a role model for me to look up to!
    \item[Wife] Fatima who got out of her comfort zone to accompany me on this adventure to the down under. Her ardent love and unshaken trust helped me to pull myself together during the gloomy times all these years.
    \item[Kids] Zainab and Umamah, the only biological and generally intelligent agents I had privilege to create, who made it clear that specifying the goals and training such agents is not so easy task!
    \item[Family] who remotely shared a part of their lives with us by making the time to join us during the lows and the highs, which would have been unbearable and colorless otherwise.
    \item[Friends] especially, Elliot Catt, Samuel Yang-Zhao, Jan Leike, Noman Akbar, Matthew Aitchison and Micheal Cohen, without whom Straya would have been a lonely place!
    \item[College Admins] especially Christie Liu and Jasmine Jury, who made the administration side of the degree a breeze.
    \item[Funding Agencies] Islamic Development Bank (IsDB) and Australian Research Council (ARC) for the sufficient financial support to fund my degree and travels.
\end{itemize}
\end{quote}
\vspace{\fill}

\frontmatter


\chapter{Abstract}

\noindent The field of artificial intelligence (AI) is devoted to the creation of artificial decision-makers that can perform (at least) on par with the human counterparts on a domain of interest. Unlike the agents in traditional AI, the agents in artificial general intelligence (AGI) are required to replicate human intelligence in \emph{almost every} domain of interest. Moreover, an AGI agent should be able to achieve this without (virtually any) further changes, retraining, or fine-tuning of the parameters.
The real world is non-stationary, non-ergodic, and non-Markovian: we, humans, can neither revisit our past nor are the most recent observations sufficient statistics. Yet, we excel at a variety of complex tasks. Many of these tasks require longterm planning. We can associate this success to our natural faculty to \emph{abstract} away \emph{task-irrelevant} information from our overwhelming sensory experience. We make \emph{task-specific} mental models of the world without much effort. Due to this ability to abstract, we can plan on a significantly \emph{compact representation} of a task without much \emph{loss of performance}.
Not only this, we also abstract our actions to produce high-level plans: the level of action-abstraction can be anywhere between small muscle movements to a mental notion of ``doing an action''. It is natural to assume that any AGI agent competing with humans (at every plausible domain) should also have these abilities to abstract its experiences and actions.
This thesis is an inquiry into the existence of such abstractions which aid efficient planing for a wide range of domains, and most importantly, these abstractions come with some \emph{optimality guarantees}.
We use a \emph{history-based} reinforcement learning (RL) setup, appropriately called \emph{general reinforcement learning} (GRL), to model such general-purpose decision-makers. We show that if such GRL agents have access to appropriate abstractions then they can perform optimally in a huge set of domains. That is, we argue that \emph{GRL with abstractions}, called abstraction reinforcement learning (ARL), is an appropriate framework to model and analyze AGI agents.
This work uses and extends beyond a powerful class of (state-only) abstractions called extreme state abstractions (ESA). We analyze a variety of such extreme abstractions, both state-only and state-action abstractions, to formally establish the \emph{representation and convergence guarantees}. We also make many minor contributions to the ARL framework along the way. Last but not least, we collect a series of ideas that lay the foundations for designing the (extreme) \emph{abstraction learning} algorithms.

\vfill

\paradot{Keywords} General Reinforcement Learning (GRL), Abstraction Reinforcement Learning (ARL), Feature Reinforcement Learning (FRL), Representation Learning, (State-only) Abstractions, Homomorphisms (State-action Abstractions), Non-Markovian Decision Processes (NMDP), Q-uniform Decision Processes (QDP), Value \& Policy-uniform Decision Processes (VPDP), Scalability, Feature Construction, Representation and Convergence Optimality Guarantees.


\tableofcontents
\listoffigures
\listoftables
\listofalgorithms
\addcontentsline{toc}{chapter}{List of Algorithms}


\chapter{Notation}

\noindent This thesis is notation heavy, but we use a consistent notation throughout. A comprehensive list of notation and (important) symbols is presented in \Cref{chap:notation-and-symbols}.

\begin{itemize}

    \iparadot{General Notation}
    The set of natural numbers is $\SetN := \{1, 2, \dots\}$, $\SetB := \{0,1\}$ is a set of binary symbols, $\SetR$ is the set of reals, and $\SetRX \coloneqq \SetR \cup \{-\infty, +\infty\}$ denotes the set of extended real numbers. We denote by $\Dist(X)$ the set of probability distributions over any set $X$. The concatenation of two objects (or strings) is expressed through juxtaposition, e.g.\ $xy$ is a concatenation of $x$ and $y$. We express a finite string with boldface, e.g.\ $\v x = x_1 x_2 \dots x_{\abs{\v x}}$ where $\abs{{}\cdot{}}$ is used to denote the length or cardinality of the object. The individual members of a string or a vector may be accessed as $\v x_n = x_n$ for any $i \leq \abs{\v x}$. A substring of length $i \leq \abs{\v x}$ is denoted as $\v x_{\leq n} = x_1x_2\dots x_{n}$ and $\v x_{<n} = x_1x_2\dots x_{n-1}$. The empty string is denoted by $\epsilon$. We interchangeably use the same notation for vectors and strings, e.g.\ $\v x \in \B^d$ is a $d$-dimensional $\B$-ary vector which may also be expressed as a string. This choice simplifies the notation and saves redundant variables. If a variable is time-indexed, we express the continuation of the variable with a prime on it, e.g.\ if $x \coloneqq x_{n}$ then $x' \coloneqq x_{n+1}$ where $\coloneqq$ denotes equality by definition. A small scalar value (usually the error tolerance) is denoted by $\eps > 0$.
    A different member of the same set is expressed with a dot on it, e.g.\ $x, \d x \in \B$. We express the fact of $\v x$ being a prefix of $\v y$, i.e.\ $x_n = y_n$ for all $n \leq \abs{\v x}$, by $\v x \sqsubseteq \v y$ or $\v y \sqsupseteq \v x$. Moreover, $\vo{xy}$ represents a vector that is point-wise joined, i.e.\ $\vo{xy}_n \coloneqq \v x_n \v y_n$. We use a shorthand notation $\forall f(x)=y$ to mean $\forall x, y: f(x)=y$. The power set is expressed as $\PowSet(X) \coloneqq \{ A \mid A \subseteq X\}$ that is the set of all subsets of $X$. We use $f_{\langle x, y \rangle}(z)$ to denote a function $f$ of $z$ which has some parameters $x$ and $y$. By ${\rm supp}(f)$ we mean the support of $f: X \to \SetR$ which is the set $\{x : f(x) > 0\}$. ${\rm uniform}(X)$ is defined for a finite set $X$ as the uniform distribution on $X$. $x = O(y)$ denotes the ``big-O'' notation which means that $x$ is of order $y$. We denote a norm of any vector $\v x$ as $\norm{\v x}$, where the type of the norm is apparent from the context. We use $\norm[\infty]{\v x} \coloneqq \sup_{i} x_i$ and $\norm[1]{\v x} \coloneqq \sum_{i} \abs{x_i}$ to denote the infinity-norm and 1-norm respectively.

    \iparadot{Function overloading}
    We use function overloading in the thesis. This saves use cluttering the notation with unnecessary related symbols. For example, we use $\mu$ to denote the true environment. This environment symbol is used to express different representations of the environment, e.g.\ $\mu(e'\| ha)$, $\mu(e'\|x')$, or $\mu(x'\|xa)$. There does not exist any ambiguity of different choices of this function. The choice is apparent from the context and the types of the input and output parameters.

    \iparadot{Functionals}
    We use $f(x\|y)$ to denote a function which takes $y$ as a parameter to produce a real-value (mostly a probability) for $x$. We use this notation to distinguish such functions from functions like $g(x|y)$ which denotes a conditional distribution of $x$ given $y$. We use the overloaded notation for such conditional distribution functions, and express $g(xy)$ to denote the joint distribution. We can not do that for the function $f(x\|y)$ form.

    \iparadot{Stochastic process}
    We denote a stochastic process $X$ over a sample space $\Omega$ as
    \beqn
    X : \Omega \to Y^{\mathcal I}
    \eeqn
    which takes values from $Y$. Moreover, ${\mathcal I}$ is any totally ordered index set, usually the time-steps. Any random variable at (time-step) $n \in \SetN$ is denoted by $X_n(\w) \coloneqq X(\w, n)$.
\end{itemize}


\mainmatter

\chapter{Introduction}\label{chap:intro}

\begin{outline}
    We start the thesis by gently introducing the concept of \emph{intelligence}. It seems like an intuitive concept for \emph{generally} intelligent beings like ourselves, but it becomes hairy when we try to formally define intelligence. In this thesis, we are interested in artificial generally intelligent (AGI) agents, which can work in a wide range of environments. We underscore the ``general'' part of the intelligence. After considering the traditional AI paradigms for such agents, we declare them insufficient for AGI. These paradigms (with their original structure) are not powerful enough to model artificial \emph{generally} intelligent agents.
    Moreover, we highlight some of the \emph{key hurdles} such as \emph{scalability}, \emph{specification}, \emph{robustness}, and \emph{assurance} that we need to address before we get to ``the promised land''. This thesis deals with the scalability issue. We informally introduce the core framework used in this work to model AGI agents. Additionally, we outline the rest of the thesis to better position this work in the big picture of AI. In the end, we conclude the chapter by clearly stating the \emph{technical contributions} we make throughout this work and beyond.
\end{outline}

\epigraph{\it ``Our intelligence is what makes us human, and artificial intelligence is an extension of that quality.''}{--- \textup{Yann LeCun}}

The concept of intelligence is unfathomable. It is easy to provide an intuitive list of \emph{characteristics} of an intelligent being, but it is not easy to formally define \emph{intelligent behavior} comprehensively \cite{Legg2007a}. For example, we can agree that a squirrel behaves ``intelligently'' when it saves nuts for the approaching winter, the ``intelligent'' ants follow through the assigned roles for the survival of the colony, and chimpanzees exhibit ``intelligent'' behavior by solving complex puzzles. However, all these examples target some part of the broader intuitive notion of intelligence, but these \emph{traits} are only useful for certain situations. A common theme in these examples, and also in general, is the actors are faced with a choice (e.g.\ eat the nuts now or save them for later, follow the personal goals or the assigned one, and make a certain move or the other to solve the puzzle). The notion of intelligence is tightly connected to the decision-making process about some alternatives \cite{Pomerol1997}.

If someone is facing a choice, they are at a junction to exercise their ability of intelligence. If we take this decision-making perspective on intelligence then an ``optimal'' decision-maker should be the most intelligent being. However, this decision-making perspective of intelligence also bring the issues related to decision theory into the theory of intelligence. For instance, there is no \emph{universal} way to compare the decisions of any pair of individuals who value their future outcomes differently \cite{Schervish1996}. We expand on this more later in \Cref{sec:goals}.
Moreover, an agent being good at a particular task does not necessarily mean that they will perform equally good on another (possibly unrelated) task. This ability of any agent to be successful on a large variety of tasks is usually associated with general intelligence  \cite{Legg2007a}.

If the agent has the ability to adapt and (re)learn novel tasks then we might label that as an intelligent behavior. The recent successful deep learning models, e.g.\ DQN  \cite{Mnih2015}, AlphaZero \cite{Silver2018} and MuZero \cite{Schrittwieser2020} just to name a few, have the ability to adapt, albeit on a special class of environments. This thesis includes discussion of such agents, which can lean and adapt. Interestingly, not all intelligences are created equal. In this thesis, we are (primarily) interested in artificial general (human-level) intelligence.

\section{Human Intelligence}

We can argue that ants are intelligent in a certain sense as they are able to build complex tunnels and function as a working colony, but, arguably, they are not as intelligent as humans. What distinguishes humans from ants is their \emph{general-purpose intelligence} \cite{Legg2007a}.
Humankind has the ability to adapt to a \emph{diverse range of situations}, and we \emph{contemplate} about our actions at lengths.
We, as human beings, can do many tasks which are exclusive to our intelligence. Unlike other intelligent species on the planet, we have been able to take (partial) control of our environment, and we can (reasonably) mold it to our needs. We are no longer dependent (only) on evolution to ensure our survival. We as a species have come a long way because of our \emph{general-purpose} intelligence.

Replicating our general intelligence in machines is the long standing goal of AI \cite{Minsky1956,Hutter2000}. If successful, we will not just make ``another'' (non-biological) human-like species in the process, but a machine that can think and act like more rationally than a human without tiring, retiring, and breaks. Such human-level intelligent general-purpose machines would be able to argue and plan as fast as the technology permits. It is unlikely that an AGI will be bounded by similar biological limits as we are. A super-fast computer will allow AGI to compute (and think) faster.\footnote{Note that an AGI can only solve \emph{computable} problems. There are problems which no agent can solve, even with an infinite amount of computation power, e.g.\ the halting problem \cite{Turing1937}.} In the following section, we list a few of the major dichotomies in AI. We argue why an AGI should be (and be not) a part of one or the other categorization.

\section{Artificial Intelligence}

Keeping human intelligence as the baseline, we embark on the journey to study artificial systems\footnote{Throughout this thesis, we try not to associate AGI with physically embodied agents. Therefore, we usually address the agents by either as systems, algorithms, or decision-makers to put the emphasis on the actual decision-making process. We try to avoid the prevalent confusion of associating bipedal robots with human-like AGI.} which have the potential to mimic this general-purpose intelligence. But first, it is important to highlight some of the key dichotomies exist in the field of AI \cite{Hutter2000}. It will help the readers better understand and position this work, and AGI in general, in the bigger picture of the field of designing the artificial intelligent agents.

\begin{itemize}
    \item\paradot{Specialized vs General}
    We do not want to design different algorithms for every task of interest. We should have a single algorithm which can adapt to different situations and hardware. We call such algorithms \emph{general} decision-makers, which are the subject of this thesis. That puts this work apart from the prevalent techniques of \emph{specialized} agents. For example, we have a very good specialized agents for playing Chess \cite{Silver2018a}, playing Go \cite{Silver2016}, recognizing faces \cite{Masi2019}, synthesizing speech-to-text \cite{Chiu2018}, driving cars \cite{Grigorescu2020}, folding proteins \cite{Senior2020}, generating/understanding natural language \cite{Brown2020}, and so on. However, we do not (yet) have a truly general agent beyond some theoretical works \cite{Hutter2000,Leike2016,Lattimore2014a,Orseau2013b}. This work, although also being dominantly theoretical in nature, is an investigation into the scalability of such general agents.

    It is easier to develop/program the specialized agents because, usually, the domain of interest provides a fair amount of structure, which simplifies the design. For example, it is easier and structurally simpler to design a (good) face recognizer by \emph{exploiting} the facial structure, but undoubtedly the resultant system would be brittle \cite{Wright2009}. It may not generalize beyond faces. On the other hand, general agents cannot assume much about the set of domains, which complicates the analysis and design of such agents \cite{Hutter2000}.

    \begin{remark}[Scheduler AGI]
        It is critical to point out that we do not dismiss the possibility of a general agent being nothing but a scheduler for a collection of specialized agents. Such agent can first do a task recognition activity, and later find the best specialized agent for the task to get the task done \cite{Montes2019}. As clear from the above construction, the notion of intelligence is further convoluted in such setups: where is the intelligence in this setup? Is it in the scheduler or in the specialized agents? In this work, we take the stance that this question is irrelevant. We should not worry about where we put the credit. It is the complete system which exhibits the general-purpose intelligence. We should \emph{not distinguish the entities} beyond this union.
    \end{remark}

    \item\paradot{Planning vs Learning}
    The uncertainty about the environment is not an exception but it is the norm in the universe we live. We need a decision-maker which can perform best under uncertainty. If the agent does not know the environment \emph{a priori} and it has to \emph{learn} the (true) environment then we, unsurprisingly, say the agent is a \emph{learning} agent. On the other hand, a \emph{planning} agent is one which has access to a (near) perfect model of the world. It ``only'' needs to find the (optimal) plan of actions for the environment. However, it is tedious to come up with a perfect model a priori. Therefore, the agents should not only be able to plan (almost) perfectly given the (perfect) model, but also be able to acquire knowledge on their own \cite{Lattimore2014a}. Ideally, we want them to be able to learn autonomously without our constant feedback \cite{Orseau2013b}.

    In the early years of AI, agents were mostly planning agents. The community was dominated by ideas of encoding the world knowledge using crisp logical rules which can later be fed to some powerful planners \cite{Weld1999}. As reality is too complex to be explained by simple rules\footnote{We are aware that there are some branches of science which are actively looking for a simple logical explanation of the universe, or more famously known as a theory of everything. However, the observable world is ``complex'' even if the governing rules turned out to be ``simple.''}, the field of AI has recently been resurrected under a (purely\footnote{Roughly, a purely learning agent is one which learns to ``behave optimally'' in the domain exclusively through real (interaction) data without using a model.}) learning regime. The recent achievements of deep learning are paramount evidence that a lot can be gained only through pure learning \cite{Mnih2015,Silver2018a,Silver2016,Senior2020}. However, the lack of world knowledge in most of these models is supplemented through big data, which is not possible for many tasks of interest. For example, we can not afford to endanger a lot of cancer patients to gather more data on the effects of different experimental drugs. A planning agent would be a better option in this situation. A sufficiently accurate model (e.g.\ a model quantifying the effects of drugs on human body) may a priori allow to prune off many candidate (dangerous) drug trails.

    The main focus of this work is agents which might have an approximate model of the world, but they are not certain about it. The agents are able to learn through experience, but they also maintain a potentially compact model of the world to plan.
\end{itemize}

The dichotomies listed above are by far not the only possible distinctions in the vast domain of AI, but we believe they are the most discriminatory \cite{Hutter2000}. Any algorithm falling under one or the other division is structurally, motivationally and logically very different. This work is in the ``middle'' of the learning and planning regimes. We deal with general agents which keep an uncertain (approximate) model of the world to plan, but learn through interactions to improve the model.\footnote{An example is Dyna \cite{Sutton1990}, but as we will formally state the problem later in \Cref{chap:grl}, this framework is much general than Dyna.}

We need to overcome some huge technical challenges to even start talking about a \emph{realizable} AGI. We call them the \emph{milestones} to AGI.

\section{Milestones to AGI}

An AGI agent is \emph{realizable} if it can be implemented on a physical system which observes physical limitations, e.g.\ limited computation power and storage. There are already some (theoretical) candidates of AGI out there \cite{Hutter2000,Leike2016,Lattimore2014a,Orseau2013b}. However, the current proposals for general agents are too demanding. They are usually \emph{incomputable} or \emph{intractable} \cite{Hutter2000}. Moreover, the policy (or the course of actions) they learn is context dependent which grows over time. The policy learned by the agent is not valid for the next time-instance. On the other hand, specialized agents are too brittle for the job; they cannot generalize much beyond the domain they are specialized for \cite{Wright2009}. In the following, we argue and list some of the key milestones which a \emph{realizable} AGI has to meet. They are listed in no particular order, and every milestone is essential for a \emph{realizable}, \emph{beneficial}, and \emph{safe} AGI \cite{Everitt2018}.

\begin{itemize}

    \iparadot{Scalability}
    We need an AGI (also known as a (universal) anytime algorithm \cite{Hutter2000}) which scales with the resources available at its disposal. Most importantly, the performance of a realizable AGI should \emph{gracefully} scale with limited resources. For example, if a realizable AGI is proving mathematical theorems (e.g.\ by searching through a space of proofs) then the probability of producing a wrong or incomplete proof should be a function of the resources. When the same AGI is run on a supercomputer, it may become appropriately powerful to produce correct and complete proofs with high probability.

    \iparadot{Specification}
    We should be able to specify what we want from the agent. As the title of the thesis says, this work builds on over the reinforcement learning (RL) paradigm. In RL, the reward signal is an essential component to specify a problem. A complete RL problem setup is provided in \Cref{chap:grl,chap:arl}.  However, there are many other ways to specify ``goals'' for machines in RL, e.g.\ by providing preferences over states (or trajectories) \cite{Christiano2017a} or iteratively specifying the rewards by human-in-the-loop \cite{Zanzotto2019}. This thesis assumes that we can specify the rewards for any task of interest, see \Cref{asmp:reward-hypothesis} which is known as the \emph{reward hypothesis}. Under this assumption, we can safely sidestep this milestone. As will be discussed more in \Cref{chap:grl}, this assumption is not very hard to satisfy in most domains of interests. If we humans, as a reward dispenser, can distinguish between a failed and successful states of the task, e.g.\ cancer is cured or not, a chess board is in the won/drawn/lost position, or the dinner is cooked with a cleaned kitchen or not, then \Cref{asmp:reward-hypothesis} could be satisfied by simply rewarding the agent in the desired state of the system in the human-in-the-loop paradigm. However, there is an active body of research, known as AI-Safety, which tries develop AGI agents without the reward hypothesis \cite{Everitt2018}.

    \iparadot{Robustness}
    The designed agents should be able to handle \emph{slight perturbations} or any \emph{small variations} in the system. A robust agent should be able to ``recover'' from any slight disturbance from the optimal settings \cite{Ortega2018}.  Throughout this thesis, we allow an error tolerance $\eps > 0$ (for many quantities of interest) in the main results. Therefore, the stipulated agents possible through this work are inherently robust.

    \iparadot{Assurance}
    Ideally, the agents are designed for the wellbeing and benefits of the society. There must be a way to monitor and control the progress of such agents. Assurance becomes critical when the agent has the ability to effect the environment in unintended ways. This work assumes a dualistic setup where there is a clear distinction between the environment and the agent. The agent can only effect the environment through a predefined ``communication/action-perception'' channel. However, a realistic scenario would demand a proper treatment of these aspects of AGI; see the \emph{AI-Safety} research on these topics \cite{Everitt2018}. Nevertheless, a dualistic setup under the reward hypothesis allows us to not worry about the assurance of the agents.
\end{itemize}

The above milestones are neither exclusive nor complete. However, they capture a majority of critical aspects of the general agents \cite{Hutter2000}. As discussed above, in this work, we \emph{exclusively} focus on the scalability of general agents. The rest of the milestones are (automatically) granted due to the reward hypothesis in a dualistic setup. However, generalizations of this work where these assumptions no longer hold are important future research directions.

\section{Scalable Framework}

It is easy to see that reality is non-stationary, i.e.\ the dynamics of the world are changing with time, and it is not \emph{ergodic}, which (roughly) means that we cannot \emph{revisit} every situation infinitely ofter. Moreover, the most recent observation is (almost) never a sufficient statistics for optimal decision-making, i.e.\ the world is non-Markovian; let alone be IID (independent and identically distributed). This means that the standard (un)supervised learning paradigms are not sufficient to model many interesting problems \cite{Roth2017}. Therefore, we base our work on the reinforcement learning (RL) paradigm which can model non-stationary, reactive, and arbitrary history-based processes \cite{Hutter2000}. However, the standard RL setup, which is defined over a finite-state Markov decision process (MDP) \cite{Sutton2018}, is not sufficient for a majority of realistic situations \cite{Spaan2012}. Moreover, if we \emph{naively} try to approximate a complex problem as an MDP by adding more states to distinguish different situations then the resultant state-space blows up \cite{Sutton2018}. In this thesis, we take a top-down approach, and start from a history-based RL setup known as the general reinforcement learning (GRL) framework \cite{Hutter2000}.

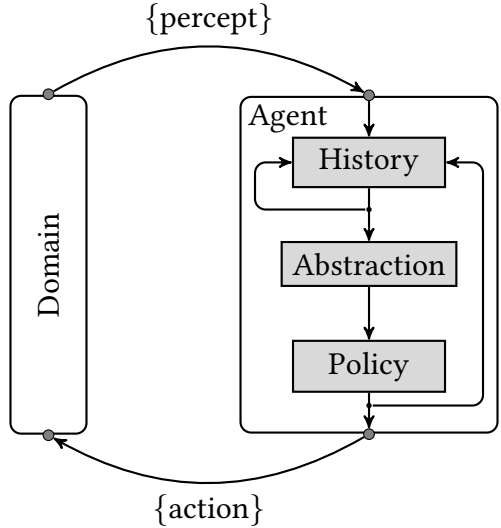
\begin{figure}[!]
    \centering
    \begin{tikzpicture}[
        node distance = 7mm and -3mm,
        innernode/.style = {draw=black, thick, fill=gray!30,
            minimum width=2cm, minimum height=0.5cm,
            align=center},
        outernode/.style = {draw=black, thick, rounded corners, fill=none,
            minimum width=1cm, minimum height=0.5cm,
            align=center},
        endpoint/.style={draw,circle,
            fill=gray, inner sep=0pt, minimum width=4pt},
        arrow/.style={->,thick,rounded corners},
        point/.style={circle,inner sep=0pt,minimum size=2pt,fill=black},
        skip loop/.style={to path={-- ++(#1,0) |- (\tikztotarget)}},
        every path/.style = {draw, -latex}
        ]
        \node (start) {Start};
        \node (h) [innernode]{History};
        \node (phi) [innernode, below=of h]{Abstraction};
        \node (pi) [innernode, below=of phi]      {Policy};
        \node [outernode, align=left, inner sep=15pt, fill=none, fit=(h) (phi) (pi)] (agent) {};
        \node[below right, inner sep=3pt, fill=none] at (agent.north west) {Agent};
        \node[outernode, left=120pt of agent, fit=(agent.north)(agent.south), inner sep=0pt] (env) {};
        \node[below right, inner sep=0pt, fill=none, rotate=90, anchor=center] at (env) {Domain};
        \node[endpoint, above= -2pt of env] (or_env) {};
        \node[endpoint, below= -2pt of env] (a_env) {};
        \node[endpoint, below= -2pt of agent] (a_agent) {};
        \node[endpoint, above= -2pt of agent] (or_agent) {};

        \path (a_agent) edge[arrow,bend left] node[below]{$\{\text{action}\}$} (a_env);
        \path (or_env) edge[arrow, bend left] node[above]{$\{\text{percept}\}$} (or_agent);
        \path (or_agent) edge[arrow] node[right]{} (h);
        \path (h) edge[arrow] node[above=0.5pt,midway,name=h_phi,point]{} node[right]{} (phi);
        \path (phi) edge[arrow] node[left]{} (pi);
        \path (pi) edge[arrow] node[above=0.5pt,midway,name=pi_a,point]{} node[left]{} (a_agent);
        \path (pi_a) edge[arrow, skip loop=1.5cm] (h.east);
        \path (h_phi) edge[->, skip loop=-1.5cm, thick, rounded corners] (h.west);
    \end{tikzpicture}
    \caption{A general reinforcement learning setup with abstractions.}
    \label{fig:scalable-grl}
\end{figure}

%

The GRL setup is arguably the most general setup, which models the interaction of an agent with the environment in the least restrictive way possible. It starts from the extreme case in which every history (a replacement of the state of the system in standard RL) is unique. The setup in its original form is not scalable, as the length of the history grows, which limits the possibilities of getting a \emph{realizable} GRL agent. This history-dependence shows up in the optimal policies, i.e.\ the optimal GRL agent, which makes learning these policies extremely hard if not impossible \cite{Hutter2000}.

In this work, we use and extend a scalable variation of GRL initiated by \citet{Hutter2014}. \citet{Hutter2016} augmented this history-based GRL setup with an abstraction map which tries to provide a compact representation of the domain, see \Cref{fig:scalable-grl}. We extend the abstraction based GRL setup of \citet{Hutter2014} to a general state-action abstraction framework, which we call the abstraction reinforcement learning (ARL) framework. The ARL setup can model \emph{almost} every prevalent RL setup in the literature. In ARL, the scalability is provided through the abstraction map. The map can mimic the resource constraints, e.g.\ the storage capacity and/or the processing power required for representation and learning, by providing different levels of abstractions. A resource-limited system might have coarser abstraction, which might need less storage and compute, than resource-rich systems. However, if there are some guarantees on the structure of the map then performance can be assured. This thesis is dominantly about the existence of such abstractions (be it either abstracting only the states or both states and actions) which are powerful enough to \emph{reliably} model \emph{any} domain of interest by a \emph{compact} state-action space.

Most of the thesis deals with representation and learning guarantees of these abstractions. However, in \Cref{chap:abs-learning} we consider the possibility of learning such abstractions from data. \citet{Hutter2014} calls learning the abstraction map feature reinforcement learning (FRL). The abstractions considered in FRL are crucially different from the usual representation learning techniques \cite{Hutter2009}. In FRL, the abstractions are functions of history, which allows for non-stationary, context based aggregation maps. This characteristic sets FRL, and by extension this work, apart from other works dealing with representation learning \cite{Bengio2013}.

\section{Thesis Statement}

This thesis can be summarized by the following \emph{thesis statement}.
\begin{tcolorbox}
    \begin{quote}
        \it \centering The abstraction reinforcement learning framework is a pathway to a scalable artificial general intelligence.
    \end{quote}
\end{tcolorbox}

The above statement is purposefully left open ended, as this thesis barely scratches the surface of the powerful ARL setup. There is plenty of support behind the statement in this thesis, and in the related literature. This thesis is a part of a series of earlier works done on this topic \cite{Hutter2016,Daswani2015a,Nguyen2012,Daswani2013a}.

\section{Thesis Outline}

To facilitate the readers in navigating through the thesis, we point to the different chapters of the thesis in this section. Through this section, the reader can have a birds-eye view of what to expect in each chapter. Even though the thesis is theoretical in nature, we resisted the temptation to put a chapter about the mathematical preliminaries. We develop the necessary mathematical machinery along the way as need and summarize notation in \Cref{chap:notation-and-symbols}. This choice may not allow independent reading of the chapters if the required notion is developed in other chapters. If it is the case, we highly recommend to first read \Cref{chap:grl,chap:arl} before reading any other chapter. This will expose the reader to necessary mathematical background required for the majority of the chapters. Nevertheless, we expect a continuous reading of the main text, which justifies our choice of not having an independent mathematical background chapter.

We start by formally laying down the foundations of a GRL framework in \Cref{chap:grl}. Then, we build on GRL and provide a setup of ARL in \Cref{chap:arl}. Because the literature on abstractions is vast, we compile some of the dominant state-only abstractions in \Cref{chap:state-only-abs} and state-action abstractions in \Cref{chap:state-action-abs}. These chapters do not just contain the well-established abstractions; we provide a number of novel abstractions in between. \Cref{chap:extreme-abs} initiates the discussion on the extreme abstractions which is a special type of state-only abstractions. We provide convergence guarantees in \Cref{chap:convergence-guarnt}. Representation guarantees for many abstractions  are listed in \Cref{chap:representation-guarnt}. \Cref{chap:action-seq} contains more novel results about action-sequentialization to improve representation guarantees of extreme abstractions. \Cref{chap:abs-learning} contains a series of ideas to build abstraction learning algorithms. The thesis style is primarily based on rigorously proving statements, but in \Cref{chap:vpdp-exp} we provide some (preliminary) supporting experimental results. We finally conclude the thesis in \Cref{chap:conclusion} with an important outlook to some of the key extensions of this work. A comprehensive list of notation is presented in \Cref{chap:notation-and-symbols}.


\section{Contributions}

During my PhD, I have made the following contributions, some of which were published in the prestigious AI conferences. This thesis is mostly written around these contributions.
\begin{enumerate}
    \item \Cref{chap:grl,chap:arl,chap:state-only-abs,chap:state-action-abs,chap:extreme-abs} are based on \citet{Majeed2021}. It put the limelight on the ARL framework as being a natural way to model the agent-environment interaction under abstractions, which can be specialized to many standard RL frameworks. This work can be considered as a shorter version of the thesis.
    \item \Cref{chap:convergence-guarnt} is an extended version of \citet{Majeed2018}. The original publication provides convergence guarantees for Q-learning history-based RL beyond MDP domains.
    \item \Cref{chap:representation-guarnt} is based on \citet{Majeed2019}, which establishes the representation guarantees for non-MDP  homomorphisms. It basically extends the results of \citet{Hutter2016} from state-only abstractions to state-action abstractions.
    \item \Cref{chap:action-seq} is an adaptation of \citet{Majeed2020}. The major contribution in this work is the usage of the technique of action-sequentialization to improve representation guarantees of extreme abstractions. The upper bound on the required state-space size of extreme abstractions was improved double exponentially through action-sequentialization.
    \item \Cref{chap:abs-learning} lists the results reported in \citet{Majeed2021a}. The manuscript collects many idea and pieces required to build abstraction learning algorithms for non-MDP state abstractions.
    \item \Cref{chap:vpdp-exp} is based on \citet{McMahon2019}, which provides a numerical treatment of some non-MDP abstractions. It highlights the possibility of going beyond Q-uniform abstractions. I supervised and made equal contributions to this work.
\end{enumerate}

The above list is for the contributions in the main trunk of the thesis. The following is the list of the rest of the contributions I made during my PhD, not included in this thesis.

\begin{enumerate}[resume]
    \item In \citet{Hutter2019}, I made equal contribution in establishing the conditions on features for convergence of natural algorithms.
    \item \citet{Majeed2018a} provides a Python framework for developing ARL agents.
    \item In \citet{Parker2019}, I contributed in numerically analyzing the looping context tree for a candidate of better model class for general agents.
\end{enumerate}

\section{Summary}

This chapter set the scene for the rest of the thesis. We introduced artificial generally intelligent agents. We sketched a utopia where AGI agents are abundant in society. It contained an informal introduction of a scalable AGI framework ARL. This chapter also provided a list of contributions I made during my PhD.

This concludes the introductory chapter of the thesis. Now, we embark on the journey to investigate the scalability of generally intelligent agents.

\chapter{General Reinforcement Learning}\label{chap:grl}

\begin{outline}
    The standard RL setup assumes that the environment is a Markov decision process (MDP), which is not sufficient to model a majority of interesting problems. In this chapter, we formalize the general reinforcement learning (GRL) framework, which is a history-based setup. At the core of this setup is a history-based decision process (HDP), which is also a countable MDP but with a special structure that no state ever repeats. Along the way, we expand on the critical aspects and assumptions of the framework. Additionally, we list many possible ways one can define ``goals'' for a GRL agent, which leads to a lot of possible GRL variants. However in the end, we declare an unnormalized $\g$-discounted GRL setup as \emph{the} GRL framework considered in the rest of the thesis.
\end{outline}

\epigraph{\it``All knowledge - past, present, and future - can be derived from data by a single, universal learning algorithm.''}{--- \textup{Pedro Domingos}}

\section{Introduction}

As highlighted in \Cref{chap:intro}, we take the ``computational perspective'' of intelligence. The intelligence of an agent is categorized by the quality of the decisions it takes \cite{Hutter2000}. Before we can even put a quality measure over a sequence of decisions, we need a formal notion for the environment and the agent. For an AGI, we would like to have an agent which can be used in a multitude of environments, just like a human. We must not put such assumptions early on which may limit the types of environments we can model. For example, if we build our framework over a finite-state MDP, i.e.\ the most recent observation is sufficient information for the agent, then we restrict ourselves to an environment class which may not \emph{reasonably} model the time-variant or long-term dependencies of the environment \cite{Sutton2018,Spaan2012,Hutter2000}.

We can rule out an MDP modeling early on because the real world is too ``messy'' to be considered as a finite-state MDP \cite{Dulac-Arnold2019}. Hence, we do not want an MDP as a starting point. However, as shown later in \Cref{chap:arl}, there is no harm if we may end up having an abstraction of the environment which happens to be an MDP. Because as long as the general agents can perform well in a wide range of environments (even if the abstraction is an MDP), we can say that the agent is a candidate for AGI. Nevertheless, if we start our modeling of an environment as an MDP we loose the ability to model a major chunk of interesting environments \cite{Spaan2012}.

Hence in the rest of this chapter, we formulate a general history-based RL setup, where the agent can \emph{potentially} use the complete history of interaction to decide its actions. In the decision theory perspective, we allow the decision-making process to be a function of history, which can model \emph{almost every} possible decision problem \cite{Hutter2000}.

\section{Dualistic Setup}

In simple terms, we are interested to model the decision-making process under uncertainty with as little assumptions as possible.
In this work, we consider a dualistic framework, see \Cref{fig:ae-loop}, where there is a clear distinction between the \emph{agent} $\pi$ (the decision-maker) and the \emph{environment} $\mu$ (the process being controlled). In the dualistic setting the entities only affect each other through a defined ``communication'' interface. The internal working of each element (i.e.\ the agent and the environment) is not directly observable to each other except through the communication interface.\footnote{See \Cref{chap:conclusion} for more discussion on the other possible setups, e.g.\ embedded agents framework \cite{Orseau2012,Everitt2015}.}

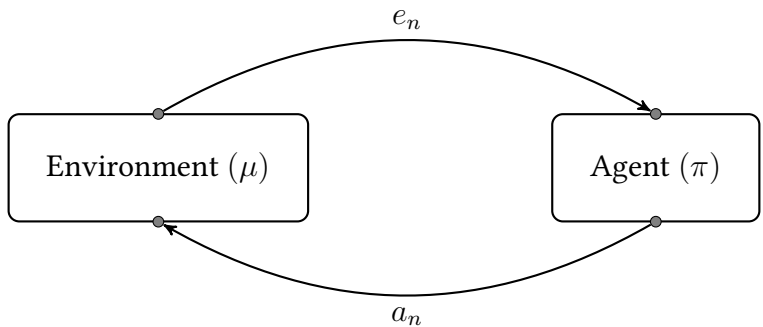
\begin{figure}[!]
    \centering
    \begin{tikzpicture}
        \node [outernode] (agent) {Agent $(\pi)$};
        \node[outernode, left=0.2\textwidth of agent] (env) {Environment $(\mu)$};
        \node[endpoint, above= -2.5pt of env] (or_env) {};
        \node[endpoint, below= -2.5pt of env] (a_env) {};
        \node[endpoint, below= -2.5pt of agent] (a_agent) {};
        \node[endpoint, above= -2.5pt of agent] (or_agent) {};
        \path (a_agent) edge[arrow,bend left] node[below]{$a_n$} (a_env);
        \path (or_env) edge[arrow, bend left] node[above]{$e_n$} (or_agent);
    \end{tikzpicture}
    \caption[The agent-environment interaction in a dualistic setup.]{The agent-environment interaction in a dualistic setup. At any time-step $n$ the agent $\pi$ receives a percept $e_n \in \OR$ from the environment $\mu$ and then it acts on the environment with an action $a_n \in \A$. In this setup, so far, there is no restriction on the percept and action generation processes. In the extremely general case both entities can produce the percept-action pair from a set of unique distributions at every time-step.}
    \label{fig:ae-loop}
\end{figure}

We assume that there \emph{always} exists a real-valued (partial) function that can guide the agent to the optimal behavior, i.e.\ the agent can learn to take the ``best'' action through this (reward) signal. This is called the reward hypothesis \cite{Glimcher2011}.

\begin{assumption}[Reward Hypothesis]\label{asmp:reward-hypothesis}
    There exists a real-valued (partial) reward function that can be \emph{optimized} to learn the \emph{(intuitive) optimal behavior}.
\end{assumption}

We argue that this hypothesis makes sense, and it may hold in a majority of domains. Actually, for most of the environments we can simply reward the agent once it has reached a situation that we \emph{want} it to get to, e.g.\ a winning end state of a chess game, the destination of a self-driving agent, a cup of hot coffee by the house-hold robot and so forth. The agent should be able to infer the required behavior if we can ``evaluate'' the final states\footnote{At this point, by the word ``state'' we denote the world configuration we end up with.} appropriately. For example, we may be interested only in the clean (state of the) room rather than how the agent actually cleans the room. Ideally, if the agent has damaged items (or cheated) in between we should not reward the agent. However, if we, the reward generators, can not distinguish between two ``good'' situations, e.g.\ two rooms cleaned by two different strategies, there is no reason for us to evaluate them differently. That is why the reward hypothesis is an assumption in our setup. Once we have a reward sequence, we assume that it is enough to optimize this signal for the agent to be aligned with the intentions of the ``designer''. There is a whole different body of research linked to the famous \emph{value-alignment problem} where the reward hypothesis is no longer true \cite{Everitt2018}. We do not consider the value-alignment issue in this work, because of \Cref{asmp:reward-hypothesis}.

So far, we have been using the term \emph{optimal behavior} to capture the intuitive notion of doing the ``best'' thing. Provided the reward hypothesis holds, we can informally state the (intuitive) optimal behavior.

\begin{definition}[Optimal Behavior]\label{def:opt-behavior}
    The (intuitive) \emph{optimal behavior} of a (G)RL agent is to \emph{keep getting} as \emph{much reward} as \emph{possible} from the environment as \emph{often} as \emph{possible}.
\end{definition}

The above definition is not yet strictly formal, hence labeled ``intuitive''. At this stage, there are many elements left unspecified in the definition. As will be the case in the following sections, there are a lot of subtleties involved when we start to formalize this intuitive notion of optimality. We will see later in the thesis that this intuitive notion of optimal behavior (sometimes) may not be a solution of the formal setup.

In the following, we start formalizing the setup. The environment-agent interaction loop in \Cref{fig:ae-loop} is formally a sequence of function evaluations. First we formally define the spaces where these functions live.

\section{Foundational Spaces}

The decision-making process for the agent and the environment is merely a process of choosing among the alternatives.\footnote{We say that the environment is of the same ``type'' as the agent. They both can be decision-makers, e.g.\ the opponents in a game can act as an environment for the (main) player.} In this section we only define the spaces for these alternatives. The decision-making process is laid down later.

\subsection{Action Space}

Let the agent have an allowable non-empty set of actions $\A$. Without loss of generality, we assume that the action-space is stationary, i.e.\ it does not change over time or under different situations.\footnote{So far, we have been using ``situation'' to vaguely mean the interaction-history up to the current time-step. This notion is crystallized shortly.} Any sequence of action-spaces $(\A_n)$ can be expressed as a single action-space, which is a union of all these spaces as $\A \coloneqq \cup \A_n$. At any time instance $n$ the agent effectively chooses an action $a_n$ only from the subset of actions $\A_n \subseteq \A$. Note that the agent needs not to know the exact subset $\A_n$ beforehand. It can try all actions, and the ``invalid'' actions $\A \setminus \A_n$ are either not possible or provide no feedback. We argue that this is a natural assumption. For example, a human being has the same set of ``actions'' in a form of the limb movements, but under different contexts different actions are ``available'' (or allowed). We can take other invalid actions, sometimes without any bad consequences, and the environment can ignore these actions. For instance, we can move our feet during a Chess game, but this ``invalid'' action, typically, does not influence the game, except when we accidentally trip the board over.

There are many ways one can model an invalid action. The invalid action $a \in \A \setminus \A_n$ may either carry a large penalty, e.g. tripping the Chess board may lead to a lost game, it could be mapped back to the set of valid actions $\A_n$, or could behave as a ``no-action'' action.

\begin{remark}[Default ``no-action'' Action]
    Sometimes, it is easy to state a problem where at some situations the agent has no action available. Maybe, the agent is ``thinking'' while the environment may transition to the next situation. We can model such situation by adding a default ``no-action'' action into the original action-space. Formally, it does not change the problem of decision-making. The agent can react when it has an available set of actions, and in between it can imagine/argue that it has been taking this ``no-action'' action. Such modeling is actually quite useful in many time-critical tasks where the agent has to react at every time-step. And, if it does not do so in time, this inaction might affect the agent adversely. In such situations, if the agent has this ``no-action'' action in the action-space then it can contemplate the effects of this inaction during the decision-making process.
\end{remark}

Without any loss of generality, we assume that the agent has this ``no-action'' action available in its action-space, if required.

\subsection{Percept Space}

At any time-step $n$, the environment produces a percept $e_n$ from a percept-space $\OR$. As the name suggests, this space is the set of possible ``sensory'' perceptions, e.g.\ the raw screen captures in Atari \cite{Mnih2015}, a board position of Chess or Go \cite{Silver2018a}, the images of the objects in a classification task \cite{Masi2019}, and so forth. The perceptions need not to be ``raw'' sensations, e.g.\ camera images or audio signals. The precept denotes any set of inputs received by the agent apart from the reward signal. Sometimes, it is easy to imagine that the reward is also a part of the percept. We call this way of expressing the reward signal as the \emph{reward embedding}. Some parts of this thesis uses this notion of reward embedding. However as we explain later, it does not matter much if the agent either observes a separate reward signal or it is embedded inside the percept-space. Under mild assumptions, both setups are formally equivalent.

\subsubsection{Reward Embedding}

In the case of reward embedding, we assume that the percept-space is a product of two spaces, an observation-space $\O$ and a reward-space $\R$, i.e.\ the percept is a tuple $e_n = (r_n \in \R, o_n \in \O) \in \OR$.

\begin{itemize}
    \iparadot{Observation Space}
    The observation part of the percept is sampled from the observation-space $\O$. It is the part of the percept without the reward embedding, which models the ``everything else'' observed by the agent apart from the reward. \todo{Expand on this section.}

    \iparadot{(Embedded) Reward Space}
    The reward signal is one of the most critical elements of the decision-making process in (G)RL. It quantifies the ``goodness'' of the most recent action performed by the agent. It measures the desirability of the action taken in the current situation, but it does not provide information about the ``best'' action at the instance, \emph{cf.}\ supervised learning, where the feedback from environment is the right action (or the class label) \cite{Masi2019}.

    \begin{assumption}[Bounded rewards]\label{asm:bounded_rewards}
        The rewards are bounded.
    \end{assumption}

    We assume rewards to be bounded.
    %
    The boundedness assumption, ubiquitous in the literature, intuitively means that the agent cannot be rewarded arbitrarily high (or low) at any time-step \cite{Sutton2018}. Therefore, the agent has to devise a plan to appropriately gain more rewards. Typically, the boundedness of the rewards is a \emph{necessary condition} to get a meaningful performance guarantee. If the agent can be arbitrarily worse off in the future, there might not be a way to come up with an ``optimal'' sequence of decisions under uncertainty \cite{Leike2016b}.
\end{itemize}

\begin{remark}[Finite Rewards are Sufficient]
    So far, we were able to avoid taking about the cardinality of the spaces. However, the reward-space $\R$ is a bit tricky. Apparently, it seems restrictive to assume that the reward set is either finite or countable. Nevertheless, even if the environment can dispense a real-valued reward $\R_\mu \subseteq \SetR$ the observation apparatus of an (implementable) agent can only produce a rational approximation, i.e.\ $\R \subseteq \SetQ$. Since the rationals are dense in the reals, the agent should not loose much \cite{Abbott2015}.

    \begin{proposition}[$\SetR \supseteq \R_\mu \approx \R \subseteq \SetQ$]\label{lem:r_rational}
        Let the environment rewards from $\R_\mu \subseteq \SetR$ but the agent can only observe a unique rational approximation of it from $\R \subseteq \SetQ$, then the approximation $\R \approx \R_\mu$ is ``good enough'' for \emph{any} performance measure.
    \end{proposition}

    The above proposition  hints that a countable reward-space is sufficient in any implementable (G)RL/AGI agent. We leave the proof of \Cref{lem:r_rational} out, as it is an easy consequence of approximating a real-valued function by the best rational approximation \cite{Abbott2015}.

    Because of \Cref{lem:r_rational} and \Cref{asm:bounded_rewards}, we can easily assume that $\R \coloneqq [0,1] \cap \SetQ$.
    Furthermore to argue for a finite reward set, for any $\eps > 0$ and bounded $\R$, we may define $\R_\eps$ as
    \beq
    \R_\eps \coloneqq \left\{ \floor*{\frac{r}{\eps}} \eps \ \middle| \ r \in \R\right\}
    \eeq
    where it is clearly the case that $\abs{\R_\eps} < \infty$. Using a similar argument as for \Cref{lem:r_rational}, any problem modeled with $\R$ is \emph{approximately} the same (with a slack of order $\eps$) problem modeled using $\R_\eps$ instead of $\R$.
\end{remark}

\subsubsection{Reward Process}

The tuple structure of the percept $e_n = (r_n, o_n)$ is not the only way to get a GRL setup. The other possible option is to assume that there exists a reward function $r : \OR \to \SetR$. This function is external to the agent that evaluates the current situation and dispatches the reward. We get an equivalent GRL setup if we put this external reward function back into the environment. We expand more on this topic in \Cref{sec:reward-process}, where we use a much more general notion for the reward process.

\subsection{Finiteness}

For brevity, throughout this work, we assume that the percept and action spaces are finite. The size of the action space $\A$ is denoted by $A$.

\begin{assumption}[Finiteness]\label{asm:finite}
    The action and percept spaces\footnote{If the rewards are embedded in the percept-space this implies that the reward spaces is finite too.} are finite.
\end{assumption}

Note that the above finiteness assumption is mostly for simplicity of exposition. The results in this work can be extended to countable or even continuous spaces. The summations over the percept and action spaces involve bounded, non-negative, and convergent series, so the sums can easily be replaced by a countable summation or an integral for that matter. However, the involved formalities over weigh the advantages.

Moreover, any physical implementation of (G)RL/AGI agents would eventually have this finite structure. However, this should not be confused with the finite-state assumption of standard RL. Even with finite action and percept spaces, a GRL agent can have an infinite number of ``states''.

\begin{remark}[Continuous spaces]
    We formulate the GRL problem in topologically discrete spaces. We purposefully avoid continuous spaces, because usually the extension to the continuous case is mere a matter of technicalities. The usual conditions to handle a continuous space are either $a)$ a (countable) parameterization of the space, or $b)$ to assume some convenient structure on the spaces to allow for a dense subset, e.g.\ assuming the space to be Polish or separable. Moreover, typically it is assumed that the spaces are locally smooth or only have bounded derivative, e.g.\ Lipschitz/Holder continuity holds. See \citet{Hasselt2012} for a survey of RL in continuous state-action space.
\end{remark}

A countable action-space formalization needs a bit of care when the rewards are \emph{unbounded}. There might be some adverse effects if the decision-making process continues for infinite number of steps. The agent may be rational on individual decisions, but it may still suffer a great loss in total. See \citet{Arntzenius2004a} for some examples of such decision problems.

\subsection{History Space}

By design, the agent-environment interaction is sequential, see \Cref{fig:ae-loop}. There are two ways to start the interaction: 1) it can either be agent, or 2) environment initiated. Both variations are prevalent in the literature, and have their own (de)merits in terms of notational convenience \cite{Hutter2000,Sutton2018}. However, both variations are logically equivalent. In this work, we prefer to use the environment initiated interaction. The agent first receives the percept and then they respond. This is closer to the natural setup in the reality. We do not start a systems with a ``random'' uninformed action. Mostly, there is a well-defined starting configuration of the system, i.e.\ the initial percept for the agent to react to.

For any time-step $n \in \SetN$, we define the set of interaction histories of length $n$ as
\beq
\H_n \coloneqq (\OR \times \A)^{n-1} \times \OR
\eeq
where the exponent denotes the number of Cartesian products of the terms enclosed. The Kleene star closure of this set is called the set of all \emph{finite} histories. We express this set as
\beq\label{eq:history}
\H \coloneqq \bigcup_{n=1}^\infty \H_n
\eeq

The above set of finite histories can be regraded as the ``state'' of the system. We will expand on this idea later, but it is essential to highlight the fact that even with \Cref{asm:finite} the resultant ``set of states'' $\H$ is countably infinite.
So far, we have only discussed a finite interaction history, which is essential to talk about the interaction at any instance of time. However by assumption, the interaction between the agent and environment never stops. Hence, theoretically the interaction history is infinite for any single run. We denote the set of such \emph{infinite} histories as $\H_\infty$. Any pair of an agent and environment jointly samples one member $\w$ from $\H_\infty$. This set of infinite histories form the natural \emph{sample-space} of the problem. An agent $\pi$ and environment $\mu$ induce a probability measure $\mu^\pi$ over (the $\sigma$-algebra of) the sample-space.

In the next section, we formulate this measurable sample-space, which is at the core of a GRL setup. Moreover, we provide the formal definitions of the agent and environment as a joint measure over this space.

\section{Actors and Interaction}

Up to this point, we have defined the sets $\A, \O, \R, \OR, \H$, and $\H_\infty$. These sets are needed for the formal definitions of the agent and environment in a GRL setup. In GRL, we allow the agents (and the environment) to use the complete history of interaction to decide the next action (or percept), if required. Formally, any agent $\pi$ is a mapping from finite histories to the action-space, which can potentially be stochastic:
\beq
\pi : \H \to \Dist(\A)
\eeq
where, recall, $\Dist(X)$ denotes a set of probability distributions over any finite set $X$.
Similarly, the environment $\mu$ is a mapping from a history-action tuple to a distribution over the next percept:
\beq
\mu : \H \times \A \to \Dist(\OR)
\eeq

The key fact about the above definitions is that both $\pi$ and $\mu$ are not restricted in any way.\footnote{Other than by the choices of $\A$, $\O$ and $\R$.} We have only specified the types of these functions. They can potentially depend on the complete interaction history not just the last percept, which is usually the assumption in standard RL \cite{Sutton2018}.
The interaction of the agent and environment generates an infinite history. We will show in \Cref{sec:measure-space} that $\mu$ and $\pi$ induce a joint \emph{measure} on the space of infinite histories.

\subsection[Measure Theoretic Basis]{Measure Theoretic Basis%
\footnote{We assume a basic understanding of measure theory for this section. See \citet{Stein2019} for a refresher. This section establishes the notion of conditional expectation and probability over the measure space. However, the section can be skipped without effecting the continuity, as in the end we recover the standard ``elementary'' notion of conditional expectation and distribution. We decided to keep this exposition for the readers who are interested to look ``under the hood''.}
}

This section provides a sound measure theoretic basis for our GRL setup. We build the framework starting from the very basic measurable space structure.

\subsubsection{Measurable Space} Let us define the \emph{measurable} space over the set of infinite histories $\H_\infty$, which constitutes the \emph{sample-space} of our setup. There are many choices of $\sigma$-algebras one can put on the space, but the following is a natural option. A $\sigma$-algebra (which is a set of \emph{events}) roughly\footnote{In our setup the $\sigma$-algebra is countably generated, so it exactly represents the information set. However, this is not always the case. See \citet{Herves-Beloso2013} for a discussion of this topic.} quantifies the notion of information available at any time-step. Arguably, the complete (finite) history of the interaction at any time instance is the maximum information available up to that point. Any function of this history can only reduce (or abstract) this information. We quantify this information using the \emph{cylinder sets}.

A cylinder set is a subset of the sample-space with the same prefix. In our setup, a cylinder set contains all infinite histories starting from the same finite history. For any finite history $h \in \H$ and action $a \in \A$, we define the cylinder set as follows:
\beq\label{eq:cyl-ha}
\cyl{ha} \coloneqq \{\w \in \H_\infty : ha \sqsubseteq \w \}
\eeq

Moreover, for notational convenience, we express the set of all such cylinder sets as
\beq
\cyl{\H_n \times \A} \coloneqq \{\cyl{ha} : ha \in \H_n \times \A\}
\eeq
for any time-step $n$.
Let $\F_n$ be a sub-$\sigma$-algebra up to and including time-step $n$, and $\F$ be the union $\sigma$-algebra of all such sub-$\sigma$-algebras:
\beq
\F_n \coloneqq \sigma\left(\cyl{\H_n \times \A} \right) \text{ and } \F \coloneqq \sigma\left(\bigcup_{n=1}^\infty \F_n \right)
\eeq
where $\sigma(X)$ denotes the smallest $\sigma$-algebra that includes set $X$, which is the intersection of all sigma algebras containing $\cyl{\H_n \times \A}$.
Using the above choice of $\sigma$-algebra, we get the following (filtered) measurable space for our GRL framework.
\beq
\langle \H_\infty, \F, \{\F_n\}\rangle
\eeq

At any history $h_n$ and action $a_n$ the cylinder set $\cyl{h_na_n} \in \F_n$ quantifies the maximum possible information one can have at $h_na_n$. This property of sub-sigma-algebra $\F_n$ provides a key relationship between the functions of finite histories-action pairs and stochastic processes adapted to the filtration.

\begin{proposition}[{$X \iff g_X$}]\label{prop:XiffgX}
    If $X : \H_\infty \to \SetR^\SetN$ is any stochastic process \emph{adapted to the filtration} then there exists a corresponding function of finite histories and actions $g_X : \H \times \A \to \SetR$ such that
    \beq
    X_n(\w) = g_X(h_n(\w)a_n(\w))
    \eeq
    for all infinite histories $\w \in \H_\infty$ and time-steps $n$. The converse also holds if all $\F_n$ are discrete sigma algebras.
\end{proposition}
\begin{proof}
    The proof is trivial, and directly follows from the fact that $X_n \in \F_n$ for each time-step $n$. Due to this measurability constraint on $X_n$, any realization $X_n(\w)$ depends only on $h_n(\w)$ and $a_n(\w)$. The converse holds since every function of $h_n a_n$ is $\F_n$ measurable if $\F_n$ is discrete.
\end{proof}

Usually finite (and even countable) percept and action spaces are equipped with discrete sigma-algebra, so the converse holds for all processes considered in this thesis.
The correspondence between the adapted stochastic processes on the measure space and the functions of finite histories gives a distinctive characteristic to our GRL setup. As discussed earlier, the main power of our GRL framework comes from the fact that both the agent and the  environment can use the complete history of interaction, if they need, which encompasses almost every class of problems prevalent in the literature \cite{Hutter2000}.

An important property of the cylinder sets $\cyl{\H_n \times \A}$ is that they form a \emph{countable partition} of the sample-space $\H_\infty$ for every time-step $n$.

\begin{proposition}[Cylinder Set Partition]\label{prep:h-partition}
    The set of cylinder sets corresponding to any countable, prefix-free\footnote{A set $X$ is prefix-free if for any $x, y \in X$, $x \sqsubseteq y$ implies $x = y$.}, and complete\footnote{A set $X \subseteq Y$ is complete (for $Y$) if for every $y \in Y$ there exists an $x \in X$ such that $x \sqsubseteq y$.} set forms a \emph{countable partition} of the sample-space.
\end{proposition}
\begin{proof}[Proof sketch]
    The proof trivially follows from the definition of cylinder sets. Let $X$ be a countable, complete, and prefix-free set. For any two $x \neq \d x \in X$, the cylinder sets $\cyl{x}$ and $\cyl{\d x}$ are disjoint, as $x$ and $\d x$ cannot be a prefix of each other. Moreover, the completeness of $X$ implies that the union of the cylinder sets is the sample-space: $\cup_{x \in X} \cyl{x} = \H_\infty$.
    The partition, the set of cylinder sets, is countable by the countable assumption on $X$.
\end{proof}

An immediate consequence of \Cref{prep:h-partition} is that a set of cylinder sets corresponding to either $\H_n$ and $\H_n \times \A$ form a countable partition of $\H_\infty$ for any time-step $n$ because both sets are countable, complete, and prefix-free. This is also obvious.

\subsubsection{Measure Space} \label{sec:measure-space}
When any agent $\pi$ and environment $\mu$ interact they induce a (probability) measure
\beq
\mu^\pi: \F \to [0,1]
\eeq
on the measurable space defined above. However, this measure is not trivial. The measure $\mu^\pi$ is a function of events in $\F$ which are subsets of infinite history space $\H_\infty$, whereas the agent and the environment are functions of only finite histories. Luckily, our choice of sigma algebra based on the cylinder sets of finite history-action pairs allows us to produce a unique measure $\mu^\pi$ for each agent $\pi$ and environment $\mu$.

Intuitively, the measure $\mu^\pi(\cyl{h_na_n})$ for any $h_na_n$-pair should coincide with the probability of the agent-environment interaction to produce this finite history-action pair. Therefore, we define the measure on the cylinder sets as:
\bqa\label{eq:mu-mu-rel}
\mu^\pi(\cyl{h_na_n})
&\coloneqq
\mu(e_1) \left(\prod_{m=1}^{n-1} \mu(e_{m+1} \| h_m a_m)\pi(a_m\|h_m)\right)\pi(a_n \| h_n) \\
\mu^\pi(\cyl{h_n})
&\coloneqq
\mu(e_1) \left(\prod_{m=1}^{n-1} \mu(e_{m+1} \| h_m a_m)\pi(a_m\|h_m)\right)
\eqa
where $\mu(e_1)$ is the probability of initial percept $e_1$.

\begin{remark}[Alternate Definition]
    We can have an equivalent (and cleaner, but recursive) condition for the cylinder set measures as:
    \bqa
    \mu^\pi(\cyl{e}) &\coloneqq \mu(e)\\
    \mu^\pi(\cyl{ha}) &\coloneqq \pi(a\|h) \mu^\pi(\cyl{h}) \\
    \mu^\pi(\cyl{hae'}) &\coloneqq \mu(e'\|ha) \mu^\pi(\cyl{ha})
    \eqa
    for all $e, e' \in \OR$, $h \in \H$ and $a \in \A$.
\end{remark}

An important and easily verifiable property of this choice is that it satisfies the additivity of measures:
\beq
\mu^\pi(\cyl{h_n}) = \mu^\pi(\cup_{a_n} \cyl{h_na_n}) = \sum_{a_n} \mu^\pi(\cyl{h_na_n})
\eeq
that is, the measure of disjoint sets is sum of measures of the individual sets. Therefore, by Carathéodory's extension theorem there exists a unique extension of $\mu^\pi$ (which, for brevity, we also denote by $\mu^\pi$) over the sigma-algebra $\F$.
Unsurprisingly, the resultant \emph{measure} space is denoted by the tuple
\beq
\langle \H_\infty, \F, \{\F_n\}, \mu^\pi\rangle
\eeq

Now, we establish the notion of conditional expectation and conditional distribution on this measure space.

\subsubsection{Conditional measure and expectation}

Since the measure $\mu^\pi$ is well-defined, the conditional measure is given as follows:
\beq\label{eq:conditional-dist}
\mu^\pi(A|B) \coloneqq \frac{\mu^\pi(A \cap B)}{\mu^\pi(B)}
\eeq
for any $A, B \in \G \subseteq \F$. For a succinct representation, we sometimes abuse the notation and drop the cylinder notation from the measure. That is, we express $\mu^\pi(ha) \coloneqq \mu^\pi(\cyl{ha})$ and $\mu^\pi(h) \coloneqq \mu^\pi(\cyl{h})$ for any $ha$ history-action pair.

To define a notion of conditional expectation, we start by defining the \emph{expected} value of any \emph{Lebesgue integrable}\footnote{In this work we exclusively deal with Lebesgue integrable random variables.} (random) function $X: \H_\infty \to \SetR$ as
\beq
\E_\mu^\pi[X] \coloneqq \int Xd\mu^\pi
\eeq
where the right hand side is the \emph{Lebesgue integral} of $X$.
If $X$ is an \emph{adapted} stochastic process, i.e.\ $X_n \in \F_n$ for all $n$, then we recover the elementary notion of expectation as
\bqan\label{eq:xn-expectation}
\E_\mu^\pi[X_n]
&\overset{(a)}{=} \sum_{h_na_n} \int X_n\ind{\cyl{h_na_n}}d\mu^\pi \\
&\overset{(b)}{=} \sum_{h_na_n} g_X(h_na_n) \int \ind{\cyl{h_na_n}}d\mu^\pi \\
&\overset{(c)}{=} \sum_{h_na_n} g_X(h_na_n) \mu^\pi(\cyl{h_na_n}) \\
&\equiv \sum_{h_na_n} g_X(h_na_n) \mu^\pi(h_na_n) \numberthis
\eqan
where $(a)$ follows from the fact that the cylinder sets at time-step $n$ partition the sample space, \Cref{prep:h-partition}, $(b)$ is due to \Cref{prop:XiffgX} and $(c)$ is true by definition. Note that the above relationship is also true for any $\F_n$-measurable random variable.

In the most general terms, the conditional expectation is a random variable with the following specifications:

\begin{definition}[Conditional Expectation Given a Sigma-Algebra]
The conditional expectation of any random variable $X: \H_\infty \to \SetR$ and given a sub-$\sigma$-algebra $\G \subseteq \F$ is a random variable $Y$ that satisfies the following conditions:
\begin{enumerate}[noitemsep]
    \item $Y \in \G$, i.e.\ $Y$ is $\G$-measurable, and
    \item $\E_\mu^\pi\left[Y\ind{A}\right] = \E_\mu^\pi\left[X\ind{A}\right]$ for all $A \in \G$.
\end{enumerate}
\end{definition}

Fortunately, the cylinder sets provide enough structure on the measure space that we can use (or recover) the ``elementary'' notion of conditional expectation.
Recall that $\F_n$ is a countably generated sigma-algebra, and by \Cref{prep:h-partition} the family of cylinder sets $\cyl{\H_n \times \A}$ forms a partition of $\H_\infty$. Therefore,
\beq
\E_\mu^\pi[X|\F_n](\w) \coloneqq \E_\mu^\pi[X|\cyl{ha}] = \frac{\E_\mu^\pi[X\ind{\cyl{ha}}]}{\mu^\pi(\cyl{ha})}
\eeq
where $h \in \H_n$ and $ha \sqsubseteq \w$. So, we only need the conditional expectation on the cylinder sets to properly define the conditional expectation for any sub-sigma-algebra $\F_n \subseteq \F$.

Let $X$ be $\F_{n+m}$-measurable, $ha \sqsubseteq \w$ and $\abs{h} = n$. The conditional expectation of such random variables can also be expressed as
\bqan
\E_\mu^\pi[X|\F_n](\w)
&= \frac{\E_\mu^\pi[X\ind{\cyl{ha}}]}{\mu^\pi(\cyl{ha})} \\
&\overset{(a)}{=} \sum_{h_{n+m} a_{n+m}} g_X(h_{n+m}a_{n+m})\fracp{\mu^\pi(\cyl{h_{n+m} a_{n+m}}\cap \cyl{h a})}{\mu^\pi(\cyl{ha})} \\
&\overset{(b)}{=} \sum_{h_{m} a_{m}} g_X(hah_{m}a_{m})\fracp{\mu^\pi(\cyl{hah_{m} a_{m}})}{\mu^\pi(\cyl{ha})} \\
&\equiv \sum_{h_{m} a_{m}} g_X(hah_{m}a_{m})\mu^\pi(h_ma_m|ha)
\eqan
where $(a)$ follows form similar steps as in \Cref{eq:xn-expectation} and $(b)$ is the result of the intersection of the cylinder sets.

We can also define the notion of conditional expectation for an arbitrary event in terms of the conditional expectation on the cylinder sets.

\begin{definition}[Conditional Expectation Given an Event]\label{def:cond-exp-event}
    Since $\F_n$ is generated by a countable partition (\Cref{prep:h-partition}) for any time-step $n$, the conditional expectation of any random variable $X: \H_\infty \to \SetR$ given any event $A \in \F_n$ is defined as
    \bqan
    \E_\mu^\pi\left[X|A\right]
    &\coloneqq \frac{\E_\mu^\pi[X\ind{A}]}{\mu^\pi(A)} \\
    &\overset{(a)}{=} \sum_{ha \in \H_n \times \A} \fracp{\E_\mu^\pi\left[X\ind{ \cyl{ha} \cap A}\right]}{\mu^\pi(A)} \frac{\mu^\pi(\cyl{ha} \cap A)}{\mu^\pi(\cyl{ha} \cap A)} \\
    &\overset{(b)}{=} \sum_{ha \in \H_n \times \A} \fracp{\E_\mu^\pi\left[X\ind{ \cyl{ha} \cap A}\right]}{\mu^\pi(\cyl{ha} \cap A)}\mu^\pi(\cyl{ha} | A) \\
    &= \sum_{ha \in \H_n \times \A} \E_\mu^\pi\left[X|\cyl{ha} \cap A\right]\mu^\pi(\cyl{ha} | A)\numberthis
    \eqan
    where $(a)$ is true because the cylinder sets from a partition and $(b)$ uses the definition of conditional distribution from \Cref{eq:conditional-dist}.
\end{definition}

As with the conditional distribution, we also abuse the notion for conditional expectation and express the cylinder set conditions as
\beq
\E_\mu^\pi[X|\F_n](\w) \eqqcolon \E_\mu^\pi[X|ha]
\eeq
where $h \in \H_n$ and $ha \sqsubseteq \w$.

Moreover, throughout this work, whenever we say that a result holds with probability 1 (w.p.1) or almost surely, we mean that the statement holds with $\mu^\pi$-probability 1:
\beq
\mu^\pi\Big(\{\w \in \H_\infty : \text{the result does not hold on } \w\}\Big) = 0
\eeq

That is, in practice the agent-environment interaction does not produce (or sample) an interaction history where the result does not hold.

\subsection{Properties of History-based Functions}

Let $X$ be any \emph{measurable} function (or random variable) defined over the measure space as:
\beq
X : \H_\infty \to Y^\SetN
\eeq
which may be a (stochastic) history-based function. We express a couple of important relationships about the conditional expectations. First, we show that a conditional expectation given a history can be expressed as a $\pi$-expected conditional expectation:
\bqan\label{eq:htoha}
\E_\mu^\pi\left[X|h\right]
&= \frac{\E_\mu^\pi\left[X\ind{\cyl{h}}\right]}{\mu^\pi(\cyl{h})} \\
&\overset{(a)}{=} \frac{1}{\mu^\pi(\cyl{h})} \int_{\H_\infty} X\ind{\cup_{a} \cyl{ha}} d\mu^\pi\\
&\overset{(b)}{=} \sum_{a}\frac{\mu^\pi(\cyl{ha})}{\mu^\pi(\cyl{h})}  \left(\frac{\int_{\H_\infty} X\ind{\cyl{ha}} d\mu^\pi}{\mu^\pi(\cyl{ha})}\right)\\
&\overset{(c)}{=} \sum_{a}\pi(a\|h) \E_\mu^\pi[X|ha] \numberthis
\eqan
where $(a)$ is due to the fact that $\cyl{h} = \cup_{a}\cyl{ha}$, $(b)$ holds because cylinder sets are disjoint, so $\sum_{a} \ind{\cyl{ha}} = \ind{\cyl{h}}$ and $(c)$ from \Cref{eq:mu-mu-rel} and \Cref{def:cond-exp-event}. Second, we express the similar relationship in the other direction.
\bqan
\E_\mu^\pi\left[X|ha\right]\label{eq:hatoh}
&= \frac{\E_\mu^\pi\left[X\ind{\cyl{ha}}\right]}{\mu^\pi(\cyl{ha})} \\
&\overset{(a)}{=} \frac{1}{\mu^\pi(\cyl{ha})} \int_{\H_\infty} X\ind{\cup_{e'} \cyl{hae'}} d\mu^\pi\\
&\overset{(b)}{=} \sum_{e'}\frac{\mu^\pi(\cyl{hae'})}{\mu^\pi(\cyl{ha})}  \left(\frac{\int_{\H_\infty} X\ind{\cyl{hae'}} d\mu^\pi}{\mu^\pi(\cyl{hae'})}\right)\\
&= \sum_{e'}\mu(e'\|ha) \E_\mu^\pi[X|hae']\numberthis
\eqan
where $(a)$ is due to the fact that $\cyl{ha} = \cup_{e'}\cyl{hae'}$ and $(b)$ holds because cylinder sets are disjoint, so $\sum_{e'} \ind{\cyl{hae'}} = \ind{\cyl{ha}}$. The above relationships hold for any random variable on the space. However, in the next section we specialize these relations to an important stochastic process that provides ``goals'' for the (G)RL agent.

\section{Goals and Targets}\label{sec:goals}

So far, we have purposefully side stepped the question of how to check (or decide) whether an agent is acting the ``right way'' in the environment. We settle this question in this section under the general heading of goals. Historically, this topic has turned out to be quite involved. There are many ways we can define a goal for an agent, but there is no single best answer for this, yet \cite{Leike2016b}. However, one thing is for sure that the agent should not try to greedily maximize the \emph{immediate} reward. The whole idea of dynamic programming and RL is built around the multi-step optimization \cite{Bertsekas1996,Sutton2018}. It is easy to imagine many simple environments where a goal of ``greedy one-step optimization'' does not lead to the ``right'' behavior, e.g.\ Chess \cite{Silver2018a} and Heaven-hell domains \cite{Hutter2000} demand long-term optimization, while greedy optimization leads to a ``poor'' performance.

However, if the agent has to do \emph{long-term} optimization then how should it prefer one reward sequence over the other? In other words which reward sequence should have higher ``utility'' for the agent? In the following, we formally answer this question, albeit \emph{not} uniquely!

\subsection{Reward Process}\label{sec:reward-process}

We provide a general structure for rewarding the agent. We allow for an unrestricted (but bounded) reward signal. The key property of the reward signal is that it is available immediately after taking the action. We formulate this as a filtered stochastic process.
Let the reward process $\{R_n: n \geq 1\}$ be a stochastic process:
\beq
R : \H_\infty \to \R^\SetN
\eeq
where $R_n$ is $\F_n$-measurable, i.e.\ $R_n \in \F_n$, for all time-steps $n$. That means, the reward process is \emph{adapted to the filtration}. The reward set $\R$ can be any bounded subset of the reals. In this work, we assume $\R \coloneqq [0,1] \subseteq \SetR$. The only necessary condition is the boundedness of the rewards. We can (re)scale the rewards from any $r \in [r_\min, r_\max]$ to $\t r \coloneqq \frac{r - r_\min}{r_\max - r_\min} \in [0,1]$ without changing the decision-making process \cite{Sutton2018}.

\begin{remark}[Reward Embedding]
    We only assume bounded rewards. However, if we allow the rewards to be embedded into the percept-space as $e_n = (o_n, r_n)$ then we also require $\R$ to be countable.
\end{remark}

As discussed above, the agent should not try to maximize the immediate reward, but some longterm sequence of rewards, which we call \emph{return} (also known as gain or utility).

\subsection{Return Process}

Let us assign a return random variable (as a proxy for the ``usefulness'' of a sequence of rewards) as
\beq
G : \H_\infty \to  \left(\SetR \cup \{\infty\}\right)^\SetN
\eeq
that is, $\{G_n : n \geq 1\}$ is a stochastic process. We allow $G_n$ to be infinite for some infinite histories and time-steps. However, for it to serve as a meaningful and well-defined notion of utility for (G)RL, it has to be finite almost surely for every policy $\pi$ and environment of interest $\mu$.

\begin{definition}[Admissible]
    A return stochastic process $\{G_n : n \geq 1\}$ is \emph{admissible} if $\sup_{\pi} \E_\mu^\pi[G_n] < \infty$ for all time-steps $n$.
\end{definition}

It is important to note that $G_n$ can be infinite even if the instantaneous rewards are bounded, i.e.\ bounded rewards are not sufficient for the admissibility of $\{G_n:n \geq 1\}$.

In the following subsection, we collect a variety of possible ways to specify goal for an agent.
As in standard RL, there are two major devisions in GRL when it comes to specify goals of the agent: a discounted and an undiscounted (G)RL.

\subsubsection{Undiscounted (G)RL}

The undiscounted formulation of (G)RL is pretty natural. However, it is difficult to formulate the problem for a broad class of environments using the undiscounted returns \cite{Mahadevan1996}. There are many different ways to define returns (or utilities) under the undiscounted umbrella.

\begin{itemize}
    \item \paradot{Infinite Sum of Rewards} In the most natural form, the agent can simply sum all future rewards. The future return $G_n(\w)$ on any $\w \in \H_\infty$ at time-step $n$ is defined as follows:
    \beq
    G_n(\w) \coloneqq \sum_{m = 1}^{\infty} R_{n+m}(\w)
    \eeq

    However, as required above, this criterion only makes sense if the sum is finite almost surely for every policy in the environments of interest. This might be a useful and objective criteria in a \emph{restrictive} set of environments where the rewards are absolutely summable for every infinite history \cite{Legg2007a}. For example, it is true in domains where there is an \emph{absorbing} ``state'' with a reward zero, and the state is reached almost surely for every policy \cite{Bertsekas1996}.

    \item \paradot{Truncated Average of Rewards} A general way to assure summability the rewards is to truncate the sum. The truncation length $T$ is known as the \emph{horizon} of the decision problem. The horizon is a finite number of steps up to which the agent should optimize the rewards:
    \beq
    G_n(\w) \coloneqq \frac{1}{T_n(\w)}\sum_{m = 1}^{T_n(\w)} R_{n+m}(\w)
    \eeq
    where $T$ is a stochastic process as $T: \H_\infty \to \left(\SetN \cup \{\infty\}\right)^\SetN$ and $T_n \in \F_n$ for every $n$. This formulation covers the standard finite horizon and episodic RL \cite{Bertsekas1996}.\footnote{The episodic formulation requires an extra condition to \emph{reset} the environment at the end of an episode.} It also includes the case where the horizon can be a function of ``state''. The horizon random process can be generated from a separate process/entity (other than the agent) which calculates the horizon for the agent, e.g.\ a human instructing the agent.

    \item \paradot{Limit Average of Rewards} The limit returns provides one of the most popular undiscounted RL setups \cite{Mahadevan1996}.
    \beq
    G_n(\w) \coloneqq \liminf_{T \to \infty} \frac{1}{T}\sum_{m = 1}^{T} R_{n + m}(\w)
    \eeq

    Interestingly, the use of $\liminf$ instead of $\lim$ in the above definition makes the return process admissible\footnote{It is not clear whether this ``cheap trick'' of using $\liminf$ instead of $\lim$ can lead to a meaningful resolve of ordering policies in an undiscounted RL setting. As $\liminf$ is same as $\lim$ if the limit exists, so we believe that this definition should help expand the undiscounted RL setup to a bit more domains. However, it is beyond the scope of this work to find such class of environment.}. However, in the undiscounted RL literature usually $\lim$ is used to define the return process. Therefore, similar to the infinite sum of rewards, to make this return process admissible, there has to be some form of (connectedness) structure on the environment \cite{Bartlett2009}. For example, this notion of return is usually used in ergodic and/or recoverable environments \cite{Osband2016a}. This (severely) limits the set of environments where the limit average of rewards is a valid notion of returns.
\end{itemize}

It is not apparent if there exists any undiscounted notion which can cover almost every domain of interest \cite{Mahadevan1996}. Another ``side effect'' of this notion is that the agents tend to be ``lazy''. This notion of undiscounted limit optimality does not distinguish an agent which gets high rewards at the end from one which has the average performance from the start. This indifference unfairly rewards the ``lazy'' behavior. These are some reasons that a discounted return formulation is preferred over the undiscounted returns \cite{Hutter2000}.

\subsubsection{Discounted (G)RL}

We borrow the notion of discounting from the economics literature to establish an admissible return process for every history without any further assumptions on the reward sequence. Let there be a discounting random process $\{\g_{\langle n, m\rangle}: n, m \geq 1\}$ such that $\g_{\langle n, m \rangle} \in \F_{n+m}$:
\beq
0 < \sup_{\pi} \E_\mu^\pi[\Gamma_n] \coloneqq \sup_{\pi} \E_\mu^\pi\left[\sum_{m=1}^\infty \g_{\langle n, m\rangle}\right] < \infty
\eeq
where $\Gamma_n$ is the ``to go discount weight'' for any time-step $n$.
We can use the following (general) discounted sum of rewards to assign a return for each infinite history $\w$ and time-step $n$:
\beq
G_n(\w) \coloneqq \sum_{m = 1}^{\infty} \g_{\langle n, m\rangle}(\w)R_{n+m}(\w)
\eeq

The above definition is quite general. It can be used to model state-based discounting \cite{Sutton2018}, time-consistent discounting \cite{Lattimore2014b}, and most importantly the standard geometric discounted RL setup.
The standard geometric discounting setup can be recovered from the above definition with $\g_{\langle n,m \rangle}(\w) = (1-\g)\g^{m-1}$ or $\g_{\langle n,m \rangle}(\w) = \g^{m-1}$ for a normalized or an unnormalized geometric discounted (G)RL respectively, for any scaler $\g \in [0,1)$.

\subsubsection{The Return Process for (G)RL}

We need a definition of return process which can be used to achieve \emph{optimal behavior} for a broad class of environments. In AGI, we prefer to do this for all domains of interest \cite{Legg2007a,Hutter2000}. At the moment, there is no consensuses which definition is the best. Both the discounted and undiscounted formulations have their own (de)merits \cite{Sutton2018,Mahadevan1996}. The preferred choice is to discount the rewards, primarily for its good convergence properties. However, there are many examples where a discounting does \emph{not} lead to optimal behavior \cite{Naik2019}. The quest for \emph{the} return process for (G)RL is still unsettled and beyond the scope of this work.

\subsection{Value Functions}

Once we have a notion for return, \emph{the goal} of a (G)RL agent, naturally, is to \emph{maximize the expected return}. The expected return is also known as the \emph{value} function(s) \cite{Sutton2018}. The history-based (action-)value functions of any finite history-action pair $ha$ are defined as:
\beq
V_\mu^\pi(h) \coloneqq \E_\mu^\pi[G_{\abs{h}} |h] \quad\text{and}\quad Q_\mu^\pi(ha) \coloneqq \E_\mu^\pi[G_{\abs{h}} |ha]
\eeq
where $V_\mu^\pi$ is the \emph{history-value} function and $Q_\mu^\pi$ is called the \emph{history-action-value} function. It is easy to see that these functions are related as:
\beq
V_\mu^\pi(h) = \E_\mu^\pi[Q^\pi_\mu(ha) |h] = \sum_{a\in\A} Q_\mu^\pi(ha)\pi(a\|h) \eqqcolon Q_\mu^\pi(h\pi(h))
\eeq
where we use the notation $Q(h\pi(h))$ to denote the expectation of any action-value function $Q$ when the actions at the current time-step are chosen by the (stochastic) policy $\pi$.
The maximally achievable values by any agent (or policy), which are known as the \emph{optimal} value functions, are defined as
\beq
V^*_\mu(h) \coloneqq \sup_{\pi} V^\pi_\mu(h) \quad\text{and}\quad Q_\mu^*(ha) \coloneqq \sup_{\pi} Q^\pi_\mu(ha)
\eeq
for any history-action pair $ha$. The value functions define the goals for a GRL agent, and the ``target'' for an agent is to find a policy which (nearly) achieves this goal.

\subsection{Value-maximizing Policies}

We now embark on the quest to formally define a target for GRL agents. The target of the agent is to behave like any policy from a set of (nearly) value-maximizing policies. In this work, we use the notion of \emph{$\eps$-optimality}. We accept a policy as an \emph{$\eps$-optimal} policy if its value is ``$\eps$-close'' to the optimal value. There are different notions of this  ``$\eps$-closeness'', which leads to different targets for the agents \cite{Leike2016b}.

\begin{itemize}
    \iparadot{Uniformly Value-maximizing Policies} An important class of policies, which are simply known as \emph{$\eps$-optimal} policies in the literature \cite{Sutton2018}, is a set of uniformly value-maximizing policies $\Pi_\eps^\sup$. Each policy in $\Pi_\eps^\sup$ has the history-value function $\eps$-close to the optimal value function for every history.
    \beq\label{eq:uniform-v-max-policy}
    \pi \in \Pi^\sup_\eps(\mu) \iff \sup_{h} \abs{V_\mu^*(h) - V^\pi_\mu(h)} \leq \eps
    \eeq

    Note that $\Pi_\eps^\sup$ is a function of the environment $\mu$. Later, we want an agent that is optimal for as many environments as possible. Ideally, the agent should be able to achieve the $\eps$-optimal value at every history.

    For a geometrically discounted GRL framework $\Pi_\eps^\sup$ is not empty \cite{Lattimore2014b} for any arbitrary action-space. Moreover, under the finite action-space assumption an optimal policy always exists, i.e.\ there exists a policy which achieves the optimal value exactly \cite{Lattimore2014b}.

    To be fair, this is too strict a criterion for the agent to meet at \emph{every} history of interaction. The agent might not have enough information at the start. For example, in the learning case the agent does not know the environment a priori. It may do some non-optimal actions in the beginning to gain some information about the environment \cite{Hutter2000}.

    \iparadot{Asymptotically Value-maximizing Policies} The policies in $\Pi_\eps^\sup$ demand the agents to be $\eps$-optimal from the beginning, which might not be a realizable target. In realistic RL situations, an agent may not know the environment a priori, so it may not be able to use a policy from $\Pi_\eps^\sup$. Therefore, it is not unreasonable to only require that the agent eventually gets to a value-maximizing policy. We expand the set of $\eps$-optimal polices $\Pi_\eps^\sup$ to the limit value-maximizing policies as
    \beq\label{eq:limit-v-max-policy}
    \pi \in \Pi^\infty_\eps(\mu) \iff \limsup_{n \to \infty} \abs{V_\mu^*(\w_{1:n}) - V_\mu^\pi(\w_{1:n})} \leq \eps \quad \text{almost surely}.
    \eeq
    This class of policies is known as the set of \emph{asymptotically $\eps$-optimal} policies \cite{Hutter2000}. This is one of the least demanding set of policies. A typical behavior in this set is that the agent ``learns'' the environment in the beginning, and later exploits the environment to achieve optimal value in the limit \cite{Leike2016}.

    Sometimes the following equivalent representation of $\Pi^\infty_\eps$ is more useful:
    \beq\label{eq:fintie-v-max-policy}
    \pi \in \Pi^\infty_\eps(\mu) \iff \sum_{n=1}^{\infty} \ind{\abs{V_\mu^*(\w_{1:n}) - V_\mu^\pi(\w_{1:n})} > \eps} < \infty \quad \text{almost surely}.
    \eeq
    which highlights the connection between the notion of optimality and the number of ``wrong'' actions a policy makes.
    An important subclass of such policies is known as Probably Approximately Correct (PAC) policies, where we demand that a PAC-policy should have an upper bound on the number of mistakes polynomial in the key parameters\footnote{The usual parameters are the discount-factor, error margin, and some connectedness parameter of the environment such as the ``diameter'' of an ergodic MDP \cite{Kearns2002}.} of the setup \cite{Kakade2003,Modi2019}, but only with high probability.

    \iparadot{Pseudo-regret Minimizing Policies} Sometimes, it is natural to control the total (average) loss an agent occurs throughout the run. For this, the notion of \emph{regret} is a well-established measure in the undiscounted RL community \cite{Mahadevan1996}. This notion is problematic in the discounted setting, though \cite{Sutton2018}. We denote the following class of policies as ``pseudo-regret'' $\eps$-optimal policies:
    \beq\label{eq:regret-min-policy}
    \pi \in \overline\Pi^\infty_\eps(\mu) \iff \limsup_{n \to \infty} \frac{1}{n}\sum_{m=1}^{n} \abs{V_\mu^*(\w_{1:m}) - V_\mu^\pi(\w_{1:m})} \leq \eps \quad \text{almost surely}.
    \eeq

    In comparison to finite sample complexity policies, pseudo-regret minimizing policies can make infinitely many mistakes but either not too many or not too large.
\end{itemize}

\begin{proposition}
    The sets of value-maximizing polices are related as
    \beq
    \Pi^\sup_\eps(\mu) \subseteq \Pi^\infty_\eps(\mu) \subseteq \overline\Pi^\infty_\eps(\mu)
    \eeq
    for any environment $\mu$ and $\eps$.
\end{proposition}
\begin{proof}
    The proof trivially follows from the definitions of $\Pi^\sup_\eps(\mu)$, $\Pi^\infty_\eps(\mu)$, and $\overline\Pi^\infty_\eps(\mu)$.
\end{proof}

Due to the above proposition, we may limit the search for a value-maximizing policy in $\overline\Pi_\eps^\infty$ with a bias towards the policies in $\Pi_\eps^\infty$ or $\Pi_\eps^\sup$. Any policy in $\overline\Pi_\eps^\infty$ has at least the guarantee to be on average a value-maximizing policy in the limit.

\begin{remark}[Almost Surely Weak Notions]
    There is a serious weakness of all the above optimality notions which use ``almost surely'' quantifier. These value-maximizing policies are guaranteed to be optimal on the trajectories (or infinite histories $\w$) that are taken by the (learning) agent $\pi$. That is, the agent fulfills the criteria with $\mu^\pi$-probability 1, which is inherently troublesome. An (oracle) agent who knows the optimal policy $\pi_\mu \in \Pi_\eps^\sup(\mu)$ from the start might take a different set of trajectories with $\mu^{\pi_\mu}$-probability 1 and the agent $\pi$ may suffer a huge loss on these trajectories. This is a well known flaw of such optimality notions \cite{Kakade2003}. For example, an agent would be (asymptotically) optimal if it jumps into a trap because given the history of being inside the trap the agent is optimal thereafter.
    This is one of the reasons that the environments are typically assume to be ergodic, so the set of supported trajectories are the same under every policy \cite{Kearns2002}. A conclusive resolution of this issue is beyond the scope of this work.
\end{remark}

In the rest of the thesis, whenever we say that an agent is \emph{behaving optimally}, we mean that the agent is acting according to one of the $\eps$-optimal policies defined above. The ``target'' of a GRL agent is to find a policy form one of these sets.

\begin{remark}[Value-maximization is Optimal Behavior]
    Ideally, the target value-maximizing policy satisfies the (intuitive) optimal behavior defined in \Cref{def:opt-behavior}. However, this is not guaranteed  \cite{Naik2019}. There are examples when the agent has found a bug in the system to get more rewards using a rather ``non-optimal'' behavior \cite{Leike2017}.
    We leave this question for future research to determine when any of the value-maximizing policies defined in \Cref{eq:uniform-v-max-policy,eq:limit-v-max-policy,eq:regret-min-policy} are indeed optimal behavior stipulated in \Cref{def:opt-behavior}.
\end{remark}

It is apparent that choice of the ``target'' policy is a design choice. It defines the nature of the agents we will get in the end. If the target policy is in $\Pi_\eps^\infty$ then the agent may only have asymptotic guarantees, whereas the agents targeting to find (and use) a policy from $\overline\Pi_\eps^\infty$ may only be asymptotically optimal on average. However, with some strong structural assumptions, these asymptotic guarantees can be converted into some approximate finite bounds \cite{Hutter2000}. In this thesis, we set the target for our GRL agents to find a policy from $\Pi_\eps^\sup$.

\section{The GRL Framework}

So far, we have listed many possible alternatives and design choices one can make in a GRL setup. In the remainder of the thesis we use the standard geometrically discounted rewards.

\begin{definition}[Geometrically discounted GRL]
    For any $\g \in [0,1)$, we define an unnormalized $\g$-discounted GRL setup which has the following return process:
    \bqan
    G_n(\w) &\coloneqq \sum_{m = 1}^{\infty} \g^{m-1}R_{n+m}(\w)
    \eqan
    for any infinite history $\w$ and time-step $n$. The choice of target policy is left open. However, usually, the target policy is required to be a member of $\Pi_\eps^\sup$-the set of uniform value-maximizing policies.
\end{definition}

It is easy to see that for a $\g$-discounted return the following recursion holds:
\beq\label{eq:grel}
G_n = R_{n+1} + \g G_{n+1}
\eeq
for any time-step $n$, which leads to the following relationship between the history-value and history-action-value functions:
\bqan
Q^\pi_\mu(ha)
&\coloneqq \E_\mu^\pi\left[G_{\abs{h}} \middle| ha\right] \\
\overset{\eqref{eq:hatoh}}&{=} \sum_{e'}\mu(e' \| ha) \left(\E_\mu^\pi\left[G_{\abs{h}} \middle| hae'\right] \right)\\
\overset{\eqref{eq:grel}}&{=} \sum_{e'}\mu(e' \| ha) \left(\E_\mu^\pi\left[R_{\abs{h}+1} \middle| hae'\right] + \g \E_\mu^\pi\left[G_{\abs{h}+1} \middle| hae'\right] \right)\\
\overset{\eqref{eq:htoha}}&{=} \sum_{e'}\mu(e' \| ha) \left(r(hae') + \g V_\mu^\pi(hae') \right) \numberthis
\eqan
where $r(h) \coloneqq \E_\mu^\pi\left[R_{\abs{h}} \middle| h\right]$ because the reward process is adapted to the filtration.
Although the above recursion is a pseudo-recursion, as no the history is ever repeated, we call the above equation the general Bellman equation (GBE) similar to the Bellman equation (BE) in standard RL \cite{Sutton2018}. Similarly, the optimal history-value and history-action-value functions satisfy the following general optimal Bellman equations (GOBE):
\bqa
V_\mu^{*}(h) &= \sup_{a} \left(r_\mu(ha) + \g \sum_{e'} \mu(e'\|ha) V_\mu^*(hae')\right)\\
Q_\mu^{*}(ha) &= r_\mu(ha) + \g \sum_{e'} \mu(e'\|ha) \sup_{a'} Q_\mu^{*}(hae'a')
\eqa
for any history action pair $ha$, where we use $r_\mu(ha) \coloneqq \sum_{e'} \mu(e'\|ha) r(hae')$.

\begin{remark}[History- vs History-action-based Reward Signals]
Recall, due to \Cref{asmp:reward-hypothesis}, we have a reward function
\beq
r : \H \to \SetR
\eeq
which is simply a reward random process adapted to the filtration. This reward signal is \emph{universal} in the sense that it evaluates each history \emph{independent} of the environment $\mu$. However, the real objective of the agent is to optimize the expected value of this reward for any history-action pair that depends on the environment:
\beq
r_\mu(ha) \coloneqq \sum_{e'} \mu(e'\|ha)r(hae')
\eeq
which quantifies the expected next reward the agent receives at the history $h$ if it takes action $a$. In the case of reward embedding we get the similar relationship as
\beq
r_\mu(ha) \coloneqq  \sum_{r'o'} r' \mu(o'r'\|ha) = \sum_{r'} r'\mu(r'\|ha)
\eeq

We use both notations, but $r_\mu$ is preferred. Using this history-action based reward function makes more sense, as the agents usually\footnote{Unless we break the dualistic setup.} do not control the next percept transition. They effectively experience the expected percept (hence, experience the expected reward) anyway, so there is little need for using a history-based reward function.
\end{remark}

This concludes the GRL setup. In its pure form, the optimal policy from any value-maximizing sets $\Pi_\eps^\sup$, $\Pi_\eps^\infty$, or $\overline\Pi_\eps^\infty$ may be a function of history. Hence, these policies may be hard if not \emph{impossible} to learn with finite computational resources. In the next chapter we augment this setup with an abstraction map, which enables us to \emph{curtail the history-dependence} of the optimal policies and other quantities such as history-(action-)value functions for many interesting abstractions with a little to \emph{no restrictions} on the class of environments.

\section{Summary}
This chapter laid down a measure theoretically sound framework of GRL for decision-making under uncertainty. There are many design choices one can make to get different variants of GRL. There are a number of possibilities to define values (goals) and optimal policies (targets), in general. We use $\g$-discounted returns to define the geometrically discounted GRL. This choice provides a series of (pseudo) recursive Bellman equations, GBE and GOBE. We left the choice of target policy open, because the setup can be used for different target policies without any change.

\chapter{Abstraction Reinforcement Learning}\label{chap:arl}

\begin{outline}
    In the previous chapter, we formulated the geometrically discounted GRL setup, which we are going to refer to simply as \emph{the} GRL framework. The GRL framework in its current from can model any domain, but the resultant optimal policy may depend on the observations from an arbitrary number of previous time-steps. This limits the practical application of GRL, as the policy learned at a history may become \emph{invalid} on the next time-step. To curtail the history-dependence of the optimal policies, this chapter introduces the idea of a GRL setup with an abstraction called abstraction reinforcement learning (ARL). In ARL, the agent has access to an abstraction which reduces the original history-action space to a finite state-action space. The ARL setup allows us to do either state-only, action-only, or state-action abstractions.
\end{outline}

\epigraph{\it ``Being abstract is something profoundly different from being vague $\dots$ the purpose of abstraction is not to be vague, but to create a new semantic level in which one can be absolutely precise.''}{--- \textup{Dijkstra}}

\section{Introduction}

As highlighted in \Cref{chap:intro}, decision-making is the most important aspect of intelligence. The ability to predict the consequences of ones actions is so vital that most artificial intelligence (AI) research is concentrated on perfecting this ability in machines \cite{Hutter2000,Pomerol1997}. An intelligent agent needs some discerning information about the present situation so it can take an appropriate action. For example, if the traffic light is red then an \emph{intelligent} driver should not enter the intersection. Ideally, this choice of action is independent of any other sensory information given the traffic light is red. We, homo sapiens, have this unparalleled ability to attend only to the \emph{useful} aspects of the world around us. The sensory data our bodies receive is overwhelming, yet we pay attention to only a fraction of it. The information we keep at \emph{present} turns out to be relevant for our \emph{future} actions.

Moreover, we make a versatile model of the world in the sense that it can be adapted from one task to another. The level of information provided (or modeled) by our mental state of the world is depends on the task at hand. For example, we do not need a lot of other information apart from the chess board while we are playing chess, but we do need to consider a wealth of factors when designing a quantum experiment. We are so fluid in adding and removing (ir)relevant information in our mental state of the world that we usually do not fully appreciate this ability. An AGI agent should be able to \emph{adapt} to every \emph{plausible}\footnote{A set of all plausible tasks is a set of all tasks we care, which is a subset of the set of all \emph{possible} tasks. The set of all possible tasks is ``too big'' to allow for a universal optimizer \cite{Hutter2000}.} task if it is has sufficiently long experience interacting with the task \cite{Lattimore2014a,Leike2016}.

From an information theoretic perspective, it is sufficient to keep the \emph{complete} interaction history to make predictions about the future \cite{Hutter2000}. However, the world we live in and the world these agents will inhabit, has a lot of structure such that for most tasks the complete history may not be \emph{necessary} for optimal behavior, i.e.\ most of the tasks could` be modeled as bounded memory sources \cite{Ryabko2008}. An \emph{implementable/realizable} AGI agent should be able to compress (or abstract) its experience in an abstract state in such a way that using this (abstract) state of the world it can plan (and predict) the  future \emph{optimal} course of actions \cite{Veness2014}.

Reinforcement learning is considered to be one of the most promising learning paradigms for AGI \cite{Hutter2000}. However, the standard RL paradigm is tightly connected to (finite-state) Markov decision processes (MDP). The MDP structure defines the nature of information the agent keeps: given the current state of an MDP (and the MDP itself) the agent can predict the future states for every contingency \cite{Sutton2018}. Usually, the ability to predict \emph{every} situation is \emph{not necessary} for optimal behavior in many domains. The agent can achieve optimal performance by predicting only ``valuable'' trajectories \cite{Hutter2016}. However, the \emph{sufficient} information required to achieve (near) optimal performance may not be Markovian in the sense that the agent can not predict next state given the current abstract state \cite{Hutter2016}.

Unlike the standard RL setup, in the GRL framework defined in \Cref{chap:grl} every (historical) situation is \emph{unique}.
In theory, GRL can model every task which can be described under the \emph{reward hypothesis}. Under the natural computability assumptions about the real-world, GRL has provided a plethora of interesting and insightful results \cite{Hutter2000,Lattimore2014a,Leike2016}. However, this line of research has overlooked an important aspect of the resultant policies. In almost every proved result, this regime provides a history-based policy. The history-based polices in GRL are difficult to learn. So, the (potential) optimal policy derived through these methods for one history may never be reused for any other history. History-dependence of optimal policy is the normal rather than the exception in the field of GRL (without abstractions). Therefore, even being a powerful (theoretical) framework, GRL without further structure is not very useful for an \emph{implementable} AGI.

In this thesis, we emphasize that GRL with an (appropriate) set of abstractions, which we call abstraction RL (ARL), is a potential pathway to the ultimate goal of an \emph{implementable} AGI. An ARL agent should be able to \emph{tractably} plan the optimal course of actions by exploiting the relatively compact model provided by an abstraction.

\section{Abstraction as a Model}

An abstraction provides a \emph{model} of the environment. However, the word ``model'' carries slightly different meaning among different sub-communities of RL. In the standard RL framework the model may imply any estimate of the state transition function (also known as a transition kernel) of the environment \cite{Kuvayev1996}. The Bayesian RL community considers a model being a Bayesian mixture over a class of transition kernels \cite{Vlassis2012}.\footnote{We deliberately used the term ``class of functions'' instead of ``class of models'' to highlight the fact that the Bayesian mixture is indeed the actual model of the environment used in Bayesian RL. The choice of the class of functions affects the Bayesian mixture, but the agent reacts based on the mixture only.}

The notion of model is not very clear among the ``model-free'' RL community though, e.g.\ Q-learning is a model-free algorithm which assumes a Markovian ``model''. Traditionally, model-free RL lacks a (full) predictive estimate of the transition kernel of the environment \cite{Mnih2015}. Such model-free algorithms do not estimate the transition function of the environment, albeit typically under the assumption that the environment is an MDP with a known state-space \cite{Sutton2018}. The state-action-value function, \emph{cf.} the history-action-value function, is directly estimated in these model-free algorithms.

Model-free algorithms are quite useful, as the state-action-value function can be used by the agent to act (optimally) even without having an explicit access to an estimation of the transition function. Therefore, in model-free RL the state-action-value function models (only) the useful aspects of the environment without being able to predict the next ``state''. It lacks the ``state-predictive'' model, but model-free algorithms estimate state-action values, which is only what matters under \Cref{asmp:reward-hypothesis}, as the optimal policy is, trivially, a maximizer of the state-action-value function.  Given this estimate, the agent does not need to know more to act optimally in the environment \cite{Majeed2018}. Therefore, in a broader sense the so-called model-free RL does have a (``state-action-value non-predictive'') model of the environment.

This thesis considers world models in a more unified and broader sense. We move away from the distinction of a model being an estimate of the environment dynamics, i.e.\ the transition kernel. We use the notion of abstraction, a map from histories to a (finite) set of states, instead of a (not well agreed upon) notion of model. In the following, we provide an extensible framework by augmenting the GRL setup with an abstraction. The nature of the abstraction dictates the properties of the overall setup. Later, we show that ARL can mimic almost all prevalent RL setups by changing the abstraction map.

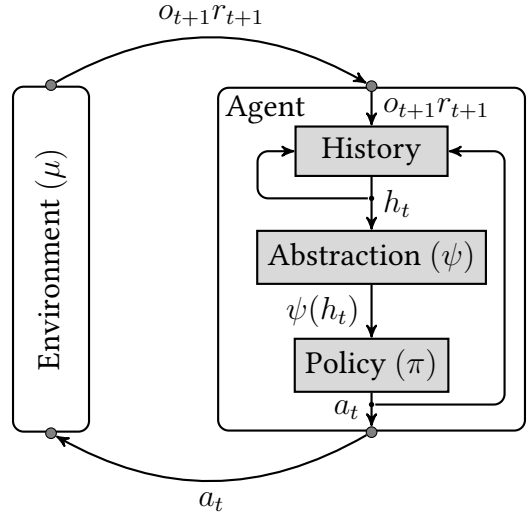
\begin{figure}[!ht]
    \centering
    \begin{tikzpicture}[
        node distance = 7mm and -3mm,
        innernode/.style = {draw=black, thick, fill=gray!30,
            minimum width=2cm, minimum height=0.5cm,
            align=center},
        outernode/.style = {draw=black, thick, rounded corners, fill=none,
            minimum width=1cm, minimum height=0.5cm,
            align=center},
        endpoint/.style={draw,circle,
            fill=gray, inner sep=0pt, minimum width=4pt},
        arrow/.style={->,thick,rounded corners},
        point/.style={circle,inner sep=0pt,minimum size=2pt,fill=black},
        skip loop/.style={to path={-- ++(#1,0) |- (\tikztotarget)}},
        every path/.style = {draw, -latex}
        ]
        \node (start) {Start};
        \node (h) [innernode]{History};
        \node (phi) [innernode, below=of h]{Abstraction $(\psi)$};
        \node (pi) [innernode, below=of phi]      {Policy $(\pi)$};
        \node [outernode, align=left, inner sep=0.5cm, fill=none, fit=(h) (phi) (pi)] (agent) {};
        \node[below right, inner sep=3pt, fill=none] at (agent.north west) {Agent};
        \node[outernode, left=120pt of agent, fit=(agent.north)(agent.south), inner sep=0pt] (env) {};
        \node[below right, inner sep=0pt, fill=none, rotate=90, anchor=center] at (env) {Environment $(\mu)$};
        \node[endpoint, above= -2pt of env] (or_env) {};
        \node[endpoint, below= -2pt of env] (a_env) {};
        \node[endpoint, below= -2pt of agent] (a_agent) {};
        \node[endpoint, above= -2pt of agent] (or_agent) {};

        \path (a_agent) edge[arrow,bend left] node[below]{$a_t$} (a_env);
        \path (or_env) edge[arrow, bend left] node[above]{$o_{t+1}r_{t+1}$} (or_agent);
        \path (or_agent) edge[arrow] node[right]{$o_{t+1}r_{t+1}$} (h);
        \path (h) edge[arrow] node[above=0.5pt,midway,name=h_phi,point]{} node[right]{$h_t$} (phi);
        \path (phi) edge[arrow] node[left]{$\psi(h_t)$} (pi);
        \path (pi) edge[arrow] node[above=0.5pt,midway,name=pi_a,point]{} node[left]{$a_t$} (a_agent);
        \path (pi_a) edge[arrow, skip loop=1.75cm] (h.east);
        \path (h_phi) edge[->, skip loop=-1.5cm, thick, rounded corners] (h.west);
    \end{tikzpicture}
    \caption{An interaction cycle in abstraction reinforcement learning (ARL).}
    \label{fig:arl}
\end{figure}


\section{Interaction Through an Abstraction}

An agent in ARL is constrained to interact through an abstraction, see \Cref{fig:arl}. It observes only the states of the abstraction, and does not directly receive the percepts from the environment. As the history of interaction is the only viable ``state'' of the underlying environment in general, the abstraction map is performing a ``state'' reduction of the environment. Trivially, the original GRL setup can be recovered if we allow for the identity abstraction, i.e.\ $\psi(h) = h$ for any history.

\begin{definition}[State-only Abstraction]\label{def:abstraction}
    An abstraction $\psi$ is a map from the \emph{history} to some \emph{abstract state}:
    \beq
    \psi: \H \to \S
    \eeq
    where $\S$ is the set of states. If finite, the size of the state space is denoted by $S$.
\end{definition}

Let the agent have access to an \emph{abstraction}. In this work, we assume that $\S$ is any finite set, and the mapping $\psi$ is deterministic.\footnote{This thesis uses a specialized notion for ARL. A fully-general ARL framework should allow for stochastic abstractions with a (countably) infinite set of states. However, the specialized problem is already non-trivial and sufficiently interesting. Therefore, we leave this extension for a future work.} In ARL, the agent may have access to a state-action abstraction, which in standard RL is known as a \emph{homomorphism} \cite{Whitt1978}.

\begin{definition}[State-action Abstraction]\label{def:homo}
    An abstraction $\psi$ is a homomorphism if it is a map from the \emph{history-action} pairs to some \emph{abstract state-action} pairs:
    \beq
    \psi: \H \times \A \to \S \times \B
    \eeq
    where $\S$ is the set of states and $\B$ is the set of abstract actions.
\end{definition}

This chapter expands on the state-only abstraction, while the state-action abstraction (or homomorphism) is explored in \Cref{chap:representation-guarnt}. One of the key differences between ARL and GRL is that in ARL the agent ``pretends'' that the states of the abstraction are Markovian, i.e.\ the agent uses abstraction as an MDP. We formally define this ``pretend'' MDP shortly. Interestingly, there are many non-MDP abstractions where this ``pretending'' is actually useful. A major part of this thesis deals with such non-MDP abstractions. The agent does not loose much in terms of performance by regarding the states of these non-MDP abstractions as Markovian, see \Cref{chap:representation-guarnt,chap:convergence-guarnt} for some examples. The following section formalizes this idea of the agent pretending the states are Markovian.

\section{Abstract Environment}

The power of the abstraction map defined in \Cref{def:abstraction} becomes apparent when we compare it to the prevalent notion of state-abstraction of MDPs \cite{Li2006,Abel2016}. If the environment $\mu$ is an MDP then the most recent observation is a sufficient information for the agent. In this MDP case, the state-abstraction is from the percepts, the state-space of the underlying MDP, to the abstract state-space of the abstracted process, i.e.\ $\psi_{\rm MDP}: \OR \to \S$. So, the map $\psi_{\rm MDP}$ is fixed and stationary, i.e.\ it does not depend on time \cite{Abel2016}. However, the abstraction $\psi$ considered in ARL is potentially non-stationary and history dependent. In ARL, we can model the situations where the abstraction map may evolve (i.e.\ change over time) as the agent gains new information. For instance the state of $\psi$ may be a (discretized) belief vector if $\mu$ is a POMDP. Moreover, the agent can switch to a completely different abstraction based on the history of interaction, which is not possible in the standard state-abstraction setup.
Therefore, the type of problems considered by ARL are significantly harder and richer than the standard state-aggregation paradigm, \emph{cf.} \citet{Abel2016}.

Interestingly, if an agent is interacting with the environment through an abstraction then it can (and should only) consider the resultant abstract process as the ``true'' environment. By design, the agent should not use (or have access to) the ``original'' interaction history $h \in \H$. However, it may use the complete ``abstract'' interaction history, which is similar to the original interaction history except the percept are replaced by the abstract states.\footnote{The reward signal is abstracted as an ``average'' reward signal using a dispersion distribution, see \Cref{def:surrogate-mdp}.}

Let $\tau \coloneqq s_1 a_1 \dots s_{n-1} a_{n-1} s_n$ be such an abstract interaction history at a time-step $n$, where the states $s_m \coloneqq \psi(h_m)$ are generated by the abstraction from the original histories $h_{m}$. We define the set of all finite abstract histories as
\beq
\H^\psi \coloneqq \bigcup_{m=1}^\infty \H^\psi_n
\eeq
where $\H^\psi_n \coloneqq (\S \times \A)^{n-1} \times \S$  denotes the set of abstract histories of length $n \in \SetN$.
Moreover, we define the (equivalence) class of original histories $[\tau]$ which are indistinguishable from the abstract history $\tau$:
\beq
[\tau] \coloneqq \{ h \in \H_{\abs{\tau}} \mid \forall m \leq \abs{\tau}.\ s_m(\tau) = \psi(h_{m}(h)), a_m(\tau) = a_m(h)\}
\eeq
where, $h_m(\cdot)$, $s_m(\cdot)$ and $a_m(\cdot)$ represent the history, state and the action at the time-step $m$.
Let us imagine that an agent is interacting with the true (history-based) environment, but it can only observe (and store/use) the abstract histories. The agent would effectively be interacting with the following abstract history-based process, which we also call \emph{an} abstract environment,\footnote{The abstraction environment depends on the choice of the history-conditional dispersion distribution, see \Cref{rem:choice-of-mixing}.} for any abstract history-action pair $\tau a$:
\beq
\mu_\psi(s' \| \tau a) \coloneqq \sum_{h \in [\tau]} \mu_\psi(s'\|ha)\Pr(h \| \tau a)
\eeq
where
\beq
\mu_\psi(s'\|ha) \coloneqq \sum_{e':\psi(hae') = s'} \mu(e'\|ha)
\eeq
and $\Pr(h \| \tau a)$ is any (arbitrary) \emph{history-conditional} dispersion distribution over the original indistinguishable (fixed-length) histories given the (same length) abstract history and action pair. It may be generated (or induced) by a sampling/behavior policy of the agent, see \Cref{rem:choice-of-mixing} for an example.
However, it is not necessary that this distribution is generated \emph{on-policy}.\footnote{By on-policy we mean that the agent uses the target policy (the one it wants to learn, ideally the optimal policy) also as the sampling policy (the one used to interact with the environment during learning).} There is no such restriction in ARL. It is possible to imagine a situation where a different sampling/behavior policy is used to generate this distribution, which is later used by the agent to come up with the abstract environment \cite{Hutter2016}.
It is critical to remember that, in general the abstract process may not be an MDP, i.e.\ $\mu_\psi(s'\|\tau a) \neq \mu_\psi(s'\|\d \tau a)$ even if $s_{\abs{\tau}}(\tau) = s_{\abs{\d \tau}}(\d \tau)$.

\begin{remark}[Generated/induced History-conditional Dispersion Distribution]\label{rem:choice-of-mixing}
    A natural choice of the history-conditional dispersion distribution $\mu^{\pi_B}_n : \H^\psi_n \times \A \to \Dist(\H_n)$ is that it is generated/induced by a behavior policy $\pi_B$ as
    \bqan
    \mu^{\pi_B}_{\abs{\tau}}(h\|\tau a)
    &\coloneqq \frac{\mu^{\pi_B}(ha)\ind{h \in [\tau]}}{\mu^{\pi_B}(\tau a)}
    \eqan
    for any abstract history-action pair $\tau a$, where, we defined the probability of visiting $\tau a$-pair as $\mu^{\pi_B}(\tau a) \coloneqq \sum_{h \in [\tau]} \mu^{\pi_B}(ha)$.
    Interestingly, this distribution can be made \emph{independent} of the choice of the behavior policy under a mild (and natural) condition that the behavior policy is only a function of the abstract history, i.e.\ $\pi_B(a \| h) = \pi_B(a \| \tau)$ for any action $a$ and history $h \in [\tau]$.
    \bqan
    \mu^{\pi_B}_{\abs{\tau}}(h\|\tau a)
    &= \fracp{\mu(e_1)\prod_{m=1}^{\abs{\tau}-1} \mu(e_{m+1}\|h_{m}a_{m}) \pi_B(a_{m}\|h_{m})\pi_B(a\|h)}{\sum_{\d h \in [\tau]} \mu(\d e_1)\prod_{m=1}^{\abs{\tau}-1} \mu(\d e_{m+1}\|\d h_{m} \d a_{m}) \pi_B(\d a_{m}\|\d h_{m})\pi_B(a\|\d h)}\ind{h \in [\tau]}\\
    &= \fracp{\mu(e_1)\prod_{m=1}^{\abs{\tau}-1} \mu(e_{m+1}\|h_{m} a_{m})}{\sum_{\d h \in [\tau]} \mu(\d e_1)\prod_{m=1}^{\abs{\tau}-1} \mu(\d e_{m+1}\|\d h_{m} \d a_{m})}\ind{h \in [\tau]}
    \eqan
    where, for brevity, we suppress the explicit history dependence in the notation, e.g.\ $e_{m+1} \equiv e_{m+1}(h)$ and $\d e_{m+1} \equiv e_{m+1}(\d h)$.
\end{remark}

\section{Surrogate MDP}

Although the abstract environment may not be an MDP, in ARL the agent ``pretends'' that the abstract states are Markovian. It results into a ``pretend'' surrogate MDP, which we formally define in this section.
As hinted in the introduction, the standard RL methods are designed to work on finite-state MDPs, and there exists a number of efficient methods to find an optimal policy in a finite-state MDP \cite{Kearns2002,Strehl2006,Strehl2009}. That is why we are interested in defining the surrogate MDP of an abstraction because if the (near) optimal policy of the original environment is \emph{representable} through the optimal policy of the surrogate MDP then we can leverage the standard RL methods for a broad class of environments; see \Cref{chap:convergence-guarnt} for an example.

Once the agent starts treating the states as Markovian states, the agent is effectively working with a \emph{surrogate MDP} $\b\mu$. We define this \emph{stationary}\footnote{It is stationary in the sense that the transition kernel $\b\mu$ does not depend on time.} MDP from the abstract environment as follows:
\bqan\label{eq:surrogate-mdp}
\b\mu(s' \| sa) &\coloneqq \sum_{\tau} \mu_\psi(s' \| \tau a) \Pr(\tau \| sa) \\
&= \sum_{\tau}\Pr(\tau \| sa) \sum_{h \in [\tau]} \mu_\psi(s'\|ha)\Pr(h \| \tau a) \\
&= \sum_{h} \mu_\psi(s'\|ha) \left(\sum_{\tau}\Pr(h \| \tau a) \Pr(\tau \| sa)\right) \\
&= \sum_{h} \mu_\psi(s'\|ha) \Pr(h \| s a) \numberthis
\eqan
where $\Pr(\tau \| sa)$ is an (arbitrary) \emph{state-conditional} dispersion distribution over the abstract histories that ends at $s$, and, for brevity, we assume that $\Pr(h\|\tau a) = 0$ if $h \not\in [\tau]$ and $\Pr(\tau \| sa) = 0$ if $s_{\abs{\tau}}(\tau) \neq s$. Recall that, in case $\Pr(h\|\tau a)$ is generated by a (behavior) policy, the history-conditional dispersion distribution could be made only a function of the environment under the mild condition that the (behavior) policy is only a function of the abstract history, see \Cref{rem:choice-of-mixing}. However if generated by a (behavior) policy, the state-conditional dispersion distribution is \emph{always} a function of the (behavior) policy and the environment, see \Cref{rem:variable-length-mixing}.

\begin{remark}[Generated/induced State-conditional Dispersion Distribution]\label{rem:variable-length-mixing}
    As the history-conditional dispersion distribution, the state-conditional dispersion distribution $\mu^{\pi_B}_w: \S \times \A \to \Dist(\H^\psi)$ may also be induced by a policy $\pi_B$, but the structure of this distribution is not straight forward. The subtlety arises from the fact that we ask for a distribution over mixed length histories given an $sa$-pair. What we are doing is basically ``averaging out'' the time information. We demand that there exists a set of weights\footnote{In general, the weights can be stochastic, see \citet{Hutter2016} for an example.} $\{w_n : n \geq 1\}$ for each $sa$-pair such that
    \bqan
    \mu^{\pi_B}_w(\tau \| sa)
    &\coloneqq w_{\abs{\tau}}(sa) \fracp{\mu^{\pi_B}(\tau a)\ind{s_{\abs{\tau}}(\tau) = s}}{\mu^{\pi_B}_{\abs{\tau}}(s a)}  \\
    &= w_{\abs{\tau}}(sa) \mu^{\pi_B}_{\abs{\tau}}(\tau \| sa) \numberthis
    \eqan
    where the probability of visiting $sa$-pair on time-step $n$ is denoted as
    \bqan
    \mu^{\pi_B}_n(sa)
    &\coloneqq \sum_{\tau \in \H^\psi_n} \mu^{\pi_B}(\tau a)\ind{s_{n}(\tau) = s} \\
    &= \sum_{h \in \H_n} \mu^{\pi_B}(h a)\ind{\psi(h) = s} \numberthis
    \eqan
    and the weights are required to be
    \beq
    \sum_{m=1}^\infty w_m(sa) = 1
    \eeq

    Moreover, the policy $\pi_B$ is assumed to be \emph{sufficiently exploratory} such that it visits each $sa$-pair infinitely often almost surely.
\end{remark}

As highlighted earlier, the above (history- and state-conditional) distribution $\Pr$ may be induced by a (behavior) policy, or the may be provided to the agent as an extra information along side with the abstraction. No matter what the \emph{source} of this distribution is, ideally, the learned (near-optimal) policy by the agent should not depend on the choice of this function. Therefore for the majority of this thesis, we collectively denote a history- and state-conditional dispersion distributions simply as a \emph{dispersion} distribution $B$:
\beq
B: \S \times \A \to \Dist(\H)
\eeq

Moreover, it is natural to require that for a valid $B$ we have $B(h \| sa) = 0$ if $\psi(h) \neq s$ for all $a$. We revisit the definition in \Cref{eq:surrogate-mdp} and define the surrogate MDP in a more general form using $B$.

\begin{definition}[Surrogate MDP]\label{def:surrogate-mdp}
    For any abstraction $\psi$ and dispersion distribution $B$, the surrogate MDP is defined as
    \bqan
    \b\mu(s'\|sa)
    &\coloneqq \sum_{h}\mu_\psi(s'\|ha)B(h\|sa) \\
    &= \sum_{m=1}^\infty \sum_{h_m \in \H_m}\mu_\psi(s'\|h_ma)B(h_m\|sa) \numberthis
    \eqan
    for any pair of states $s,s' \in \S$ and action $a \in \A$. Additionally, we define the \emph{abstract} reward signal\footnote{In the case of reward embedding the notation is slightly different as
        \beq
        \b r_\mu(sa) \coloneqq \sum_{h} \sum_{r'} r'\mu(r'\|ha)B(h\|sa) = \sum_{h}\sum_{r'o'} r' \mu(o'r'\|ha)B(h\|sa)
        \eeq
        for any state-action pair $sa$.} for the surrogate MDP as follows:
    \bqan
    \b r_\mu(sa) &\coloneqq \sum_{h} r_\mu(ha)B(h\|sa)
    \eqan

    The pair $\langle \b \mu, \b r_\mu \rangle$ collectively denotes the surrogate-MDP.

\end{definition}

The idea of a dispersion distribution $B$ as a function of state and action is crucially different from the other notions of ``weighting function'' considered in the literature \cite{Li2006,Abel2016}. Therefore, we discuss more this key quantity in the following sub-section. Importantly in \Cref{def:surrogate-mdp} we have explicitly separated the time index to highlight the fact that $B$ is a distribution over \emph{mixed-length} finite histories, which makes it semantically different than a measure over cylinder sets. For example, $\mu^{\pi_B}(h) = \mu^{\pi_B}(\cyl{h})$ is the probability of reaching $h$ by $\pi_B$, but $B(h\|sa)$ is a discrete ``belief probability'' of being at history $h$ if the only information we have is the $sa$-pair. We recollect and discuss the differences between different probability distributions further in \Cref{rem:cheat-sheet}.

\section{Dispersion Distribution}

So far, we have only discussed about aggregating the histories to a common state. This direction makes sense, as the agent only has access to the states through an abstraction. The agent does not directly ``observe'' the history of interaction (only the abstraction map has access to this history). However in ARL, as the agent interacts with the environment it effectively lump histories into state labels because it does not distinguish between histories which end at the same state. This effectively induces a dispersion distribution over the set of histories for each state-action pair, see \Cref{rem:dispersion}. The agent must not have explicit access to this dispersion distribution, because it will defeat the purpose of the abstraction should the agent have access.

\begin{remark}[Generated/induced Dispersion Distribution]\label{rem:dispersion}
    Similar to the history- and state-conditional dispersion distributions, a dispersion distribution may also be generated/induced by a (behavior) policy $\pi_B$ as
    \bqan
    \mu^{\pi_B}_w(h \| sa)
    &\coloneqq \sum_{\tau} \mu^{\pi_B}_{\abs{\tau}}(h\|\tau a) \mu^{\pi_B}_w(\tau \| sa) \\
    &= \sum_{\tau} \mu^{\pi_B}_{\abs{\tau}}(h\|\tau a) w_{\abs{\tau}}(sa)\mu^{\pi_B}_{\abs{\tau}}(\tau \| sa) \\
    &= \sum_{\tau} \mu^{\pi_B}_{\abs{\tau}}(h\|\tau a) w_{\abs{\tau}}(sa) \fracp{\mu^{\pi_B}(\tau a)\ind{s_{\abs{\tau}}(\tau) = s}}{\mu^{\pi_B}_{\abs{\tau}}(s a)} \\
    &= \sum_{\tau} \frac{w_{\abs{\tau}}(sa)\mu^{\pi_B}(ha)\mu^{\pi_B}(\tau a)\ind{h \in [\tau]} \ind{s_{\abs{\tau}}(\tau) = s}}{\mu^{\pi_B}(\tau a)\mu^{\pi_B}_{\abs{\tau}}(s a)} \\
    &= \fracp{w_{\abs{h}}(sa)\mu^{\pi_B}(ha)}{\mu^{\pi_B}_{\abs{h}}(sa)} \sum_{\tau} \ind{h \in [\tau]} \ind{s_{\abs{\tau}}(\tau) = s} \\
    &= w_{\abs{h}}(sa) \fracp{\mu^{\pi_B}(ha)\ind{\psi(h) = s}}{\mu^{\pi_B}_{\abs{h}}(s a)} \\
    &= w_{\abs{h}}(sa) \mu^{\pi_B}_{\abs{h}}(h\|sa) \numberthis
    \eqan
    where the (behavior) policy is assumed sufficiently exploratory, i.e.\ it visits each $sa$-pair infinitely ofter almost surely.
\end{remark}

Apart from a few exceptions in the standard (MDP) state-abstraction literature, this dispersion distribution is usually parameterized only by the state, i.e. it is considered to be a distribution on the state-space of the underlying (original) MDP as a function of the abstract state \cite{Li2006,Abel2016,Roy2006,Hostetler2014}. The rational behind this structure is that if the environment admits a stationary distribution\footnote{The stationary distribution $d^\pi$ induced by a policy $\pi$ over a finite state MDP $\mu$ is a probability distribution which satisfies $d^\pi(s') = \sum_{s} \mu^\pi(s'\|s)d^\pi(s)$ for all underlying states $s'$.} over the underlying states then the dispersion distribution is simply the stationarity distribution on the underlying states conditioned on the abstract state. This construction clearly restricts the covered use cases; not all environments allow for a stationary distribution for any arbitrary (behavior) policy \cite{Hutter2000}. Moreover, in GRL it is hard to even have a notion of such stationary distribution on the history-space, as no history ever repeats.

Therefore, we consider the most general structure possible for the dispersion distribution $B$ which parametrized by the state-action pairs.
It is critical to note that in general $B$ may neither be directly \emph{estimable} nor \emph{accessible} \cite{Hutter2016}. However as shown in \Cref{eq:surrogate-mdp}, $B$ has the right ``type signature'' of being a probability mass function over the set of finite histories $\H$, see \Cref{rem:cheat-sheet}. Additionally, the dispersion distribution $B$ can also be a function of the (behavior) policy (\Cref{rem:dispersion}). And, using this (behavior induced) distribution we can even argue about some non-stationary environments which do not admit a limit stationary distribution on the underlying state-space \cite{Hutter2000}.

\begin{remark}[Probability Distributions Cheat Sheet]\label{rem:cheat-sheet}
    We use a variety of probability mass functions (PMFs) and measures to define probability distributions over the subsets of finite histories. For ease of access, we enumerate some of such important distributions below.
    \begin{itemize}
        \iparadot{$\mu^\pi$ (Probability Measure)} It is a probability measure over the subsets of \emph{infinite} histories induced by a policy $\pi$.
        \beq
        \mu^\pi: \F \to [0,1]
        \eeq

        We abuse notation and use the measures over cylinder sets as a PMF over finite histories, i.e.\ $\mu^\pi(ha) \coloneqq \mu^\pi(\cyl{ha})$ for any $ha$-pair. We can evaluate this PMF using \Cref{eq:mu-mu-rel} for any environment $\mu$ and policy $\pi$. It is critical to note that
        \beq
        \sum_{h_na_n \in \H_n \times \A} \mu^\pi(h_n a_n) = 1
        \eeq
        for any time-step $n$, but it may not be summable over mixed-length finite histories, i.e.\ $\sum_{m=1}^{\infty} \sum_{h_ma_m} \mu^\pi(h_ma_m) = \infty$.

        \iparadot{$\mu^\pi_n$ (Induced PMF,  Fixed-length)} Using the above measure as the starting point, we define the PMF $\mu^\pi_n$ over the set of finite histories of length $n \in \SetN$ as
        \beq
        \mu^\pi_n(ha) \coloneqq \mu^\pi(ha) \ind{\abs{h} = n}
        \eeq
        for \emph{any} finite history $h$ and action $a$. Now this is a properly defined PMF over the set of finite histories because $\sum_{m=1}^\infty \sum_{h_ma_m} \mu^\pi_n(h_ma_m) = 1$ for all time-steps $n$. We renormalized and marginalize this PMF to get many useful ``derived'' distributions, e.g.\ $\mu^\pi_n(sa)$, $\mu^\pi_n(h\|\tau a)$ and $\mu^\pi_n(h\|sa)$ for any state $s$ and abstract history $\tau$ of length $n$.

        \iparadot{$\mu^\pi_w$ (Induced PMF,  Mixed-length)} The above fixed-length PMF $\mu^\pi_n$ is also properly defined probability distribution over the set of mixed-length finite histories. However, it simply does not ``mix'' histories and put zero probability weight over any finite history other than length $n$. We use a weighting factor to get a ``truly'' defined PMF over the finite histories of mixed-length parametrized by an $sa$-pair as
        \beq
        \mu^\pi_w(h\|sa) \coloneqq w_{\abs{h}}(sa) \mu^\pi_{\abs{h}}(h\|sa)
        \eeq
        where $\sum_{m=1}^\infty w_m(sa) = 1$ for each $sa$-pair. The PMF is normalized as
        \beq
        \sum_{h \in \H} \mu^\pi_w(h\|sa) = \sum_{m=1}^\infty w_m(sa)\sum_{h_m}\mu^\pi_m(h_m\|sa) = 1
        \eeq
        for every state-action pair $sa$.

        \iparadot{$B$ (Dispersion Distribution)} We express the surrogate-MDP in \Cref{def:surrogate-mdp} in terms of a ``generic'' dispersion distribution $B$ which has the same ``type signature'' as $\mu^\pi_w$ described above. This is also a PMF over the set of mixed-length finite histories $\H$ paramterized by state-action pairs. $\mu^\pi_w$ is a possible choice of $B$ in \Cref{def:surrogate-mdp}. However, we develop the theory of ARL using this ``generic'' dispersion distribution $B$ without any reference to a $\pi$-induced distribution.

        \iparadot{$\Pr$ (Placeholder PMF)} We use $\Pr$ to denote any \emph{arbitrary choice} of a PMF in an expression. It is used as a ``placeholder'' for many $\pi$-induced distributions, e.g.\ $\mu^\pi_n$ and $\mu^\pi_w$, listed above.

    \end{itemize}

\end{remark}

\section{Learning a Surrogate MDP}

Let us assume that the agent has access to an abstraction $\psi$ of the environment, i.e.\ an oracle which, typically, is an abstraction learning algorithm,\footnote{We consider the abstraction learning case in \Cref{chap:abs-learning}.} has provided the agent a ``feature map'' that extracts the features of every history into a finite set of states. For this section, we assume that the environment is at least $\psi$-communicating, i.e.\ it is possible to visit every $sa$-pair infinitely often by a \emph{sufficiently exploratory} policy $\pi_B$. The agent can estimate a surrogate MDP using this policy, even if the original environment is history-based and non-stationary. Using the dispersion distribution of \Cref{rem:dispersion}, \citet{Hutter2016} proved that a simple frequency estimate of the following surrogate MDP converges due to the law of large numbers under ``weak conditions'':
\beq\label{eq:surrogate-mdp-by-behavior}
\b\mu(s'\|sa) = \sum_{m=1}^\infty w_m(sa)\mu^{\pi_B}_m(s'|sa)
\eeq
where $\mu_m^{\pi_B}(s'|sa) = \frac{\mu_m^{\pi_B}(sas')}{\mu_m^{\pi_B}(sa)}$. The structure of above surrogate MDP shows that it is a ``time-average'' of an ``effective'' state-process at every time instance. So, the estimation is trivial. However, the above structure allows the estimation of the surrogate MDP even for some non-stationary processes too \cite{Hutter2016}.

\section{Planning with Surrogate MDP}

Principally, once the agent has either estimated (or has been provided) a surrogate MDP, the agent can plan using the surrogate MDP process. In this work, by ``planning'' we mean that the agent estimates the optimal (state-)action-value functions (or a value-maximizing policy) of the surrogate MDP environment.

We use the corresponding abstract reward function $\b r_\mu$ to define the corresponding (state-)action-value functions and Bellman equations for the surrogate MDP. Moreover, we use the same discount factor $\g \in [0,1)$ from the underlying true environment. For a fixed policy $\pi: \S\to \Dist(\A)$, we define the action-value and state-value function as follows:
\bqa
q^\pi_\mu(sa) &\coloneqq \b r_\mu(sa) + \g  \sum_{s'}\b\mu(s' \| sa) v^\pi_\mu(s')   \\
v^\pi_\mu(s) &\coloneqq \sum_{a} \pi(a\|s)q_\mu^\pi(sa) \eqqcolon q_\mu^\pi(s\pi(s))
\eqa

These are the Bellman equations (BE) for the surrogate MDP. The corresponding optimal Bellman equations (OBE) for the process are the following:
\bqa
q^*_\mu(sa) &\coloneqq \b r_\mu(sa) + \g  \sum_{s'}\b\mu(s' \| sa) \sup_{a'} q^*_\mu(s'a') \\
v^*_\mu(s) &\coloneqq \sup_{a} \left(\b r_\mu(sa) + \g  \sum_{s'}\b\mu(s' \| sa) v_\mu^*(s')\right)
\eqa
where we use the short hand notation $q^\pi_\mu(s\pi(s))$ to denote the application of stochastic policy at state $s$.
The agent in ARL uses this OBE of the surrogate MDP to find an optimal policy $\pi_{\b\mu}$. After having access to this policy, the agent \emph{uplifts} this policy to the true environment by
\beq
\u\pi_{\b\mu}(a\|h) \coloneqq \pi_{\b\mu}(a\|\psi(h))
\eeq
for any action $a$,  history $h$ and state $\psi(h)$. The uplifting of the policy simply means that the agent takes actions dictated by the optimal policy of the surrogate MDP at the states provided by the abstraction.

The key question is when it is the case that $\u\pi_{\b\mu} \in \Pi^{\sup}_\eps(\mu)$. That is, the optimal policy found through the surrogate MDP is actually good. By definition the uplifted policy is a function of $\psi$ and $B$, so the question naturally applies to the choice of abstraction and dispersion distribution.
In this work, we put emphasis on and prove that there exist \emph{non-MDP} abstractions for which this is possible. One of the most important aspects of ARL is the use of a surrogate MDP to get a ``candidate'' policy for the environment. Since the literature is teeming with finite-state MDP solution algorithms, the agent can plan in both a model-based and model-free manner on the surrogate MDP. For completeness, we list some of the standard dynamic programming (DP) algorithms \cite{Bertsekas1996}, which are later referenced in the thesis.

\begin{itemize}
\iparadot{Model-based Planning} For model-based planning, let the agent have access to (an estimate of) the surrogate MDP $\langle \b\mu, \b r_\mu \rangle$ through an abstraction and any dispersion probability\footnote{There are many abstractions where the choice of $B$ is irrelevant. \Cref{chap:state-only-abs,chap:state-action-abs,chap:extreme-abs} list some examples of such abstractions.}.
The agent may either use \Cref{alg:policy-iteration} of policy iteration (PI), \Cref{alg:state-action-value-iteration} of state-action-value iteration (AVI), or an interleaving application of \Cref{alg:state-value-iteration} and \Cref{alg:policy-evaluation} with adaptive policies\footnote{By adaptive policies we mean the sequence of policies which progressively becomes ``greedy'' \cite{Sutton2018}.} to get an optimal policy of the surrogate MDP.\footnote{Note that each error tolerance $\theta$ in these algorithms corresponds to a particular $\eps$-optimality. These are some of the standard well-understood DP algorithms. We direct interested readers to consult \cite{Bertsekas1996} for their detailed treatment and convergence proofs.}

\begin{algorithm}[!]
    \caption{State-action-value Iteration (AVI)}
    \begin{algorithmic}[1]
        \Input finite-state MDP $M = \langle \b\mu, \b r_\mu\rangle$, discount factor $\g$ and small error tolerance $\theta$
        \Output (anytime) optimal state-action-value function $\h q \to q_\mu^*$
        \State Set $\Delta = \infty$ \Comment{any starting value greater than $\theta$ works}
        \State Set $\h q = \v 0$ \Comment{any initial value works}
        \Repeat \Comment{until the iteration converges}
        \ForAll{$s \in \S$}
        \ForAll{$a \in \A$}
        \State Set $q(sa) = \b r_\mu(sa) + \g \sum_{s'} \b \mu(s'\|sa) \max_{a'} \h q(s'a')$
        \EndFor
        \EndFor
        \State Set $\Delta = \norm{\h q - q}$ \Comment{typically, the infinity norm}
        \State Set $\h q = q$ \Comment{new estimate of $q_\mu^*$}
        \Until $\Delta \leq \theta$
    \end{algorithmic}
    \label{alg:state-action-value-iteration}
\end{algorithm}

\begin{algorithm}[!]
    \caption{Policy Iteration (PI)}
    \begin{algorithmic}[1]
        \Input finite-state MDP $M = \langle \b\mu, \b r_\mu\rangle$, discount factor $\g$ and error tolerance $\theta$
        \Output (anytime) uniform value-maximizing policy $\h \pi \to \pi_\mu^\theta$
        \State Set $q = \v 0$ \Comment{any initial value works}
        \State Set $\h \pi = {\rm uniform}(\A)$ \Comment{any random policy works}
        \Repeat \Comment{until the iteration converges}
        \State Set $\Delta = {\rm False}$
        \ForAll{$s \in \S$}
        \ForAll{$a \in \A$}
        \State Solve $q(sa) = \b r_\mu(sa) + \g \sum_{s'} \b \mu(s'\|sa) q(s'\h\pi(s'))$ w.r.t. $q$
        \EndFor
        \State Set $v(s) = \max_{a} q(sa)$
        \State Set $\pi[s] = {\rm uniform}(\{a \mid v(s) - q(sa) \leq \theta \})$ \Comment{any fixed prob. dist. rule}
        \If {${\rm supp}(\h \pi[s]) \neq {\rm supp}(\pi[s])$}
        \State Set $\Delta = {\rm True}$
        \EndIf
        \EndFor
        \State Set $\h \pi = \pi$ \Comment{new estimate of $\pi_\mu^\theta$}
        \Until $\Delta = {\rm False}$
    \end{algorithmic}
    \label{alg:policy-iteration}
\end{algorithm}

\begin{algorithm}[!]
    \caption{State-value Iteration (VI)}
    \begin{algorithmic}[1]
        \Input finite-state MDP $M = \langle \b\mu, \b r_\mu\rangle$, discount factor $\g$ and small error tolerance $\theta$
        \Output (anytime) optimal state-value function $\h v \to v_\mu^*$
        \State Set $\Delta = \infty$ \Comment{any starting value greater than $\eps$ works}
        \State Set $\h v = \v 0$ \Comment{any initial value works}
        \Repeat \Comment{until the iteration converges}
        \ForAll{$s \in \S$}
        \State Set $v(s) = \sup_{a} \left(\b r_\mu(sa) + \g \sum_{s'} \b \mu(s'\|sa) \h v(s')\right)$
        \EndFor
        \State Set $\Delta = \norm{\h v - v}$ \Comment{typically, the infinity norm}
        \State Set $\h v = v$ \Comment{new estimate of $v_\mu^*$}
        \Until $\Delta \leq \theta$
    \end{algorithmic}
    \label{alg:state-value-iteration}
\end{algorithm}

\begin{algorithm}[!]
    \caption{Policy Evaluation (PE)}
    \begin{algorithmic}[1]
        \Input finite-state MDP $M = \langle \b\mu, \b r_\mu\rangle$, discount factor $\g$, policy $\pi$ and small error tolerance $\eps$
        \Output (anytime) state-value function of the policy $\h v \to v_\mu^\pi$
        \State Set $\Delta = \infty$ \Comment{any starting value greater than $\eps$ works}
        \State Set $\h v = \v 0$ \Comment{any initial value works}
        \Repeat \Comment{until the iteration converges}
        \ForAll{$s \in \S$}
        \State Set $v(s) = \b r_\mu(s\pi(s)) + \g \sum_{s'} \b \mu(s'\|s\pi(s)) \h v(s')$
        \EndFor
        \State Set $\Delta = \norm{\h v - v}$ \Comment{typically, the infinity norm}
        \State Set $\h v = v$ \Comment{new estimate of $v_\mu^*$}
        \Until $\Delta \leq \eps$
    \end{algorithmic}
    \label{alg:policy-evaluation}
\end{algorithm}

\iparadot{Model-free Planning} Interestingly, the agent may not (directly) need $\langle \b\mu, \b r_\mu \rangle$ to find an optimal policy of the surrogate MDP $\pi_{\b\mu}$. The agent may use any model-free algorithm, e.g. Monte-Carlo, Q-learning, SARSA, DQN \cite{Sutton2018}, to directly estimate the state-action-value function $q^*_\mu$ of the surrogate MDP by interacting with the true environment. However, a subtlety is that the abstract process may not be an MDP, but we show in \Cref{chap:convergence-guarnt} that under some non-MDP abstractions Q-learning may still converge.  Once the agent has (an estimate of) $q^*_\mu$ then the optimal policy is trivially a maximizer of this function.
\end{itemize}

\section{Summary}

This chapter provided the formal definitions required for our ARL setup. We established the notion of a surrogate MDP and the corresponding abstract quantities. We stated what it means to uplift an optimal policy to the true environment. We highlighted that the planning (and learning) algorithms of standard RL designed for finite-state MDPs may be used to plan on the surrogate MDP(s). Once learned, the optimal policy of these surrogate MDP(s) is used in the original environment as the ``uplifted'' policy.

\chapter{Abstraction Zoo: State-only Abstractions}\label{chap:state-only-abs}
\begin{outline}
    This chapter collects different types of state-only abstractions prevalent in the literature. We also define some new types of abstractions along the way. The definitions of these abstractions are history-based, which makes them more general than their classical counterparts. Because of this general setup, these abstractions can readily be applied and compared with the ARL setup.
\end{outline}

\epigraph{\it ``All models are wrong, but some are useful.''}{--- \textup{George Box}}

\section{Introduction}

Recall that a state-only abstraction is simply a map from histories to some set of states. The abstraction map effectively groups histories to some information set,
which is the set of states. In ARL, the agent has only access to the abstract history. The information gained through a state signal becomes critical for the agent. If the state is abstracted too much (i.e.\ it distinguishes only a few histories) then the optimal policy might not be representable by the surrogate MDP, defined in the last chapter. In this case, the agent might not be able to learn the optimal behavior. On the other hand, if the abstraction provides too much information (i.e.\ it is too discriminatory) then the learning slows down, as there are too many state-action pairs \cite{McCallum1996}.

We argue that a state is nothing but a predictor about some quantity of interest (e.g.\ the next state, a future reward (or state) sequence, or expected future reward etc.) In MDPs, the current state is sufficient for the agent to make a prediction about the next state of the environment for every possible action. An agent can benefit a lot if it has access to a state which potentially helps to (only) make ``useful'' predictions, e.g.\ it predicts the future expected reward.
For example, if the ``state'' is the \emph{speed}, \emph{acceleration} and \emph{color} of the car in an autonomous driving problem, then the agent may not need to know the color of the car to be able to predict if the car is going to stop or not, which is a ``useful'' prediction in this context.

This predictive perspective allows us to determine which information might be useful to gain more rewards. Under the reward hypothesis of \Cref{asmp:reward-hypothesis}, we stress that the history-action-value function quantifies the ``useful'' aspects of a problem. Any abstraction which represent this function in the state-space should, in theory, be sufficient to represent the optimal behavior.

In this chapter, we collect the major branches of state-only abstractions used in this work. Recall that the history dependence of the abstraction map is one of the most important aspects of the abstractions considered in this thesis. This complete history dependence of the abstractions sets them apart from the standard state-abstractions in MDPs \cite{Abel2016}. These abstractions subsume the (possible) non-stationary nature of the environment, as they ``extract'' the state (features) from the complete history. Therefore under these history-based state-only abstractions, the mapping from a percept to a state may \emph{change} based on the history that has led to the precept, which is not the case in standard state-aggregation methods \cite{Abel2016}. In standard state-aggregation methods for MDPs, the most recent percept is mapped to a state and that mapping is fixed, i.e. it is a stationary mapping. It is critical to point out that the history notation used in this thesis may understate the generality of these abstractions, as a history may be mistaken with a set of states of any countably-infinite state MDP. One can say that $\H$ is a countably infinite set of states of an MDP, and we can use the existing theory of countably infinite MDPs. However as discussed earlier, history-based decision problems (HDP) are crucially different from other countably-infinite state MDP frameworks considered in the literature \cite{Melo2008}, as no state ever repeats in an HDP. This inherent non-recurrence of histories is a peculiarity of GRL \cite{Hutter2000} which has no direct counterpart in the theory of RL with MDPs \cite{Sutton2018}.

In this work, we do not try to provide a comprehensive or exhaustive list of abstraction maps, as the literature is vast (for their standard state-aggregation counterparts) with many (minor) variations. However, we try to provide a categorization of this field, which, we believe, conveys the main message about the structure of abstractions without worrying about the minor improvements possible with the variations.

\section{States as Predictors}

In general, a state is simply a label for some features\footnote{In general, a \emph{feature} is any function of history which should, ideally, capture some useful aspect of the history. For example, the arithmetic mean, sum and standard deviation are some \emph{features} of an IID sequence of numbers which are quite useful in predicting the next number in the sequence.} of the history. For example, an abstraction might put the quantized frequencies of the percepts as the label of the state. However, this historical information about the frequencies of the percepts might not be useful for future predictions. Unless the environment is a bandit problem,\footnote{A problem is an RL problem with only one state. See \citet{Lattimore2020} for a readable text about the bandit problem..} the frequency of the past percepts may not be of much help to make accurate predictions about the next percept. However, if the environment is an MDP then the abstraction which treats the states simply as the most recent percept provides a sufficient statistics for the future trajectories \cite{Sutton2018}. Importantly, the agent can predict the future rewards (or expectations of any other random process) from an MDP state label.

We use this view of state information being a statistics for future predictions of useful signals in a much broader sense. We define some state-only abstractions below according to the information preserved by the state labels. In particular, we are interested in making predictions about the optimal policy. It is easy to see that a sufficient distinction is the optimal action itself, i.e.\ states are distinguished only if they have map histories with different optimal actions. However, the resultant state-space through this information is not sufficient for learning even in simple MDPs. So, we need a bit more ``redundancy'' in the state-space than encoding the optimal actions to be able to make agents which can work without a priori knowing the abstractions  \cite{McCallum1996}.

The convergence and representation guarantees of such abstractions can be provided by sidestepping the question of learning these abstractions. It is not, yet, clear which is the right level\footnote{We use the ``right level'' to mean the smallest possible ``useful'' abstraction which an agent can learn by interacting with the environment.} of information, so the agent is able to learn the abstraction map in an online manner. We consider the abstraction learning problem in \Cref{chap:abs-learning}. Armed with the information preservation and prediction view of an abstract state, we start listing some of the major abstractions prevalent in the literature \cite{Hutter2016,Abel2016,Li2006}. It it important to note that the variants of such abstractions considered in the standard RL (non-history based) setup are fundamentally different. As discussed in \Cref{chap:arl}, the standard counterparts of these abstractions are stationary and history-independent. Whereas, we allow them to be non-stationary and history-based.

\section{Transition Kernel Abstractions}

The abstractions based on the transition kernel of the environment are subject of this section. Recall that the history-based process has as \emph{transition kernel} an action-conditional function from the history-space to a distribution over the same space.
\beq
\mu : \H \times \A \to \Dist(\OR)
\eeq

Any abstraction which (directly) tries to preserve the structure of this function, i.e.\ the transition kernel is representable as a function of states of the abstraction, is subject of this section. It is important to note that the definitions below are more general than what is considered in the standard RL \cite{Abel2016}.

\begin{definition}[$\eps$-MDP Abstraction]\label{def:mdp}
    An abstraction $\psi$ is called an $\eps$-MDP abstraction if for every $\psi(h) = \psi(\d h)$ the following holds for every action $a \in \A$:
    \beq
    \sum_{s'\in \S} \abs{\mu_\psi(s'\|ha) - \mu_\psi(s'\|\d h a)} \leq \eps_1 \text{ and } \abs{r_\mu(ha) - r_\mu(\d h a)} \leq \eps_2
    \eeq
    where $\eps \coloneqq \eps_1 + \eps_2$.
\end{definition}

The above conditions are also known as bounded-parameter MDP (BMDP) in the literature \cite{Givan2000}. In this thesis, this is one of the most restrictive abstraction in the sense that to be satisfied it has to distinguish histories too much. In terms of the predictive power of the abstraction, it can predict \emph{any} future continuations of the history. However as discussed above, the agent only needs to learn the optimal behavior which achieves maximum longterm expected rewards.

\citet{Hutter2016} considered (a variant of) such abstractions to prove representation guarantees of $\eps$-MDP abstractions \cite[Theorem 2]{Hutter2016}, which we extend to homomorphisms in \Cref{thm:psimdpstar}. The main focus of this thesis is to prove that we can abstract further (to get non-MDP abstractions) and still maintain representation (\Cref{thm:psiQstar}) and convergence guarantees (\Cref{thm:non-MDP-convergence}), see details in \Cref{chap:representation-guarnt,chap:convergence-guarnt}.

The following history-action-value based abstractions do exactly that. They do not preserve the state-predictive information, but they abstract based on the history-action-values to preserve only the ``useful'' information.

\section{Action-value Abstractions}

In ARL, the agents are designed to predicts only the valuable future, i.e.\ they are able to make prediction about the high optimal value states. So rather than going through the transition kernel, which can be used to predict history-action-value function of any policy, it suffices to abstract the optimal history-action-value function (also knows as the optimal Q-function). In the following abstractions, we choose maps based on the structure of the optimal Q-function. It turns out that these types of abstractions are very compact as compared to MDP abstractions \cite{McCallum1996,Abel2016,Hutter2016}. The resultant abstractions have many benefits over the transition kernel based abstractions, e.g.\ $\eps$-MDP abstractions. For example, optimal Q-function abstractions allow Q-learning to converge beyond MDPs, see \Cref{chap:convergence-guarnt}. Moreover, since the space of Q-function is (a priori) well defined and bounded due to our bounded reward assumption, there is an upper bound on the maximum number of states required to represent the optimal Q-function \cite{Hutter2016}. Such a bound is not possible for $\eps$-MDP abstraction stated above.

\subsection{Q-uniform Decision Process (QDP)}

The QDP, which we now introduce, is one of the most important abstractions considered in this work. In a $\eps$-QDP abstraction histories are mapped to a same state if they have nearly the same optimal history-action-value function.

\begin{definition}[$\eps$-QDP Abstraction]\label{def:qdp}
    A map $\psi$ is an $\eps$-QDP if for every $\psi(h) = \psi(\d h)$ the following holds:
    \beq
    \norm{Q_\mu^*(ha) - Q_\mu^*(\d ha)} \leq \eps
    \eeq
    for every action $a \in \A$.
\end{definition}

The $\eps$-QDP abstraction defined above is a major class of non-MDP abstractions. As said earlier, the majority of the thesis is based on (variants of) such abstractions. Later in \Cref{chap:extreme-abs,chap:convergence-guarnt,chap:representation-guarnt,chap:action-seq}, we prove some very important properties of these abstractions. \citet{Hutter2016} has proved representation guarantees for $\eps$-QDP abstraction \cite[Theorem 8]{Hutter2016}. Later in the thesis, we extend these results to homomorphisms \Cref{thm:psiQpi,thm:psiQstar}. We also improve (\Cref{thm:bin-esa}) an important upper-bound on the number of states of surrogate-MDP proved by \citet[Theorem 11]{Hutter2016}.

However, in the literature, there are many sub-classes of this abstraction class which lie between the $\eps$-MDP and $\eps$-QDP abstractions \cite{Abel2016}. The commonality in these sub-classes is the normalization since the absolute scale of the optimal Q-function does not alter the optimal policy. The absolute scale of the history-action-value functions is not preserved. That is, the mapped histories may have a mixture of history-action-values of different absolute values. The following is a non-exhaustive list of such sub-classes \cite{Abel2016}:

\begin{itemize}
    \iparadot{Boltzmann QDP}
    As apparent from the name, these abstraction preserves a Boltzmann like structure of history-action-value functions. The Q-functions can be normalized, since the absolute scale of a Q-function does not matter. An abstraction map $\psi$ is \emph{Boltzmann $\eps$-QDP} if for all actions $a \in \A$ the following holds:
    \beq
    \psi(h) = \psi(\d h) \implies \norm{\frac{\e^{Q_\mu^{*}(ha)}}{\sum_{\d a} \e^{Q_\mu^{*}(h\d a)}} - \frac{\e^{Q_\mu^{*}(\d ha)}}{\sum_{\d a} \e^{Q_\mu^{*}(\d h\d a)}}} \leq \eps
    \eeq

    \iparadot{Multinomial QDP}
    Similar to the Boltzmann $\eps$-QDP, the history-action-values are also normalized in the multinomial $\eps$-QDPs, but slightly differently. An abstraction $\psi$ is a multinomial $\eps$-QDP if the following holds:
    \beq
    \psi(h) = \psi(\d h) \implies \norm{\frac{Q^*_\mu(ha)}{\sum_{a} Q^*_\mu(ha)}- \frac{Q^*_\mu(\d ha)}{\sum_{a} Q^*_\mu(\d ha)}} \leq \eps
    \eeq
    for all $a \in \A$

    \iparadot{Normalized QDP}
    A natural way to normalize a history-action-value is to normalize it by the history-value, and not by the total ``area'' of a Q-function, as is the case above. An abstraction $\psi$ is called normalized $\eps$-QDP, if for all actions $a \in \A$ the following holds:
    \beq
    \psi(h) = \psi(\d h) \implies \norm{\frac{Q_\mu^*(ha)}{V_\mu^*(h)}- \frac{Q_\mu^*(\d ha)}{V_\mu^*(\d h)}} \leq \eps
    \eeq

    \iparadot{Advantage QDP}
    Another popular way to normalize a Q-function is to consider the relative (dis)advantage of each action from the optimal actions. The grouped histories have similar (dis)advantage for each action as a function of state, but the absolute value may still depend on the history. Formally, an abstraction $\psi$ is an advantage $\eps$-QDP if for all actions $a \in \A$ the following holds:
    \beq
    \psi(h) = \psi(\d h) \implies \norm{A_\mu^*(ha) - A_\mu^*(\d ha)} \leq \eps
    \eeq
    where $A_\mu^*(ha) \coloneqq  Q_\mu^*(ha) - V_\mu^*(h)$ is known as the advantage function in the literature \cite{Sutton2018}.
\end{itemize}

The above classes of Q-function abstractions may be beneficial in designing a learning algorithm \cite{Abel2018}, but, as highlighted at many places in the thesis, the (standard) $\eps$-QDP class is interesting enough. We do not use these specializations any further.

\subsection{Value \& Policy-uniform Decision Process (VPDP)}

History-action-value functions are critical for the agent to learn the optimal behavior, see \Cref{chap:convergence-guarnt}. However, distinguishing the states based on the full history-action-value functions might not be necessary. For example, if the non-optimal actions have smaller values relative to the optimal values then the abstraction might ``safely'' mix multiple histories with same values and optimal actions only. The surrogate MDP might not lead to a sub-optimal policy. However, it is not a rigorously proven fact, yet. We have empirical support\footnote{The empirical setup used was similar to what is considered in \Cref{chap:vpdp-exp}. However, we did not find any counter-example to the conjecture in our experiments, which has lead us to believe that the conjecture is true. It calls for further investigation.} that it might be the case. We put this down as a conjecture later, see \Cref{conj:vpdp}. We call such class of abstractions the value \& policy-uniform decision processes (VPDP).

\begin{definition}[$\eps$-VPDP Abstraction]\label{eq:vpdp}
    An abstraction map $\psi$ is an $\eps$-VPDP if for every $\psi(h) = \psi(\d h)$ the following holds:
    \beq
    \norm{V_\mu^*(h) - V_\mu^*(\d h)} \leq \eps_1 \text{ and } \A_{\eps_2} = \A_{\eps_2}
    \eeq
    where $\eps \coloneqq \eps_1 + \eps_2$ and $\A_{\eps}(h) \coloneqq \{a \in \A : V^*(h) - Q^*(ha) \leq \eps\}$ is the set of all $\eps$-optimal actions.
\end{definition}

The above definition is crucially different from the value-uniform decision process ($\eps$-VDP) considered by \Cite{Hutter2016} and in \Cref{chap:representation-guarnt}.

\begin{definition}[$\eps$-VDP Abstraction]\label{eq:vdp}
    An abstraction map $\psi$ is an $\eps$-VDP if for every $\psi(h) = \psi(\d h)$ the following holds:
    \beq
    \norm{V_\mu^*(h) - V_\mu^*(\d h)} \leq \eps \text{ and } \pi^*(h) = \pi^*(\d h)
    \eeq
\end{definition}

As apparent from the definitions, $\eps$-VDP only demands for the optimal actions to be $\psi$-uniform. Whereas, $\eps$-VPDP merges histories together which have same set of $\eps_2$-optimal actions therefore allows a finer reduction than $\eps$-VDP.
\Cref{fig:classes} shows a graphical relationship between the main classes of abstractions considered in this thesis. The $\eps$-PDP class of abstractions are the ones where
\beqn
\norm{V_\mu^*(h) - Q_\mu^*( h\pi_\mu(\d h))} \leq \eps
\eeqn
for all pair of histories such that $\psi(h) = \psi(\d h)$. Trivially, the $\eps$-VPDP class is the intersection of the $\eps_1$-VDP and $\eps_2$-PDP classes.

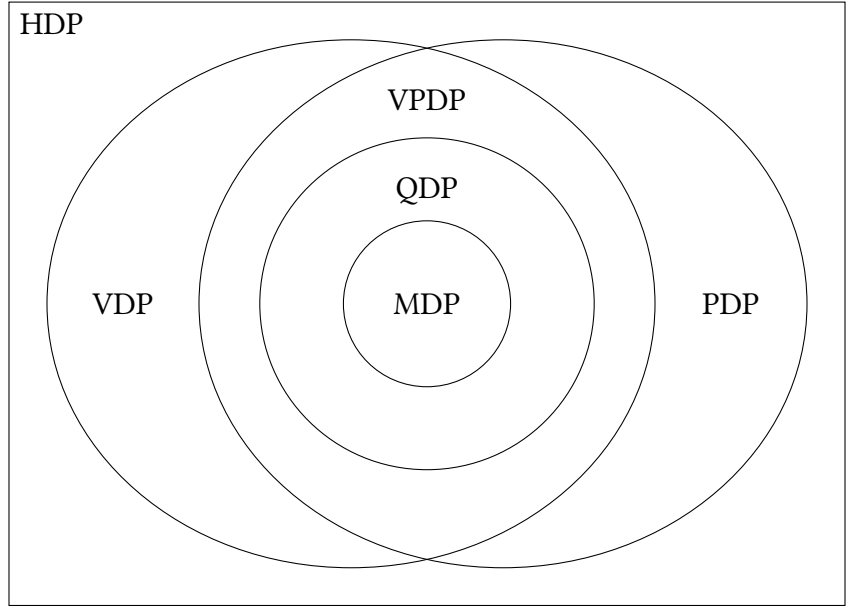
\begin{figure}
    \centering
    \begin{tikzpicture}
        \draw[black] (-5.5,-4) rectangle (5.5,4);
        \node[below right] at (-5.5,4) {HDP};
        \draw[black] (0,0) circle (2.2cm);
        \node[] at (0,1.5cm) {QDP};
        \draw[black] (0,0) circle (1.1cm);
        \node[] at (0,0) {MDP};
        \draw[black] (-1,0) ellipse (4cm and 3.5cm);
        \node[] at (-4,0) {VDP};
        \draw[black] (1,0) ellipse (4cm and 3.5cm);
        \node[] at (4,0) {PDP};
        \node[] at (0,2.7) {VPDP};
    \end{tikzpicture}
    \caption{Major abstraction classes for abstraction reinforcement learning.} \label{fig:classes}
\end{figure}

\section{Other Abstractions}

In this section, we list further abstractions used by the community. We hint their connections to the ARL framework. However, a proper treatment of these abstraction to cast them into a history map is left as a future work.

\subsection{Factored MDP (FMDP)}

The above BMDP condition (which is $\eps$-MDP defined in \Cref{def:mdp}) does not assume any structure on the state-space, but one can go further. Sometimes the environment allows for a natural factorization of the state-space \cite{Boutilier2000}. The state-space $\S \coloneqq \S_1 \times \S_2 \times \dots \times \S_N$ can be considered as a product of many different factors $\S_i$, e.g.\ in an Atari game the shape, color, and position of the objects might be enough to produce the state-space.

A factored MDP represents the environment with a factored state-space, where the next factor of the state depends on a subset of the state factors in the previous step. These are known as the parents of the factor. Dynamic Bayesian networks (DBN) \cite{Kearns1999a} and decision trees \cite{Strehl2007} are actively used to model such factored MDP environments. In a factored MDP representation the surrogate-MDP has the following structure:
\beq
\b \mu(\v s'\| \v sa) = \prod_{i=1}^{d} \mu(s_i'\| {\rm Parents}_i(\v s)a)
\eeq
where ${\rm Parents}_i(\v s)$ is the set of factors which affects the next factor $s'_i$.
\citet{Hutter2008c} considers such factored representations as  $\phi$DBN abstractions.

\subsection{Partially Observable MDP (POMDP)}

The most prevalent extension of MDP framework is the class of partially observable Markov decision processes (POMDP). In a POMDP model, it is assumed that there is an underlying MDP which is not observable by the agent \cite{Kaelbling1996}. To see how we can fit this model into ARL, let there be a (finite) set of states $\X$ of the environment. The true occupation probability\footnote{The occupation probability is the probability of the POMDP being in a certain underlying, unobserved state after a certain history of interaction of the agent with the environment.} is expressed by a map $\chi : \H \to \Dist(\X)$. By design, the true environment has an underlying MDP $\mu_{\rm MDP}$ and an emission process $\mu_{\rm E}$ with the following structure:
\bqan
\mu(e'\|ha)
&\coloneqq \sum_{x' \in \X}\mu_{\rm E}(e'\|x')\sum_{x \in \X}\mu_{\rm MDP}(x'\|xa)\chi(x\|h) \\
&= \sum_{x' \in \X}\mu_{\rm E}(e'\|x')\mu_{\rm MDP}(x'\|\v x(h) a) \\
&= \mu(e'\|\v x(h) a) \numberthis
\eqan
where $\v x (h)$ is the true occupancy probability of the underlying MDP $\mu_{\rm MDP}$. For a fixed starting distribution, there is a unique $\v x(h)$ for each $h$. The mapping $\chi$ is a kind of soft state-aggregation from the set of histories to the distribution over the underlying state-space \cite{Singh1995}.

It is clear from the above construction that if we allow for continuous-state abstractions then we can model the POMDP class exactly with ARL. However as shown in \Cref{chap:representation-guarnt}, we do not need a continuous state vector to fully represent the useful aspects of the environment. Therefore, we side step the POMDP construction, and only consider QDP and VPDP classes in this thesis.

\subsection{Predictive State Representation (PSR)}

The POMDP model of the environment involves the information about the true, unobservable state-space $\X$. Typically, the agent maintains a belief distribution over $\X$, so it knows $\X$ \cite{Kaelbling1996}, which is a strong assumption. Predictive state representations try to circumvent this complication by building a representation by using only observable quantities \cite{Littman2002}. In its standard form, a PSR models the environment using a (finite) set of \emph{core tests} $\T \subseteq \A\times\H$, which are nothing but finite set of future trajectories (starting from an action). The idea is that if the probability of these so-called core tests is known then we can \emph{represent} the probability of \emph{any} future trajectory $\vo{ae}' \in \A \times \H$ by some (linear or non-linear) combination of the core tests. We can connect the theory of PSR with ARL by noting that a PSR state maps histories $h$ and $\d h$ together if they ``agree'' on their predictions on the core tests. Formally,
\beq
\sup_{\vo{ae}'_n \in \T} \abs{\mu(\v e'_n \| h, \v a_n) - \mu(\v e'_n \| \d h, \v a_n)} \leq \eps
\eeq
for any $\psi(h) = \psi(\d h)$, where $\mu(\v e'_n \| h, \v a_n)$ is an action-conditional distribution defined as
\beq
\mu(\v e'_n \| h, \v a_{n}) \coloneqq \prod_{m=1}^n \mu(e'_{m}\| ha_1e'_1\dots e'_{m-1} a_m)
\eeq

Furthermore, it is easy to see that a PSR state $\v s(h)$ at each history $h$ is a finite vector from $\SetR^\T$. A PSR requires the existence of a weight vector $w: \A \times \H \to \SetR^\T$
for each future trajectory which can be used to express the action-conditional probability of that trajectory. In linear PSR it is expressed as follows:
\beq
\mu(\v e'_n \| h, \v a_n) \approx \sum_{t \in \T} w_t(\vo{ae}'_n)\v s_t(h)
\eeq
for all trajectories $\vo{ae}'_n \in \A \times \H$ strarting from the history $h$.

\section{Summary}
This chapter listed some of the major state-only abstractions used in the literature. We also highlighted some minor variations. The chapter provided formal definitions of important classes of abstractions, $\eps$-MDP, $\eps$-QDP, and $\eps$-VPDP used in the rest of the thesis. We also explored the connection of ARL with other dominant state-representation and abstraction methods.

\chapter{Abstraction Zoo: State-Action Abstractions}\label{chap:state-action-abs}

\begin{outline}
    The standard RL framework suffers from the curse of dimensionality of both the state and action spaces, i.e.\ the standard RL algorithms scale badly with the increased size of the state-action space. In this chapter we collect many prominent techniques of state-action abstractions in RL. However, the problem of formalizing a state-action abstraction is more subtle than the state-only abstraction. That is the reason that there is less variety among state-action abstractions (also known as homomorphisms) as compared to the state-only abstractions.
\end{outline}

\epigraph{\it ``Stop learning tasks, start learning skills.''}{--- \textup{Satinder Singh}}

\section{Introduction}
In \Cref{chap:state-only-abs} we categorized a number of state-only abstractions, which produce a finite state representation of the environment by mapping histories to states. In this chapter, we extend the definition of some of these abstractions to a state-action abstractions. As highlighted in \Cref{chap:arl}, in this way we can abstract actions jointly with the states using a single map. This method of jointly reducing the (original) state-action space is typically studied under the homomorphism framework \cite{Whitt1978}. Just like state-abstractions, in the homomorphism setup the original problem is solved by finding the solution in the smaller state-action abstract problem. The abstract action-space is assumed to be reasonably small to facilitate planning using the abstraction map.

Recall that a state-action abstraction is a mapping from history-action pairs to state and abstract action pairs:
\beq
\psi: \H \times \A \to \S \times \B
\eeq
where $\B$ is some set of abstract actions. As for the state-only abstractions, we assume that state-action abstraction is deterministic. This chapter mirrors most of the abstractions considered in \Cref{chap:state-only-abs}.

\section{Transition Kernel Homomorphisms}

The most natural similarity measure is the similarity of the transition kernel and the immediate rewards. This mirrors the $\eps$-MDP state-only abstraction of \Cref{def:mdp}. The resultant abstraction maps the history-action process to a finite abstract state-action MDP, which nearly preserves the Markovian structure of the origin problem \cite{Ravindran2004}.

\begin{definition}\label{def:mdp-homo}
    An abstraction $\psi$ is an $\eps$-MDP homomorphism if for any $\psi(ha) = \psi(\d h \d a)$ the following holds:
    \beq
    \sum_{s' \in \S} \abs{\mu_\psi(s'\|ha) - \mu_\psi(s'\|\d h\d a)} \leq \eps_1 \land \abs{r_\mu(ha) - r_\mu(\d h \d a)} \leq \eps_2
    \eeq
    where $\eps \coloneqq \eps_1 + \eps_2$.
\end{definition}

Usually, the $\eps$-MDP homomorphism is studied under a parametrized metric \cite{Taylor2008}. The conditions in \Cref{def:mdp-homo} are joined together as
\beq
\sum_{s' \in \S} \abs{\mu_\psi(s'\|ha) - \mu_\psi(s'\|\d h\d a)} + \alpha \abs{r_\mu(ha) - r_\mu(\d h \d a)} \leq \eps
\eeq
where $\alpha > 0$ is some preference parameter which trades-off between the reward and the next state distribution mismatches. As these homomorphisms model the transitional similarities, they are useful for domains where there are regions of similar transitional structures, e.g.\ navigational and grid-world domains \cite{Ravindran2004}.

\section{Action-value Function Homomorphisms}

Again, we can argue, and show in \Cref{chap:representation-guarnt}, that an $\eps$-MDP homomorphism is not necessary for a good representation. We only need to preserve the ``useful'' structure of the environment. There is nothing better than the history-action-value function to quantify such ``usefulness''. We can easily mirror the definition of $\eps$-QDP state-only abstraction to define an $\eps$-QDP homomorphism.

\begin{definition}
    An abstraction $\psi$ is an $\eps$-QDP homomorphism if for any $\psi(ha) = \psi(\d h \d a)$ the following holds:
    \beq
    \norm{Q^*_\mu(ha) - Q^*_\mu(\d h\d a)} \leq \eps
    \eeq
\end{definition}

We can replicate the sub-class of $\eps$-QDP state-only abstractions from \Cref{chap:state-only-abs} by $\B = \A$ and $\psi(ha) = sa$. To the best of our knowledge, this $\eps$-QDP homomorphism is quite underexplored. Moreover, the notion of $\eps$-VPDP homomorphism is very convoluted and weak. We do not list it as a potential homomorphism here. However, we address a variation of it in \Cref{chap:representation-guarnt}.

\section{Options Framework}

The state-action homomorphism is one of a few formal methods to reduce the complexity of an RL problem. The options framework is a competing formulation to homomorphisms. However, the standard options framework does not involve state abstraction, and it is a framework for temporal action abstractions \cite{Sutton1999a}. An options framework coupled with state abstraction would be a framework of joint state-action abstraction. For completeness, we provide a bit more formal introduction of this framework below.

\subsection{Standard Options}
In the standard options framework, we say that a set of abstract actions $\B$ is a set of behaviors (or plans). There are three main components of any behavior $b = \langle \pi_b, I_b, \beta_b \rangle \in \B$ defined as follows:
\bqs
\pi_b &: \H \to \Dist(\A) \\
I_b &: \H \to \{0,1\} \\
\beta_b &: \H \to \SetRP
\eqs
where $I_b$ is a set of states where this option can be ``started'', and $\beta_b$ is the ``continuation probability''. Its complement $\b \beta \coloneqq 1-\beta$ is the ``termination probability'' probability \cite{Sutton1999a}.
The agent at any history $h$ selects an option $b \in \B$. As an option could take multiple steps before terminating, the agent follows $\pi_b$ until the option terminate before deciding again for the next option. There is no ``decision-making'' during the option. Once an option is started the agent follows the associated policy without further decision making about the option. However, once the option is terminated the agent takes back the control.

The ARL setup with homomorphic abstraction and the options framework with action-abstraction are both trying to reduce the complexity of the underlying RL problem. We believe that the connection between the ARL setup with a homomorphism may aid to a improve the understanding of the options framework. However, it requires more investigation.

\subsection{Relativized Options}

The major limitation of the standard options framework above is the options' dependency on the original history $h$. Recently in an MDP context, this framework has been extended to use a state-only abstraction \cite{Ravindran2003,Abel2019}. The resultant options which work ``well'' with the (state-only) abstraction are called \emph{relative} options. The relativized options are the ``macro'' options which are simply the labels for a set of options which respect the state-only abstraction, e.g.\ the options start in the abstract state and terminate in another abstract state \cite{Abel2019}. It is easy to see that the framework of relativized options can also be emulated by ARL with a set of appropriate abstractions and policies for the agents.

\section{Summary}

This chapter is a close replication of \Cref{chap:state-only-abs}. We defined the two most important state-action abstractions, $\eps$-MDP and $\eps$-QDP homomorphisms. Moreover, we touched upon the alternate framework of options for state-action reduction.

\chapter{Abstraction Zoo: Extreme State Abstractions}\label{chap:extreme-abs}

\begin{outline}
    This chapter expands on some of the non-MDP abstractions of \Cref{chap:state-only-abs}. These abstractions have special properties: they allow to cast any RL problem into a fixed state abstraction. This is known as the extreme state aggregation (ESA). We define what it means to have an extreme abstraction, and list some of the key properties of such abstractions.
\end{outline}

\epigraph{``Approximations, after all, may be made in two places - in the construction of the model and in the solution of the associated equations. It is not at all clear which yields a more judicious approximation.''}{--- \textup{Richard Bellman}}

\section{Introduction}

Recall that we assume that the rewards are bounded between a unit interval, i.e. $r \in [0,1]$. This implies that in the unnormalized discounted GRL setup the history-value and history action-value functions are also bounded between $[0, 1/(1-\g)]$. The boundedness of this ``value space'' is an important property for non-MDP abstractions grouping histories based on value similarity, e.g.\ $\eps$-QDP and $\eps$-VPDP abstractions. As defined later, the state-space of such abstractions can be embedded into (a discretized version of) a bounded hypercube. Interestingly, adding more observations to $\OR$ does not increase the size of the state-space. This can not be done for $\eps$-MDP abstractions. If we fix the state-space of an $eps$-MDP and increase the number of percepts then the resultant abstract process may not be an MDP. In general, the dimension of the kernels matrix $\mu$ increases with the addition of more percepts. The bounded range of $\eps$-QDP and $\eps$-VPDP abstractions provide some of the most interesting properties to these non-MDP models. Before we discuss these extreme abstractions, we formalize the notion of (universal) extreme state-space.

\section{Extreme State-space}

In general, an abstraction map $\psi$ may have any arbitrary state-space $\S$. However, as discussed above, if histories are mapped based on the (optimal) Q-function then we can (trivially) label the state by a (representative) action-value vector. That is, the state-space is $\S \subseteq [0, (1-\g)\inv]^A \subset \SetR^A$. Furthermore, we are interested in $\eps$-close Q-functions, so a discretized version of above space is sufficient our setup.

\begin{definition}[Extreme State-space]
    For any finite action-space of size $A$ and $\eps > 0$, we define the \emph{extreme state-space} as an $\eps$-grid of the hypercube $[0, (1-\g)\inv]^A$.
\end{definition}

We consider the case of learning such extreme abstractions in \Cref{chap:abs-learning}.
In the following sections, we go over the extreme versions of $\eps$-QDP and $\eps$-VPDP abstractions. We discuss some of the key consequences for using such extreme abstractions.

\section{Extreme QDP Abstractions}

The $\eps$-QDP abstractions defined in \Cref{def:qdp} can have an arbitrary state-space, but we call them extreme $\eps$-QDP abstractions if they are defined over the (universal) extreme state-space.

\begin{definition}[Extreme $\eps$-QDP]
    For any environment $\mu$, an abstraction $\psi_\mu^*$ is an extreme $\eps$-QDP if
    \beq\label{eq:extreme-qdp}
    \psi_\mu^*(h) \coloneqq \left(\ceil{Q_\mu^*(ha)/\eps}\right)_{a \in \A}
    \eeq
    for any history $h$.
\end{definition}

In the above definition we use the special notation $\psi_\mu^*$ to highlight the fact that the abstraction is tailored to the individual environment. Later in \Cref{chap:abs-learning} when we try to design a learning algorithm, this notation comes handy to denote the estimated (or candidate) abstraction maps.

The following theorem is an adaptation of \citet[Theorem 11]{Hutter2016}, which proves the usefulness of these abstractions.\footnote{We casually call it ``extremely useful QDP'' just to reference the fact that it is an extreme abstraction.}

\begin{theorem}[Extremely ``Useful'' QDP]
    Any uplifted optimal policy of any surrogate-MDP of an extreme $\eps'$-QDP abstraction is an $\eps$-optimal policy in the original environment. The number of states is bounded as
    \beq
    S \leq \fracp{3}{\eps(1-\g)^3}^A
    \eeq
    where $A$ is the size of the action-space and $\eps' \coloneqq 3\inv(1-\g)^2 \eps$.
\end{theorem}
\begin{proof}
    The proof can be found in the original reference \cite[Theorem 8 \& 11]{Hutter2016}.
\end{proof}

To intuitively see why the above bound holds, consider we are discretizing each action-value into $\eps'$ wide bins. Due to $\g$-discounting and unit bounded rewards, the values can range from zero to $1/(1-\g)$. So, we end up with $\ceil{1/\eps'(1-\g)}$ bins per action-value. Since we need to do the discretization for each action-value, we have in total $\ceil{1/\eps'(1-\g)}^A$ bins, which are the number of different states of the abstraction. The final piece in the bound the relationship between $\eps'$ and $\eps$ that requires a further fine graining of the grid by a factor of $3(1-\g)^{-2}$.

In \Cref{chap:action-seq} we improve the above upper bound on the number of states from an exponential to a mere logarithmic dependency in the size of the action-space, see \Cref{thm:bin-esa}. The above fact is one of the most powerful aspects of extreme abstractions. If the agent has access to (or it learns one of) them then it can use the surrogate-MDP (defined over a finite state-space of size $S$) to plan optimally. Importantly, \emph{any} surrogate MDP defined in \Cref{eq:surrogate-mdp} will lead to a near-optimal policy of the original process. This allows the agent to use any behavior sufficiently exploratory policy of the environment to learn the surrogate MDP and then use the optimal policy of the learned process \cite{Hutter2016}. Note that the resultant surrogate MDP will be a function of the behavior, see \Cref{eq:surrogate-mdp-by-behavior} for example, but the above theorem assures that the optimal policies of these (possible) surrogate MDPs are near-optimal in the original environment.

\section{Extreme VPDP Abstractions }

The idea of extreme state abstraction (ESA) can also be used on $\eps$-VPDP abstractions.

\begin{definition}[Extreme $\eps$-VPDP]
    For any environment $\mu$, an abstraction $\psi_\mu^*$ is an extreme $\eps$-VPDP if
    \beq
    \psi_\mu^*(h) \coloneqq \left(\ceil{V_\mu^*(h)/\eps_1}, \pi^*(h), \A_{\eps_2}(h))\right)
    \eeq
    for any history $h$, where $\A_\eps(h) \coloneqq \{a \in \A: V^*_\mu(h) - Q^*_\mu(ha) \leq \eps \}$ denotes the set of all $\eps$-optimal actions.
\end{definition}

\citet[Theorem 10]{Hutter2016} provides a counter-example for the optimality of $\eps$-VDP abstractions. Note that the notion of value and policy-uniformity considered by \citet{Hutter2016} is crucially different than what is defined in $\eps$-VPDP. He demands an $\eps$-VPDP to preserve \emph{an} optimal action in each state. Whereas, our notion preserves \emph{all} $\eps_2$-optimal actions. We map the histories together only if all $\eps_2$-optimal actions in these histories are the same. The notion of \citet{Hutter2016} is coarser than ours. That is the reason the counter-example of \citet[Theorem 10]{Hutter2016} does not apply to our definition of $\eps$-VPDP. The offending state in \citet[Theorem 10]{Hutter2016} can not be aggregated in $\eps$-VPDP abstraction for an arbitrarily large $R$, which is a crucial parameter for the counter-example, see \citet[Theorem 10]{Hutter2016}  for more details.

We have done an empirical analysis to check if there exists some random MDP which violates the optimality condition for $\eps$-VPDP, but have found none so far. We provide more details about the experimental setup in \Cref{chap:vpdp-exp}. The following is an empirically demonstrated conjecture.

\begin{conjecture}[Extremely ``Useful'' VPDP]\label{conj:vpdp}
    There exist constants $c_1$, $c_2$ and $\eps' = O(\eps)$ such that any uplifted optimal policy of any surrogate-MDP of an extreme $\eps'$-VADP abstraction is an $\eps$-optimal policy in the original environment. The number of states of the surrogate MDP is bounded as
    \beq
    S \leq O(\eps^{-c_1} \cdot (1-\g)^{-c_2} \cdot 2^A)
    \eeq
    where $A$ is the size of the action-space.
\end{conjecture}

We tried generating a number of random ergodic MDPs which were amenable to $\eps$-VPDP abstractions. A range of sampling distributions were used to produce the surrogate MDPs. However, none of these runs produced a counter-example to our conjecture. Another reason why we believe the above conjecture should hold is that \emph{all} $\eps_2$-optimal actions are function of the state. This makes the estimates of the history-action-values of these preserved optimal actions (somewhat) independent of the dispersion distribution, i.e.\ how the histories are mixed. Since the optimal action-value estimates are more or less independent of the behavior policy, the sub-optimal actions should not be rewarded higher enough to ``fool'' the surrogate MDP to select a sub-optimal action for the state.

However, a formal proof (or disproof) for the conjecture is a highly desirable addition to the theory of ARL. We leave this for another work in the future. It is important to reiterate, the counter-example of \citet[Theorem 10]{Hutter2016} is not applicable to our conjecture as the abstraction is \emph{not} $\eps$-VPDP as defined in \Cref{eq:vpdp}.

\section{Summary}

This chapter focused on the special cases of $\eps$-QDP and $\eps$-VPDP abstractions. These are the extreme variants of these abstraction, which are defined (or can be embedded) into the universal state-space. The universal state-space is nothing but an $\eps$-discretized version of $[0, 1/(1-\g)]^A$. The extreme $\eps$-QDP abstraction has been shown to carry some remarkable properties, e.g.\ a uniformly fixed state-space abstraction where any policy of any surrogate MDP is near-optimal. We conjecture that the same result holds for the extreme $\eps$-VPDP abstractions. The conjecture is empirically tested, and hints are provided to find a formal (dis)proof.

\chapter[Convergence Guarantees for TD-based Algorithms]{Convergence Guarantees for TD-based Algorithms {\\ \it \small This chapter is an adaptation of \citet{Majeed2018}}}\label{chap:convergence-guarnt}

\begin{outline}
    Temporal-difference (TD) learning is an attractive, computationally efficient framework for model-free RL. Q-learning is one of the most widely used TD learning techniques that enables an agent to learn the \emph{optimal} action-value function. Contrary to its widespread use, Q-learning has only been proven to converge on Markov Decision Processes (MDPs) and Q-uniform abstractions of finite-state MDPs.
    On the other hand, most real-world problems are inherently non-Markovian: the full true state of the environment is not revealed by recent observations.
    \ifshort\else
    In addition to the finite state MDP assumption, ergodicity is also typically required for Q-learning convergence which may also not be satisfied in most of real-world problems.
    \fi
    In this chapter, we investigate the behavior of Q-learning when applied to non-MDP and non-ergodic domains which may have infinitely many underlying states. We prove that the convergence guarantee of Q-learning can be extended to a class of such non-MDP problems, in particular, to some non-stationary domains. As defined in \Cref{chap:state-only-abs}, we model these domains with $\eps$-QDP abstractions with $\eps = 0$, which we also call the ``exact'' QDP aggregations.
    We show that state-uniformity of the optimal Q-value function is a necessary and sufficient condition for Q-learning to converge even in the case of infinitely many (underlying) states.
\end{outline}

\section{Introduction}

Temporal-difference learning \cite{Sutton1988} is a model-free learning framework in reinforcement learning. In TD learning, an agent learns the optimal action-value function of the underlying problem without explicitly building or learning a model of the environment. The agent can learn the optimal behavior from the learned Q-value function: the optimal action maximizes the Q-value function. It is generally assumed that the environment is Markovian and ergodic for a TD learning agent to converge \cite{Tsitsiklis1994,Bertsekas1996}.

The TD learning agents, apart from a few restrictive cases\footnote{See Section \ref{sec:related-work} for exceptions.}, are not proven to learn\footnote{In this chapter we use the term ``learn a domain'' in the context of learning to act optimally and not to learn a model/dynamics of the domain.} non-Markovian environments, whereas most real-world problems are inherently non-Markovian: the full true state of the environment is not revealed by the last observation \cite{Agarwal2021}, and the set of true states can be infinite, e.g. as effectively in non-stationary domains.
Therefore, it is important to know if the agent performs well in such non-Markovian domains to work with a broad range of real-world problems.

In this chapter, we investigate convergence of one of the most widely used TD learning algorithms, Q-learning \cite{Watkins1992}.
Q-learning has been shown to converge in MDP domains \cite{Tsitsiklis1994,Bertsekas1996}, whereas there are empirical observations that Q-learning sometimes also work in some non-MDP domains \cite{Sutton1984}. First non-MDP convergence of Q-learning has been reported by \cite{Li2006} for the environments that are Q-uniform abstractions of finite-state MDPs. It was shown in \Cref{chap:extreme-abs} that under certain conditions there exists a deterministic, near-optimal policy for non-MDP environments which are not required to be abstractions of any finite-state MDP. These positive results motivated this work to extend the non-MDP convergence proof of Q-learning to a larger class of infinite internal state non-MDPs.

We show that QDP is the largest class where Q-learning can converge, i.e. QDP provides the necessary and sufficient conditions for Q-learning convergence.
Apart from a few toy problems, it is always a leap of faith to treat real-world problems as MDPs. An MDP \emph{model} of the underlying true environment is implicitly \emph{assumed} even for model-free algorithms. Our result helps to relax this assumption: rather assuming the domain being a \emph{finite-state} MDP, we can suppose it to be a QDP, which is a much weaker implicit assumption.
\ifshort\else
We develop our proof in a state-aggregation context. We show that our result holds for any aggregation map (i.e. a model of real-world problem) that respects the exact state-uniformity condition of the optimal action-value function.
\fi
The positive result of this chapter can be interpreted in a couple of ways; {\bf a)} it provides theoretical grounds for Q-learning to be applicable in a much broader class of environments or {\bf b)} if the agent has access to a QDP aggregation map as a potential model of the true environment or the agent has a companion map learning/estimation algorithm to build such a model, then this combination of the aggregation map with Q-learning converges.
It is an interesting topic to learn such maps. We will consider learning such an aggregation map in \Cref{chap:abs-learning}.




\section{Setup}\label{sec:setup}
In this chapter, we call the following a \emph{state-process} \footnote{It is technically a state and reward process, but since rewards are not affected by the mapping $\psi$, we suppress the reward part to put more emphasis on the contrast between history and state dependence. } induced by the map $\psi$.
\begin{definition}[State-process]
    For a history $h$ that is mapped to a state $s$, a state-process $\mu_h$ is a stochastic mapping from a state-action pair with the fixed state $s$ to state-reward pairs. Formally, $\mu_h : \{s\} \times \A \to \Dist(\S \times \R)$.
\end{definition}

Recall that in a reward embedding setting, the relationship between the underlying HDP and the induced state-process for an $s=\psi(h)$ is formally defined as:
\beq \label{eq:state-process}
\mu_h(s'r'\|sa) := \sum_{o': \psi(hao'r')=s'}\mu(o'r'\|ha).
\eeq

We denote the action-value function of the state-process by $q$, and the optimal Q-value function is given by $q^*_\mu$:
\beq
q^*_\mu(sa\|h) := \sum_{s'r'} \mu_h(s'r'|sa) \left(r' + \g \max_{\d a} q^*_\mu(s'\d a\|h)\right) \label{eq:q}
\eeq

It is clear that $\mu_h(s'r'\|sa)$ may not be same as $\mu_{\dot{h}}(s'r'\|sa)$ for any two histories $h$ and $\dot{h}$ mapped to a same state $s$. If the state-process is an MDP, then $\mu_h$ is independent of history and so is $q^*_\mu$, and convergence of Q-learning follows from this MDP condition \cite{Bertsekas1996}. However, we do not assume such a condition and go beyond MDP mappings. We later show --- by constructing examples --- that $q^*_\mu$ can be made independent of history while the state-process is still history dependent, i.e. non-MDP.

Now we formally define Q-learning in the finite-state MDP setting: At each time-step $n$ the agent maintains an action-value function estimate $q_n$. The agent in a state $s := s_n$ takes an action $a := a_n$ and receives a reward $r' := r_{n+1}$ and the next state $s' := s_{n+1}$. Then the agent performs an action-value update to the $sa$-estimate with the following Q-learning update rule:
\beq \label{eq:q-iteration}
q_{n+1}(sa) :=  q_n(sa) + \alpha_n(sa)\left(r' + \g \max_{\d a} q_n(s'\d a) - q_n(sa)\right)
\eeq
where $(\alpha_n)_{n \in \SetN}$ is a learning rate sequence. \ifconf\else This concludes our problem setup. In the next section, we describe QDP, a reasonably large class of decision problems (Figure \ref{fig:QDP}). A pseudo-algorithm of Q-learning in MDP domains is shown in \Cref{alg:q-learning}. Later we prove that Q-learning converges if the state-process is a QDP.\fi

\begin{algorithm}
    \caption{Q-learning in MDPs (Q-learning)}
    \begin{algorithmic}[1]
        \Input MDP environment $\mu$ for sampling, behavior policy $\pi$, and error tolerance $\eps$, learning rate sequence $(\alpha_n)$
        \Output (anytime) optimal state-action-value function $\h q \to q_\mu^*$
        \State Set $\Delta = \infty$ \Comment{any starting value greater than $\eps$ should work}
        \State Set $\h q = \v 0$ \Comment{any initial value should work}
        \State Observe the current state $s$ from the environment $\mu$
        \Repeat \Comment{until the iteration converges}
        \State Increment time-step to $n$
        \State Take action $a$ from the behavior policy $\pi$ \Comment{sufficiently exploratory policy}
        \State Observe the next state $s'$ and reward $r'$ from the environment $\mu$
        \State Set $q = \h q$ \Comment{save old estimate of $q_\mu^*$}
        \State Set $\h q(sa) = \h q(sa) + \alpha_n(sa) \left(r' + \g \max_{\d a} \h q(s'\d a) - \h q(sa)\right)$ \Comment{new estimate of $q_\mu^*$}
        \State Set $\Delta = \norm{\h q - q}$ \Comment{typically, the infinity norm}
        \Until $\Delta \leq \eps$
    \end{algorithmic}
    \label{alg:q-learning}
\end{algorithm}

\section{MDP vs QDP vs POMDP} \label{sec:QDP}

Recall from \Cref{chap:state-only-abs} the class of environments called Q-value uniform decision processes, i.e.\ QDP class. This class is substantially larger than MDP and has a non-empty intersection with POMDP and HDP (Figure \ref{fig:QDP}). The state-uniformity condition of QDP is weaker than the MDP condition: the latter implies the former but not the other way around \cite{Hutter2016}. Therefore, trivially MDP~$\subseteq$~{\small\stackanchor{QDP}{POMDP}}~$\subseteq$~HDP. The example in Figure \ref{fig:example-mrp} shows that QDP $\cap$ POMDP$\setminus$MDP $\neq\emptyset$. The example in Figure \ref{fig:example-mdp} proves that QDP $\setminus$POMDP $\neq\emptyset$. Moreover, it is easy to find examples in POMDP$\setminus$QDP and HDP$\setminus$QDP.

It is easy to see that QDP is a much larger class than MDP since the former allows non-stationary domains: it is possible to have $\mu_{h_n}(\cdot | sa) \neq \mu_{h_m}(\cdot | sa)$ for some histories $\psi(h_n)=\psi(h_m)=s$ at two different time-steps but still maintaining $Q^*_\mu(h_na)=Q^*_\mu(h_ma)$ (Figure \ref{fig:example-mdp}). Moreover, the definition of QDP enables us to approximate most if not all problems as a QDP model: any number of \emph{similar} $Q^*_\mu$-value histories can be merged into a single QDP state \cite{Hutter2016}. In particular, a QDP model of the environment can provide more compression in terms of state space size as compared to an MDP model: multiple MDP states with the same/similar Q-value can be merged into a single QDP state but not necessarily the other way around. Thus, QDP allows for more compact models for an environment than its MDP counterparts.

In general, we can say that a POMDP has both the dynamics and Q-values as functions of history. Whereas, the definition of QDP provides models where only the dynamics can be history-dependent. Therefore, QDP captures a subset of POMDP models that have history-independent Q-values.

%
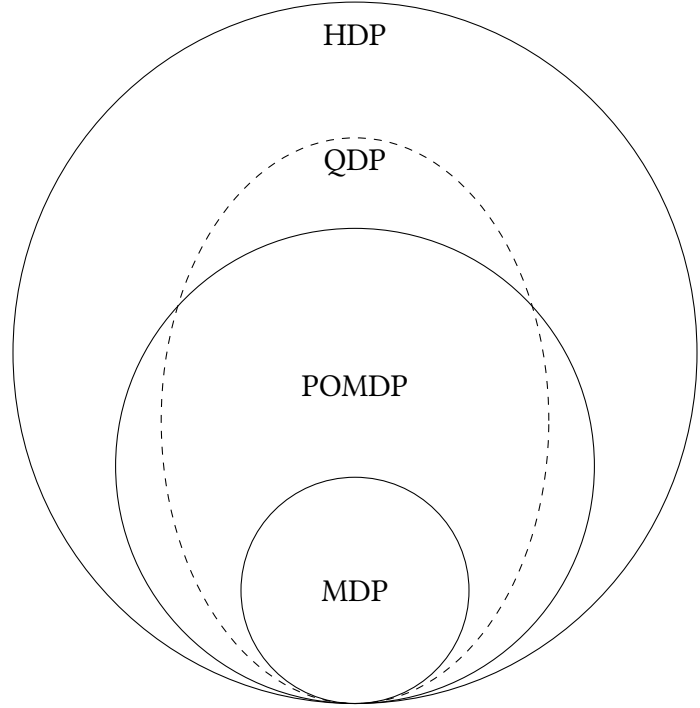
\begin{figure}[!]
    \centering
    \begin{tikzpicture}[scale=1.5]
    \draw[black] (0,-0.4) ellipse (3cm and 3.1cm);
    \node[below] at (0,2.6cm) {HDP};
    \draw[black] (0,-1.4cm) ellipse (2.1cm and 2.1cm);
    \node[below] at (0,-0.5cm) {POMDP};
    \draw[black] (0,-2.5cm) circle (1cm);
    \node[] at (0,-2.5cm) {MDP};
    \draw[black, dashed] (0,-1cm) ellipse (1.7cm and 2.5cm);
    \node[below] at (0,1.5cm) {QDP};
    \end{tikzpicture}
    \caption{QDP in the perspective of other decision problem classes.} \label{fig:QDP}
\end{figure}

\section{Q-learning Convergence in QDP}\label{sec:main-result}

In this chapter, we assume that the state-process is ergodic --- i.e. all states are reachable under any policy from the current state after sufficiently many steps.

\begin{assumption}[Ergodicity]\label{asmp:state-ergodicity}
    The state-process is ergodic.
\end{assumption}

Because of the ergodicity assumption we can suppose the following standard stochastic approximation conditions on each state-action $sa$-pair's learning rate sequence $(\alpha_n)_{n \in \SetN}$\footnote{Note that $\alpha_n(sa):=0, \forall sa \neq s_na_n$.}.
\beq \label{eq:infinite-visit}
\sum_{n=0}^{\infty} \alpha_n(sa) = \infty, \qquad \sum_{n=0}^{\infty} \alpha_n^2(sa) < \infty.
\eeq

The above conditions on the learning rate ensure that the agent asymptotically decreases the learning rate to converge to a fixed point but never stops learning to avoid local maxima \cite{Bertsekas1996}. It is critical to note that we assume ergodicity of the state-process but not of the underlying HDP: a state can be reached multiple times from different histories but any history is only reached once. We assume that the state-process is a QDP.

\begin{assumption}[QDP] \label{asmp:phi-uniformity} The state-process is a QDP.
\end{assumption}

It is important to consider that we only assume the optimal action-value to be a function of states. We do not suppose any structure on the intermediate action-value estimates \emph{cf.}\ $q_n(sa) \neq Q_n(ha)$. Recall that we also assume that rewards are bounded. This is a standard condition for Q-learning convergence.
We have all the components in place to extend Q-learning convergence in QDP.

\begin{theorem} [Q-learning Convergence in QDP] \label{thm:non-MDP-convergence}
    Under Assumptions \ref{asmp:state-ergodicity}, \ref{asmp:phi-uniformity}, and bounded rewards, and with a learning rate sequence $(\alpha_n)_{n \in \SetN}$ satisfying \Cref{eq:infinite-visit}, the sequence $(q_n)_{n \in \SetN}$ generated by the iteration \eqref{eq:q-iteration} converges to $q^*_\mu=Q^*_\mu$.
\end{theorem}

Hence, the agent learns the optimal action-value function of a QDP state-process. We provide a proof of Theorem \ref{thm:non-MDP-convergence} below. Superficially, the proof looks similar to a standard MDP proof \cite{Bertsekas1996}, however, a subtlety is involved in the definition of the contraction map: the contraction map is a function of history. This history-dependence lets the proof scale to non-MDP domains, especially to non-stationary domains.
\ifconf\else Nevertheless, the key contribution of this work is to show that the MDP convergence guarantee of Q-learning can be extended to QDP rather easily with our setup.\fi

\paradot{Proof of Theorem \ref{thm:non-MDP-convergence}} At a time-instant $n$ with a history $h_n$ we  rewrite \Cref{eq:q-iteration} in terms of an operator and a noise term.
\beq
q_{n+1}(sa) = \left(1-\alpha_n(sa)\right)q_n(sa) +  \alpha_n(sa)\left(F_{h_n}q_n(sa) + w_{h_n}(sa)\right)
\eeq
where, the noise term is defined as follows,
\beq
w_{h_n}(sa) := r' + \g \max_{\d a} q_n(s'\d a) - F_{h_n}q_n(sa).
\eeq

Since the agent samples from the underlying HDP, the operator $F_{h_n}$ is defined to be a history-based operator.
\beq
F_{h_n}q_n(sa) := \E_{\mu_{h_n}}\left[r' + \g \max_{\d a} q_n(s'\d a)  \middle| \F_n \right]
\eeq
where $\F_n$ is a complete history of the algorithm up to time-step $n$ that signifies all information including $h_n$, $(\alpha_k)_{k \leq n}$ and the state sequence $(s_k)_{k \leq n}$. We use $\E_{\mu_{h_n}}$ as an expectation operator with respect to $\mu_{h_n}$.

\paradot{Noise is bounded} Now we show that the noise term is not a significant factor that affects the convergence of Q-learning.  By construction it has a zero mean value:
\beq
\E_{\mu_{h_n}}\left[w_{h_n}(sa)| \F_n \right] = \E_{\mu_{h_n}}\left[r' + \g \max_{\d a} q_n(s' \d a) -F_{h_n}q_n(sa)  \middle| \F_n\right] = 0.
\eeq

Due to the bounded reward assumption the variance of the noise term is also bounded.

\bqan
\E_{\mu_{h_n}}\left[w_{h_n}^2(sa) \middle| \F_n \right]
&=  \E_{\mu_{h_n}}\left[\left(r' + \g \max_{\d a} q_t(s'\d a)\right)^2 \middle| \F_n\right]  - \E_{\mu_{h_n}}\left[r' + \g \max_{\d a} q_n(s'\d a) \middle| \F_n\right]^2 \\
&\overset{(a)}{\leq} \frac{1}{4}\left(\max_{s'r'}\left(r' + \g \max_{\d a} q_n(s'\d a)\right) -  \min_{s'r'}\left(r' + \g \max_{\d a} q_n(s'\d a)\right)\right)^2 \\
&\overset{(b)}{\leq} \frac{1}{4}\left( \frac{r_{\max}-r_{\min}}{1-\g} + \g \norm[\infty]{q_n}\right)^2 \\
&\overset{(c)}{\leq} A + B \left(\norm[\infty]{q_n}\right)^2 \numberthis
\eqan
$(a)$ follows from Popoviciu's inequality, $\Var(X) \leq \frs{1}{4}(\max X - \min X)^2$, $(b)$ is due to the bounded rewards assumption; and $(c)$ results from some algebra with constants $A := \frs{\Delta}{4}\left(\frs{2\g}{1-\g} + \Delta \right)$ and $B := \frs{\g^2}{4}$, where $\Delta := \left(r_{\max} - r_{\min}\right)/(1-\g)$. We denote a \emph{sup-norm} by $\norm[\infty]{\cdot}$.

\paradot{$F_{h}$ is a contraction} For a fixed history $h$, we show that an operator $F_{h}$ is a contraction mapping.
\bqan\label{eq:contraction}
\norm[\infty]{F_{h}q - F_{h}q'}
&\overset{}{=} \max_{sa} \left| \E_{\mu_{h_n}}\left[r' + \g \max_{\d a} q(s'\d a) \middle| \F_n\right]  - \E_{\mu_{h_n}}\left[r' + \g \max_{\d a} q'(s'\d a) \middle| \F_n\right]\right| \\
&\overset{(a)}{=} \max_{sa} \left| \E_{\mu_{h_n}}\left[r' + \g \max_{\d a} q(s'\d a) \middle| sa\right] - \E_{\mu_{h_n}}\left[r' + \g \max_{\d a} q'(s'\d a) \middle| sa\right]\right| \\
&\overset{}{\leq} \max_{sa} \max_{s'} \g \left| \max_{\d a} q(s'\d a) - \max_{\d a} q'(s'\d a)\right| \\
&\overset{}{\leq} \g \max_{sa} \left|q(sa) - q'(sa)\right|
= \g \norm[\infty]{q - q'}\numberthis
\eqan
$(a)$ for a fixed history, the expectation only depends on the $sa$-pair rather than the complete history $\F_n$. Therefore $\norm[\infty]{F_{h}q - F_{h}q'} \leq \g \norm[\infty]{q - q'}$, hence, $F_{h}$ is a contraction.

\paradot{Same fixed point} We show that for any history $h$, the contraction operator $F_h$ has a fixed point $q^*_\mu$. Let $h$ be mapped to state $s$:
\bqan\label{eq:fixed-point}
q^*_\mu(sa) &\overset{(a)}{=} Q^*_\mu(ha) \\
&\equiv \sum_{o'r'} \mu(o'r'\|ha)\left(r' + \g \max_{\d a} Q^*_\mu(h'\d a)\right) \\
&\overset{(b)}{=} \sum_{s'r'} \mu_h(s'r'\|sa)\left(r' + \g \max_{\d a} q^*(s'\d a)\right) \\
&\equiv F_h q^*_\mu(sa) \numberthis
\eqan
$(a)$ is the QDP assumption; and $(b)$ follows from \Cref{eq:state-process} and again using the QDP assumption. We also show that for any history $h$ the operator $F_h$ has a same contraction factor $\g$.
\beq
\norm[\infty]{F_hq-q^*_\mu} \overset{(a)}{=} \norm[\infty]{F_hq - F_hq^*_\mu} \overset{(b)}{\leq} \g \norm[\infty]{q-q^*_\mu}
\eeq
$(a)$ follows from \Cref{eq:fixed-point}; and $(b)$ is due to \Cref{eq:contraction}. Therefore, for any history $h$ the operator $F_h$ has the same fixed point $q^*_\mu$ with the same contraction factor $\g$.

We have all the conditions to invoke a convergence result from \cite{Bertsekas1996}. We adopt\footnote{The original proposition is slightly more general than we need for our proof. It has an extra diminishing noise term which we do not have/require in our formulation.} and state Proposition 4.5 from \cite{Bertsekas1996} without reproducing the complete proof.

\begin{proposition}[Prop. 4.5 \cite{Bertsekas1996}] \label{prpo:4.5}
    Let $(q_n)_{n\in\SetN}$ be the sequence generated by the iteration \eqref{eq:q-iteration}. We assume the following.
    \begin{enumerate}
        \item[(a)] The learning rates $\alpha_n(sa)$ are nonnegative and satisfy,
        \beqn
        \sum_{n=0}^{\infty} \alpha_n(sa) = \infty, \qquad \sum_{n=0}^{\infty} \alpha_n^2(sa) < \infty.
        \eeqn
        \item[(b)] The noise term $w_{n}(sa)$ satisfies,
        \bqan
        \E_{\mu_{n}}\left[w_{n}(sa)| \F_n \right] &= 0, \\ \E_{\mu_{n}}\left[w_{n}^2(sa)| \F_n \right]  &\leq A + B \norm[\infty]{q_n}^2, \quad \forall s,a,n
        \eqan
        where, $A$ and $B$ are constants.
        \item[(c)] There exists a vector $q^*_\mu$, and a scalar $\g \in [0,1)$, such that,
        \beqn
        \norm[\infty]{F_{n}q_n - q^*_\mu} \leq \g \norm[\infty]{q_n-q^*_\mu}, \quad \forall n.
        \eeqn
    \end{enumerate}
    Then, $q_n$ converges to $q^*_\mu$ with probability 1.
\end{proposition}

\paradot{Proof Sketch} We have a sequence of maps $(F_{n})_{n \in \SetN}$. At any time-step $n$, the map $F_n$ is a contraction and every map moves the iterates toward the same fixed point $q^*_\mu$. Since, the contraction factor is the same, the rate of convergence is not affected by the order of the maps. Every map contracts the iterates by a factor $\g$ with respect to the fixed point that asymptotically converges to $q^*_\mu$.\qed

Proposition \ref{prpo:4.5} uses a sequence of maps $(F_{n})_{n\in \SetN}$ with a same fixed point. In our case, we have this sequence based on histories, i.e.$(F_{h_n})_{n\in \SetN}$. Similarly, for $w_n$ and $\mu_n$ we have corresponding history-based instances. Therefore, Proposition \ref{prpo:4.5} with $\mu_n = \mu_{h_n}$, $F_n = F_{h_n}$ and $w_n = w_{h_n}$ provides the main result. This can be done, since $\alpha_n, q_n, w_n$ and $F_n$ are allowed to be random variables. \qed



Obviously, state-uniformity is a necessary condition, since otherwise $Q^*_\mu(ha)$ can not even be represented as $q^*_\mu(sa)$.
%
%
%
Typically, the state-process is assumed to be an MDP. This makes the state-process $\mu_h$ independent of history, and leads to a history-independent operator $F := F_h$ for any history $h$, which (trivially) all have the same fixed point.
We relax this MDP assumption, and only demand state-uniformity of the optimal value-function. The proof shows that this condition is sufficient to provide a unique fix point for the \emph{history-dependent} operators. Therefore, the state-uniformity is not only a necessary but also a sufficient condition for Q-learning convergence.

\section{Empirical Evaluation}\label{sec:evaluation}
In this section, we empirically evaluate two example non-MDP domains to show the validity of our result.

\paradot{Non-Markovian Reward Process} Let us consider our first example from \cite{Hutter2016} to demonstrate Q-learning convergence to a non-Markovian reward processes\footnote{An MPR is an MDP with only one action, i.e.\ $\abs{\A} = 1$.} (non-MRP). We consider that the underlying HDP is an MDP over the observation space $\O$ (in fact an action-independent MDP, i.e. an MRP) with a transition matrix $T$ and a deterministic reward function $R$. The state diagram of the process is shown in Figure \ref{fig:example-mrp}.
\beq
T =
\begin{bmatrix}
    0 & 1/2 & 1/2 & 0  \\
    1/2 & 0 & 0 & 1/2  \\
    0 & 1 & 0 & 0  \\
    1 & 0 & 0 & 0
\end{bmatrix}, \qquad
R = \begin{bmatrix}
    \frac{\g/2}{1+\g}  \\
    \frac{1+\g/2}{1+\g}  \\
    0  \\
    1
\end{bmatrix}.
\eeq

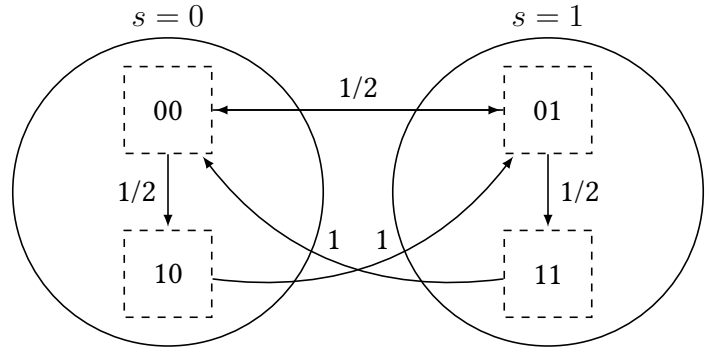
\begin{figure}[!]
    \centering
        \begin{tikzpicture}[->,>=stealth',shorten >=1pt,auto,node distance=5cm,
        semithick,square/.style={regular polygon,regular polygon sides=4,dashed},
        outer/.style={inner sep=6pt, column sep=1cm, row sep=1cm}
        ]
        \tikzstyle{every state}=[]

        \matrix (0) [matrix of nodes, state, outer, nodes={state,square}, label ={$s=0$}]{
            00 \\
            10 \\
        };

        \matrix (1) [matrix of nodes, outer, state, right of=0, nodes={state,square},label={$s=1$}]{
            01 \\
            11 \\
        };

        \path (0-1-1) edge     node {1/2} (1-1-1)
        edge   node[left] {1/2} (0-2-1)
        (1-1-1) edge  (0-1-1)
        edge  node{1/2} (1-2-1)
        (0-2-1) edge [bend right] node[above, pos=0.5]{1} (1-1-1)
        (1-2-1) edge[bend left] node[above, pos=0.5]{1} (0-1-1);
        \end{tikzpicture}
    \caption[An example MRP aggregated to a non-MDP (non-MRP).]{An example MRP aggregated to a non-MDP (non-MRP). The square nodes represent the states of the underlying MRP. The circles are the aggregated states. The solid arrows represent the only available action $x$ and the transition probabilities are shown at the transition edges.}\label{fig:example-mrp}
    \centering
\end{figure}

\begin{figure}[!]
    \centering
    \includegraphics[width=0.8\linewidth]{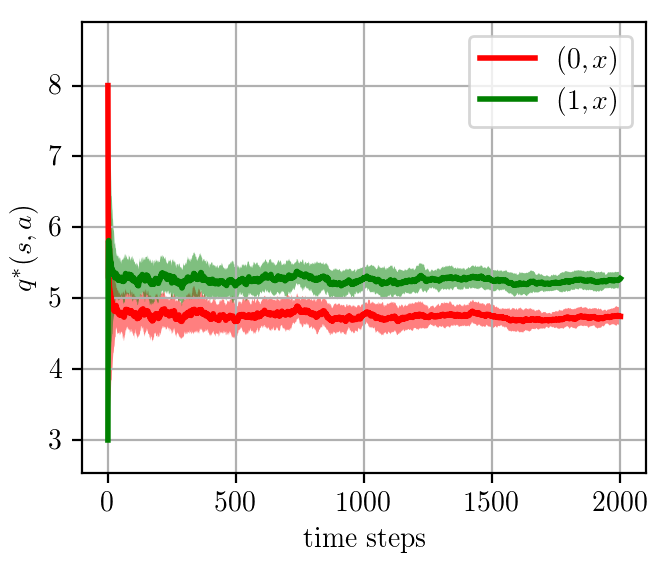}
    \caption[The learning curves of Q-learning of the first example domain.]{The learning curves of Q-learning are averaged over 40 independent runs with the parameters, $\g = 0.9$, $q_0(s=0,x)=8$ and $q_0(s=1,x)=3$.}\label{fig:convergence-mrp}
\end{figure}

Due to this structure, the HDP is expressible as, $\mu(o'r'\|ha) = T_{oo'} \cdot \llbracket r' = R(o)\rrbracket$, such that $h$ has a last observation $o$, where $\llbracket \cdot \rrbracket$ denotes an Iverson bracket.  The observation space is $\O = \{00,01,10,11\}$. Let us consider the state space $S = \{0,1\}$, and the agent experiences the state-process under the following aggregation map:
\beqn
s_t := \psi(h_t) :=
\begin{cases}
    0, \qquad \text{if } o_t = 00 \text{ or } 10, \\
    1, \qquad \text{if } o_t = 01 \text{ or } 11.
\end{cases}
\eeqn

It is easy to see that the resultant state-process is not an MDP (MRP):
\bqan
\mu_{00}(s'=0\|s=0) &= T_{00,00} + T_{00,10} = 0 + \frs{1}{2} = \frs{1}{2} \\
\mu_{10}(s'=0\|s=0) &= T_{10,00} + T_{10,10} = 0 + 0 = 0
\eqan
which implies $\mu_{00} \neq \mu_{10}$, but this state-process satisfies the optimal Q-value function state-uniformity condition (see below). Hence, the state-process is a QDP $\in$ POMDP$\setminus$MDP: the underlying HDP has a finite set of hidden states, i.e.\ the states of the underlying MDP. Since it is an action-independent process, the action-value function is the same for any action $a \in \A$. We denote the only available action with $x$:
\bqan
q^*_\mu(s=0,x) &:= Q^*_\mu(00,x) = Q^*_\mu(10,x) =  \frac{\g}{1-\g^2}\\
q^*_\mu(s=1,x) &:= Q^*_\mu(01,x) = Q^*_\mu(11,x) = \frac{1}{1-\g^2}.
\eqan

\begin{figure}[!]
    \centering
    \includegraphics[width=0.8\linewidth]{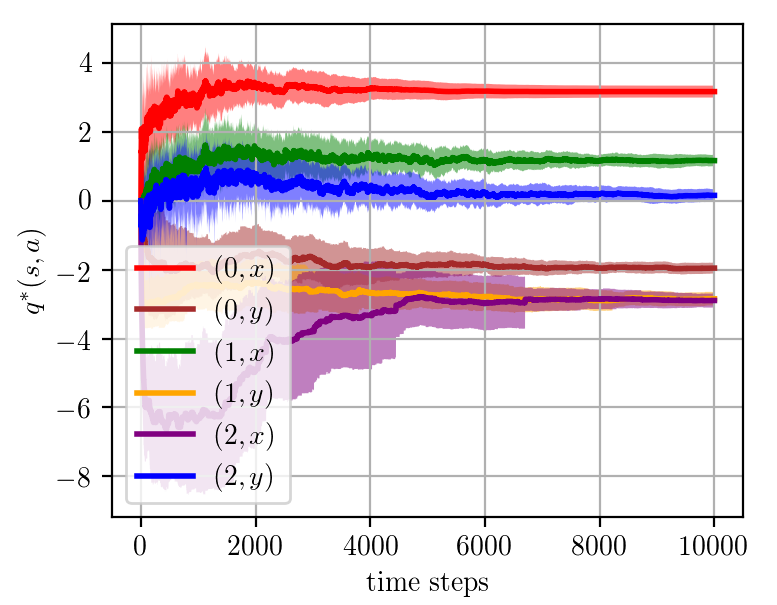}
    \caption[The learning curves of Q-learning of the second example domain.]{The learning curves of Q-learning are averaged over 50 independent runs with the parameters, $\g = 0.9$, $p_{\min} = 0.01$ and $q_0(sa)=0$ for all $s$ and $a$.}\label{fig:convergence-mdp}
\end{figure}

We apply Q-learning to the induced state-process and the learning curves plot is shown in Figure \ref{fig:convergence-mrp}. The plot shows that Q-learning is able to converge to the optimal action-value function of the process, despite the fact that the state-process is a non-MDP (non-MRP).

\begin{figure}[!]
    \centering
    \resizebox{\textwidth}{!}{%
        \begin{tikzpicture}[->,>=stealth',shorten >=1pt,auto,node distance=4cm,inner sep=3pt,
        semithick]
        \tikzstyle{every state}=[]

        \node[state,label={above:$s=0$}]              (0) {
            $h= \begin{cases}
            \epsilon \\
            \dots v
            \end{cases}
            $
        };
        \node[state,label={$s=1$}] [right of=0] (1) {
            $h= \begin{cases}
            xi \\
            yi \\
            \dots vxi \\
            \dots vyi
            \end{cases}
            $
        };
        \node[state,label={above:$s=2$}] [right of=1] (2) {
            $\underset{\text{s.t. } \abs{\d h} \geq 2 \text{ and } v \notin \d h}{h=\dots \d h}$
        };

        \path (0) edge [loop left,looseness=6]     node{$p_{v}(h)$}      (0)
        (0) edge [bend left]      node[above]{$1-p_{v}(h)$}    (1)
        (0) edge [dashed]         node{$1$}           (1)
        (1) edge [bend left]      node[below]{$p_{v}(h)$}      (0)
        (1) edge [bend left]      node[above]{$1-p_{v}(h)$}    (2)
        (1) edge [dashed]         node{$1$}           (2)
        (2) edge [bend left=50, dashed]   node[below]{$p_{v}(h)$}      (0)
        (2) edge [loop right, dashed, looseness=6]     node{$1-p_{v}(h)$}    (2)
        (2) edge [loop below, looseness=6]     node{$1$}           (2)
        ;

        \end{tikzpicture}
    }
    \caption[A complete history-dependent process is aggregated to a 3-state non-MDP.]{A complete history-dependent process is aggregated to a 3-state non-MDP. The circles are states. Inside the states are the corresponding history patterns mapped to the state. For clarity, the rewards are not shown in the history patterns. The action $x$ is denoted by the solid arrows while the action $y$ is denoted by the dashed arrows. The transition probabilities are indicated at the transition edges.}\label{fig:example-mdp}
\end{figure}
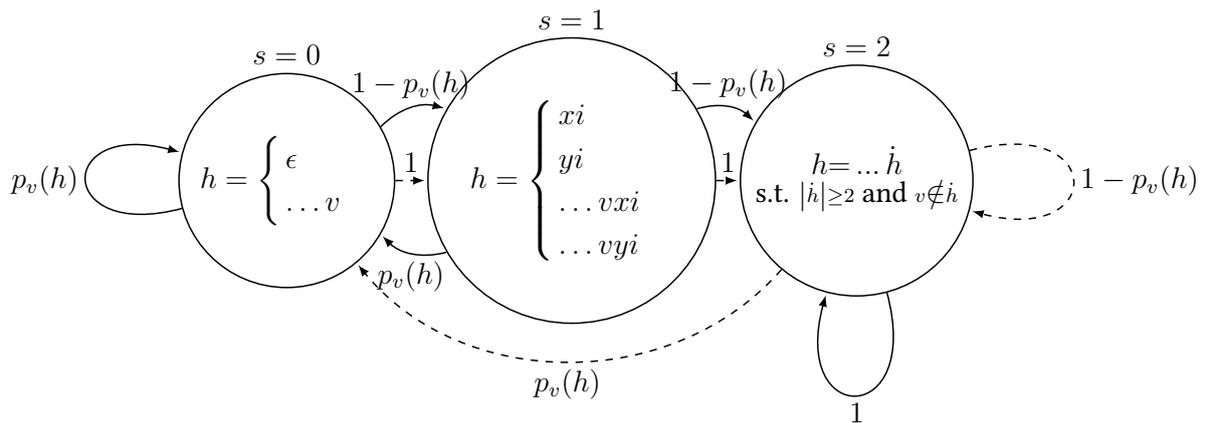

\paradot{Non-Markovian Decision Process} The previous example demonstrated that Q-learning is able to learn a non-MRP $\in$ QDP. Now we provide an example QDP which is a two-action non-MDP $\in$ HDP$\setminus$POMDP: the state space of the underlying HDP is infinite. The agent is facing a \emph{non-stationary} state-process with state space $\S = \{0,1,2\}$ and action space $\A = \{x,y\}$. The agent has to input a right \emph{key}-action $k_s$ at a state $s = \psi(h)$ but the environment accepts the key action with a certain history-dependent probability $p_{v}(h)$ by providing an observation from $\O = \{v, i\}$, where $v$ and $i$ indicate acceptance or rejection of an input, respectively.
\beq
p_{v}(h) := \max\{p_{\min} , \%(v,h)\}
\eeq
where, $\%(v,h)$ is the percentage of accepted keys in $h$ and $p_{\min}$ is a minimum acceptance probability. Without loss of generality, we use the key sequence $k_0 := x, k_1 := x, k_2 := y$. The history to state mapping is defined as:

\beq
\psi(h) = \begin{cases}
    0 \, &\text{if, } h= \begin{cases}
        \epsilon \\
        \dots v
    \end{cases} \\
    1 \, &\text{if, } h= \begin{cases}
        xi \\
        yi \\
        \dots vxi \\
        \dots vyi
    \end{cases} \\
    2 \, &\text{if, } h=\dots \dot{h} \text{ such that } |\dot{h}| \geq 2 \text{ and } v \notin \dot{h}
\end{cases}
\eeq

It is apparent from the mapping function that \emph{state-0} is the start state, and it is also the case when a key is accepted in the last time-step, \emph{state-1} is defined when a key input is rejected once, and \emph{state-2} is reached when the key input has been recently rejected at least twice in a row.

The transition probabilities are formally given as follows (see Figure \ref{fig:example-mdp} for a graphical representation):
\beq
\mu_h(s'\|sa) = \begin{cases}
    p_{v}(h) \, &\text{if, } s'=0, s=0|1|2, a = k_s \\
    1-p_{v}(h) \, &\text{if, } \begin{cases}  s'=1, s = 0, a = k_s \\ s'=2, s = 1, a = k_s \\ s'=2, s = 2, a = k_s \end{cases} \\
    1 \, &\text{if, } \begin{cases}  s'=1, s = 0, a \neq k_s \\ s'=2, s = 1, a \neq k_s \\ s'=2, s = 2, a \neq k_s \end{cases} \\
    0 \, &\text{otherwise.}
\end{cases}
\eeq

The reward is also a function of the complete history.
\beq
r'(ha) = \begin{cases}
    3-\g-2\g p_{v}(h) \; &\text{if, } \psi(h)=0, a = k_0, \\
    1-3\g p_{v}(h) \; &\text{if, } \psi(h)=1, a = k_1, \\
    -3\g p_{v}(h) &\text{if, } \psi(h)=2, a = k_2, \\
    -3 \; &\text{if, } \psi(h) = s, a \neq k_s
\end{cases}
\eeq

It is easy to see that Q-values are only a function of state-action pairs as follows for $\psi(h) = s$:
\beq
q^*_\mu(sa) = Q^*_\mu(ha) = \begin{cases}
    3 \; &\text{if, } s=0, a = x, \\
    -2 \; &\text{if, } s=0, a = y, \\
    1 \; &\text{if, } s=1, a = x, \\
    -3 \; &\text{if, } s=1, a = y, \\
    -3 \; &\text{if, } s=2, a = x, \\
    0 \; &\text{if, } s=2, a = y.
\end{cases}
\eeq

Despite the history-based dynamics, Figure \ref{fig:convergence-mdp} shows that Q-learning is able to learn the Q-values of the \emph{non-stationary} state-process due to the fact that it is a QDP.

\section{Related Work}\label{sec:related-work}

A similar result was first reported by \citet{Li2006} in a finite state MDP setting. We confirm and extend the findings by considering a more general class of history-based environments. Also, we do not assume a weighting function to define the state-process \emph{cf.} \cite[Definition 1]{Li2006}, and our proof is based on time-dependent contraction mappings (see Section \ref{sec:main-result} for the details).

A finite-state POMDP is the most commonly used extension of an MDP. It is well-known that the class of finite-state POMDPs is a subset of HDP class \cite{Leike2016b}. One prevalent approach to handle the non-Markovian nature of a POMDP is to estimate a Markovian model with a state estimation method  \cite{Whitehead1995,Lin1992,Cassandra1994,Cassandra1994a} or use a finite subset of the recent history as a state to form a $k$-order MDP ($k$-MDP) \cite{McCallum1995}.
Then Q-learning is applied to learn this resultant state-based MDP.
This is a different approach to ours. We do not try to estimate an MDP or $k$-MDP representation of the underlying HDP.

\citet{Singh1994} investigate a direct application of model-free algorithms to POMDPs without a state estimation step akin to our setup but limited to finite state POMDPs only. They show that an optimal policy in a POMDP may be non-stationary and stochastic. However, the learned policy in direct model-free algorithms, such as Q-learning, is generally stationary and deterministic by design. Moreover, they also show that the optimal policy of a POMDP can be arbitrarily better than the optimal stationary, deterministic policy of the corresponding MDP. These negative results of \citet{Singh1994} are based on counter-examples that violate the state-uniformity assumption of the optimal Q-value function. Similar negative findings are reported by \citet{Littman1994}. Our positive convergence result holds for a subset of POMDPs that respects the state-uniformity condition.

\citet{Pendrith1998} show that a direct application of standard reinforcement learning methods, such as Q-learning, to non-Markovian domains can have a stationary optimal policy, if undiscounted return is used as a performance measure.
\citet{Perkins2002} prove existence of a fixed point in POMDPs for continuous behavior policies. However, this fixed point could be significantly worse than the average reward achieved by an oscillating discontinuous behavior policy. It signifies the effect of behavior policy on the learning outcome. On the other hand, our convergence result is valid as long as all the state-action pairs are visited infinitely often. The nature of the behavior policy does not directly affect our convergence result. A comprehensive survey of solution methods for POMDPs is provided by \citet{Murphy2000} and more recently by \citet{Spaan2012}.

\section{Summary}\label{sec:conclusion}

In this chapter, we proved that Q-learning convergence can be extended to a much larger class of decision problems than finite-state MDPs. In QDP, the optimal action-value function of the state-process is still only a function of states, but the dynamics can be a function of the complete history (in effect, a function of time). That enables QDP to allow non-stationary domains in contrast to finite-state MDPs that can only model stationary domains. We also showed that this state-uniformity condition is not only a necessary but also a sufficient condition for Q-learning convergence.
An empirical evaluation of a few non-MDP domains is also provided.

The proof in this chapter relies on the \emph{exact} state-uniformity of the optimal Q-value function --- also known as \emph{exact} state-aggregation condition \cite{Hutter2016}. It is intriguing to explore if the proof can be extended to an \emph{approximate} aggregation since it can be considered as a special case of function approximation. It is well-known that \emph{off-policy}, model-free learning with function approximation sometimes diverges \cite{Baird1995,Baird1999}. Therefore, it is not a priori clear if our approach can be extended to an \emph{approximate} aggregation case under some interesting condition(s) that can exclude such counter examples. Nevertheless, it is a potential extension of this work.
%
%

\chapter[Representation Guarantees for non-MDP Homomorphisms]{Representation Guarantees for non-MDP Homomorphisms{\\ \it \small This chapter is an adaptation of \citet{Majeed2019}}}\label{chap:representation-guarnt}

\begin{outline}
    A majority of \emph{real-world} problems either have a huge state or action space or both. Therefore, a naive application of existing \emph{tabular}\footnote{In a tabular setup the quantities of interest, e.g.\ $\mu$, $r_\mu$, $\pi_\mu$, $Q^*_\mu$, and $V_\mu^*$, can be represented as a finite table of values.} solution methods is not tractable on such problems. Nevertheless, these solution methods are quite useful if the agent has access to a relatively small state-action space homomorphism of the true environment. The agent can plan on the homomorphic model and uplift the policy to the actual environment.
    A plethora of research is focused on the case when the homomorphism is a Markovian representation of the underlying process. However, we show that the near-optimal performance is sometimes guaranteed even if the homomorphism is non-Markovian, i.e.\ a non-Markovian homomorphism can represent the environment fairly well for planning.
    \ifshort\else Moreover, non-Markovian representations are coarser than their Markovian counterparts. In this work, we extend the analysis of the Extreme State Aggregation (ESA) framework to homomorphisms. We also lift the policy uniformity condition for aggregation in ESA which allows even coarser modeling of the environment. \fi
\end{outline}

\section{Introduction}\label{sec:Intro}

Recall that it is typically assumed that the agent is facing a small state-action space Markov Decision Process (MDP) so the agent can advise a stationary policy as a function of state \cite{Puterman2014}. Unfortunately, the number of state-action pairs in most of real-world problems is prohibitively large, e.g.\ driving a car, playing Go, personal assistance, controlling a plant with real-valued inputs and so forth. The agent can neither simply visit each state-action pair nor can it keep the record of these visits to learn a near-optimal behavior, which is usually the assumption for convergence of many learning algorithms \cite{Watkins1992}. This intractability of state-action space is known as the \emph{curse of dimensionality} in RL \cite{Sutton2018}. Therefore, it is essential for the agent to generalize over its experiences in such problem.

Although most of the abstraction proposals concentrate on the state space reduction \cite{Abel2016,Li2006}, there is another equally important dimension of action space that hinders the application of traditional RL methods to real-world problems. The problem with a small state but huge --- sometimes continuous --- action space is equally challenging for learning and planning, {\em cf.} continuous bandit problem \cite{Bubeck2012}.

A \emph{homomorphism} framework originated by \citet{Whitt1978} is a well-studied solution to handle the state-action space curse of dimensionality. In the homomorphism framework a problem of a large state-action space is \emph{solved} by using an \emph{abstract} problem with a relatively small state-action space. The \mbox{(near-)optimal} policy of the abstract problem is a \emph{solution} if it is also a \mbox{(near-)optimal} policy in the true environment.

It is important to highlight that homomorphism is not the only technique for abstracting actions. The \emph{options} framework is a competing method for temporal action abstractions \cite{Sutton1999a}. In the option/macro-action framework, the original action space is augmented with long-term/built-in policies \cite{McGovern1997}. The agent using an option/macro-action commits to execute a fixed set of actions for a fixed (expected) time duration. This temporal action abstraction framework is arguably more powerful but beyond the scope of this work. Because, to the best of our knowledge, there are no theoretical performance guarantees available for such methods, and most probably such bounds might not exist.

In the homomorphism framework it is typically assumed that the abstract problem is an MDP \cite{Ravindran2003,Ravindran2004,Taylor2008}. However, the size of the abstract state-action space can be significantly reduced if non-MDP abstractions are possible \cite{Abel2016,Li2006}. Moreover, the reduction of abstract state-action space roughly\footnote{Although reduction of state-action space is necessary for faster learning/planning but not sufficient \cite{Littman1995}.} translates into faster learning and planning \cite{Strehl2009,Lattimore2014}.

In this chapter, we use similar notation and techniques of ESA but investigate and prove optimality bounds for non-MDP state-action homomorphisms in ARL. Since state abstraction is a special case of homomorphism (the action space is not reduced/mapped), our work is a generalization of ESA \cite{Hutter2016}.

The homomorphism framework has been extended beyond MDPs to finite-state POMDPs \cite{Wolfe2010}. As mentioned earlier,  GRL has an infinite set of histories and no two histories are alike. We can represent a finite-state POMDP environment as a history-based process by imposing a structure that there is an internal MDP that generates the observations and rewards. The GRL framework, by design, is more powerful and expressive than a finite-state POMDP \cite{Wolfe2010,Leike2016b}. Therefore, our results are more general than finite-state POMDP homomorphisms.

In this chapter we expand ESA representation guarantees to homomorphisms that scale beyond MDPs. We make another important technical contribution by relaxing the policy uniformity condition in ESA. In ESA, the states are aggregated together if they have the same policy. We show that this requirement can be relaxed and states with approximately similar policies can also be aggregated together with little performance loss. It enables us to have near-optimal maps with, considerably, more coarsely aggregated state-action pairs.

\section{Homomorphism Setup}

Recall that we define a homomorphism as a surjective mapping $\psi$ from the original state-action space $\H \times \A$ to the abstract state-action space $\S \times \B$.
For a succinct exposition, we also define a few marginalized mapping functions. These marginalized maps do not have any special significance other than making the notation a bit simpler.

\paradot{Histories mapped to an $sb$-pair} For a given abstract action $b \in \B$, we define a marginalized abstract state map as
\beq
\psi^{-1}_b(s) := \left\lbrace h \in \H \mid \exists a \in \A : \psi(ha)=sb \right\rbrace.
\eeq

\paradot{Actions mapped to an $sb$-pair} Similarly, we also define a marginalized abstract action map for any abstract state $s \in \S$ and history $h \in \H$ as
\beq
\psi^{-1}_s(b) := \left\lbrace a \in \A \mid \psi(ha)=sb \right\rbrace.
\eeq

It is important to note that $\psi^{-1}_s(b)$ is also a function of history. This dependence is always clear from the context, so we suppress it in the notation. 

\paradot{Abstract states mapped by a history} By a slight abuse of notation we overload $\psi$, and define a history to abstract state marginalized map as
\beq
\psi(h) := \{s \in \S \mid \exists a \in \A, b \in \B : \psi(ha)=sb\}.
\eeq

\paradot{Histories mapped to an abstract state} Finally, an abstract state to history marginalized map is defined as
\beq
\psi^{-1}(s) := \{h \in \H \mid \exists a \in \A, b \in \B : \psi(ha)=sb\}.
\eeq

\ifshort\else In this chapter, we assume a structure for the aggregation map $\psi$. The general unstructured case is left for future research. \fi

\begin{assumption}{$(\psi(h)=s)$}\label{asm:state-history}
    We assume that an abstract state is determined only by the history, i.e.\ $\psi(ha) := (s=f(h),b)$, where $f$ is any fixed \emph{surjective} function of history and is independent of actions $a$ and $b$.
\end{assumption}

The above assumption implies that $\psi(h)$ is \emph{singleton}. This is not only a technical necessity but a requirement to make the mapping \emph{causal}, i.e.\ the current history $h$ corresponds to a unique state $s$ independent of the next action taken by the agent. If we drop this assumption then the current history might resolve to a different state based on the next (future) action taken by the agent.

A homomorphic map $\psi$ lets the agent merge the experiences from $\mu$ and induces a \emph{history-based} abstract process $\mu_\psi$. Formally, for all $\psi(ha) = sb$ and any next abstract state $s'$, we express $\mu_\psi$ as
\beq\label{eq:marginP}
\mu_\psi(s'r'\|ha) := \sum_{o' : \psi(hao'r')=s'} \mu(o'r'\|ha).
\eeq

The map $\psi$ also induces a \emph{history-based} abstract policy $\pi_{\psi}$ as
\beq\label{eq:dpih}
\pi_\psi(b\|h) := \sum_{a \in \psi^{-1}_{s}(b)} \pi(a\|h).
\eeq

It is clear from \Cref{eq:marginP,eq:dpih} that the induced abstract process and policy are in general non-Markovian, i.e.\ both are functions of the history $h$ and not only the abstract state $s$.

\section{Motivation}\label{sec:Example}

In this section we motivate the importance of non-MDP homomorphisms by an example. We show that a non-MDP homomorphism can cater to a large set of domains and allows more compact representations.

\paradot{Navigational Grid-world} Let us consider a simplified version of the asymmetric grid-world example by \citet{Ravindran2004} in Figure \ref{fig:grid-world-original}. In this navigational domain, the goal of an agent $\pi$ is to navigate the grid to reach the target cell $T$. The unreachable cells are grayed-out. The agent receives a large positive reward if it enters the cell $T$, otherwise a small negative reward is given to the agent at each time-step. The agent is capable of moving in the four directions, i.e.\ up, down, left and right. This domain has an \emph{almost} similar transition and reward structure across a diagonal axis. We call this an approximate MDP axis and denote it by $\approx$MDP. This axis of symmetry enables us to create a homomorphism of the domain using approximately half of the original state-space (see Figure \ref{fig:grid-world-abstract}).

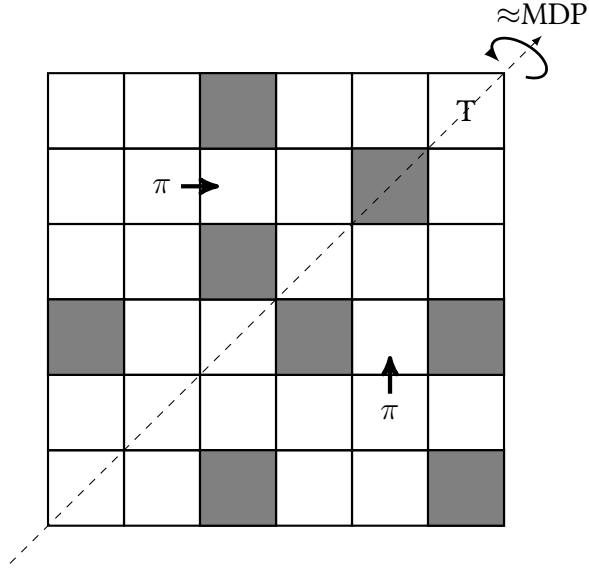
\begin{figure}[h]
    \centering
        \begin{tikzpicture}
            [
            box/.style={rectangle,draw=black,thick, minimum size=1cm},
            gray-box/.style={box, fill=gray}
            ]

            \foreach \x in {0,1,...,5}{
                \foreach \y in {0,1,...,5}
                \node[box] at (\x,\y){};
            }
            \node[gray-box] at (2,5){};
            \node[gray-box] at (4,4){};
            \node[gray-box] at (2,3){};
            \node[gray-box] at (3,2){};
            \node[gray-box] at (0,2){};
            \node[gray-box] at (5,2){};
            \node[gray-box] at (2,0){};
            \node[gray-box] at (5,0){};

            \node at (5,5){T};
            \draw[->, ultra thick] (1.25,4) -- (1.75,4);
            \node at (1,4){$\pi$};
            \draw[->, ultra thick] (4,1.25) -- (4, 1.75);
            \node at (4,1){$\pi$};

            \node[above] at (6,6){$\approx$MDP};
            \draw[dashed] (-1,-1) -- (6,6);
            \node at (5.7,5.7) {\AxisRotator[x=0.2cm,y=0.4cm,->,rotate=60]};
        \end{tikzpicture}
    \caption[The original navigational grid-world with the axis of approximate symmetry.]{The original navigational grid-world with the axis of approximate symmetry. The gray cells are not reachable. The target cell is at the top right corner. The figure shows two possible positions of the agent and corresponding optimal actions.}\label{fig:grid-world-original}
\end{figure}
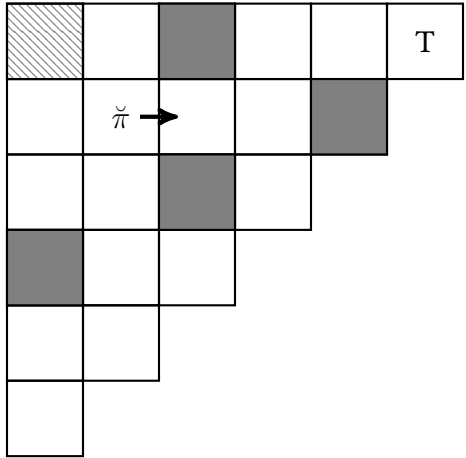
\begin{figure}[h]
    \centering
        \begin{tikzpicture}
            [
            box/.style={rectangle,draw=black,thick, minimum size=1cm},
            gray-box/.style={box, fill=gray},
            stripe-box/.style={box, pattern=north west lines, pattern color=gray}
            ]

            \foreach \y in {1,...,5}{
                \foreach \x in {0,1,...,\y}
                \node[box] at (\x,\y){};
            }
            \node[box] at (0,0){};
            \node[gray-box] at (2,5){};
            \node[gray-box] at (4,4){};
            \node[gray-box] at (2,3){};
            \node[gray-box] at (0,2){};
            \node[stripe-box] at (0,5){};

            \node at (5,5){T};
            \draw[->, ultra thick] (1.25,4) -- (1.75,4);
            \node at (1,4){$\u\pi$};

    \end{tikzpicture}
    \caption[A possible MDP homomorphism by merging the mirror state-action pairs together.]{A possible MDP homomorphism by merging the mirror state-action pairs together. The presence of a hashed cell indicates that it is not an exact homomorphism. The agent $\u\pi$ solves the problem in this abstract domain.}\label{fig:grid-world-abstract}
    \hfill
\end{figure}

This grid-world example has primarily been studied in the context of either exact, approximate or Bounded parameter MDP (BMDP) homomorphisms \cite{Ravindran2004}: the abstract model \emph{approximately} preserves the one-step dynamics of the original environment. However, as we later prove in this chapter (see Theorem \ref{thm:psiQstar}ii), some non-MDP homomorphisms can also be used to find a near-optimal policy in the original process. We motivate the need of non-MDP homomorphisms, first, by highlighting the fact\footnote{This section is an informal motivation, we formally deal with this fact in the main results section (Theorem \ref{thm:psimdpstar}i).} that in the grid-world domain, the states with similar dynamics have similar optimal action-values. Afterwards, we modify the grid-world domain such that the modified grid-world does not have an approximate MDP symmetry axis, but still has the same approximate optimal action-values symmetry.

We apply Value Iteration (VI) \cite{Bellman1957} algorithm (\Cref{alg:state-value-iteration}) with some fixed but irrelevant parameters on the grid world (see Figure \ref{fig:optimal-values-gw}). The grid world has the same approximate symmetry axis for the optimal values, denoted by $\approx$Q-uniform axis. It is easy to see that each merged state in Figure \ref{fig:grid-world-abstract} has the same action-values. Hence, the $\approx$MDP axis is also an $\approx$Q-uniform axis in the grid-world.

\begin{figure}
    \centering
        \begin{tikzpicture}
            [
            box/.style={rectangle,draw=black,thick, minimum size=1cm},
            gray-box/.style={box, fill=gray},
            highlight/.style={box,draw=black,ultra thick,dashed}
            ]

            \foreach \x in {0,1,...,5}{
                \foreach \y in {0,1,...,5}
                \node[box] at (\x,\y){};
            }
            \node[gray-box] at (2,5){};
            \node[gray-box] at (4,4){};
            \node[gray-box] at (2,3){};
            \node[gray-box] at (3,2){};
            \node[gray-box] at (0,2){};
            \node[gray-box] at (5,2){};
            \node[gray-box] at (2,0){};
            \node[gray-box] at (5,0){};

            \node at (0,0){1.74};
            \node at (0,1){1.94};
            \node at (0,3){2.42};
            \node at (0,4){2.70};
            \node at (0,5){\textbf{2.42}};

            \node at (1,0){1.94};
            \node at (1,1){2.17};
            \node at (1,2){2.42};
            \node at (1,3){2.70};
            \node at (1,4){3.01};
            \node at (1,5){\textbf{2.17}};

            \node at (2,1){2.42};
            \node at (2,2){2.17};
            \node at (2,4){3.35};

            \node at (3,0){2.42};
            \node at (3,1){2.70};
            \node at (3,3){3.35};
            \node at (3,4){3.74};
            \node at (3,5){4.16};

            \node at (4,0){2.70};
            \node at (4,1){3.01};
            \node at (4,2){3.35};
            \node at (4,3){3.74};
            \node at (4,5){4.64};

            \node at (5,1){\textbf{2.70}};
            \node at (5,3){4.16};
            \node at (5,4){4.64};
            \node at (5,5){5.16};

            \node[highlight] at (4,2){};
            \node[highlight] at (2,4){};

            \node[above] at (6,6){$\approx$Q-uniform};
            \draw[dashed] (-1,-1) -- (6,6);
            \node at (5.7,5.7) {\AxisRotator[x=0.2cm,y=0.4cm,->,rotate=60]};

    \end{tikzpicture}
    \caption[The optimal values at each approachable cell.]{The optimal values at each approachable cell. The bold-faced values are not exactly matched across the symmetry axis.} \label{fig:optimal-values-gw}
\end{figure}

\paradot{Modified Navigational Grid-world} Now we modify the grid world such that it does not have an $\approx$MDP axis (Figure \ref{fig:grid-world-original}) but it still has the same $\approx$Q-uniform axis (Figure \ref{fig:optimal-values-gw}). The idea is to take a pair of merged states from Figure \ref{fig:grid-world-original} and change the reward and transition probabilities such that the states no longer have similar one-step dynamics but still have similar action-values. For example, let us consider the cells highlighted with dashed borders in Figure \ref{fig:optimal-values-gw} and denote the cell in the bottom half with $s_{23}$. Let $u$, $d$, $p_u$ and $p_d$ denote the actions up and down, and the probabilities to reach the desired cell by taking the corresponding action, respectively. Let $r_u$ and $r_d$ be the expected rewards for each action in the state $s_{23}$. In general, we get an \emph{under-determined} set of equations for the action-value function at state $s_{23}$ as
\beq\label{eq:q23}
Q^*_\mu(s_{23},a) = \begin{cases}
    r_u + 0.73\g p_u  + 3.01\g   &\quad \text{if }a = u\\
    r_d - 0.73\g p_d + 3.74\g  &\quad \text{if } a = d.
\end{cases}
\eeq

In the original navigational grid-world problem $p_u = p_d = 1$, i.e.\ each action leads deterministically to the indented reachable cell, and $r_u=r_d = r_n$, where $r_n$ is a fixed small negative reward. We can break the $\approx$MDP similarity by setting\footnote{$p_u = 0$ implies that the action $u$ now takes the agent to the down cell and vice versa for the action $d$.} $p_u = p_d := 0$, i.e.\ the actions behave in the opposite way in the lower half, $r_u := r_n + 0.73\g$ and $r_d := r_n - 0.73\g$, without changing the $\approx$Q-uniform similarity. In fact, we can have infinite combinations of rewards and transitions to get a set of modified domains since the set of equations \eqref{eq:q23} is under-determined.

This set of \emph{modified} domains, by design, no longer allows the approximate MDP homomorphism of Figure \ref{fig:grid-world-abstract}. Every state is different in terms of reward and transition structure across the $\approx$MDP axis of Figure \ref{fig:grid-world-original}. Any one-step model similarity abstraction would be approximately of the same size as the original problem. However, if we consider Q-uniform homomorphisms, i.e.\ state-action pairs are merged if the action-values are close, then the set of modified domains has a same Q-uniform homomorphism.

In GRL it is natural to assume that the (expected) rewards are function of realized history. The above modification argument is more likely to hold in a GRL setting: the reward and transition similarity might be hard to satisfy. Therefore, a GRL agent is better to consider such non-MDP homomorphisms to cover more domains with a single abstract model. Now we ask the main question, does such a non-MDP homomorphism, e.g.\ Q-uniform homomorphism, have a guaranteed solution for the original problem? In the next section, we answer this question in affirmative for Q-uniform homomorphisms (Theorem \ref{thm:psiQstar}ii), but in negative for V-uniform
homomorphisms with a weaker positive result (Theorem \ref{thm:psiVstar}ii).

As discussed earlier in \Cref{chap:extreme-abs}, we are primarily interested in the optimal policies of the surrogate-MDP. However, it is also interesting to consider a general policy case (e.g.\ Theorems \ref{thm:psimdppi}, \ref{thm:exactmdp}, \ref{thm:psiQpi} and \ref{thm:psiVpi}) akin to an on-policy result where we uplift a representative policy. We use any arbitrary member as a representative policy $\piR$ on the abstract state $s$.
\beq\label{eq:dpi}
\piR(\cdot\|s) := \pi_{\psi}(\cdot\|h), \quad  \text{for some } h \in \psi\inv(s).
\eeq

This arbitrary choice of representative introduces a policy representation error $\DPi$ for each abstract state $s$, expressed as
\beq\label{eq:Delta}
\DPi(s) := \sup_{h \in \psi^\inv(s)} \norm{ \piR(\cdot\|s) - \pi_\psi(\cdot\|h)}_1.
\eeq

This representation error is small/zero when the induced abstract policy $\pi_{\psi}$ is approximately/piecewise constant, i.e.\ $\pi_{\psi}(\cdot\|h) = \pi_\psi(\cdot|\d h)$ for all $\psi(h) = \psi(\d h)$.

\begin{remark}
    It is easy to see that the state aggregation mapping function $\phi$ of ESA setup\footnote{The reader is encouraged to see \cite{Hutter2016} for more details about $\phi$.} is a special case of our generalized mapping function $\psi$. We can mimic any $\phi$ by using a $\psi_\phi$ with identity transformation over action spaces, i.e.\ $\psi_\phi(ha) := (\phi(h)a)$.
\end{remark}

In the next section we provide the main results of this chapter. We construct a near-optimal policy for the original process from the surrogate-MDP even if the homomorphism is non-MDP.

\section{Representation Guarantees}\label{sec:Aggregation}

\ifshort\else In this section, we prove that under some conditions an optimal policy of the abstract process loses only a faction of the value when uplifted in the original process. Also, these results hold even when the marginalized process in not an MDP. \fi
We analyze three types of homomorphisms in this work: MDP, Q-uniform and V-uniform homomorphisms\ifshort\footnote{Due to the limited space, we omit the proofs. But the proofs can be found in the extended version of this paper \cite{Majeed2018}.}\fi. Both Q and V-uniform homomorphisms are non-Markovian by definition. In general, MDP and Q-uniform homomorphisms admit a \emph{deterministic} near-optimal policy of the original process, while V-uniform homomorphisms do not.

\subsection{MDP Homomorphisms}

A homomorphism is an MDP homomorphism if the induced abstract process $\mu_\psi$ is an MDP. This would mean there exists a process $\mu_{\mathrm{MDP}}$ such that for all $\psi(ha) = sb$ and for all $s'$ and $r'$, it holds:
\beq\label{eq:MDP}
\mu_\psi(s'r'\|ha) = \mu_{\mathrm{MDP}}(s'r'\|sb).
\eeq

Using the above condition renders $\b \mu_B = \mu_{\mathrm{MDP}}$ and independent of $B$.
The condition $\eqref{eq:MDP}$ is a  \emph{stronger} version of the bisimulation condition \cite{Givan2003} that is generalized to joint history-action pairs. This condition is strong enough to preserve the optimal \mbox{(action-)value} functions of the original process (see Theorem \ref{thm:psimdpstar}). But, it is not strong enough to preserve arbitrary policy \mbox{(action-)value} functions (see Theorem \ref{thm:psimdppi}). Unless we define a notion of action-value function representative and a corresponding representation error. For an abstract state-action pair, the representative action-value is defined as
\beq\label{eq:Qrep}
Q^\pi_\mu(\psi^{-1}(sb)) := Q^\pi_\mu(ha), \,  \text{for some } \psi(ha) = sb
\eeq
and the representation error of the action-value function is expressed as
\beq\label{eq:delta}
\DQ(s) := \sup_{h,a,b:\psi(ha) = sb} \abs{ Q^\pi_\mu(\psi\inv(sb)) - Q^\pi_\mu(ha)}.
\eeq

Similar to $\DPi$, this representation error is small/zero if the action-value function is approximately/piecewise constant.
At this point, we have all the required components properly defined to state the first theorem of the chapter. Before we give the first theorem of this chapter, we establish a couple of important lemmas that bound the (optimal) value loss when we evaluate the action-value function only at the representatives.

\begin{lemma}\label{lem:Vrep}
    For any policy $\pi$, and $\psi(h)=s$ the following holds,
    \beqn
    \left|V^\pi_\mu(h) - Q^\pi_\mu(\psi^\inv(s\piR(s))) \right| \leq \DQ(s) + \frac{\DPi(s)}{1-\g}.
    \eeqn
\end{lemma}


\begin{proof}
    For any $\psi(h)=s$, we start from the value function of the original process,
    \bqan
    V^\pi_\mu(h) &\overset{}{\equiv} \sum_{a \in \A} Q^\pi_\mu(ha) \pi(a \|h) \\
    &\overset{(a)}{=} \sum_{b \in \B} \sum_{a \in \psi^{-1}_s(b)} Q^\pi_\mu(ha) \pi(a \|h) \\
    &\overset{}{=} \sum_{b \in \B} \sum_{a \in \psi^{-1}_s(b)} \left( Q^\pi_\mu(ha) + Q^\pi_\mu(\psi^{-1}(sb)) - Q^\pi_\mu(\psi^{-1}(sb))  \right)  \pi(a \|h) \\
    &\overset{}{=} \sum_{b \in \B} Q^\pi_\mu(\psi^{-1}(sb)) \sum_{a \in \psi^{-1}_s(b)} \pi(a \|h) + \sum_{b \in \B} \sum_{a \in \psi^{-1}_s(b)} \left( Q^\pi_\mu(ha) - Q^\pi_\mu(\psi^{-1}(sb)) \right)  \pi(a \|h) \\
    &\overset{(b)}{\leq} \sum_{b \in \B} Q^\pi_\mu(\psi^{-1}(sb)) \pi_\psi(b \|h) + \DQ(s) \\
    &\overset{}{=} \sum_{b \in \B} Q^\pi_\mu(\psi^{-1}(sb)) \left(\pi_\psi(b \|h) + \piR(b \|s) - \piR(b\|s) \right) + \DQ(s) \\
    &\overset{}{=} \sum_{b \in \B} Q^\pi_\mu(\psi^{-1}(sb)) \piR(b \|s) + \sum_{b \in \B} Q^\pi_\mu(\psi^{-1}(sb)) \left(\pi_\psi(b \|h) - \piR(b \|s) \right) + \DQ(s) \\
    &\overset{(c)}{\leq} Q^\pi_\mu(\psi^{-1}(s \piR(s))) + \frac{\DPi(s)}{1-\g} + \DQ(s) \\
    \eqan
    $(a)$ follows from the fact that the mapping is defined to be surjective; $(b)$ from \Cref{eq:dpih,eq:delta}, $(c)$ uses \Cref{eq:Delta} and the fact that action-value function is bounded, so is the representative. It is easy to get the other side of the inequality from similar steps.
\end{proof}

Using the above Lemma, the following theorem states that an (approximately) MDP homomorphism (approximately) represents the (action-)value function of the true environment for a fixed policy $\pi$, i.e.\ it has the representation guarantees for the value functions.

\begin{theorem}{$(\psi_{\mathrm{MDP}\pi})$}\label{thm:psimdppi}
    Let $\psi$ be a homomorphism such that $\mu_\psi$ is an MDP, then for any policy $\pi$ and all $\psi(h,a) = (s,b)$ it holds:
    \beqn
    \left|q^\piR_\mu(sb) - Q^\pi_\mu(ha)\right| \leq  \frac{\g \eps_{\max}}{1-\g} \quad \text{and} \quad
    \left|v^\piR_\mu(s) - V^\pi_\mu(h)\right| \leq  \frac{\eps_{\max}}{1-\g}
    \eeqn
    where $\eps_{\max} := \max_{s \in \S} \left(\DQ(s) + \frac{\DPi(s)}{1-\g}\right)$.
\end{theorem}
\begin{proof}
    Let $\delta := \underset{h,a,s,b: \psi(ha)=sb}{\sup}\left|q^\piR_\mu(sb) - Q^\pi_\mu(ha)\right|$, and for any $\psi(h)=s$ we have,
    \bqa\label{eq:vVdif}
    \left|v^\piR_\mu(s) - V^\pi_\mu(h)\right|
    &\overset{(a)}{\leq} \left|\sum_{b \in \B} q^\piR_\mu(sb){\piR}(b\|s) - \sum_{b \in \B} Q^\pi_\mu(\psi^{-1}(sb)){\piR}(b\|s)\right| + \DQ(s) + \frac{\DPi(s)}{1-\g} \nonumber \\
    &\overset{}{\leq} \sum_{b \in \B} \left|q^\piR_\mu(sb) - Q^\pi_\mu(\psi^{-1}(sb))\right| {\piR}(b\|s) + \DQ(s) + \frac{\DPi(s)}{1-\g} \nonumber \\
    & \leq \delta + \DQ(s) + \frac{\DPi(s)}{1-\g}
    \eqa
    $(a)$ follows from the definition of $v^\piR_\mu(s)$ and Lemma \ref{lem:Vrep}. Now for any $\psi(ha)=sb$, we have,
    \bqan
    Q^\pi_\mu(ha) \overset{}&{\equiv} \sum_{o' \in \O, r' \in \R} \mu(o'r' \|ha)(r' + \g V^\pi_\mu( h'))  \qquad [h' = hao'r'] \\
    \overset{(b)}&{\lessgtr} \sum_{s' \in \S, r' \in \R} \mu_\psi(s'r' \|ha)\left(r' + \g v^\piR_\mu(s')\right) \pm \g\left(\delta + \DQ(s') + \frac{\DPi(s')}{1-\g}\right) \\
    \overset{(\ref{eq:MDP})}&{=} \sum_{s' \in \S, r' \in \R} \mu_{\rm MDP}(s'r' \|sb)\left(r' + \g v^\piR_\mu(s')\right) \pm \g\left(\delta + \DQ(s') + \frac{\DPi(s')}{1-\g}\right) \\
    \overset{}&{\lessgtr} q^\piR_\mu(sb) \pm \g(\delta + \eps_{\max})
    \eqan
    $(b)$ follows from value function error bound (\ref{eq:vVdif}) and definition of $\mu_\psi$ given by (\ref{eq:marginP}).
    Since, $\delta \equiv \sup \abs{q^{\piR}_\mu(sb) - Q^\pi_\mu(ha)} \leq \g(\delta + \eps_{\max})$ therefore, $\delta \leq \frac{\g \eps_{\max}}{1-\g}$, hence completes the proof.
\end{proof}

The above theorem shows that the surrogate MDP \emph{approximately} preserves the \mbox{(action-)value} functions of the original process for any arbitrary policy. However, these \mbox{(action-)value} functions are preserved \emph{exactly} if we further impose a policy uniformity condition in addition to an MDP assumption.

\begin{theorem}{$(\psi_{\mathrm{MDP}\pi=})$} \label{thm:exactmdp}
    Let $\psi$ be a homomorphism such that $\mu_\psi$ is an MDP and $\pi(\cdot\|h) = \pi(\cdot|\d h)$ (i.e.\ the policy similarity condition holds) for some policy $\pi$ and for all $\psi(h) = \psi(\d h)$. Then for all $\psi(ha) = sb$ it holds:
    \beqn
    q^\piR_\mu(sb) = Q^\pi_\mu(ha) \quad \text{and} \quad v^\piR_\mu(s) = V^\pi_\mu(h).
    \eeqn
\end{theorem}
\begin{proof}
    Let $\delta := \underset{a, \d a, h,\d h: \psi(h)=\psi(\d h)}{\sup} |Q^\pi_\mu(ha) - Q^\pi_\mu(\d h\d a)|$, then for all $\psi(h) = \psi(\d h)$,
    \bqa\label{eq:vdifmdp}
    \left|V^\pi_\mu(h) - V^\pi_\mu(\d h)\right|
    &\overset{}{\equiv} \left|\sum_{a \in \A} Q^\pi_\mu(ha) \pi(a\|h) - \sum_{\d a \in \A} Q^\pi_\mu(\d h\d a)\pi(\d a|\d h) \right| \nonumber \\
    &\overset{(a)}{=} \left|\sum_{a \in \A} \left(Q^\pi_\mu(ha)-Q^\pi_\mu(\d ha)\right) \pi(a\|h) \right| \leq \delta
    \eqa
    $(a)$ follows from the assumption. Now for all $\psi(ha) = \psi(\d h\d a)=sb$,
    \bqan\label{eq:qdifmdp}
    \left|Q^\pi_\mu(ha) - Q^\pi_\mu(\d h\d a)\right|
    \overset{(b)}&{=} \abs{\sum_{o' \in \O, r' \in \R} \mu(o'r' \|ha) \left(r' + \g V^\pi_\mu(h')\right) - \sum_{o' \in \O, r' \in \R} \mu(o'r' | \d h \d a) \left(r' + \g V^\pi_\mu(\dot{h'})\right)} \\
    \overset{(\ref{eq:MDP})}&{=} \g\left|\sum_{s' \in \S, r' \in \R} \mu_{\rm MDP}(s'r' \|sb) \left( V^\pi_\mu(h') - V^\pi_\mu(\dot{h'})\right)\right| \\
    \overset{(\ref{eq:vdifmdp})}&{\leq} \g\delta\numberthis
    \eqan
    $(b)$ follows from the definition and $h' = hao'r'$ and $\dot{h'} = \d h\d a o' r'$, and $\psi(h') = \psi(\dot{h'}) = s'$. From the inequality (\ref{eq:qdifmdp}), we have $\delta \leq \g \delta \Rightarrow \delta = 0$. Therefore, $Q^\pi_\mu(ha) = Q^\pi_\mu(\d h\d a)$ for all $\psi(ha) = \psi(\d h\d a)$. Note that this also implies, $\DQ = 0$ and $\DPi = 0$ by assumption.
\end{proof}

Theorems \ref{thm:psimdppi} and \ref{thm:exactmdp} are interesting results for a fixed policy, but, primarily, we are interested in the \mbox{(near-)optimal} policies of the original process. And, we want to find the abstract policies that can be lifted with a performance guarantee from the abstract state-action space to the original history-action space.

Before we prove the analogous theorem for the optimal policy, we need the following lemma. This lemma establishes the $\psi$-uniformity of the optimal action-value function.

\begin{lemma}\label{lem:dzeromdp}
    Let $\mu_\psi$ be an MDP then $Q^*_\mu(ha) = Q^*_\mu(\d h\d a)$ for all $\psi(ha) = \psi(\d h\d a)$.
\end{lemma}
\begin{proof}
    Let $\delta := \underset{a, \d a,h,\d h, : \psi(h)=\psi(\d h)}{\sup} |Q^*_\mu(ha) - Q^*_\mu(\d h\d a)|$, then for all $\psi(h) = \psi(\d h)$,
    \bqa\label{eq:vdifstarmdp}
    \left|V^*_\mu(h) - V^*_\mu(\d h)\right| \overset{}&{\equiv} \left|\max_{a \in \A} Q^*_\mu(ha) - \max_{a \in \A} Q^*_\mu(\d ha)\right| \nonumber \\
    \overset{(a)}&{\leq} \max_{a \in \A} \abs{Q^*_\mu(ha) - Q^*_\mu(\d ha)} \leq \delta
    \eqa
    $(a)$ follows from simple mathematical fact of maximum value. Now for all $\psi(ha) = \psi(\d h\d a) = sb$,
    \bqan\label{eq:qdifstarmdp}
    \left|Q^*_\mu(ha) - Q^*_\mu(\d h\d a)\right|
    \overset{(b)}&{=} \abs{\sum_{o' \in \O, r' \in \R} \mu(o'r' \|ha) (r' + \g V^*_\mu(h')) - \sum_{o' \in \O, r' \in \R} \mu(o'r' \| \d h \d a) (r' + \g V^*_\mu(\dot{h'}))} \\
    \overset{(\ref{eq:MDP})}&{=} \g \left|\sum_{s' \in \S, r' \in \R} \mu_{\rm MDP}(s'r' \|sb) \left( V^*_\mu(h') - V^*_\mu(\dot{h'})\right) \right| \\
    \overset{(\ref{eq:vdifstarmdp})}&{\leq} \g\delta \numberthis
    \eqan
    $(b)$ follows from the definition and $h' = hao'r'$ and $\dot{h'} = \d h \d ao'r'$, and $\psi(h') = \psi(\dot{h'}) = s'$. By (\ref{eq:qdifstarmdp}) we have $\delta \leq \g \delta \Rightarrow \delta = 0$.
\end{proof}

We bound the error of evaluating the action-value function only at the representatives of the mapping when the agent is following the optimal policy in the original process.

\begin{lemma}\label{lem:VStarrep}
    For the optimal policy $\pi^*$, and for all $\psi(h)=s$ the following holds,
    \beqn
    \left|V^*_\mu(h) - \max_{b \in \B} Q^*_\mu(\psi^{-1}(sb))\right| \leq \DQ(s).
    \eeqn
\end{lemma}
\begin{proof}
    For any $\psi(h)=s$ we have,
    \bqan
    V^*_\mu(h) \overset{}&{\equiv} \max_{a \in \A} Q^*_\mu(ha) \overset{(a)}{=} \max_{b \in \B} \max_{a \in \psi^{-1}_s(b)} Q^*_\mu(ha)\\
    \overset{(\ref{eq:delta})}&{\lessgtr} \max_{b \in \B} Q^*_\mu(\psi^{-1}(sb)) \pm \DQ(s)
    \eqan
    $(a)$ follows from the fact that mapping is defined to be surjective.
\end{proof}

Moreover, we need Lemma 8 from \cite{Hutter2016}. For completeness, we state the lemma in this work without repeating the proof. It gives value loss-bounds between the optimal policy and a policy that has bounded one-step action-value loss.

\begin{lemma}{$($\cite{Hutter2016} Lemma 8. $\pi(h) \neq \pi^*(h))$} \label{lem:different-optimal-action}
    If $Q^*_\mu(h\pi(h)) \geq V^*_\mu(h) - \eps$ for all $h$ and some policy $\pi$, then for all $h$ and $a$ it follows:
    \bqan
    0 \leq Q^*_\mu(ha) - Q^\pi_\mu(ha) \leq \frac{\g\eps}{1-\g} \quad \text{and} \quad
    0 \leq V^*_\mu(h) - V^\pi_\mu(h) \leq \frac{\eps}{1-\g}.
    \eqan
\end{lemma}

By using the above lemmas the following theorem provides representation guarantee for the optimal value functions.

\begin{theorem}{$(\psi_{\mathrm{MDP}*})$}\label{thm:psimdpstar}
    Let $\psi$ be a homomorphism such that $\mu_\psi$ is an MDP, then for all $\psi(ha)=sb$ it holds:
    \begin{enumerate}[(i)]
        \item $q^*_\mu(sb) = Q^*_\mu(h,a)$ and  $v^*_\mu(s) = V^*_\mu(h)$.
        \item $V^*_\mu(h)=V^{\u\pi}_\mu(h)$

        where $\breve{\pi}(h) :\in \psi^{-1}_s\left({\u\pi}^*(s)\right)$ for any $\psi(h)=s$.
    \end{enumerate}
\end{theorem}
\ifshort\else
\begin{proof}
    {\bf(i)}
    Let $\delta := \underset{h,a,s,b: \psi(h,a)=(s,b)}{\sup} \abs{q^*_\mu(sb) - Q^*_\mu(ha)}$. Now for any $\psi(h)=s$ we have,
    \bqa\label{eq:vVdifstar}
    \left|v^*_\mu(s) - V^*_\mu(h)\right|
    \overset{(a)}&{\leq} \left|\max_{b \in \B} q^*_\mu(sb) - \max_{b \in \B} Q^*_\mu(\psi^{-1}(sb))\right| + \DQ(s)  \nonumber \\
    \overset{(b)}&{\leq} \delta
    \eqa
    $(a)$ follows from the definitions of $v^*_\mu(s)$ and Lemma \ref{lem:VStarrep}, and $(b)$ is due to Lemma \ref{lem:dzeromdp}.
    \bqan
    Q^*_\mu(ha) \overset{}&{\equiv} \sum_{o' \in \O, r' \in \R} \mu(o'r' \|ha)(r' + \g V^*_\mu(h')) \qquad [h' = hao'r'] \nonumber \\
    \overset{(\ref{eq:vVdifstar})}&{\lessgtr} \sum_{s' \in \S, r' \in \R} \mu_\psi(s'r' \|ha)(r' + \g v^*_\mu(s')) \pm \g\delta \nonumber \\
    \overset{(\ref{eq:MDP})}&{=} \sum_{s' \in \S, s' \in \R} \mu_{\rm MDP}(s'r' \|sb)(r' + \g v^*_\mu(s'))  \pm \g\delta \nonumber \\
    \overset{}&{\equiv} q^*_\mu(sb) \pm \g\delta
    \eqan
    Therefore, $\delta \leq \g\delta$, therefore $\delta = 0$ which completes the proof.

    {\bf(ii)}
    For $\psi(h)=s$ and $\breve{\pi}(h) :\in \psi^{-1}_s({\u\pi}^*(s))$,
    \bqan
    V^*_\mu(h) \overset{(i)}{=} v^*_\mu(s) \overset{}{\equiv} q^*_\mu(s,{\u\pi}^*(s))  \overset{(i)}{=} Q^*_\mu(h\breve{\pi}(h))
    \eqan
    which implies $Q^*_\mu(h\breve{\pi}(h)) = V^*_\mu(h)$ and Lemma \ref{lem:different-optimal-action} concludes the proof.

\end{proof}

For any MDP homomorphism, the performance guarantee is provided by Theorem \ref{thm:psimdpstar}ii. The abstract optimal policy ${\u\pi}^*$ is also a near-optimal policy for the original process when lifted to the original history-action space.

\subsection{Q-uniform Homomorphisms}\label{sec:Extreme}

In this section, we relax the MDP condition (see Equation \ref{eq:MDP}) on the abstract-process provided by the homomorphism. We show that there still exists an abstract policy that is near-optimal in the original process (see Theorem \ref{thm:psiQstar}ii). We denote the $B$-average of the action-value function of the original process as
\beq
\b Q^\pi_\mu(\psi^{-1}(sb)) := \sum_{h \in \H, a \in \A} Q^\pi_\mu(ha)B(ha\|sb).
\eeq

The following lemma is a stepping stone for the main results. It establishes a link between the value functions loss and the action-value functions loss.

\begin{lemma}{$(\b Q - q)$}\label{lem:QBqDiff}
    Let $|V^\pi_\mu(h) - v^{\u\pi}_\mu(s)| \leq \eps$ for all $\psi(h) = s$. Then for all $s \in \S$ and $b \in \B$ it holds:
    \beqn
    \left| \b Q^\pi_\mu(\psi^{-1}(sb)) - q^{\u\pi}_\mu(sb) \right| \leq \g\eps.
    \eeqn
\end{lemma}
\begin{proof}
    We begin with the action value of the representative of any $(s,b)$ as
    \bqan
    \b Q^\pi_\mu(\psi^{-1}(sb))
    \overset{}&{\equiv} \sum_{h \in \H, a \in \A} B(ha\|sb) \sum_{o' \in \O, r' \in \R} \mu(o'r' \|ha)\left(r' + \g V^\pi_\mu(hao'r')\right) \\
    \overset{(a)}&{=} \sum_{h \in \H, a \in \A} B(ha\|sb) \sum_{s' \in \S} \sum_{r' \in \R, o': \psi(hao'r')=s'} \mu(o'r' \|ha)\left(r' + \g V^\pi_\mu(hao'r')\right) \\
    \overset{(b)}&{\lessgtr} \sum_{h \in \H, a \in \A} B(ha\|sb) \sum_{s' \in \S, r' \in \R} \mu_\psi(s'r' \|ha)\left(r' + \g v^{\u\pi}_\mu(s') \pm \g\eps \right) \\
    \overset{}&{\equiv} \sum_{s' \in \S, r' \in \R} \b\mu_B(s'r' \|sb)\left(r' + \g v^{\u\pi}_\mu(s') \right)  \pm \g\eps \equiv q^{\u\pi}_\mu(sb) \pm \g\eps
    \eqan
    $(a)$ follows since $\psi$ is surjective; $(b)$ follows from the assumption.
\end{proof}

We start with proving a value loss bound for an arbitrary policy when the action-value function of the original process is \emph{approximately} $\psi$-uniform.

\begin{theorem}{$(\psi_{Q^\pi_\mu})$}\label{thm:psiQpi}
    Assume $\left|Q^\pi_\mu(ha)-Q^\pi_\mu(\d h\d a)\right|\leq \eps$ for some policy $\pi$ and for all $\psi(ha)=\psi(\d h\d a)$. Then for all $\psi(ha) = sb$ it holds:
    \beqn
    \left|Q^\pi_\mu(ha) - q^\piR_\mu(sb)\right| \leq  \eps + \frac{\g\eps(s)}{1-\g} \quad \text{and} \quad  \left|V^\pi_\mu(h) - v^\piR_\mu(s)\right| \leq \frac{\eps(s)}{1-\g}
    \eeqn
    where $\eps(s) := 2\eps + \frac{\DPi(s)}{1-\g}$.
\end{theorem}
\begin{proof}
    Let $\delta := \underset{h,a,s,b:\psi(ha)=sb}{\sup}\left|Q^\pi_\mu(ha) - q^\piR_\mu(sb)\right|$, and for any $\psi(h)=s$ we have,
    \bqa
    V^\pi_\mu(h) - v^\piR_\mu(s)
    \overset{Lem. \ref{lem:Vrep}}&{\lessgtr} \sum_{b \in \B} \left( Q^\pi_\mu(\psi^{-1}(s b)){\piR}(b\|s) - q^\piR_\mu(sb){\piR}(b\|s) \right) \pm \left(\DQ(s) + \frac{\DPi(s)}{1-\g}\right) \nonumber \\
    \overset{}&{=} \sum_{b \in \B} \left( Q^\pi_\mu(\psi^{-1}(sb)) - q^\piR_\mu(sb)\right) {\piR}(b\|s) \pm \left(\DQ(s) + \frac{\DPi(s)}{1-\g}\right) \nonumber \\
    \overset{(a)}&{\lessgtr} \pm \left(\delta + \DQ(s) + \frac{\DPi(s)}{1-\g}\right) \label{eq:VvDiff}
    \eqa
    $(a)$ follows from the definition of $\delta$ and the fact that ${\piR}(b\|s)$-average is smaller than the $b$-supremum. Using the above inequality (\ref{eq:VvDiff}) and Lemma \ref{lem:QBqDiff} we get,
    \beq \label{eq:BAvgqDiff}
    | \b Q^\pi_\mu(\psi^{-1}(sb)) - q^\piR_\mu(sb) | \leq \g\left(\delta + \DQ(s) + \frac{\DPi(s)}{1-\g}\right).
    \eeq

    We exploit the theorem's assumption and derive a key relationship between the $B$ average and any instance of action value.

    \bqa \label{eq:BAvgRepDiff}
    \b Q^\pi_\mu(\psi^{-1}(sb))
    \overset{}&{\equiv} \sum_{h \in \H, a \in \A}   Q^\pi_\mu(ha)B(ha\|sb) \nonumber \\
    \overset{(a)}&{\lessgtr} \sum_{h \in \H, a \in \A}  \left( Q^\pi_\mu(\psi^{-1}(sb)) \pm \eps \right) B(ha\|sb) \nonumber \\
    \overset{}&{=} Q^\pi_\mu(\psi^{-1}(sb)) \pm \eps
    \eqa
    $(a)$ follows from the theorem's assumption. Since, $Q^\pi_\mu(\psi^{-1}(sb))$ is a representative member in the pre-image set of $(sb)$; it is equivalent to say $Q^\pi_\mu(\psi^{-1}(sb)) = Q^\pi_\mu(ha)$ for any $\psi(ha) = sb$. Therefore, combining \eqref{eq:BAvgqDiff} and \eqref{eq:BAvgqDiff} we get $\left|Q^\pi_\mu(ha) - q^{\u\pi}_\mu(sb)\right| \leq \g(\delta + \DQ(s) + \frac{\DPi(s)}{1-\g}) + \eps$, hence $\delta \leq \frac{\eps + \g \DQ(s) + \frac{\g\DPi(s)}{1-\g}}{1-\g}$.
\end{proof}

We can further improve the above bound for the optimal policy. The key is the fact that if we have an approximately $\psi$-uniform action-value function then the state dependent representation error has an upper bound.

\begin{lemma}{$(\eps_{Q\pi})$}\label{lem:eQPi}
    Let $|Q^\pi_\mu(ha)-Q^\pi_\mu(\d h\d a)|\leq \eps$ for all $\psi(ha)=\psi(\d h\d a)$. Then $\DQ(s) \leq \eps$ for all $s \in \S$.
\end{lemma}
\begin{proof}
    For all $s \in \S$,
    \bqan
    \DQ(s)
    \overset{}{\equiv} \sup_{h, a, b : \psi(ha) = sb} | Q^\pi_\mu(\psi^{-1}(sb)) - Q^\pi_\mu(ha)| \overset{(a)}{\leq} \eps
    \eqan
    $(a)$ follows from the assumption and the fact that $Q^\pi_\mu(\psi^{-1}(sb))$ can be any member in the pre-image set of $sb$-pair.
\end{proof}

The following theorem \emph{improves} the optimal policy value loss bounds, \emph{cf.} Theorem \ref{thm:psiQpi}, and establishes the existence of a near-optimal policy of the original history-action space in the abstract state-action space.

\begin{theorem}{$(\psi_{Q^*_\mu})$}\label{thm:psiQstar}
    Let $|Q^*_\mu(ha)-Q^*_\mu(\d h\d a)|\leq \eps$ for all $\psi(ha)=\psi(\d h\d a)$, then for all $\psi(ha) = sb$ it holds:
    \begin{enumerate}[(i)]
        \item
        $|Q^*_\mu(ha) - q^*_\mu(sb)| \leq \frac{2\eps}{1-\g}$ and
        $|V^*_\mu(h) - v^*_\mu(s)| \leq \frac{2\eps}{1-\g}$.
        \item
        $0 \leq V^*_\mu(h) - V^{\u\pi}_\mu(h) \leq \frac{4\eps}{(1-\g)^2}$

        where $\breve{\pi}(h) :\in \psi^{-1}_s\left({\u\pi}^*(s)\right)$ for any $\psi(h)=s$.
    \end{enumerate}
\end{theorem}
\begin{proof}
    {\bf(i)} The proof follows the same steps as the proof of Theorem \ref{thm:psiQpi}, replacing $\pi$ with $\pi^*$ and ${\u\pi}$ with ${\u\pi}^*$ and using Lemma \ref{lem:VStarrep} instead of Lemma \ref{lem:Vrep}. In the end we use Lemma \ref{lem:eQPi} to conclude the proof.


    {\bf(ii)} For $\psi(h)=s$ and $\breve{\pi}(h) :\in \psi^{-1}_s({\u\pi}^*(s))$,
    \bqan
    V^*_\mu(h) \pm \frac{2\eps}{1-\g}
    &\overset{(i)}{\lessgtr} v^*_\mu(s) \\
    &\overset{}{\equiv} q^*_\mu(s{\u\pi}^*(s))\\
    &\overset{(i)}{\lessgtr} Q^*_\mu(h \breve{\pi}(h)) \pm \frac{2\eps}{1-\g}
    \eqan
    which implies $\abs{Q^*_\mu(h \breve{\pi}(h)) - V^*_\mu(h)} \leq \frac{4\eps}{1-\g}$ and Lemma \ref{lem:different-optimal-action} concludes the proof.
\end{proof}

It is important to note that Theorem \ref{thm:psiQstar} holds for any stochastic inverse $B$. Every choice of $B$ gives a different surrogate MDP $\b \mu_B$, so the theorem provides a \emph{near-optimal} performance guarantee for the \emph{uplifted} abstract optimal policies of \emph{any} possible surrogate MDP. Therefore, for any non-MDP Q-uniform homomorphism and a fixed $B$ there exists an uplifted near-optimal policy ($\u\pi$ from Theorem \ref{thm:psiQstar}ii).

\subsection{V-uniform Homomorphisms}

All the previous results are valid for any choice of the stochastic inverse $B$. However, for V-uniform homomorphisms, the results are explicitly dependent on $B$ (see Theorem \ref{thm:psiVpi} and \ref{thm:psiVstar}).
The choice of $B$ is eventually be dictated by the behavior policy \cite{Hutter2016} which can be understood as a conditional distribution generate from a joint distribution over $\H \times \A$ space. Furthermore, we can decompose $B$ into two distinct parts: action dependent and independent.  With an abuse of notation, assume an arbitrary joint distribution $B$ over $\H,\A,\S$ and $\B$. By using the chain rule of probability distributions on $B$,
\bqa
B(ha\|sb)
&= B(h\|sb)B(a|bhs) \nonumber \\
&= \frac{B(hs)B(b\|hs)}{B(sb)}B(a|bhs) \nonumber \\
&\overset{(a)}{=} \frac{B(hs)B(b\|h)}{B(sb)}B(a|bh) \nonumber \\
&= B(h\|s)\frac{B(b\|h)}{B(b\|s)}B(a|bh) \nonumber \\
&= \underbrace{B(h\|s)}_{\text{action-independent}}
\cdot\overbrace{\left(\frac{B(ab\|h)}{B(b\|s)}\right)}^{\text{action-dependent}} \label{eq:decomposed-B}
\eqa
$(a)$ follows from Assumption \ref{asm:state-history}, the state is determined only by the history.

Using the action-dependent part from \eqref{eq:decomposed-B}, we define a history and state based induced measure on the original action space for any $B$ and an abstract state based policy ${\u\pi}$ as
\beq
B^{\u\pi}(a\|hs) := \sum_{b \in \B}\left(\frac{B(ab\|h)}{B(b\|s)}\right) {\u\pi}(b \|s).
\eeq

This seemingly complex and arbitrary relationship has a well-structured explanation. If \emph{approximately}, the $B$ distribution is linked to the actual dynamics of an agent ${\u\pi}$ acting in the abstract state-action space, i.e.\ $B(b\|s) \approx {\u\pi}(b\|s)$, then $B^{\u\pi}(a\|hs) \approx B(a\|h)$, which is effectively a \emph{shadow} agent induced by the agent ${\u\pi}$ on the original history-action space.

To prove a result analogous to Theorem \ref{thm:psiQpi} for a V-uniform homomorphism, we need to impose an extra condition on $B$, \emph{ cf.} Theorem \ref{thm:psiQpi}, which requires a structure on $B$ and/or on the underlying original process. For general $B$, there exist some known counter examples \cite{Hutter2016}.

\begin{theorem}{$(\psi_{V^\pi_\mu})$}\label{thm:psiVpi}
    Let $\abs{V^\pi_\mu(h)-V^\pi_\mu(\d h)}\leq \eps$ for some policy $\pi$ and for all $\psi(h)=\psi(\d h)$, and $\abs{\sum_{a \in \A} Q^\pi_\mu(ha)B^\piR(a\|hs) - V^\pi_\mu(h)} \leq \eps_B$ for all $s=\psi(h)$, then it holds:
    \beqn
    \left|\b Q^\pi_\mu(\psi^{-1}(sb)) - q^\piR_\mu(sb)\right| \leq \frac{\g(\eps+ \eps_B)}{1-\g} \quad \text{ and} \quad
    \left|V^\pi_\mu(h) - v^\piR_\mu(s)\right| \leq \frac{\eps + \eps_B}{1-\g}.
    \eeqn
\end{theorem}
\begin{proof}
    Let $\delta := \underset{h,a,s,b:\psi(ha)=sb}{\sup}\left|Q^\pi_\mu(ha) - q^\piR_\mu(s b)\right|$, and for any $\psi(h)=s$ we have,
    \bqa
    V^\pi_\mu(h) - v^\piR_\mu(s)
    \overset{Lem. \ref{lem:Vrep}}&{\lessgtr} \sum_{b \in \B} \left( Q^\pi_\mu(\psi^{-1}(s b)){\piR}(b\|s) - q^\piR_\mu(sb){\piR}(b\|s) \right) \pm \left(\DQ(s) + \frac{\DPi(s)}{1-\g}\right) \nonumber \\
    \overset{}&{=} \sum_{b \in \B} \left( Q^\pi_\mu(\psi^{-1}(sb)) - q^\piR_\mu(sb)\right) {\piR}(b\|s) \pm \left(\DQ(s) + \frac{\DPi(s)}{1-\g}\right) \nonumber \\
    \overset{(a)}&{\lessgtr} \pm \left(\delta + \DQ(s) + \frac{\DPi(s)}{1-\g}\right) \label{eq:VvDiff2}
    \eqa
    $(a)$ follows from the definition of $\delta$ and the fact that ${\piR}(b\|s)$-average is smaller than the $b$-supremum. Using the above inequality (\ref{eq:VvDiff2}) and Lemma \ref{lem:QBqDiff} we get,
    \beq \label{eq:BAvgqDiff2}
    | \b Q^\pi_\mu(\psi^{-1}(s,b)) - q^\piR_\mu(s,b) | \leq \g\left(\delta + \DQ(s) + \frac{\DPi(s)}{1-\g}\right).
    \eeq

    We exploit the theorem's assumption and derive a key relationship between the $B$ average and any instance of action value.
    \bqa \label{eq:BAvgRepDiff2}
    \b Q^\pi_\mu(\psi^{-1}(sb))
    \overset{}&{\equiv} \sum_{h \in \H, a \in \A}   Q^\pi_\mu(ha)B(ha\|sb) \nonumber \\
    \overset{(a)}&{\lessgtr} \sum_{h \in \H, a \in \A}  \left( Q^\pi_\mu(\psi^{-1}(sb)) \pm \eps \right) B(ha\|sb) \nonumber \\
    \overset{}&{=} Q^\pi_\mu(\psi^{-1}(sb)) \pm \eps
    \eqa
    $(a)$ follows from the theorem's assumption. Since, $Q^\pi_\mu(\psi^{-1}(sb))$ is a representative member in the pre-image set of $sb$-pair; it is equivalent to say $Q^\pi_\mu(\psi^{-1}(sb)) = Q^\pi_\mu(ha)$ for any $\psi(ha) = sb$. Therefore, combining \eqref{eq:BAvgqDiff2} and \eqref{eq:BAvgqDiff2} we get $\left|Q^\pi_\mu(ha) - q^\piR_\mu(sb)\right| \leq \g(\delta + \DQ(s) + \frac{\DPi(s)}{1-\g}) + \eps$, hence $\delta \leq \frac{\eps + \g \DQ(s) + \frac{\g\DPi(s)}{1-\g}}{1-\g}$.
\end{proof}

In Theorem \ref{thm:psiQpi}, we had an absolute loss-bound for action-value functions but in Theorem \ref{thm:psiVpi} we only have a $B$-average relationship. So far, we were able to get a near-optimal performance guarantee when the optimal policy of a surrogate MDP is uplifted to the original process (see Theorems \ref{thm:psimdpstar}ii and \ref{thm:psiQstar}ii). However, there does not exist such a near-optimal performance guarantee for V-uniform homomorphisms. A \emph{deterministic} abstract policy could be arbitrarily worse off when uplifted to the original process \cite[Theorem 10]{Hutter2016} in V-uniform state-abstractions, which are a special case of V-uniform homomorphisms. However, a relatively weak result is still possible.

\begin{theorem}{$(\psi_{V^*_\mu})$}\label{thm:psiVstar}
    Let $|V^*_\mu(h)-V^*_\mu(\d h)|\leq \eps$ for all $\psi(h)=\psi(\d h)$ and $|\sum_{a \in \A} Q^*_\mu(ha)B^{{\u\pi}^*}(a\|hs) - V^*_\mu(h)| \leq \eps_B$ for all $s=\psi(h)$, then for all $\psi(ha) = sb$ it holds:
    \begin{enumerate}[(i)]
        \item $|\b Q^*_\mu(\psi^{-1}(sb)) - q^*_\mu(sb)| \leq \frac{3\g(\eps + \eps_B)}{(1-\g)^2}$ and
        $|V^*_\mu(h) - v^*_\mu(s)| \leq \frac{3(\eps + \eps_B)}{(1-\g)^2}$.
        \item If $\eps + \eps_B = 0$ then $\psi(h\pi^*(h)) = s{\u\pi}^*(s)$ for all $\psi(h)=s$.
    \end{enumerate}
\end{theorem}
\begin{proof}
    Let us define ${\u\pi}_h(s)$ such that $s{\u\pi}_h(s) := \psi(h\pi^*(h))$ for $\psi(h)=s$. Then,
    \beq
    q^{{\u\pi}_h}\left(s{\u\pi}_h(s)\right) = v^{{\u\pi}_h}(s) \overset{(a)}{\lessgtr} V^*_\mu(h) \pm \frac{\eps + \eps_B}{1-\g} \label{eq:qVe1}
    \eeq
    $(a)$ follows from Theorem \ref{thm:psiVpi} applied to $\pi = \pi^*$ (with ${\u\pi} = {\u\pi}_h$). Now we derive a bound for any $b \in \B$.
    \bqa \label{eq:qVe2}
    q^{{\u\pi}_h}(sb) - \frac{\g(\eps + \eps_B)}{1-\g}
    \overset{Thm. \ref{thm:psiVpi}}&{\leq} \b Q^*_\mu\left(\psi^{-1}(sb)\right) \equiv \sum_{h \in \H, a \in \A} Q^*_\mu(ha) B(ha\|sb) \nonumber \\
    \overset{}&{=} \sum_{h \in \H} B(h\|s) \sum_{a \in \A} Q^*_\mu(ha) B(a \|sbh) \nonumber \\
    \overset{(b)}&{\leq} \sum_{h \in \H} B(h\|s) \sum_{a \in \A} Q^*_\mu(h\pi^*(h)) B(a \|sbh) \nonumber \\
    \overset{}&{=} \sum_{h: \psi(h)=s} B(h\|s) V^*_\mu(h) \nonumber \\
    \overset{(c)}&{\leq} V^*_\mu(h) + \eps
    \eqa
    $(b)$ is due to the definition of optimal value and $(c)$ follows form the theorem's assumptions. Together (\ref{eq:qVe1}) and (\ref{eq:qVe2}) imply,
    \beq \label{eq:qVe3}
    v^{{\u\pi}_h}_\mu(s) = q^{{\u\pi}_h}_\mu\left(s, {\u\pi}_h(s)\right) \overset{}{\leq} \max_{b \in \B} q^{{\u\pi}_h}_\mu(sb) \overset{(\ref{eq:qVe2})}{\leq} V^*_\mu(h) + \frac{\eps + \g\eps_B}{1-\g} \overset{(\ref{eq:qVe1})}{\leq} v^{{\u\pi}_h}_\mu(s) + \frac{2(\eps + \eps_B)}{1-\g}
    \eeq

    \paragraph{(ii)} For $\eps = \eps_B = 0$, the previous equation implies $v^{{\u\pi}_h}_\mu(s) = \max_{b \in \B} q^{{\u\pi}_h}_\mu(sb)$. It shows that ${\u\pi}_h(s) = {\u\pi}^*(s)$ for all $\psi(h)=s$.

    \paragraph{(i)} Now for general $\eps + \eps_B > 0$ case. For all $s \in \S$ and $b \in \B$ we have,
    \bqan
    0
    &\overset{}{\leq} q^*_\mu(sb) - q^{{\u\pi}_h}_\mu(sb) \\
    &\overset{}{\equiv} \g \sum_{s' \in \S, r' \in \R} \b \mu_B(s'r'\|sb) \left(v^*_\mu(s') - v^{{\u\pi}_h}_\mu(s')\right) \\
    &\overset{(d)}{\leq} \g\max_{s' \in \S}\left(v^*_\mu(s') - v^{{\u\pi}_h}_\mu(s')\right)
    \eqan
    And,
    \bqan
    0 \overset{}&{\leq} v^*_\mu(s) - v^{{\u\pi}_h}_\mu(s) \\
    \overset{(e)}&{\leq} \max_{b \in \B} q^*_\mu(sb) - \max_{b \in \B} q^{{\u\pi}_h}_\mu(sb) + \frac{2(\eps + \eps_B)}{1-\g} \\
    \overset{(f)}&{\leq} \max_{b \in \B}\left(q^*_\mu(sb) - q^{{\u\pi}_h}_\mu(sb)\right)+ \frac{2(\eps + \eps_B)}{1-\g}
    \eqan
    $(d)$ expectation is replace by maximum operation; $(e)$ follows from the definition of $v^*_\mu(s)$ and (\Cref{eq:qVe3}); $(f)$ is a simple mathematical fact of maximization operation. Together this implies,
    \bqa
    \max_{s \in \S} \left( v^*_\mu(s) - v^{{\u\pi}_h}_\mu(s) \right) &\leq \g\max_{s \in \S} \left( v^*_\mu(s) - v^{{\u\pi}_h}_\mu(s) \right) + \frac{2(\eps + \eps_B)}{1-\g} \nonumber \\
    \Rightarrow \quad \max_{s \in \S} \left( v^*_\mu(s) - v^{{\u\pi}_h}_\mu(s) \right) &\leq \frac{2(\eps + \eps_B)}{(1-\g)^2} \label{eq:qVe4}
    \eqa
    Hence for any $\psi(h)=s$, we have,
    \bqan
    V^*_\mu(h) - \frac{\eps + \eps_B}{1-\g} \overset{(\ref{eq:qVe1})}&{\leq} v^{{\u\pi}_h}_\mu(s) \\
    \overset{(g)}&{\leq} v^*_\mu(s) \\
    \overset{(\ref{eq:qVe4})}&{\leq} v^{{\u\pi}_h}_\mu(s) + \frac{2(\eps + \eps_B)}{(1-\g)^2} \\
    \overset{(\ref{eq:qVe1})}&{\leq} V^*_\mu(h) + \frac{\eps + \eps_B}{1-\g} + \frac{2(\eps + \eps_B)}{(1-\g)^2} \leq V^*_\mu(h) + \frac{3(\eps + \eps_B)}{(1-\g)^2}
    \eqan
    $(g)$ holds by definition. Hence, the main results follows by using Lemma \ref{lem:QBqDiff} with $\pi$ replaced by $\pi^*$ and ${\u\pi}$ by ${\u\pi}^*$.
\end{proof}

In the \emph{approximate} case, i.e.\ $\eps + \eps_B > 0$, Theorem \ref{thm:psiVstar} is not as useful as Theorem \ref{thm:psiQstar} because of the missing performance guarantee, {\em cf.} Theorem \ref{thm:psiQstar}ii. However, it is still an important theorem for the \emph{exact} V-uniform homomorphisms, i.e.\ $\eps + \eps_B = 0$. In that case, it is guaranteed that the optimal actions of \emph{all} member histories are mapped to the \emph{same} abstract optimal action (see Theorem \ref{thm:psiVstar}ii).

That concludes the main results of this chapter. However, before we close this chapter, we provide an example non-MDP homomorphism below.

\section{A Simple Example}\label{sec:Example-2}

In this section, we provide a toy example to illustrate the possibility of approximately joint state-action aggregation beyond MDPs. In the example, contrary to ESA, we also aggregate the approximately similar policy state-action pairs.  These pairs are not aggregated in ESA due to the exact policy similarity condition, \emph{cf.} Theorem 5 of \citet{Hutter2016}.

For simplicity, we assume that actions and histories are mapped independently and the true environment $\mu$ is an MDP. We define the original action and observation spaces as $\A = \O = \{1,2,3,4\}$. Moreover, $\S = \{X, Y\}$, $\B = \{\alpha, \beta\}$. The original state-action pairs are represented by dots and the shaded regions indicate the mapping function. The dynamics are (fictitious) region constant (see Figure \ref{fig:example-homo})\footnote{Only the dots are elements of the joint space and the regions are fictitious. The aggregated pair $(s,b)$ is indicated adjacent to each region.}.

We assume an arbitrary policy which may depend on the history rather than only on the last observation. This allows complex dynamics for the proof of concept. We express the environment as
\bqan
\mu(o'\|oa) &= \sum_{a' \in \A} \mu(o'\|oa) \pi(a'\|oao') \\
&= \sum_{a' \in \A} \mu^\pi(o'a'\|oa).
\eqan

We express the dynamics with a joint measure $\mu^\pi$ and do not distinguish between the policy and the environment unless otherwise stated. Let the rewards be a function of the originating region and the problem has a finite set of real-valued rewards. We present three cases in this example: non-MDP dynamics, approximate action-values and policy disagreement.

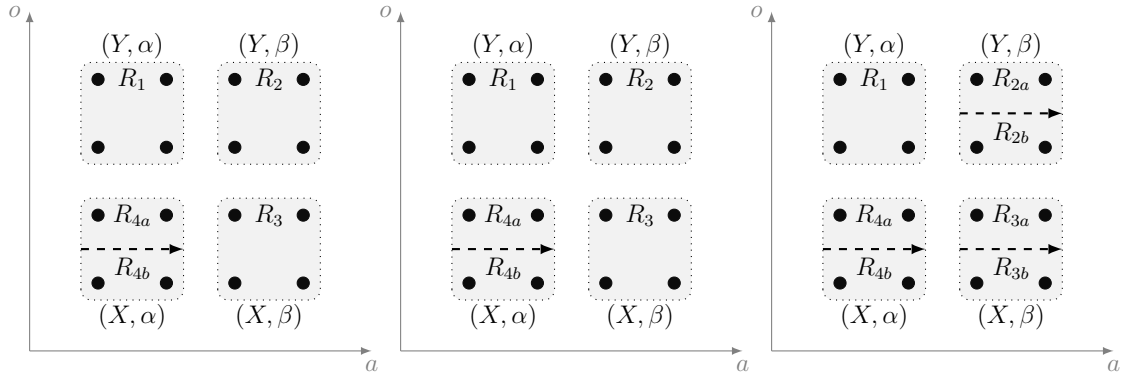
\begin{figure*}
    \begin{subfigure}[b]{0.30\textwidth}
        \centering
        \begin{tikzpicture}[scale=0.45, every node/.style={scale=0.8}]
            \coordinate (Origin)   at (0,0);
            \coordinate (XAxisMin) at (0,0);
            \coordinate (XAxisMax) at (10,0);
            \coordinate (YAxisMin) at (0,0);
            \coordinate (YAxisMax) at (0,10);
            \draw [thin, gray,-latex] (XAxisMin) -- (XAxisMax) node [below] {$a$};
            \draw [thin, gray,-latex] (YAxisMin) -- (YAxisMax) node [left] {$o$};
            \foreach \x in {1,2,...,4}{
                \foreach \y in {1,2,...,4}{
                    \node[draw,circle,inner sep=2pt,fill] at (2*\x,2*\y) {};
                }
            }
            \filldraw[fill=gray, fill opacity=0.1, draw=black, dotted, rounded corners] (1.5,1.5)
            rectangle (4.5,4.5);
            \filldraw[fill=gray, fill opacity=0.1, draw=black, dotted, rounded corners] (5.5,1.5)
            rectangle (8.5,4.5);
            \filldraw[fill=gray, fill opacity=0.1, draw=black, dotted, rounded corners] (1.5,5.5)
            rectangle (4.5,8.5);
            \filldraw[fill=gray, fill opacity=0.1, draw=black, dotted, rounded corners] (5.5,5.5)
            rectangle (8.5,8.5);

            \draw[dashed, thick] (1.5,3) -- (4.5,3);

            \node at (3,8) {$R_1$};
            \node at (7,8) {$R_2$};
            \node at (7,4) {$R_3$};
            \node at (3,4) {$R_{4a}$};
            \node at (3,3) [below] {$R_{4b}$};

            \node at (3,1) {$(X,\alpha)$};
            \node at (7,1) {$(X,\beta)$};
            \node at (3,9) {$(Y,\alpha)$};
            \node at (7,9) {$(Y,\beta)$};
        \end{tikzpicture}
    \end{subfigure}
    \begin{subfigure}[b]{0.30\textwidth}
        \centering
        \begin{tikzpicture}[scale=0.45, every node/.style={scale=0.8}]
            \coordinate (Origin)   at (0,0);
            \coordinate (XAxisMin) at (0,0);
            \coordinate (XAxisMax) at (10,0);
            \coordinate (YAxisMin) at (0,0);
            \coordinate (YAxisMax) at (0,10);
            \draw [thin, gray,-latex] (XAxisMin) -- (XAxisMax) node [below] {$a$};
            \draw [thin, gray,-latex] (YAxisMin) -- (YAxisMax) node [left] {$o$};
            \foreach \x in {1,2,...,4}{
                \foreach \y in {1,2,...,4}{
                    \node[draw,circle,inner sep=2pt,fill] at (2*\x,2*\y) {};
                }
            }
            \filldraw[fill=gray, fill opacity=0.1, draw=black, dotted, rounded corners] (1.5,1.5)
            rectangle (4.5,4.5);
            \filldraw[fill=gray, fill opacity=0.1, draw=black, dotted, rounded corners] (5.5,1.5)
            rectangle (8.5,4.5);
            \filldraw[fill=gray, fill opacity=0.1, draw=black, dotted, rounded corners] (1.5,5.5)
            rectangle (4.5,8.5);
            \filldraw[fill=gray, fill opacity=0.1, draw=black, dotted, rounded corners] (5.5,5.5)
            rectangle (8.5,8.5);

            \draw[dashed, thick] (1.5,3) -- (4.5,3);

            \node at (3,8) {$R_1$};
            \node at (7,8) {$R_2$};
            \node at (7,4) {$R_3$};
            \node at (3,4) {$R_{4a}$};
            \node at (3,3) [below] {$R_{4b}$};

            \node at (3,1) {$(X,\alpha)$};
            \node at (7,1) {$(X,\beta)$};
            \node at (3,9) {$(Y,\alpha)$};
            \node at (7,9) {$(Y,\beta)$};

        \end{tikzpicture}
    \end{subfigure}
    \begin{subfigure}[b]{0.30\textwidth}
        \centering
        \begin{tikzpicture}[scale=0.45, every node/.style={scale=0.8}]
            \coordinate (Origin)   at (0,0);
            \coordinate (XAxisMin) at (0,0);
            \coordinate (XAxisMax) at (10,0);
            \coordinate (YAxisMin) at (0,0);
            \coordinate (YAxisMax) at (0,10);
            \draw [thin, gray,-latex] (XAxisMin) -- (XAxisMax) node [below] {$a$};
            \draw [thin, gray,-latex] (YAxisMin) -- (YAxisMax) node [left] {$o$};
            \foreach \x in {1,2,...,4}{
                \foreach \y in {1,2,...,4}{
                    \node[draw,circle,inner sep=2pt,fill] at (2*\x,2*\y) {};
                }
            }
            \filldraw[fill=gray, fill opacity=0.1, draw=black, dotted, rounded corners] (1.5,1.5)
            rectangle (4.5,4.5);
            \filldraw[fill=gray, fill opacity=0.1, draw=black, dotted, rounded corners] (5.5,1.5)
            rectangle (8.5,4.5);
            \filldraw[fill=gray, fill opacity=0.1, draw=black, dotted, rounded corners] (1.5,5.5)
            rectangle (4.5,8.5);
            \filldraw[fill=gray, fill opacity=0.1, draw=black, dotted, rounded corners] (5.5,5.5)
            rectangle (8.5,8.5);

            \draw[dashed, thick] (1.5,3) -- (4.5,3);
            \draw[dashed, thick] (5.5,3) -- (8.5,3);
            \draw[dashed, thick] (5.5,7) -- (8.5,7);

            \node at (3,8) {$R_1$};
            \node at (7,8) {$R_{2a}$};
            \node at (7,7) [below] {$R_{2b}$};
            \node at (7,4) {$R_{3a}$};
            \node at (7,3) [below] {$R_{3b}$};
            \node at (3,4) {$R_{4a}$};
            \node at (3,3) [below] {$R_{4b}$};

            \node at (3,1) {$(X,\alpha)$};
            \node at (7,1) {$(X,\beta)$};
            \node at (3,9) {$(Y,\alpha)$};
            \node at (7,9) {$(Y,\beta)$};
        \end{tikzpicture}
    \end{subfigure}
    \caption[A non-MDP homomorphism example.]{\textbf{(Left):} Non-MDP aggregation, \textbf{(Middle):} Approximate aggregation, \textbf{(Right):} Violating policy uniformity condition}
    \label{fig:example-homo}
\end{figure*}

\paradot{Non-MDP Example} In the first case, this example (see Figure \ref{fig:example-homo}:Left) demonstrates a non-MDP aggregation. Let the problem have the following region uniform transition probability matrix in the observation-action space:
\beq
\mu^\pi =
\begin{bmatrix}
    0 & 1 & 0 & 0  & 0 \\
    0 & 0 & 1 & 0  & 0 \\
    0 & 0 & 0 & 1/2 & 1/2 \\
    1 & 0 & 0 & 0  & 0 \\
    1/2 & 0 & 0 & 1/4 & 1/4
\end{bmatrix}
\eeq
where $\mu^\pi_{ij}$ is the probability of reaching region $R_j$ if the current $(o,a) \in R_i$\footnote{The indexes should be read in order --- i.e.\ $1,2,3,4a,4b$.}. Formally, $\mu^\pi_{ij} = \sum_{(o'a') \in R_j} \mu^\pi(o'a'\|oa \in R_j)$. The marginalized process is expressed as
\beqn
\mu_\psi(X\|ha) = \mu_\psi(X\|oa) = \sum_{o' : \psi(oao')=X, a' \in \A} \mu^\pi(o'a'\|oa).
\eeqn

This is not even approximately an MDP. It is evident from the following probabilities of reaching state $X$ from itself.
\bqan
\mu_\psi(X\|oa \in R_{4a}) &= \mu^\pi_{4a3} + \mu^\pi_{4a4a} + \mu^\pi_{4a4b} = 0 \\
\mu_\psi(X\|oa \in R_{4b}) &= \mu^\pi_{4b3} + \mu^\pi_{4b4a} + \mu^\pi_{4b4b} = \frac{1}{2}
\eqan

The above equations show that reaching state $X$ from the regions $R_{4a}$ and $R_{4b}$ are different. But, we can still aggregate the regions if they have similar action-values.
Let the regional rewards be $r = \begin{bmatrix} 0 & 0 & 0 & \g & 0 \end{bmatrix}$. The regional action values can be expressed as
\beq
Q^\pi_\mu(oa) = r(oa) + \g \sum_{o' \in \O,a' \in \A} Q^\pi_\mu(o'a')\mu^\pi(o'a'\|oa).
\eeq

It translates into a vectorization form with each region $i$ has the action value expressed as
\beq
Q_i = r_i + \g \sum_{j} \mu^\pi_{ij}Q_j.
\eeq
By solving this system of equations we get the following region uniform action-value vector,
\beq
Q = \begin{bmatrix}
    c-2 & \g^2c & \g c & c & c
\end{bmatrix}
\eeq
where $c = \frac{2}{1-\g^3} \geq 2$. Hence, the example shows that even the regions $R_{4a}$ and $R_{4b}$ have non-MDP dynamics, they can still be aggregated due to the same action-values.

\paradot{Approximate Q-value Example} Now for the second case, we perform an approximate aggregation in region $R_3$ (see Figure \ref{fig:example-homo}:Middle). Let the reward in $R_{3b}$ be $\eps$. The updated reward vector is $r = \begin{bmatrix} 0 & 0 & 0 & \eps & \g & 0 \end{bmatrix}$. By keeping everything same as in the first example, the new transition matrix is given as
\beq
\mu^\pi =
\begin{bmatrix}
    0 & 1 & 0 & 0  & 0 & 0 \\
    0 & 0 & 1/2 & 1/2  & 0 & 0\\
    0 & 0 & 0 & 0 & 1/2 & 1/2 \\
    0 & 0 & 0 & 0 & 1/2 & 1/2 \\
    1 & 0 & 0 & 0  & 0 & 0 \\
    1/2 & 0 & 0 & 0 & 1/4 & 1/4
\end{bmatrix}
\eeq

Similar to the first case, we get the aggregated action-values as
\beq
Q = \begin{bmatrix}
    c-2 & \frac{\g\eps}{2}+ \g^2c & \g c & \g c + \eps & c & c
\end{bmatrix}
\eeq
where $c = \frac{\g^2\eps + 4}{2(1-\g^3)} \geq 2$. It shows that the regions $R_{3a}$ and $R_{3b}$ can be approximately aggregated together.

\paradot{Approximate Policy Example} As the last case, let us divide the region $R_2$ into two regions with different policies (see Figure \ref{fig:example-homo}:Right). Let the policy be approximately similar, i.e.\ $\abs{\pi(a\|o \in R_{2a}) - \pi(a\|o \in R_{2b})} = \eps'$ for all $\psi(oa) = (Y,\beta)$. The region $R_1$ is approachable from itself because the $\eps'$ weight of the aggregated action $\beta$ in region $R_{2b}$ is distributed to the aggregated action $\alpha$ in region $R_{1}$. This effectively translates the transition matrix into the following region uniform transition matrix,
\beqn
\mu^\pi =
\begin{bmatrix}
    \eps' & 1/2 & 1/2-\eps' & 0 & 0  & 0 & 0 \\
    0 & 0 & 0 & 1/2 & 1/2  & 0 & 0\\
    0 & 0 & 0 & 1/2 & 1/2  & 0 & 0\\
    0 & 0 & 0 & 0 & 0 & 1/2 & 1/2 \\
    0 & 0 & 0 & 0 & 0 & 1/2 & 1/2 \\
    1 & 0 & 0 & 0 & 0  & 0 & 0 \\
    1/2 & 0 & 0 & 0 & 0 & 1/4 & 1/4
\end{bmatrix}
\eeqn

The reward structure is the same as in the previous case, $r = \begin{bmatrix} 0 & 0 & 0 & 0 & \eps & \g & 0 \end{bmatrix}$. Finally, we get the action-value vector as
\beq
Q = \begin{bmatrix}
    \g^2 g(\eps')(\eps/2 + \g c) \\ \frac{\g\eps}{2} + \g^2 c \\ \frac{\g\eps}{2} + \g^2 c \\ \g c \\ \g c + \eps \\ c \\ c
\end{bmatrix}^T
\eeq
where $c = \frac{4 + \g^2\eps g(\eps')}{2(1-\g^3g(\eps'))}$ and $g(\eps') := \frac{1-\eps'}{1-\g\eps'}$. This final case shows that although the regions $R_{2a}$ and $R_{2b}$ have different policies they still have same action-values. Hence, the regions can be aggregated together exactly. It allows us to have coarser maps than ESA permits.

\section{Summary}\label{sec:Disc}

In this chapter we analyzed \emph{approximate} homomorphisms of a general history-based environment. The main idea was to find a \emph{deterministic} policy in the abstract state-action space such that, when uplifted, it is a near-optimal policy in the original problem. Using the surrogate MDP technique, we proved near-optimal performance bounds for both MDP (Theorem \ref{thm:psimdpstar}ii) and Q-uniform homomorphisms (Theorem \ref{thm:psiQstar}ii). In general, there does not exist a near-optimal \emph{deterministic} uplifted policy for V-uniform homomorphisms. However, we proved a weaker result (Theorem \ref{thm:psiVstar}ii) for the \emph{exact} V-uniform homomorphisms: the optimal actions of the member histories are mapped to the same abstract optimal action at the corresponding state of the surrogate MDP.

\paradot{Versus ESA}
We borrow some notation and techniques from \citet{Hutter2016}. But this work is crucially different from ESA. Apart from the obvious difference of being a generalization to homomorphisms, there are also some other key differences. In ESA, the policy $\pi$ is required to be state uniform for various of the main results \cite[Theorems 1,5,6 and 9]{Hutter2016}, whereas we do not make any such assumption.
Moreover, at the first instance our results might look \emph{almost} similar to ESA but the important difference is in the definition of $\eps(s)$ which is \emph{not} a simple addition of both state and action representation errors. It is a non-trivial weighted average of representation errors.
The extra conditions on Theorems \ref{thm:psiVpi} and \ref{thm:psiVstar} are \emph{weaker} than the policy-uniformity condition, \emph{cf.} \cite[Theorems 6 and 9]{Hutter2016}, and do not have direct counterparts in ESA.

\paradot{Versus Options} As briefly addressed in the introduction section, the options framework does not have any provable performance guarantees, yet. Whereas our restriction of uplifting a state-based policy and using a deceptively ``spatial-looking'' abstraction of actions have such guarantees. Since we allow the action mapping part of $\psi$ to be a function of history, which is arguably a function of time, our framework also admits temporal dependencies. It enables $\psi$ to model much more than mere renaming of the original action space distributions. A thorough comparison between these two approaches is left for future work.

\paradot{Special Environment Classes} In general, we do not use/exploit structure of the underlying original process. However, effects of a specific model class can be expressed in terms of the \mbox{(action-)value} functions. For example, if the original process is a finite state POMDP then our results provide the performance-loss guarantee by representing a belief-state based value function of the POMDP by a state-based value function. A similar argument can be rendered for various other types of model classes. Since the results in this work are general, they are not expected to gracefully scale down to some class specific tight performance bounds. Nevertheless, it is an important agenda to get the scaled-down variants of these results for some specific model classes.

%

\paradot{Fully Generalized Homomorphism} In a sense our results are not \emph{fully} general since we assumed a structure on the homomorphism. A fully generalized homomorphism formulation with no $\psi(ha) = f(h)b$ assumption would be an interesting extension of this work. However, lifting this condition may lead to some bizarre non-causal effects, e.g.\ the \emph{current} abstract state would be decided by the \emph{next} action!

\chapter[Action Sequentialization]{Action Sequentialization{\\ \it \small This chapter is an adaptation of \citet{Majeed2020}}}\label{chap:action-seq}

\begin{outline}
    Many real-world problems have large action-spaces such as in Go, StarCraft, protein folding, and robotics or are non-Markovian, which cause significant challenges to RL algorithms.
    In this chapter we address the large action-space problem by sequentializing actions, which can reduce the action-space size significantly, even down to two actions at the expense of an increased planning horizon. We provide explicit and exact constructions and equivalence proofs for all quantities of interest for arbitrary history-based processes. In the case of MDPs, this could help RL algorithms that bootstrap.
    In this chapter we show how action-binarization in the non-MDP case can significantly improve ESA bounds. As shown in \Cref{chap:extreme-abs}, ESA allows casting any (non-MDP, non-ergodic, history-based) RL problem into a fixed-sized non-Markovian state-space with the help of a surrogate Markovian process. On the upside, ESA enjoys similar optimality guarantees as Markovian models do. But a downside is that the size of the aggregated state-space becomes exponential in the size of the action-space. In this chapter, we patch this issue by binarizing the action-space. We provide an upper bound on the number of states of this binarized ESA that is logarithmic in the original action-space size, a double-exponential improvement.
\end{outline}


\section{Introduction}

Recall that an RL setting can be described by an agent-environment interaction \cite{Sutton2018}. The agent $\pi$ has an action-space $\A$ to choose its actions from while the environment $\mu$ reacts to the action by dispensing an observation and a reward from the sets $\O$ and $\R \subseteq \SetR$, respectively, see \Cref{fig:interaction}. Under a suitable definition of the ``state'' of environment, the resultant set of states might be huge or even infinite \cite{Powell2011}.

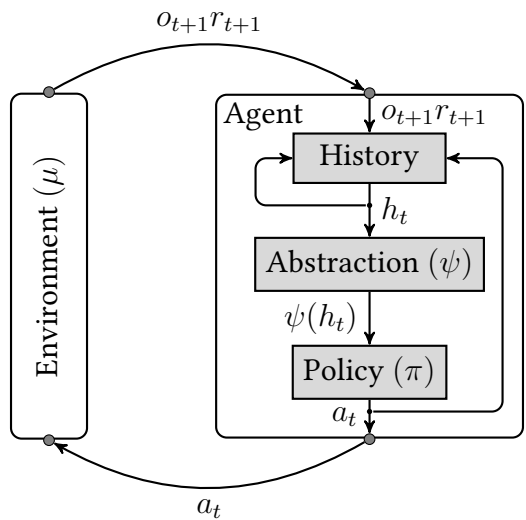
\begin{figure}[!ht]
    \centering
    \begin{tikzpicture}[
        node distance = 7mm and -3mm,
        innernode/.style = {draw=black, thick, fill=gray!30,
            minimum width=2cm, minimum height=0.5cm,
            align=center},
        outernode/.style = {draw=black, thick, rounded corners, fill=none,
            minimum width=1cm, minimum height=0.5cm,
            align=center},
        endpoint/.style={draw,circle,
            fill=gray, inner sep=0pt, minimum width=4pt},
        arrow/.style={->,thick,rounded corners},
        point/.style={circle,inner sep=0pt,minimum size=2pt,fill=black},
        skip loop/.style={to path={-- ++(#1,0) |- (\tikztotarget)}},
        every path/.style = {draw, -latex}
        ]
        \node (start) {Start};
        \node (h) [innernode]{History};
        \node (phi) [innernode, below=of h]{Abstraction $(\psi)$};
        \node (pi) [innernode, below=of phi]      {Policy $(\pi)$};
        \node [outernode, align=left, inner sep=0.5cm, fill=none, fit=(h) (phi) (pi)] (agent) {};
        \node[below right, inner sep=3pt, fill=none] at (agent.north west) {Agent};
        \node[outernode, left=120pt of agent, fit=(agent.north)(agent.south), inner sep=0pt] (env) {};
        \node[below right, inner sep=0pt, fill=none, rotate=90, anchor=center] at (env) {Environment $(\mu)$};
        \node[endpoint, above= -2pt of env] (or_env) {};
        \node[endpoint, below= -2pt of env] (a_env) {};
        \node[endpoint, below= -2pt of agent] (a_agent) {};
        \node[endpoint, above= -2pt of agent] (or_agent) {};

        \path (a_agent) edge[arrow,bend left] node[below]{$a_t$} (a_env);
        \path (or_env) edge[arrow, bend left] node[above]{$o_{t+1}r_{t+1}$} (or_agent);
        \path (or_agent) edge[arrow] node[right]{$o_{t+1}r_{t+1}$} (h);
        \path (h) edge[arrow] node[above=0.5pt,midway,name=h_phi,point]{} node[right]{$h_t$} (phi);
        \path (phi) edge[arrow] node[left]{$\psi(h_t)$} (pi);
        \path (pi) edge[arrow] node[above=0.5pt,midway,name=pi_a,point]{} node[left]{$a_t$} (a_agent);
        \path (pi_a) edge[arrow, skip loop=1.75cm] (h.east);
        \path (h_phi) edge[->, skip loop=-1.5cm, thick, rounded corners] (h.west);
    \end{tikzpicture}
    \caption{The agent-environment interaction, revisited.}
    \label{fig:interaction}
\end{figure}

To the best of our knowledge, extreme state aggregation (ESA), a non-MDP abstraction framework, is the only method which provides a provable upper bound on the size of required state-space uniformly, which depends only on the size of the action-space, discount factor and the optimality gap, for all problems \cite{Hutter2016}. However, a \emph{downside} of ESA is that the size of the aggregated state-space is \emph{exponential} in the size of the action-space, see \Cref{thm:esa}.
In this chapter, we move the research further in this direction. We provide a variant of ESA that can help provide much more compact representations as compared to MDP abstractions. Our approach improves the key upper bound on the size of the state-space in the original ESA framework.

The key trick to achieve this improvement is to sequentialize the actions. Often $\A$ already has a natural vector structure $\B^d$, e.g.\ real valued activators in robotics ($\B = \SetR$) or (padded) words ($\B = \{a, \dots, z, \textvisiblespace\}$), or more generally $\B_1 \times \dots \times \B_d$, where $\B$ denotes a finite set of \emph{decision symbols}.
In this case, sequentialization is natural, but one may further want to binarize $\B$ to $\SetB^{d'}$ esp. for ESA (\Cref{thm:bin-esa}).
If actions are (converted to) $\B$-ary strings,
the RL agent could execute the action ``bits'' sequentially with fictitious dummy observations in-between.
%

The example in \Cref{fig:example} provides a naive way of \emph{sequentializing} the actions in an MDP. Apparently, it might seem that such sequentialization of the action-space would be of no help, as the state-space would blow up, and it is simply shifting the problem from the actions to the states. However, we prove that this can be avoided. Most importantly, the universal upper bound on the effective state-space of ESA remains valid. Our scheme of sequentializing the actions achieves a \emph{double exponentially} improved bound; compare \Cref{thm:bin-esa} with \Cref{thm:esa}.

%
%
%

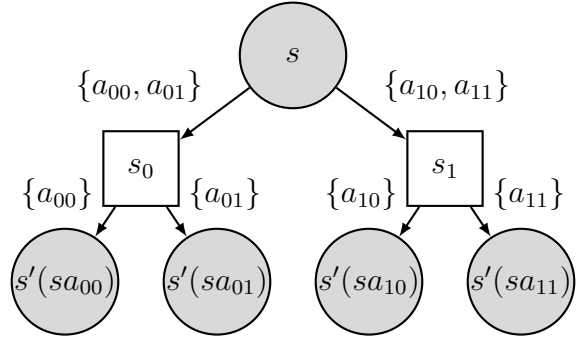
\begin{figure}
    \centering
    \begin{tikzpicture}[level/.style={sibling distance=40mm/#1,thick},square/.style={regular polygon,regular polygon sides=4,minimum size=1mm, inner sep=0,thick}]
        \node [graycircle] (z) {$s$}
        child {node [square,draw,inner sep=0, minimum size=14mm,thick] (a) {$s_0$}
            child {node [graycircle] (b) {$s'(sa_{00})$}
                edge from parent node[above left] {$\{a_{00}\}$}}
            child {node [graycircle] (g) {$s'(sa_{01})$}
                edge from parent node[above right] {$\{a_{01}\}$}}
            edge from parent node[above left] {$\{a_{00}, a_{01}\}$}
        }
        child {node [square,draw,inner sep=0, minimum size=14mm,thick] (j) {$s_1$}
            child {node [graycircle] (k) {$s'(sa_{10})$}
                edge from parent node[above left] {$\{a_{10}\}$}}
            child {node [graycircle] (l) {$s'(sa_{11})$}
                edge from parent node[above right] {$\{a_{11}\}$}}
            edge from parent node[above right] {$\{a_{10}, a_{11}\}$}
        };
    \end{tikzpicture}
    \caption[A simple sequentialization example in an MDP.]{A simple sequentialization example in an MDP. To see how the actions sequentialized, consider an agent which has to choose among four alternatives, e.g.\ $\A = \{a_{00}, a_{01}, a_{10}, a_{11}\}$. Let the agent receive a state signal $s$ from the environment. It first decides between a partition of actions, say two actions each, $\{a_{00}, a_{01}\}$ and $\{a_{10}, a_{11}\}$. \emph{After} it has decided on the bifurcation, the \emph{extended} state becomes $s_x$, where $x$ is the decision of the first stage. Now the agent \emph{on} this \emph{extended} state $s_x$ makes its second decision to choose from the \emph{short-listed} set of actions, and moves to the next state $s'(sa)$. This way, the agent only selects among \emph{two} alternatives at each stage by \emph{tripling} the effective state-space.}
    \label{fig:example}
\end{figure}

Along the way, we also establish some other key results, which are interesting and useful on their own. We provide explicit and exact constructions and equivalence proofs for all quantities of interest (\Cref{sec:bin-esa}) for arbitrary history-based processes, which are then used to double-exponentially improve the previous ESA bound (\Cref{thm:bin-esa}).
In the special case of MDPs, we show that through a sequentialized scheme (of augmenting observations with partial decision vectors) the resultant ``sequentialized process'' preserves the Markov property (\Cref{pro:mdpismdp}), which should help RL algorithms that bootstrap, though demonstrating or proving this is left for future work. Moreover in \Cref{lem:uplift}, we prove that the stipulated sequentialization scheme preserves near-optimality, i.e.\ a near-optimal policy of the sequentialized process is also near-optimal in the original process.

\section{Problem Setup}\label{sec:setup}

In \Cref{chap:grl,chap:arl}, we setup the GRL and ARL frameworks respectively. Now, we setup the scheme of sequentializing the decision-making process to reduce the effective action-space for the agent. Especially for our main result about ESA, we sequentialize the action-space to \emph{binary} decisions.

For this chapter we assume that the size of the action-space is finite and $\abs{\A} = \abs{\B}^d$ for some $d \in \SetN$. The latter assumption is not restrictive, as we can extend the set of actions by repeating some of the actions. It is important to note that these repeated actions should be labeled distinctly. This way we can have a bijection between the original (extended) action-space and the sequentialized one. For example, let an action set be $\{a_1,a_2,a_3,a_4,a_5\}$. One possible extended set, with repetition, for $\abs{\B} = 2$ and $d = 3$ is $\A := \{a_1,a_2,a_3,a_4,a_5,a_{5_1},a_{5_2},a_{5_3}\}$. Where, the actions $a_{5_i}$ for $i \leq 3$ are functionally the same as $a_5$, i.e.\ taking $a_5$ or any $a_{5_i}$ action has the same effect, but they are labeled distinctly.

\begin{remark}
Note that continuous action-spaces could be approximately sequentialized/binarized by using the binary expansion of reals to some desired precision, say $\delta$. Our main bound will only depend logarithmically on $\delta$.
\end{remark}

In the next section, we formulate another agent which only works in the ``sequentialized'' action-space, i.e.\ it takes decisions in a sequence of $\B$-ary choices, and responds \emph{only} to the histories generated by this $\B$-ary interaction, see \Cref{fig:bianary-interaction}. In the extreme case, this agent may only take binary decisions by sequentializing the action-space to binary sequences, i.e.\ $\B = \SetB$.

\section{Sequential Decisions}

We want to transform the action-space into a sequence of $\B$-ary decision code words, which are decided sequentially. To map the actions between the original action-space and the sequentialized decision-space, we define a pair of encoder and decoder functions. Let $C$ be any encoding function that maps the actions to a $\B$-ary decision code of length $d$, i.e.\ $C : \A \to \B^d$. A decoder function $D : \B^d \to \A$ sends the $\B$-ary decision sequences generated by $C$ back to the actions in the (original) action-space.
In this work, the choices of $C$ and $D$ do not matter\footnote{The choice could matter in practical implementation of such agents. For example, a clever choice of such functions might produce sparse $\B$-ary decision sequences for the optimal actions, hence it may facilitate in learning such optimal $\B$-ary decision sequences.} as long as they are bijections such that $D(C(\A)) = \A$.


This sequentialization of the action-space changes the interaction history. The generated histories are no longer members of $\H$. The goal of this paper is to argue that an agent can still work with the sequentialized histories only. The agent can plan, learn and interact with the environment using $\B$-ary actions and keeping sequentialized histories. Hence, the agent can be agnostic to the original action-space and with the state provided through an appropriate abstraction it can only take $\B$-ary decisions at every time step, see \Cref{fig:bianary-interaction}.

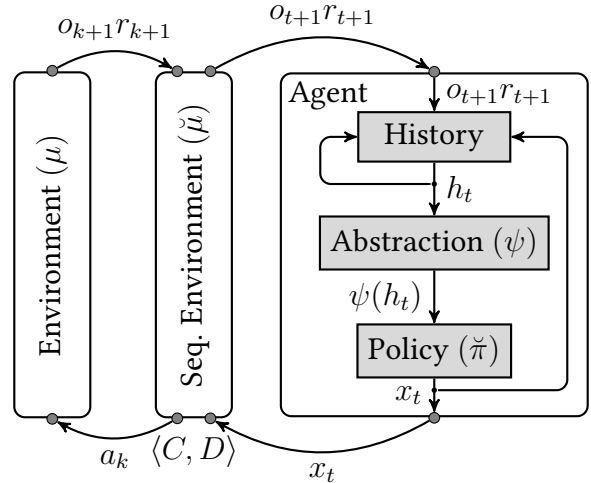
\begin{figure}[!ht]
    \centering
    \begin{tikzpicture}[
        node distance = 7mm and -3mm,
        innernode/.style = {draw=black, thick, fill=gray!30,
            minimum width=2cm, minimum height=0.5cm,
            align=center},
        outernode/.style = {draw=black, thick, rounded corners, fill=none,
            minimum width=1cm, minimum height=0.5cm,
            align=center},
        endpoint/.style={draw,circle,
            fill=gray, inner sep=0pt, minimum width=4pt},
        arrow/.style={->,thick,rounded corners},
        point/.style={circle,inner sep=0pt,minimum size=2pt,fill=black},
        skip loop/.style={to path={-- ++(#1,0) |- (\tikztotarget)}},
        every path/.style = {draw, -latex}
        ]
        \node (start) {Start};
        \node (h) [innernode]{History};
        \node (phi) [innernode, below=of h]{Abstraction $(\psi)$};
        \node (pi) [innernode, below=of phi]      {Policy $(\u \pi)$};
        \node [outernode, align=left, inner sep=0.5cm, fill=none, fit=(h) (phi) (pi)] (agent) {};
        \node[below right, inner sep=3pt, fill=none] at (agent.north west) {Agent};
        \node[outernode, left=90pt of agent, fit=(agent.north)(agent.south), inner sep=0pt] (bin_env) {};
        \node[outernode, left=143pt of bin_env, fit=(agent.north)(agent.south), inner sep=0pt] (env) {};
        \node[below right, inner sep=0pt, fill=none, rotate=90, anchor=center] at (env) {Environment $(\mu)$};
        \node[below right, inner sep=0pt, fill=none, rotate=90, anchor=center] at (bin_env) {Seq. Environment $(\u \mu)$};
        \node[below= 2pt of bin_env] (D) {$\langle C, D\rangle$};
        \node[endpoint, above= -2pt of env] (or_env) {};
        \node[endpoint, below= -2pt of env] (a_env) {};
        \node[endpoint, left=4pt of bin_env.north] (or_bin_env_in) {};
        \node[endpoint, right=4pt of bin_env.north] (or_bin_env_out) {};
        \node[endpoint, left=4pt of bin_env.south] (a_bin_env_out) {};
        \node[endpoint, right=4pt of bin_env.south] (a_bin_env_in) {};
        \node[endpoint, below= -2pt of agent] (a_agent) {};
        \node[endpoint, above= -2pt of agent] (or_agent) {};

        \path (a_agent) edge[arrow,bend left] node[below]{$x_t$} (a_bin_env_in);
        \path (a_bin_env_out) edge[arrow,bend left] node[below]{$a_k$} (a_env);
        \path (or_bin_env_out) edge[arrow, bend left] node[above]{$o_{t+1}r_{t+1}$} (or_agent);
        \path (or_env) edge[arrow, bend left] node[above]{$o_{k+1}r_{k+1}$} (or_bin_env_in);
        \path (or_agent) edge[arrow] node[right]{$o_{t+1}r_{t+1}$} (h);
        \path (h) edge[arrow] node[above=0.5pt,midway,name=h_phi,point]{} node[right]{$h_t$} (phi);
        \path (phi) edge[arrow] node[left]{$\psi(h_t)$} (pi);
        \path (pi) edge[arrow] node[above=0.5pt,midway,name=pi_a,point]{} node[left]{$x_t$} (a_agent);
        \path (pi_a) edge[arrow, skip loop=1.75cm] (h.east);
        \path (h_phi) edge[->, skip loop=-1.5cm, thick, rounded corners] (h.west);
    \end{tikzpicture}
    \caption[The agent-environment interaction through the sequentialization scheme.]{The agent-environment interaction through the sequentialization scheme. Note that the sequentialized-environment block (or a $\B$-ary ``mock'') manages two different time-scales $t$ and $k$. It is simply a buffer block which knows (de)coders $D$ and $C$ (see text for details). It buffers the input $\B$-ary actions and dispatches the buffered observation and reward. Once a complete $\B$-ary decision sequence is produced by the agent the $\B$-ary mock decodes the encoded actions to the original environment to continue the interaction loop. We can consider this sequentialized environment as a ``middle layer'' between the agent and the original environment.}
    \label{fig:bianary-interaction}
\end{figure}

We construct a history transformation function which maps the original histories from $\H$ to some sequentialized histories in $\u \H$, where
\beq
\u \H := \bigcup_{t=1}^\infty \underbrace{\O \times \R \times \B \times \dots \times \O \times \R \times \B}_{(t-1)-\text{step interactions}}\times\O\times\R
\eeq

It is worth noting that $\u \H$ does not (directly) contain any information about $\A$, cf.\ \Cref{eq:history}. The agent experiencing histories from this set would not be aware of $\A$.

\begin{definition}[History transformation function]
    A history transformation function is expressed with $g : \H \to \u \H$. The map is recursively defined for any history $h$, action $a$, next observation $o'$ and next reward $r'$ as
    \beq
    g(hao'r') := g(h)\v x_1or_\bot\v x_2or_\bot \dots \v x_do'r' \ \text{and} \ g(e) := e
    \eeq
    where $\v x := C(a)$, $o$ is the last observation of the history $h$, $e$ denotes the ``initial'' history\footnote{The initial history $e \in \O \times \R$ is similar to the initial state in standard RL. It is dispatched by environment without any input at the start.}, and $r_\bot$ is any fixed real-value. In this work, we assume\footnote{This assumption is not much of a restriction, if $r_\bot \notin \R$ then we can extend the reward space by $\R \cup \{r_\bot\}$.} that $r_\bot \in \R$ and $r_\bot = 0$.
\end{definition}

In the above construction, we chose to repeat the last observation $o$ in between the real interactions with the environment. This is not the only possible choice, we can choose a \emph{dummy} observation $o_\bot \in \O$ instead without affecting the claims. For brevity, we define $\v o$ and $\v r_\bot$ as $d$-dimensional constant vectors of $o$ and $r_\bot$, respectively. These vectors are then ``welded'' with $\v x$ to succinctly replace $x_1or_\bot\dots x_ior_\bot$ with $\vo{xor_\bot}_{\leq i}$. Note that we do not sequentialize the observations. It can be done, but we believe it is not useful in any way.

However, if the original process $P$ is an MDP, i.e.\ the most recent observation is the state of $P$, then there is another interesting option possible for $o_\bot$: extend the observation space $\O$ with $\O \times \cup_{i=0}^{d-1} \B^{i} =: \t \O$, and let the $\B$-ary mock dispatch an appropriate observation at every partial $\B$-ary decision vector $\v x_{<i}$ as:
\beq\label{eq:tobs}
\t o_\bot \coloneqq (o, \v x_1, \v x_2, \dots, \v x_{i-1}) \in \t \O
\eeq

It is not hard to show that with this sequentialization scheme the resultant sequentialized decision process is also an MDP over $\t \O$, see \Cref{pro:mdpismdp}.
By doing so, we end up with a state-space of size $|\t \O| = \abs{\O}(\abs{\A}-1) \leq \abs{\O \times \A}$. It is clear that this recasting of the original problem might not be very helpful for some Monte-Carlo like tree search methods, however, it might significantly improve the performance of some temporal-difference like algorithms,, e.g.\ Q-learning \cite{Watkins1992}, when applied to huge action-spaces.

Note that $g$ is injective, but it may not be a bijection. There are many sequentialized histories $\tau \in \u \H$ which are not mapped by $g$, i.e.\ there does not exist any history in $\H$ such that $\tau = g(h)$. For such sequentialized histories we define $g\inv(\tau) := \bot$, which formally allows us to talk about $g\inv$ without worrying about it being undefined on some arguments. The choice of this definition is not important. As a matter of fact, there is no particular significance of the symbol $\bot$. What makes this choice insignificant is the fact that the environment does not react until the agent has taken $d$ $\B$-ary actions. Some histories not covered by $g$ are such ``partial'' sequentialized histories where the actual environment does not react. Note that the rewards of the sequentialized setup are zero ($r_\bot := 0$) unless the sequentialized history length is a multiple of $d$, i.e.\ a ``complete'' sequentialized history. See \Cref{fig:example-2} for an example sequentialized/binarized setup for $\B = \SetB$ and $d=2$.

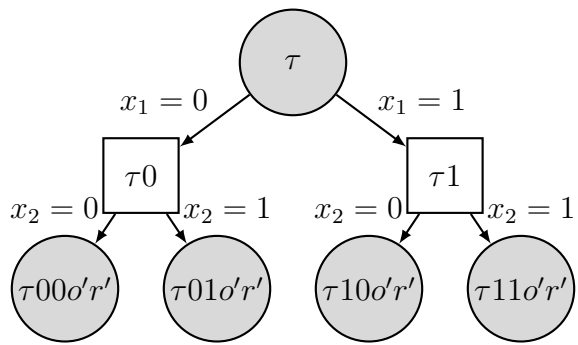
\begin{figure}
    \centering
    \begin{tikzpicture}[level/.style={sibling distance=40mm/#1,thick, inner sep=2pt},square/.style={regular polygon,regular polygon sides=4,minimum size=1mm, inner sep=0,thick}]
        \node [graycircle] (z){$\tau$}
        child {node [square,draw, minimum size=14mm,thick] (a) {$\tau 0$}
            child {node [graycircle] (b) {$\tau 00 o' r'$} edge from parent node[above left] {$x_2 = 0$}}
            child {node [graycircle] (g) {$\tau 01 o' r'$} edge from parent node[above right] {$x_2 = 1$}}
            edge from parent node[above left] {$x_1 = 0$}
        }
        child {node [square,draw, minimum size=14mm,thick] (j) {$\tau 1$}
            child {node [graycircle] (k) {$\tau 10 o' r'$} edge from parent node[above left] {$x_2 = 0$}}
            child {node [graycircle] (l) {$\tau 11 o' r'$} edge from parent node[above right] {$x_2 = 1$}}
            edge from parent node[above right] {$x_1 = 1$}
        };
    \end{tikzpicture}
    \caption[A simple sequentialization/binarization example in a deterministic history-based process.]{A simple sequentialization/binarization example in a deterministic history-based process. The $\B$-ary/binary decisions are on the edges. For brevity, we do not represent $o_\bot$ and $r_\bot$ in the figure. For example, it should be apparent that $\tau 1 o_\bot r_\bot \equiv \tau 1$. The circles represent complete histories while the squares indicate partial histories.}
    \label{fig:example-2}
\end{figure}

Any agent which interacts with the environment through this sequentialized scheme would effectively experience the following sequentialized environment.
\begin{definition}[Sequentialized environment]\label{def:uP}
    For any $\B$-ary action $x \in \B$, sequentialized history $\tau \in \u \H$, and any partial extension $\vo{xor_\bot}_{<i}$ for $i \leq d$ the probability of receiving $o'$ and $r'$ as the next observation and reward is as follows:
    \bqan
    \u \mu(o'r'\|\tau\vo{xor_\bot}_{<i} x)
    := \begin{cases}
        \mu(o'r'\|ha) &\text{if } \tau\vo{xor_\bot}_{<i} xo'r' = g(hao'r') \\
        1 &\text{if } o'r' = or_\bot \\ &\phantom{}\text{ and } g^{-1}(\tau\vo{x or_\bot}_{<i}  x o' r') = \bot \\
        0 &\text{otherwise}
    \end{cases} \numberthis
    \eqan
    where $\mu$ is the actual environment.
\end{definition}

As highlighted before, the history $h$ can not be empty, so the above definition is well-defined.

The next step is to define the (action) value functions for this sequentialized agent-environment interaction. Let $\u \pi$ be a policy such that $\u \pi : \u \H \to \Dist(\B)$. Then, we define the (action) value functions similar to the original agent-environment interaction case. For any $\tau \in \u \H$ and $x \in \B$, the action-value function is defined as
\beq
\u Q_\mu^{\u \pi}(\tau x) := \sum_{o'r'} \u \mu(o'r'\|\tau x) \left(r' + \lambda \u V_\mu^{\u \pi}(\tau x o'r')\right)\label{eq:ube}
\eeq
where $\u V_\mu^{\u \pi}(\tau) := \sum_{x \in \B} \u Q_\mu^{\u \pi}(\tau x)\u \pi(x\|\tau)$ and $\lambda$ is the discount factor of this sequentialized problem. Similar to the original optimal (action) value functions, $\u Q_\mu^*$ and $\u V_\mu^*$ denote the optimal (action) value functions of the sequentialized problem. The discount factor $\lambda$ plays an important role in trading off the size of the action-space with the planning horizon. Recall that the size of the original action-space is $\abs{\A} = \abs{\B}^d$. Therefore, if the agent has to make $d$ $\B$-ary decisions for each original action the discount factor after $d$ $\B$-ary actions should be $\g$, i.e.\ $\lambda^d = \g$. This implies that $\lambda = \g^{1/d} < 1$ as $\g < 1$ and $d < \infty$.

This completes the problem setup. We have defined an agent $\u \pi$ which only makes $\B$-ary decisions and reacts to sequentialized histories, see \Cref{fig:bianary-interaction}. As expected, the set of sequentialized histories $\u \H$ is blown out in comparison with $\H$. However, in \Cref{sec:esa}, we show that, under certain non-Markovian abstractions of either $\H$ or $\u \H$, this expansion is not ``harmful''.

\section{Sequentialized Processes and Values}\label{sec:bin-esa}

In this section we formally define the sequentialized process and related value functions. But first we need a couple of important quantities to state our main results. For any $\B$-ary vector $\v x \in \B^i$ where $i \leq d$, we define $\A(\v x) := \{ a \in \A : \v x \sqsubseteq C(a) \}$ a \emph{restricted} set of actions. Moreover, for any history $h$, an action-value function maximizer on this restricted set is defined as $\pi^*(h,\v x) \in \argmax_{a \in \A(\v x)} Q_\mu^*(ha)$ where ties are broken uniformly randomly.

We start off the section by noting an important relationship between the sequentialized process and the original process.

\begin{proposition}[Sequentialized Process]\label{pro:uPtoP}
    For any $o' \in \O, r' \in \R, h \in \H$ and $D(\v x) =: a \in \A$, the following relationship holds between $\u \mu$ and $\mu$:
    \beq
    \u \mu(o'r'|g(h)\vo{xor_\bot}_{<d} \v x_d) =\mu(o'r'\|ha)
    \eeq
\end{proposition}
\begin{proof}
    The proof trivially follows from \Cref{def:uP} by evaluating the definition at $i = d$ with $D(\v x_{<d} \v x_d) = a$.
\end{proof}

When the original process is an MDP then there exists a sequentialization scheme such that the sequentialized process is also Markovian.

\begin{theorem}[Sequentialization preserves Markov property]\label{pro:mdpismdp}
    If $\mu$ is an MDP over $\O$, and the observations from the $\B$-ary mock are $\t \O \coloneqq  \O \times \cup_{i=0}^{d-1} \B^{i}$, then $\u \mu$ is an MDP over $\t \O$.
\end{theorem}
\begin{proof}
    In the case of augmenting the observation space, the definition of $\u \mu$ becomes slightly more verbose than \Cref{def:uP} as $o_\bot$ is different for each partial history as defined in \Cref{eq:tobs}.
    \bqan
    \u \mu(\t o'r'\|\tau\vo{x\t or_\bot}_{<i} x)
    := \begin{cases}
        \mu(o'r'\|ha) &\text{if } \tau\vo{x\t or_\bot}_{<i} x\t o'r' = g(hao'r') \\
        1 &\text{if } \t o'r' = o\v x_{<i} xr_\bot \\ &\phantom{}\text{ and } g^{-1}(\tau\vo{x\t or_\bot}_{<i}  x \t o' r') = \bot \\
        0 &\text{otherwise}
    \end{cases}\numberthis
    \eqan
    for any $i \leq d$, $\t o, \t o' \in \t \O$, and $o \in \O$ is the most recent observation in $h$. At any $h$ the sufficient information is $o$, so $\mu(o'r'\|ha) \equiv \mu(o'r'\|oa)$. Therefore, from the above (expanded) definition of $\u \mu$, it is clear that:
    \beqn
    \u \mu(\t o'r'\|\tau \vo{x\t or_\bot}_{<i}x) \equiv \u \mu(\t o'r'\|o\v x_{<i}x) = \u \mu(\t o'r' \|o x)
    \eeqn
    hence proves the proposition.
\end{proof}

The following proposition proves that the action-values of the ``partial'' histories of the sequentialized problem are related. This fact later helps us to show that these action-value functions respect the Q-uniform structure of the original environment.

\begin{proposition}[$\u Q_\mu^*$ $\max$-relationship]\label{prep:expandsion}
    For any sequentialized history $\tau \in \u \H$ such that $g\inv(\tau) \in \H$, the following holds
    \beq
    \max_{x \in \B} \u Q_\mu^*(\tau, x) = \lambda^{d-1} \max_{\v x \in \B^d} \u Q_\mu^*(\tau \vo{xor_\bot}_{<d}, \v x_d)
    \eeq
\end{proposition}
\begin{proof}
    The proof is straight forward. We successively apply the definition of $\u Q_\mu^*$.
    \bqan
    \max_{x_1 \in \B} \u Q_\mu^*(\tau, x_1)
    &= \max_{x_1 \in \B} \sum_{o'r'} \u \mu(o'r'\|\tau x_1)\\ &\phantom{=} \left(r' + \lambda \max_{x_2 \in \B} \u Q_\mu^{*}(\tau x_1 o'r', x_2)\right)\\
    &\overset{(a)}{=} \lambda \max_{x_1 \in \B} \max_{x_2 \in \B} \u Q_\mu^*(\tau x_1 or_\bot, x_2)\\
    &\vdotswithin{=} \text{(continue unrolling for $d-1$-steps)}\\
    &= \lambda^{d-1} \max_{\v x \in \B^d} \u Q_\mu^*(\tau \vo{xor_\bot}_{<d}, \v x_d) \numberthis
    \eqan
    where $(a)$ follows from the definition of $\u \mu$ and the fact that $r' = r_\bot = 0$ when $\u \mu \neq 0$.
\end{proof}

Now, using \Cref{prep:expandsion} we can prove a relationship between the action-value functions of the actual environment and the sequentialized environment.

\begin{lemma}[$\u Q_\mu^*$ $\v x$-relationship]\label{lem:q-relation}
    For any history $h$ with the corresponding sequentialized history $\tau = g(h)$ and $\B$-ary decision vector $\v x \in \B^d$, the following holds for any $i \leq d$.
    \beqn
    \u Q_\mu^*(\tau\vo{xor_\bot}_{<i}, \v x_i) = \g^{\frac{d-i}{d}} Q_\mu^*(h,\pi^*(h,\v{x}_{\leq i}))
    \eeqn
\end{lemma}
\begin{proof}
    Before we prove the general result, we show that the result holds for $i = d$, i.e.\ the sequentialized problem has same optimal action-values at the ``real'' decision steps. Note that $\pi^*(h,\v x_{\leq d}) = D(\v x)$. Let $\v x := C(a)$ and $\tau := g(h)$. Using the fact that $r_\bot = constant = 0$, we get

    \bqan
    f_{r_\bot}(h,a)
    &:= \u Q_\mu^*(\tau \vo{xor_\bot}_{< d}, \v x_d) \\
    &\overset{(a)}{=} \sum_{o'r'} \u \mu(o'r'\|\tau \vo{xor_\bot}_{< d}\v x_d) \left( r' + \lambda \max_{x'} \u Q_\mu^{*}(\tau \vo{xor_\bot}_{< d} \v x_d o'r', x')\right) \\
    &\overset{(b)}{=} \sum_{o'r'} \mu(o'r'\|ha) \left(r' + \lambda \max_{x'} \u Q_\mu^{*}(\tau \vo{xor_\bot}_{< d} \v x_d o'r', x')\right) \\
    &\overset{(c)}{=} \sum_{o'r'} \mu(o'r'\|ha) \left(r' + \lambda^d \max_{\v x' \in \B^d} \u Q_\mu^{*}(\tau \vo{xor_\bot}_{< d} \v x_d o'r' \vo{xor_\bot}'_{<d}, \v x'_d)\right) \\
    &\overset{(d)}{=} \sum_{o'r'} \mu(o'r'\|ha) \left(r' + \g \max_{a' \in \A} f_{r_\bot}(hao'r',a')\right)\numberthis\label{eq:obe-2}
    \eqan
    where $(a)$ is just \Cref{eq:ube} with the optimal policy, $(b)$ follows by \Cref{pro:uPtoP}, $(c)$ is given by \Cref{prep:expandsion}, $(d)$ is true by rearranging the argument, the definition of $f_{r_\bot}$ and by using the relation $\lambda^d = \g$. Note that \Cref{eq:obe-2} is the OBE of the original problem. The solution of the OBE is unique \cite{Lattimore2014b}, hence $f_{r_\bot}$ is indeed $Q_\mu^*$.

    Having proved the claim for $i = d$, we show that it also holds for any $i < d$.
    \bqan
    \u Q_\mu^*(\tau\vo{xor_\bot}_{<i}, \v x_i)
    &\overset{(a)}{=} \sum_{o'r'} \u \mu(o'r'\|\tau \vo{xor_\bot}_{< i}\v x_i) \left(r' + \lambda \max_{x_{i+1}} \u Q_\mu^{*}(\tau \vo{xor_\bot}_{< i} \v x_i o'r', x_{i+1})\right) \\
    &\overset{(b)}{=} \lambda \max_{x_{i+1}} \u Q_\mu^{*}(\tau \vo{xor_\bot}_{< i} \v x_i or_\bot, x_{i+1}) \\
    &\vdotswithin{=} \text{(continue unrolling for $d-i-1$-steps)}\\
    &\overset{}{=} \lambda^{d-i} \max_{x_{i+1}} \dots \max_{x_{d}} \u Q_\mu^{*}(\tau \vo{xor_\bot}_{< i} \v x_i \\&\phantom{\tau \vo{xor_\bot}_{< i}}or_\bot x_{i+1}or_\bot \dots x_{d-1}or_\bot, x_d) \\
    &\overset{(c)}{=} \lambda^{d-i} \max_{a \in \A(\v x_{\leq i})} Q_\mu^*(h, a) \numberthis
    \eqan
    where, again $(a)$ is \Cref{eq:ube} with the optimal policy, $(b)$ follows from the definition of $\u \mu$ and $r_\bot = 0$, and $(c)$ is true from the fact that the claim holds for $i=d$ and the maximum is over the restrictive set of actions.
\end{proof}

What we have proven so far is that the sequentialization scheme produces action-value functions which (at the ``partial'' histories) are rescaled versions of the original action-value function. They agree with the original $Q_\mu^*$ at the decision points (at the ``complete'' histories) where the sequentialized policy $\u \pi$ completes an action code.

We also show that a similar relationship as proved in \Cref{lem:q-relation} holds for a fixed policy $\u\pi$. However, we use a different proof method for the following lemma. Note that $\u \pi$ induces a policy $\b \pi$ on the original environment, which can trivially be expressed as follows:
\beq\label{eq:bpi}
\b \pi(a\|h) := \prod_{i=1}^d \u \pi(\v x_i \|\tau \vo{xor_\bot}_{<i}) =: \u \pi(\v x \|\tau)
\eeq
for any $a = D(\v x)$ and $\tau = g(h)$.

\begin{lemma}[$\u Q_\mu^{\u\pi}$ $\v x$-relationship]\label{lem:fixed-policy}
    For any arbitrary policy $\u \pi$ the following relationship is true:
    \beq
    \u Q_\mu^{\u\pi}(\tau \vo{xor_\bot}_{<d}, \v x_d) = Q_\mu^{\b \pi}(h, D(\v x))
    \eeq
    for any history $\tau = g(h)$ and $\v x \in \B^d$.
\end{lemma}
\begin{proof}
    Before we prove the main result of the lemma, we show that the following relationship holds for the value-functions of the sequentialized and the original environment:
    \beq\label{eq:v-v}
    V_\mu^{\u\pi}(\tau) = \lambda^{d-1}V_\mu^{\b \pi}(h)
    \eeq
    for any $\tau = g(h)$.
    We use a different argument than \Cref{lem:q-relation} to prove the above statement. Lets imagine the sequentialized environment is at the history $\tau = g(h)$. The agent starts to follow the policy $\u\pi$. The following is the $(\text{expected-reward}, \text{discount-factor})$ sequence it generates from this history.
    \bqan
    &(0,\lambda^0),(0,\lambda^1),\dots,(0,\lambda^{d-2}),(\b r,\lambda^{d-1}),\\
    &(0,\lambda^d),(0,\lambda^{d+1}), \dots,(0,\lambda^{2d-2}),(\b r',\lambda^{2d-1}),\\ &(0,\lambda^{2d}), \dots
    \eqan
    where $\b r$ is the expected reward. The sum of the reward part of the above sequence returns $\u V_\mu^{\u\pi}(\tau)$.
    Now, if we re-scale the discount part of the above sequence by $\lambda^{d-1}$ we get $V_\mu^{\b\pi}(h)$ as the sum of the reward part.

    \bqan
    &(0,\lambda^{1-d}),(0,\lambda^{2-d}),\dots,(0,\lambda^{-1}),(\b r,\lambda^{0}),\\
    &(0,\lambda^1),(0,\lambda^{2}), \dots,(0,\lambda^{d-1}),(\b r',\lambda^{d}), \\
    &(0,\lambda^{d+1}), \dots
    \eqan
    which proves \Cref{eq:v-v} when $\lambda^d = \gamma$. Now, let $a := D(\v x)$.
    \bqan
    Q_\mu^{\u\pi}(\tau\vo{xor_\bot}_{<d},\v x_d)
    &= \sum_{o'r'} \u\mu(o'r'\|\tau\vo{xor_\bot}_{<d}\v x_d) \left(r' + \lambda V_\mu^{\u\pi}(\tau\vo{xor_\bot}_{<d}\v x_do'r')\right) \\
    \overset{(a)}&{=} \sum_{o'r'} \mu(o'r'\|ha) \left(r' + \lambda V_\mu^{\u\pi}(\tau\vo{xor_\bot}_{<d}\v x_do'r')\right) \\
    \overset{\eqref{eq:v-v}}&{=} \sum_{o'r'} \mu(o'r'\|ha) \left(r' + \lambda^d V_\mu^{\b\pi}(hao'r')\right) \\
    &= \sum_{o'r'} \mu(o'r'\|ha) \left(r' + \g V_\mu^{\b\pi}(hao'r')\right) = Q_\mu^{\b\pi}(h,D(\v x))
    \eqan
    where $(a)$ is due to \Cref{pro:uPtoP}.
\end{proof}

The following theorem proves the usefulness of our sequentialization framework. We show that the optimal policy of the sequentialized environment is also optimal in the original environment when it is lifted back using the decoding function $D$.


\begin{theorem}[Sequentialization preserves $\eps$-optimality]\label{lem:uplift}
    Any $\lambda^{d-1}\eps$-optimal policy of the sequentialized environment is $\eps$-optimal in the original environment.
\end{theorem}
\begin{proof}
    Let $\u \pi$ be an $\eps'$-optimal policy of the sequentialized environment, where $\eps' := \lambda^{d-1}\eps$. It implies the following:
    \beq
    \u V_\mu^*(\tau\vo{xor_\bot}_{<i}) - \u V_\mu^{\u \pi}(\tau\vo{xor_\bot}_{<i}) \leq \eps'
    \eeq
    for any complete sequentialized history $\tau = g(h)$ and $\v x \in \B^{i-1}$ where $i \leq d$. Especially, we are interested in the case when $i=1$, i.e.\ values at the complete histories.
    \beq\label{eq:near-opt}
    \u V_\mu^*(\tau) - \u V_\mu^{\u \pi}(\tau) \leq \eps'
    \eeq
    With simple algebra, we can show that the following relationship holds for the optimal policies of the sequentialized and original processes:
    \bqan
    \u V_\mu^*(\tau)
    &\overset{(a)}{=} \max_{x} \u Q_\mu^*(\tau, x) \\
    &\overset{(b)}{=} \lambda^{d-1} \max_{\v x \in \B^d} \u Q_\mu^*(\tau \vo{xor_\bot}_{<d}, \v x_d) \\
    &\overset{(c)}{=} \lambda^{d-1} \max_{\v x \in \B^d} Q_\mu^*(h, D(\v x)) = \lambda^{d-1} V_\mu^*(h) \numberthis\label{eq:opt-v-v}
    \eqan
    where $(a)$ is the definition of the value function, $(b)$ holds due to \Cref{prep:expandsion}, and $(c)$ is true by applying \Cref{lem:q-relation} for $i = d$.

    Now, by simply using \Cref{eq:v-v} and \Cref{eq:opt-v-v}, we can prove the claim.
    \beq
    V_\mu^*(h) - V_\mu^{\b \pi}(h) \overset{(a)}{=} \lambda^{1-d}\left(\u V_\mu^*(\tau) - \u V_\mu^{\u \pi}(\tau)\right) \overset{\eqref{eq:near-opt}}{\leq} \eps
    \eeq
    for any $\tau = g(h)$, where $(a)$ is due to \Cref{eq:v-v} and \Cref{eq:opt-v-v}.
\end{proof}

We are done formally defining the setup. In the next section we put everything together under the context of ESA to establish the validity of our sequentialization setup.

\section{ESA with Binarized Actions}\label{sec:esa}

The $\eps$-Q-uniform, non-MDP abstractions lead to the following important result due to \citet{Hutter2016}. We only state the result without a proof for the closure of exposition, see the previous chapters and \citet{Hutter2016} for more details about ESA and proofs.

In the following theorems we assume that the rewards are bounded in the unit interval, i.e.\ $\R \subseteq [0,1]$. This is done for brevity, and it is not a necessary condition. The rescaling of the rewards does not affect the decision-making process in (G)RL. In general, let the range of the rewards be $R := \max \R - \min \R$. Then, the scalars in the nominators of \Cref{thm:esa,thm:bin-esa} are replaced by $2R$ and $4R^2$ respectively.

\begin{theorem}[ESA{\cite[Theorem 11]{Hutter2016}}] \label{thm:esa}
    For every environment $\mu$ there exists a reduction $\phi$ and a surrogate-MDP whose optimal policy\footnote{See \citet{Hutter2016} of how to learn this policy, the surrogate-MDP, $Q_\mu^*$, and $\phi$.} is an $\eps$-optimal policy for the environment. The size of the surrogate-MDP is bounded (uniformly for any $\mu$) by\footnote{The 2 instead of a 3 in the original theorem is a trivial improvement by removing the grid point at 0 in the construction.}
    \beqn
    \abs{\S} \leq \left(\frac{2}{\eps (1-\g)^3}\right)^{\abs{\A}}
    \eeqn
\end{theorem}

This is a powerful result, but it suffers from the exponential dependence on the action-space size. We now put our action sequentialization framework to work and dramatically improve this dependency from exponential to only a logarithmic dependency in $\abs{\A}$.

So far, we have considered an arbitrary $\B$-ary decision set to sequentialize the action-space. However, in the following theorem we go to the extreme case of sequentializing the action-space to binary decisions ($\B = \SetB$) to squeeze out the maximum improvement possible through the framework.



\begin{theorem}[Binary ESA] \label{thm:bin-esa}
    For every environment there exists an abstraction and a corresponding surrogate-MDP for its binarized version ($\B = \SetB$) whose optimal policy is $\eps$-optimal for the true environment. The size of the surrogate-MDP is uniformly bounded for \emph{every} environment as
    \beqn
    \abs{\S} \leq \frac{4\ceil{1 - \g + \lb\abs{\A}}^6}{\g^2 \eps^2 (1-\g)^6}
    \eeqn
\end{theorem}
\begin{proof}
    Consider the agent that is interacting with the sequentialized/binarized environment $\u \mu$. By \Cref{lem:uplift}, we know that a near-optimal policy of this sequentialized environment is also near-optimal in the original environment. Now, if we use ESA on the binarized problem and get an $\eps'$-optimal policy through the surrogate-MDP by \Cref{thm:esa}, we are sured to be $\eps$-optimal in the original environment $\mu$ as explained above. Additionally, the size of the state-space is bounded as
    \beq\label{eq:bound}
    \abs{\S} \overset{\Cref{thm:esa}}{\leq} \left(\frac{2}{\eps' (1-\lambda)^3}\right)^2 = \frac{4}{{\eps'}^2 (1-\lambda)^6}
    \eeq
    where $\lambda$ is the discount factor of the sequentialized problem.
    Next, we upper bound \Cref{eq:bound} by using the fact that $\lambda^d = \g$. Let $\del := 1 - \g < 1$. So,
    \bqan
    1 - \lambda
    &= 1- (1-\del)^{1/d} = 1 - \e^{\frac{\ln (1-\del)}{d}}\\
    &\overset{(a)}{\geq} 1- \frac{1}{1-\ln (1-\del)/d}
    \overset{(b)}{\geq} 1 - \frac{1}{1+\del/d} \\
    &\overset{}{=} \frac{\del}{d + \del} = \frac{1-\g}{d + 1 - \g}
    \numberthis \label{eq:bound2}
    \eqan
    where $(a)$ holds due to $\frac{1}{\e^{-\alpha}} \leq \frac{1}{1-\alpha}$, $(b)$ is true by using the fact that $\del < 1$, hence $\ln(1-\del) \leq - \del$.
    Therefore, using \Cref{eq:bound}, \Cref{eq:bound2}, and $\eps' = \lambda^{d-1}\eps \geq \lambda^d \eps = \g \eps$
    we get,
    \beq
    \abs{\S} \leq \frac{4}{{\eps'}^2 (1-\lambda)^6} \leq  \frac{4 (1-\g+d)^6}{\g^2\eps^2 (1-\g)^6}
    \eeq
    which proves the claim.
    %
\end{proof}

Superficially, it might seem that we have simply replaced the original discount factor with a larger value. But, it is not the case. If we simply scaled the discount factor (without sequentializing the actions) then the resulting bound would indeed deteriorate, see \Cref{thm:esa}, but on the contrary, with sequentialization/binarization and our analysis the bound (dramatically) improves.

Usually in RL the discount factor $\g$ is close to 1. In that case, the bound in \Cref{thm:bin-esa} can be tightened further as:
\beq
\abs{\S} \lesssim \frac{4\ceil{\lb\abs{\A}}^6}{\eps^2 (1-\g)^6}
\eeq
which agrees with the bound in \Cref{thm:esa} for the case when $\abs{\A} = 2$, i.e.\ when the original problem already has a binary action-space.

\ifshort\else
We conclude the section by reminding the fact that our results do not assume that the agent has access to the original history sequence from $\H$, the agents can solely work in the binarized history-space $\u \H$, if they have access to an ESA map, see \Cref{fig:bianary-interaction}.
\fi

\section{Summary}\label{sec:conclusion}

This chapter contributes to the study of the GRL problem. We have provided a reduction to handle large state and action spaces by sequentializing the decision-making process.
This helped us improve the upper bound on the number of states in ESA from an exponential dependency in $\abs{\A}$ to logarithmic. The gain is \emph{double exponential} in terms of the action-space dependence at no other cost.
Our result carries a broader impact on the implementation of \emph{general} RL agents. The required storage for such agents, which have access to a non-MDP, approximate Q-uniform abstraction, can be reasonably bounded which only scales logarithmically in the size of the original action-space.

This work analyses the case when the agent has a fixed aggregation map. \citet{Hutter2016} provides an outline for a learning algorithm to learn such abstractions which can be combined with our sequentialization framework.

Another direction, which we also did not touch in this work, is to explore the connection, if any, between the surrogate-MDPs of a map on the original environment, and its extension on the sequentialized problem. By lifting the small binary ESA map, say $\psi$, back to $\H$, one obtains a small map directly on $\H$, say $\phi$. While $\psi$ used sequentialization/binarization for the construction of $\phi$, the map $\phi$ can be used without further referencing to sequentialization. This suggests that a bound logarithmic in $\abs{\A}$ should be possible without a detour through the sequentialization. This deserves further investigation.

We sequentialize the action-space through an \emph{arbitrary} coding scheme $C$, so the main result does not depend on this choice. Sometimes, it is possible that the action-space may allow “natural” sequentialization, e.g.\ in a video game controller the “macro” action might be a binary vector where the first bit might represent the left/right direction, the second bit indicates up/down, and so on. The exact nature of these “binary decisions” depends on the domain which is reflected by the choice of encoding $C$. Sequentialization was our path to double-exponentially improve that bound. Whether there are more direct/natural aggregations with the same bound is an open problem. Moreover, if the agent is learning an abstraction through interaction, the choice of these functions may become critical.

This chapter focused on rigorously formalizing and proving the main improvement result. One can also try to empirically show the effectiveness of our improved upper bound. To do this, we need a problem domain where ESA requires more states than the sequentialized/binarized version of it. But a point of caution is that the upper bound still scales badly in terms of $\g$ and $\eps$. Any reasonable value of these parameters would imply a huge upper bound.
Even with Markovian abstractions, a cubic dependency on the discount factor is the best achievable.
We considered a general underlying process and non-Markovian abstractions,
and dramatically improved the previously best bound $(1-\g)^{-3|\A|}$ to $(1-\g)^{-3\cdot 2}$.
Indeed it would be interesting to see whether this can be further improved to the optimal $(1-\g)^{-3}$ rate.

\chapter[Abstraction Learning Methods]{Abstraction Learning Methods {\\ \it \small This chapter is an adaptation of \citet{Majeed2021a}}}\label{chap:abs-learning}

\begin{outline}
    In this chapter we consider a couple of techniques that can be explored further to craft an abstraction learning algorithm. The primary objective of these ideas is to build an algorithm that learns an (extreme) $\eps$-QDP representation of the environment. Unlike the rest of the thesis, the (pseudo-)algorithms in this chapter have no rigorously proven theoretical guarantee. We explore the possibility of an extension of the \emph{partial order} of \citet{Hutter2016}
    and a potential candidate based on the \emph{algorithmic complexity} of the abstraction map.
\end{outline}

\section{Introduction}

Up to this point, the primary focus of this thesis was to provide existence proofs and usefulness results for some given (and fixed) MDP and/or non-MDP abstractions. These results are the stepping stones for the big problem of creating agents which can learn these abstractions from data \cite{Hutter2016,Hutter2009}. A truly general agent should be able to abstract the experiences on their own without us (the designers) providing them the abstraction map \cite{McCallum1996}. This thesis can be considered as a case for the existence of such \emph{universal} abstractions (e.g.\ extreme $\eps$-QDP abstraction) which can be used to tractably plan near optimally for any environment. In this chapter, we consider the case where the agent may search from a set of candidate abstractions and/or might build an abstraction from the data on the fly, e.g.\ refining a coarser model of the world as it gets more experience.

The contribution of this chapter is to initiate discussion about a collection of several different ideas and methods which together may be used to get a sound non-MDP abstraction learning algorithm. However, this topic requires further investigation. We discuss a multitude of technically involved aspects of the problem. We merely scratch the surface of the problem of \emph{feature reinforcement learning} (FRL) which aims at \emph{learning} the abstraction map $\psi$ from the historical data \cite{Hutter2009,Hutter2008c}.

In the majority of cases, we consider that the algorithm can sample an infinite amount of data, i.e.\ it interacts with the environment infinitely long. Once the soundness of any of these algorithms is established, it is not hard to create a finite sample variant, as most of the (potential) algorithms are in the form of an \emph{anytime} algorithm. An anytime algorithm has a monotonically shrinking error bound around the estimates. We can stop the algorithm at any instance, and we will get an estimate with a confidence bound which improves over time. By using some structural assumptions, we can prove the finite sample complexity results \cite{Kakade2003}.

The following is a brief introduction of each algorithm we consider in this chapter.  All these algorithms assume that the environment can be modeled by some bounded memory source \cite{Ryabko2008}. Otherwise, the estimation might not be possible. The full details about each algorithm are provided in the following sections.

\begin{itemize}

    \iparadot{Ordering a Set of Abstractions}
    \citet{Hutter2016} considered a partial order over a set of (candidate) abstractions. The resultant order can be extended to make the order ``less'' partial. A less partial order is better for FRL, as the agent can simply compare ``most'' maps in any sequence. If that is the case, then there is a better chance that the agent may find the optimal map. We provide a series of order relations, which culminates to a \emph{total order} over any set of maps.

    \iparadot{Minimum Description Length Abstractions}
    The state-space of an abstraction can be considered as a partition of the history-space. The abstraction is ``computing'' the output for each history to decide which state to put this in. This view can potentially lead to an algorithm which can make the optimal compromise/trade-off between the ``representation power'' of an abstraction versus the algorithmic complexity of the map. This is also known as the minimum description length principle \cite{Grunwald2007}.
\end{itemize}

In the following sections, we go through each of these algorithms individually.

\section{Ordering a Set of Abstractions}

One way to learn an abstraction is to put an order over a class of abstractions $\Psi$. If we have a total order minimized by the coarsest $\eps$-QDP abstraction $\psi^*$, which also call the \emph{optimal} abstraction, then the algorithm is trivially easy: we simply compare any pair of maps to sequentially find the minimal element. Typically, we assume that the coarsest $\eps$-QDP abstraction is in the model class, i.e.\ $\psi^{*} \in \Psi$. In this section, we provide a couple of novel orders, one of which is (potentially) a total order.

\begin{algorithm}
    \caption{Abstraction Learning with Extended Order (A-LEO$_\infty$)}
    \begin{algorithmic}[1]
        \Input a countable model class $\Psi$ with optimal $\psi^{*} \in \Psi$
        \Output a QDP model $\widehat \psi$, where $\widehat \psi = \psi^{*}$ in the limit
        \State Initialize the history $h = \epsilon$
        \State Select any abstraction $\h\psi \in \Psi$
        \State Initially no map is rejected $\Psi_{\rm Rejected} = \emptyset$
        \Repeat \Comment{Forever}
        \State Extend (to some length) the sample trajectory $h$ using $\pi_{\h\psi}$
        \State Select a competitor map $\psi \in \Psi\setminus\Psi_{\rm Rejected}$
        \State Estimate state-action function using Q-learning with $\h\psi$ using samples from $h$
        \State Estimate state-action function using Q-learning with $\psi$ using samples from $h$
        \State Estimate state-action function using Q-learning with $\psi \times \h\psi$ using samples from $h$
        \If{${\rm order}(\psi, \h\psi \| \psi \times \h \psi)$}\Comment{$\psi$ is preferred over $\h\psi$}
        \State Put $\h\psi$ to $\Psi_{\rm Rejected}$
        \State Set $\h\psi = \psi$
        \Else
        \State Put $\psi$ to $\Psi_{\rm Rejected}$
        \EndIf
        \If{$\Psi \setminus \Psi_{\rm Rejected} = \{\h\psi\}$}\Comment{In a rare event, all other maps have been rejected}
        \State Reset $\Psi_{\rm Rejected} = \emptyset$
        \EndIf
        \Until \textbf{false}
    \end{algorithmic}
    \label{alg:leo}
\end{algorithm}

In the following subsections, we provide a ``template'' algorithm (\Cref{alg:leo}) which internally uses an order relation to pick candidate abstractions. So, every order defined later can be used in this algorithm. Obviously, the performance and convergence guarantees (if any) of the algorithm depend on the order. After providing some basic definitions, we start developing and refining the order relation. We define four order relations in total: $(1)$ an order based on maps being mutually refinable and exact Q-uniformity (\Cref{def:order-exact-q}), $(2)$ an order using Cartesian product maps with exact Q-uniformity (\Cref{def:order-extend-exact-q}), $(3)$ an order also based on Cartesian product but using approximate Q-uniformity (\Cref{def:exact-cart-product-rel}), and $(4)$ an order using approximate Cartesian product distance and approximate Q-uniformity (\Cref{def:cpd-rel}).

\subsection{Maps and Algorithms}

Before we talk about the actual algorithm which exploits the ordering of over a class of maps, we need to introduce some notation about the relationships of maps. The concept of \emph{refinement} captures a special structural relationship among the maps.

\begin{definition}[Refinement]
    We say a map $\phi: \H \to \S_\phi$ is a refinement of a map $\psi: \H \to \S_\psi$ if for \emph{all} pair of histories $h$ and $\d h$ the following holds:
    \beq
    \phi(h) = \phi(\d h) \implies \psi(h) = \psi(\d h)
    \eeq

    In this case, there exists a coarsening map $\chi: \S_\phi \to \S_\psi$, and we usually express the refinement relationship as $\psi = \chi(\phi)$.
\end{definition}

It is possible that a large number of maps are not refinements of each other. We extend the refinement relationship to Cartesian products, which is a special type of intermediate map that refines any (constituent) pair of maps.

\begin{definition}[Cartesian Product of Maps]
The Cartesian product of two maps $\psi$ and $\psi'$ is defined to be a map $\psi \times \psi' \eqqcolon \phi$. The states of this map are the ordered pairs of states of the constituent maps:
\beq
\S_{\phi} \coloneqq \{(s_\psi, s_{\psi'}) \mid \forall s_\psi \in \S_\psi, s_{\psi'} \in \S_{\psi'}\}
\eeq
and $\phi(h) \coloneqq (\psi(h), \psi'(h))$ for any history $h$.
\end{definition}

In the following subsections, we produce a variety of order relations over a set of abstractions using above structural relations. \Cref{alg:leo} is a single ``wrapper'' algorithm which can use any of the orders defined in this section as a function call ${\rm order}(\psi, \h \psi \| \psi \times \h\psi)$, where $\psi \times \h\psi$ is the side information required for some order relations below. The algorithm is very simple in itself. We start from a candidate abstraction, and pick a competitor from the class. After calculating the required quantities for the order, we compare the maps with the order relation. The successful map proceeds to the next iteration. Whereas, the rejected map is removed from the class. In a rare event, there could be the case that all maps are rejected\footnote{Note that bounding the probability of this event is one of the critical aspects to be considered in future formal analysis of the algorithm.}. In that case we restart the algorithm with all maps.
The convergence properties and the quality of converged map depends on the choice of the order. We expand on this topic further for each order in the corresponding subsection.

\begin{remark}[Formally an Order]
    In the rest of the chapter, we use the word ``order'' in a commonly understood term in order theory, without formally proving that the stipulated ``orders'' are indeed order relations. Any binary relation is an order if it is \emph{reflexive}, \emph{transitive}, and \emph{anti-symmetric} \cite{Davey2002}. Along with the formal analysis, we defer rigorous proofs of these properties to future work. However, we do provide intuitive justifications of various orders considered in this chapter.
\end{remark}

\subsection{Ordering Through Exact Q-uniformity}

This section is based on the  (extended) order relation defined by \citet{Hutter2016}. He used the above Cartesian product refinement structure to generate a partial oder over any set of maps $\Psi$, which, under some conditions, may lead \Cref{alg:leo} to $\psi^{*}$.

We start with a simple observation that if any pair of mutually refinable maps $\psi$ and $\psi'$ (e.g.\ $\psi = \chi(\psi')$) have the same state-action-value function, i.e.\ $q_{\langle \mu,\psi' \rangle}^*(sa) = q_{\langle \mu,\psi \rangle}^*(\chi(s)a)$ for all $s \in \S_{\psi'}$ and $a \in \A$, then the coarser map should be preferred over the finer one because both maps ``model'' the same state-action-value function. We formally express this as the following order relation:
\begin{definition}[$\preceq_{\chi}$]\label{def:order-exact-q}
    Any two maps $\psi$ and $\psi'$ are partially ordered as
    \begin{numcases}{\psi \preceq_{\chi} \psi' :\iff}
        \textsf{true}, & if $\psi = \chi(\psi') \land \forall s \in \S_{\psi'},a.\ q_{\langle \mu,\psi' \rangle}^*(sa) = q_{\langle \mu,\psi \rangle}^*(\chi(s)a)$ \label{eq:Rchia}\\
        \textsf{true}, & if $\chi'(\psi) = \psi' \land \exists s \in \S_{\psi},a.\ q_{\langle \mu,\psi' \rangle}^*(\chi'(s)a) \neq q_{\langle \mu,\psi \rangle}^*(sa)$ \label{eq:Rchib}\\
        \false, & otherwise \label{eq:Rchid}
    \end{numcases}
    where the comparison of the state-action-value functions is exact.
\end{definition}

In simple words:
\begin{itemize}
    \item Equation~\eqref{eq:Rchia} says that $\psi$ is a coarsening of $\psi'$ with the same state-action-value function, so $\psi$ is better, and
    \item Equation~\eqref{eq:Rchib} says that $\psi$ is refining $\psi'$ with a different state-action-function, again $\psi$ is better because the change in state-action-value function is an indication that the coarser map is not Q-uniform. So, $\psi'$ needs further refinement to get to the optimal partition.
    \item Equation~\eqref{eq:Rchid} says that $\psi$ is \emph{not} preferred over $\psi'$. It is triggered by a variety of cases. Some of which are the ``right'' comparisons in the sense that $\lnot (\psi \prec_{\chi} \psi') \implies \psi' \prec_{\chi} \psi$. For example, if $\psi$ is a refinement of $\psi'$ but has the ``same'' state-action-value function as $\psi'$ then it is \emph{not} preferred over $\psi'$. Or, if $\psi$ is a coarsening of $\psi'$ but has ``different'' state-action-value function then it is also \emph{not} preferred over $\psi'$ by Equation~\eqref{eq:Rchid}.
    \item Moreover, it is easy to see that the order is partial. There are cases when both directions of the order are false.
    \beq
    \not\exists \chi,\chi'.\ \psi = \chi(\psi') \lor \chi'(\psi) = \psi' \implies \lnot (\psi \preceq_{\chi} \psi') \land \lnot (\psi' \preceq_{\chi} \psi)
    \eeq
    which means the order can not compare the maps which are not mutually refinable.
\end{itemize}

It is clear from the above definition that $\prec_{\chi}$ is ``very'' partial. In general, there could be an overwhelming majority of maps which are not comparable under this order. Therefore, this relation can not be used in \Cref{alg:leo} as is. \citet{Hutter2016} extended the above relation to cover more maps.

\begin{definition}[$\preceq_{\times}$]
Any two maps $\psi$ and $\psi'$ are ordered as
    \begin{numcases}{\psi \preceq_{\times} \psi' :\iff}
        \true, & if $\psi \preceq_{\chi} \phi \preceq_{\chi} \psi'$ \label{eq:Rmha}\\
        \true, & if $\psi \preceq_{\chi} \phi \succeq_{\chi} \psi' \land S_\psi \leq S_{\psi'}$ \label{eq:Rmhb}\\
        \true, & if $\psi \succeq_{\chi} \phi \preceq_{\chi}  \psi' \land S_\psi \leq S_{\psi'} \land \phi \in \Psi$ \label{eq:Rmhc}\\
        \false, & otherwise
    \end{numcases}
where $\phi \coloneqq \psi \times \psi'$.
\label{def:order-extend-exact-q}
\end{definition}

The relation $\preceq_{\times}$ uses $\prec_{\chi}$ in it, hence, it extends the latter.
Again in simple terms:
\begin{itemize}
    \item Equation~\eqref{eq:Rmha} implies the ``original'' $\prec_{\chi}$ comparison. To see the equivalence, we know that $\psi \preceq_{\chi} \psi'$ holds for mutually refinable maps. So, if $\psi \preceq_{\chi} \psi'$ then either $\psi \equiv \phi$ or $\psi' \equiv \phi$, which trivially implies $\psi \preceq_{\times} \psi'$ holds because of  Equation~\eqref{eq:Rmha}.
    \item Equation~\eqref{eq:Rmhb} says both $\psi$ and $\psi'$ have exactly the same state-action-value functions as their Cartesian product map $\phi$, but $\psi$ has smaller state-space size, so select $\psi$, and
    \item Equation~\eqref{eq:Rmhc} says that both $\psi$ and $\psi'$ have different state-action-value function than their Cartesian product map $\phi$, and $\phi$ is in the class, so select the smaller state-space map $\psi$ under the pretext that when $\phi$ will be compared with $\psi$ (later) then $\phi$ will be selected.
    \item There are still some incomparable cases. The order is partial, albeit ``less'' partial than $\preceq_{\chi}$.
    \beq
    \psi \succeq_{\chi} \phi \preceq_{\chi}  \psi' \land \phi \notin \Psi \implies \lnot (\psi \preceq_{\times} \psi') \land \lnot (\psi' \preceq_{\times} \psi)
    \eeq
    which means the order cannot compare the case when the Cartesian product map $\phi$ is better than both $\psi$ and $\psi'$, but it is not in the class.
\end{itemize}

It is important to note that the above order $\preceq_{\times}$ is still partial if the class of maps $\Psi$ is not closed under Cartesian products. If $\Psi$ is closed under Cartesian products then \Cref{alg:leo} is a reasonable algorithm to find an \emph{exact} QDP abstraction using $\preceq_{\times}$.

\begin{conjecture}[Sufficiency of $\psi^{*} \in \Psi$]
    If $\psi^{*} \in \Psi$ then we do not need to care about the class being closed under the Cartesian products. In that case, we can use the following total order variant of $\preceq_{\times}^*$:
    \begin{numcases}{\psi \preceq_{\times}^* \psi' :\iff}
        \true, & if $\psi \preceq_{\chi} \phi \preceq_{\chi} \psi'$ \label{eq:Rmhxb}\\
        \true, & if $\psi \preceq_{\chi} \phi \succeq_{\chi} \psi' \land S_\psi \leq S_{\psi'}$ \label{eq:Rmhxa}\\
        \true, & if $\psi \succeq_{\chi} \phi \preceq_{\chi}  \psi' \land S_\psi \leq S_{\psi'}$ \label{eq:Rmhxc}\\
        \false, & otherwise
    \end{numcases}
    which will allow \Cref{alg:leo} to converge to $\psi^{*}$.

\end{conjecture}

The above conjecture makes sense, since if the optimal map is in the class then selecting any map by Equation~\eqref{eq:Rmhxc} would eventually be replaced by the optimal map. However, we may never know if the converged map is indeed optimal, unless we have ruled out every other map.

\subsection{Ordering Through Approximate Q-uniformity}

The order $\prec_{\chi}$ is based on the \emph{exact} similarity of the state-action-value functions, which limits the use cases of this order. It can not be used in the realistic cases where we only have limited data and approximate estimates. Moreover, the required abstraction class $\Psi$ for this (or $\preceq_{\times}$ and $\preceq_{\times}^*$) order has to contain the \emph{exact} QDP abstraction. In this work, we extend the order to approximate similarity in state-action-value functions, we call it $\eps$-Q-isomorphism.

\begin{definition}[$\eps$-Q-similarity]\label{def:q-similarity}
    Any two abstractions are \emph{$\eps$-Q-similar} if their state-action-value functions are $\eps$-close on the Cartesian product space.\footnote{We sometimes simply say that the state-action-value functions are ``similar''.} Formally,
    \beq
    \psi \approx_\eps \psi' :\iff \max_{s \in \S_\phi, a \in \A} \abs{q^*_{\langle\mu,\psi \rangle}(\chi(s)a) - q^*_{\langle\mu,\psi' \rangle}(\chi'(s)a)} \leq \eps
    \eeq
    where $\phi \coloneqq \psi \times \psi'$, and $\chi$ and $\chi'$ are the coarsening (projection) maps such that $\psi = \chi(\phi)$ and $\psi' = \chi'(\phi)$.
\end{definition}

It is easy to see that \Cref{def:q-similarity} can also be used with the estimated state-action-value functions. Let $\h q_{\langle\mu,\psi \rangle}$ and $\h q_{\langle\mu,\psi' \rangle}$ be some $\eps$-close\footnote{We use the same $\eps$ for simplicity. The argument is not affected by choosing different error tolerances for different estimates.} estimates of the optimal state-action-value functions $q^*_{\langle\mu,\psi \rangle}$ and $q^*_{\langle\mu,\psi' \rangle}$ respectively. Then, a $\eps$-Q-similarity based on the estimates implies a $3\eps$-Q-similarity for the optimal state-action-value functions. Formally,
\bqan
\abs{q^*_{\langle\mu,\psi \rangle}(\chi(s)a) - q^*_{\langle\mu,\psi' \rangle}(\chi'(s)a)}
&\leq \abs{\h q_{\langle\mu,\psi \rangle}(\chi(s)a) - \h q_{\langle\mu,\psi' \rangle}(\chi'(s)a)} \\
&\phantom{\leq q^*_{\langle\mu,\psi \rangle}} + \abs{q^*_{\langle\mu,\psi \rangle}(\chi(s)a) - \h q_{\langle\mu,\psi \rangle}(\chi(s)a)} \\
&\phantom{\leq q^*_{\langle\mu,\psi \rangle}(\chi(s)a)} + \abs{q^*_{\langle\mu,\psi' \rangle}(\chi'(s)a) - \h q_{\langle\mu,\psi' \rangle}(\chi'(s)a)} \\
&\leq 3\eps
\eqan
for each $sa$-pair. Therefore, any algorithm using $\psi \approx_\eps \psi'$ based on the estimated state-action-value functions is in ``reality'' evaluating $\psi \approx_{3\eps} \psi'$ in terms of the optimal state-action-value functions.

As defined in \Cref{def:q-similarity}, $\eps$-Q-similarity relation $\approx_\eps$ does not induce a partition on the class of maps $\Psi$. Imagine having an ``$\eps$-cover'' of the class and putting same (partition) label on the maps in a same $\eps$-ball. There will be some maps which are ``$\eps$-close'' to multiple ``partitions''. Therefore, the space is not partitioned. This leads to $\approx_\eps$ not being transitive. This is a serious issue. Many important order relations (built on $\approx_\eps$) rely on this transitive property of $\approx_\eps$. Not all hope is lost. We can recover transitivity by algorithmically partitioning the space in a sequential procedure.

Let $\Psi$ be a countable set of maps that is well-ordered by an arbitrary but fixed index set $\mathcal I$ that is $\Psi = \{\psi_i : i \in {\mathcal I} \}$. Lets assume we want to assign each map to a partition where every map in the partition is $\eps$-Q-similar to each other.
Let $\psi_j$ be an unlabeled map. We go over each (non-empty) labeled partition. We compare it with every map in the partition. If we find a single map \emph{not} $\eps$-Q-isomorphic with $\psi_j$, we try the next labeled partition. If it is not $\eps$-Q-similar with all maps (already labeled) in all non-empty partitions, we put $\psi_j$ in its own empty partition. This procedure works with any $\eps > 0$. This generates a similarity relation (defined below in \Cref{def:q-isomorphism}) slightly finer than $\approx_\eps$, which can be shown to be an isomorphism.
\Cref{alg:q-isomorphism} provides a pseudo-code for this procedure.

\begin{algorithm}
    \caption{Partition Labeling (PartitionLabel$_\eps$)}
    \begin{algorithmic}[1]
        \Input a countable model class $\Psi$, candidate map $\psi$, (estimated) state-action-value function $\h q_{\langle \mu, \psi \rangle}$, error tolerance $\eps$
        \Output partition label of $\psi$
        \State Persistent partition labels $\mathcal{L}$
        \State If $\psi$ is already labeled (in the previous runs) return the label
        \ForAll{labeled partitions $l \in \mathcal{L}$}
        \State Set $\textsf{labled} = \textsf{true}$ \Comment{assuming $l$ will be the label}
        \ForAll{$\psi'$ labeled $l$}
        \If{$\h q_{\langle \mu, \psi\rangle} \not\approx_\eps \h q_{\langle \mu, \psi'\rangle}$}
        \State Set $\textsf{labled} = \textsf{false}$ \Comment{$l$ is not the label}
        \EndIf
        \EndFor
        \If{$\textsf{labled} = \textsf{true}$}\Comment{$\psi$ is $\approx_\eps$ to every map in partition $l$}
        \State Return the label $l$, and save $\h q_{\langle \mu, \psi\rangle}$ along with $l$ for the future runs
        \EndIf
        \EndFor
        \State Pick any label $l_{\rm new} \notin \mathcal{L}$ \Comment{$\psi$ needs a ``new'' partition}
        \State Return the label $l_{\rm new}$, and save $\h q_{\langle \mu, \psi\rangle}$ along with $l_{\rm new}$ for the future runs
    \end{algorithmic}
    \label{alg:q-isomorphism}
\end{algorithm}

For the remainder of this chapter, we will use the following $\eps$-Q-isomorphic relation based on $\approx_\eps$ and the partition labeling method defined above.

\begin{definition}[$\eps$-Q-isomorphism]\label{def:q-isomorphism}
    We define \emph{$\eps$-Q-isomorphism} as
    \beq
    \psi \cong_\eps \psi' :\iff \mathrm{PartitionLabel}_\eps(\psi) = \mathrm{PartitionLabel}_\eps(\psi')
    \eeq
    for any pair of abstraction $\psi$ and $\psi'$.
\end{definition}

Now, similar to the exact case, we use $\eps$-Q-isomorphism to define a binary relation $\preceq_{\eps}$ over the class of maps.

\begin{definition}[$\preceq_{\eps}$]
    Any pair of maps $\psi$ and $\psi'$ are related as
    \begin{numcases}{\psi \preceq_{\eps} \psi' :\iff}
        \textsf{true}, & if $\psi \cong_\eps \psi' \land S_\psi \leq S_{\psi'}$ \label{eq:Rxa}\\
        \true, & if $\psi \ncong_\eps \psi' \land \psi \cong_\eps \phi$ \label{eq:Rxb}\\
        \true, & if $\psi \ncong_\eps \psi' \land \psi \ncong_\eps \phi \ncong_\eps \psi' \land S_\psi \leq S_{\psi'}$ \label{eq:Rxc}\\
        \false, & otherwise
    \end{numcases}
    where $\phi \coloneqq \psi \times \psi'$.
        \label{def:exact-cart-product-rel}
\end{definition}

The above relation is has nearly the same structure as $\preceq_{\times}^*$, except it uses an approximate similarity measure under the hood. So, we can also use this relation in \Cref{alg:leo} with similar arguments as for $\preceq_{\times}^*$. The advantage of this relation is that we can learn the relatively realistic $\eps$-QDP abstractions. We can use this relation with a decreasing sequence of ``error tolerance'' which decreases as our estimation of history-action-value function $Q^*_\mu$ improves through a series of candidate abstraction maps.

So far, we have (non-rigorously) argued that \Cref{alg:leo} converges to an abstraction in $\Psi$, but this is not strictly true. The order(s) may have equivalence classes, the set of maps which are mutually preferred over each other. The algorithm will converge to one of these equivalence classes, and it can not distinguish further. Therefore, it is desired that the equivalence classes contain the maps which lead to a ``similar'' optimal behavior in the original environment.

\begin{definition}[Equivalence Classes {$[\psi]_{\eps}$}]\label{def:eqv-classx}
    The equivalence class of $\preceq_{\eps}$ for any map $\psi$ is defined as
    \bqan
    [\psi]_{\eps}
    :=& \{\psi' : \psi \preceq_{\eps} \psi' \land \psi' \preceq_{\eps} \psi\} \\
    =& \{\psi' : (\psi \cong_\eps \psi' \land S_\psi = S_{\psi'}) \\
    &\phantom{\{\psi' : (\psi}\lor (\psi \ncong_\eps \psi' \land \psi \cong_\eps \phi \cong_\eps \psi') \\
    &\phantom{\{\psi' : (\psi \cong_\eps \psi'}\lor (\psi \ncong_\eps \psi' \land \psi \ncong_\eps \phi \ncong_\eps \psi' \land S_\psi = S_{\psi'})\}\\
    =& \{\psi' : (\psi \ncong_\eps \psi' \land \psi \cong_\eps \phi \cong_\eps \psi') \lor ((\psi \cong_\eps \psi' \lor \psi \ncong_\eps \phi \ncong_\eps \psi') \land (S_\psi = S_{\psi'}))\}
    \eqan
    where $\phi \coloneqq \psi \times \psi'$.

\end{definition}
As evident from the above definition, the equivalence class may contain three distinct types of maps: 1) $\eps$-Q-isomorphic maps with the same number of states, 2) maps which are not $\eps$-Q-isomorphic but they are $\eps$-Q-isomorphic to their Cartesian product map, or 3) maps which are not $\eps$-Q-isomorphic to each other, nor to their Cartesian product map, but have the same number of states. Importantly, the equivalence class of the optimal map has only the first kind of maps:
\beq
[\psi^{*}]_{\eps} = \{\psi : \psi^{*} \cong_\eps \psi \land S_{\psi^{*}}= S_{\psi} \}
\eeq
because for the second type of maps we need $\psi^* \ncong_\eps \psi \land \psi^* \cong_\eps \phi \cong_\eps \psi$, but we know that any refinements of $\psi^*$ have the same state-action-value function, so it cannot be the case. And, we also get a contradiction for the third type of maps because we know $\psi^* \cong_\eps \psi^* \times \psi$ for any $\psi$. Hence, third type of maps can also not be in the equivalence class of $\psi^*$.

So in the case when  $\psi^{*} \in \Psi$, we can use \Cref{alg:leo} with the confidence that if it has converged to the equivalence class of $\psi^{*}$ then the converged abstraction is $\eps$-Q-isomorphic to $\psi^{*}$ with the same number of states.

\begin{theorem}
    Given $\cong_{\eps}$ is transitive, the equivalence class of the optimal map is the minimal element of the order.
    \beq
    [\psi^{*}]_{\eps} \preceq_{\eps} [\psi]_{\eps}
    \eeq
    for all $\psi \in \Psi$, where we abused the notation to indicate the order over the equivalence classes by the same symbol.
\end{theorem}
\begin{proof}
    Let $\psi \in \Psi$ be any map such that $\psi \not\in [\psi^*]_{\eps}$. Therefore, either $\psi^* \ncong_\eps \psi$ or $S_{\psi^*} \neq S_\psi$. For a contradiction, let us assume $\psi \preceq_{\eps} \psi^*$, which can only happen in the following cases:
    \begin{enumerate}
        \item It could hold by Equation~\eqref{eq:Rxa}, i.e.\ $\psi \cong_\eps \psi^* \land S_{\psi} < S_{\psi^*}$. However, this cannot happen because by definition $\psi^*$ is the coarsest $\eps$-QDP abstraction in the class.
        \item By Equation~\eqref{eq:Rxb} we get $\psi \ncong_\eps \psi^* \land \psi \cong_\eps \psi \times \psi^*$. However, we know that any refinement of $\psi^*$ is $\eps$-Q-isomorphic to $\psi^*$. Hence a contradiction (provided $\cong_\eps$ is transitive).
        \item By Equation~\eqref{eq:Rxc} we get $\psi \ncong_\eps \psi^* \land \psi \ncong_\eps \phi \ncong_\eps \psi^* \land S_\psi \leq S_{\psi^*}$. However, this case also cannot happen as $\psi^* \times \psi \eqqcolon \phi \cong_\eps \psi^*$.
    \end{enumerate}
    Hence, we proved the claim.
\end{proof}

If the above conjecture is true, \Cref{alg:leo} may lead to a sound algorithm which selects the optimal map in the limit. However, if the optimal map is not in the class then the there is no reason to believe that the converged equivalence class is a set of some ``meaningful'' maps. From \Cref{def:eqv-classx}, we can see that an arbitrary equivalence class could be a mix of two different types of maps, which may either be $\eps$-Q-uniform or may be arbitrary different apart from having the same number of states.

In the following subsection, we rectify this issue to provide a much more ``powerful'' order, which does not require the optimal abstraction to be in the class. The equivalence classes of this order are collections of ``meaningful'' maps.

\subsection{Ordering Through Cartesian Product Distance}

The requirement that $\psi^*$ is in the model class is a restriction. Usually, we can only have an ``approximation'' of the true abstraction. In this section, we lift this requirement by ordering the maps ``relative'' to their ``distance'' from their Cartesian product map.
We define an asymmetric state-action-value ``distance'' between maps and their corresponding Cartesian product map as follows:

\begin{definition}[Cartesian Product Distance]
    For any pair of maps $\psi$ and $\psi'$ the Cartesian product distance of $\psi$ from $\psi'$ is defined as
    \beq
    d_\psi(\psi') \coloneqq \max_{s \in \S_{\psi \times \psi'}, a \in \A} \abs{q^*_{\langle\mu, \psi\rangle}(\chi(s)a) - q^*_{\langle\mu, \psi \times \psi'\rangle}(sa)}
    \eeq
    where $\chi$ is a coarsening maps such that $\psi = \chi(\psi \times \psi')$.
\end{definition}

We can use this (pseudo) distance to define a ``meaningful'' total order, which may not require $\psi^{*}$ to be in the class to let \Cref{alg:leo} converge to the ``best'' possible QDP abstraction.

\begin{definition}[ $\preceq_\eps^d$]\label{def:cpd-rel}
    Any pair of maps $\psi$ and $\psi'$ are related as
    \begin{numcases}{\psi \preceq_{\eps}^d \psi':\iff}
        \textsf{true}, & if $\psi \cong_\eps \psi' \land S_\psi \leq S_{\psi'}$ \label{eq:cpda}\\
        \true, & if $\psi \ncong_\eps \psi' \land d_{\psi}(\psi') \leq d_{\psi'}(\psi)$ \label{eq:cpdb}\\
        \false, & otherwise
    \end{numcases}
\end{definition}

The relation is simple and self-explanatory. It prefers the abstractions which are either coarser and $\eps$-Q-isomorphic, or have smaller ``distance'' to the action-value-function of the Cartesian product map. So, if the algorithm converges to a minimum element of above order then it will be guaranteed that there is no ``better'' substitute map which can ``improve'' the state-action-value function of the converged map.

\begin{definition}[Equivalence Classes {$[\psi]_{\eps}^d$}]\label{def:eqv-class-eps}
    The equivalence class of $\preceq_{\eps}^d$ for any $\psi$ is defined as
    \bqan
    [\psi]_\eps^d
    :=& \{\psi' \mid \psi \preceq_\eps \psi' \land \psi' \preceq_\eps \psi\} \\
    =& \{\psi'\mid (\psi \cong_\eps \psi' \land S_\psi = S_{\psi'}) \lor (\psi \ncong_\eps \psi' \land d_\psi(\psi') = d_{\psi'}(\psi)) \}
    \eqan
    where $\phi \coloneqq \psi \times \psi'$.
\end{definition}

So, the equivalence classes contain the maps which are either $\eps$-Q-isomorphic maps with the same number of states, or they are mutually at the same ``distance'' from the action-value-function of their Cartesian product maps. We conjecture that the relation is indeed a total order.

\begin{conjecture}
    The relation $\preceq_{\eps}^d$ is a total order (over the equivalence classes).
\end{conjecture}

Using $\preceq_{\eps}^d$ in \Cref{alg:leo} would guide the algorithm to converge to an equivalence class of maps which constitute the ``best'' $\eps$-QDP approximation possible in the class of maps. The resultant (minimal) abstraction will be either an $\eps$-Q-isomorphic member with the same number of states, or it will be part of a class which are on the same mutual ``distance'' from $\psi^*$. This outcome is optimal in the sense that we cannot hope to ``improve'' the state-action-value function further. For example, let the algorithm converge to a ``different'' approximation then these two solutions must be at the same distance from their (mutually refined) Cartesian product map.

\section{Minimum Description Length Abstractions}

An abstraction map should not only model the observed state-action-value dataset accurately, but should also be able to generalize by exploiting structural similarities in the histories. A ``computable'' abstraction map should respect the physical computation constraints. From the computational perspective, a map is deciding (or computing) the state for a history, so the histories which lead to the similar states may have similar structural properties. An abstraction which is able to distinguish ``complex'' structures in the mapped histories may be hard to compute, but it may generalize better. On the other hand, a ``simple'' abstraction could be easier to compute, but it may be of limited use.

\citet{Hutter2009} considered a similar setting for learning $\eps$-MDP abstractions. He called the setup $\phi$MDP to emphasis the use of an abstraction $\phi$ to map histories to a set of Markov states. Our focus is learning non-MDP abstractions. Although our setup in this section is similar in flavor, it differs in terms of our aim to learn $\eps$-QDP abstraction. In $\phi$MDP the model is ``scored'' based on the length of the ``codes'' required to encode the observed reward and state sequences. The codeword for the reward sequence is conditioned on the observed state sequence. So, a model which does a coarser coding of the state sequence, e.g.\ maps every history to a single state, would need a ``complex'' reward encoding, and vice versa. Note that $\phi$MDP is not the traditional ``model+data'' complexity minimization. On the other hand, our approach in this section is directly using MDL principle.

Let an abstraction learning algorithm try to \emph{find} a candidate map from a set of \emph{computable} maps $\Psi$. An \emph{identity} abstraction which distinguishes every history, i.e.\ $\S \cong \H$, is algorithmically the most ``complex'' abstraction to be verified, i.e.\ this is indeed an $\eps$-QDP abstraction. Even though it trivially ``best fits'' the history-action-value-function of the environment, the estimation of its state-action-value function is not possible. On the other extreme, a \emph{single-state} abstraction, which does not distinguish any history, can be verified fairly quickly, but the resultant state-action-value function is an ``aggregate'' (non-Q-uniform mixture) over all history-action-value functions. Therefore, the ``best'' abstraction lies somewhere in between the identity and the single-state abstractions.

Ideally, the ``best'' $\eps$-QDP abstraction $\psi^{*}$ is the coarsest possible partition of the histories which satisfy the history-action-value uniformity condition, see \Cref{def:qdp}. The partitioning function of $\psi^*$ may not a algorithmically ``simple'' function.
%
Therefore, an algorithm that has some form of complexity constraints may need to settle on some ``simpler'' approximation of $\psi^{*}$. Undoubtedly, the algorithm should try to not ``mix'' (or merge) vastly different history-action-value functions. The states of the abstraction should have ``reasonably'' small variation among the constituent history-action-value functions mapped to the same state.

We quantify the ``intra-state'' return variation by the mean square error (MSE) between the \emph{observed} return
\beq
G(h) \coloneqq \sum_{m=1}^{\abs{h}} \g^{m-1} r_m(h)
\eeq
and the predicted state-action-value function as
\beq
\mse_{\psi}(h) \coloneqq \frac{1}{n}\sum_{m=1}^{n} \left(\h q_\psi(\psi(h_{1:m})a_m(h)) - G(h_{m:\abs{h}})\right)^2
\eeq
for any history $h$, where $n \coloneqq \abs{h} - T(\eps)$, $T(\eps)$ is the $\eps$-horizon that indicates the number of time-steps after which the sum of the discount factors is less than $\eps$, and $\abs{h} > T(\eps)$.

However, it may algorithmically be expensive to minimize MSE with respect to $\psi$, as it may require more ``complex'' abstractions to better predict the individual returns. This trade-off between MSE and complexity of the abstraction map can neatly be described by the minimum description length (MDL) principal \cite{Grunwald2007}. The ``best'' MDL abstraction should simultaneously be able to minimize MSE of the estimates and the (algorithmic) complexity of the abstraction.

\begin{algorithm}
    \caption{Abstraction Learning through Minimum Description Length (AL-MDL$_\infty$)}
    \begin{algorithmic}[1]
        \Input (countable) set of Turing machine-based abstractions $\Psi$, trade-off parameter $\alpha$
        \Output QDP abstraction $\h \psi \to \psi^{*}$
        \State Pick the ``simplest'' abstraction $\h\psi \in \Psi$ \Comment{maps every history to the same state}
        \State Initialize the history $h = \epsilon$
        \Repeat \Comment{forever}
        \State Extend the sample trajectory $h$ using \Call{ExpVsExp}{$\langle\text{current state of the algorithm}\rangle$}
        \ForAll{$\psi \in \Psi$}
        \State Calculate $\h q_\psi$ using Q-learning with $h$
        \State Calculate $\mse_{\psi}(h)$ as if the states are mapped by $\psi$
        \EndFor
        \State Select the candidate abstraction $\h \psi = \argmin_{\psi} \left(\mse_\psi(h) + \alpha K_U(\psi)\right)$
        \Until \textbf{false}
    \end{algorithmic}
    \label{alg:al-mdl}
\end{algorithm}

So far, we have not defined what it means for an abstraction to be ``complex''. We may use any set of abstractions which allow for a measure of complexity, i.e.\ we can quantify if an abstraction is more complex than the other. For this work, we us a set of Turing machines (TM) as the set of abstractions \cite{Sipser2013}. The input of the Turning machines is a history and the output is the state of the abstraction, i.e.\ the TM computes $\psi$.

For a fixed universal Turing machine (UTM) $U$, the algorithmic complexity $K_U(\psi)$ for abstraction $\psi$ is the length of the shortest program possible on $U$ which can produce the same output as $\psi$. Formally,
\beq
K_U(\psi) \coloneqq \min \{\abs{p} : \forall h.\ U(p,h) = \psi(h)\}
\eeq

Once the notion of complexity\footnote{It is helpful to consider $K_U(\psi)$ as any ``complexity number'' assigned to $\psi$. We do not require the knowledge of complexity theory to understand the rest of the section. The interested reader should see \citet{Sipser2013} for more details on the computation complexity theory and \citet{Li2013} for algorithmic complexity theory.} is in place, we can use MSE of the estimates and the algorithmic complexity as a quality measure of the abstraction. More complex abstractions may distinguish more complex history patterns, but they may not be preferred if they do not (significantly) reduce MSE of estimated state-action-value functions. \Cref{alg:al-mdl} is a candidate (pseudo-)algorithm which provides the abstraction with the ``best'' MDL trade-off.

Like other algorithms in this chapter, we have delegated the exploration vs. exploitation dilemma to a function call $\rm ExpVsExp$. One possible choice is to use the optimal policy of the candidate abstraction for a fixed duration with occasional sub-optimal actions in between. Although simple, this $\eps$-greedy policy has been shown to be useful in many use cases \cite{Sutton2018}. However, we may also use some directed exploration methods \cite{Lattimore2014a}. The formal analysis of the algorithm will dictate the exact choice of the exploration policy, which we defer to future work.

\section{Summary}

This chapter is a collection of (pseudo-)algorithms, which can be used as a guiding principles for some practical algorithms to learn an abstraction from data. Unlike the rest of the thesis, the components used in this chapter are not rigorously formulated. However, each algorithm is discussed in a way to instigate future research on this subject.


\chapter[Value-uniform Abstractions: Empirical Analysis]{Value-uniform Abstractions: Empirical Analysis {\\ \it \small This chapter is an adaptation of the work I did jointly with \citet{McMahon2019}}}\label{chap:vpdp-exp}

\begin{outline}
    In this chapter, we provide a preliminary empirical analysis about a class of non-MDP abstractions in ARL. We empirically investigate the viability of a weaker notion of $\eps$-VDP abstractions, we call, VA-uniformity. We introduce and empirically test two conjectures on the theoretical performance bounds of VA-uniformity, but ultimately provide evidence that they may not hold.  This indicates that VA-uniformity may not lead to a viable solution to the problem of reducing large state-spaces. Moreover, we ran the same set of experiments for extreme $\eps$-VADP abstractions. We did not find any counter-example to negate \Cref{conj:vpdp}.
\end{outline}

\section{Introduction}

Recall that we can use (extreme) $\eps$-QDP abstractions to solve large decision problems in general, even if the true aggregated process is not an MDP.  \citet{Hutter2016} has a similar result when the aggregating histories have approximately equal history-value functions:

\begin{theorem}[VA-uniform Aggregation {\cite[Theorem 9]{Hutter2016}}]
    \label{theorem:v-agg}
    For any $\eps$-VDP abstraction $\psi$ of an environment $\mu$, dispersion distribution $B$ and the corresponding surrogate MDP $\b \mu$ the following holds:
    \begin{enumerate}
        \item $|V_\mu^*(h) - v_\mu^*(s)| \leq \frac{3\epsilon}{(1-\gamma)^2}$
        \item $|\b Q_\mu^*(sa) - q_\mu^*(sa)| \leq \frac{3\epsilon \gamma}{(1-\gamma)^2}$
        \item if $\eps = 0$ then $\pi^*(h) = \pi^*(s)$
    \end{enumerate}
    for all actions $a$, history $h$ and state $s=\psi(h)$, where $\b Q_\mu^*(sa) \coloneqq \sum_{h \in \H} B(h\|sa)Q^*_\mu(ha)$.
\end{theorem}

Similar to $\eps$-QDP abstractions, \Cref{theorem:v-agg} means that if we have a feature map $\psi$ that aggregates those histories with approximately equal values, and which maps to the same action under \emph{an}\footnote{This is an important distinction between $\eps$-VPDP and VA-uniformity. An $\eps$-VPDP abstraction aggregates histories which have same near-optimal actions.} optimal policy $\pi^*$,  then we can construct a smaller surrogate MDP $\b \mu$ which has approximately equal state-action-value function to the original environment $\mu$. However, \Cref{theorem:v-agg} has one crucial weakness; item 2 of the theorem relates $\b Q^*_\mu$ instead of $Q^*_\mu$ with $q^*_\mu$. Since the surrogate MDP is optimizing $q^*_\mu$, it is possible to ``cheat'' the surrogate MDP by a ``tuned'' $B$ distribution which helps make a non-optimal action look optimal on $B$-average. The agent will choose that ``bad'' action and could suffer huge loss on histories with low $B$-probability. That is why, we do not have equivalent results which bound the difference between the optimal history-value function $V^*_\mu$ and the history-value function $V^{\u\pi}_\mu$ under the uplifted policy of the surrogate MDP $\u \pi$ \cite[Theorem 10]{Hutter2016}.  This lack is somewhat concerning.  It is not particularly useful if the surrogate has approximately equal optimal state-action-value function if we are unable to learn optimal behavior through it.

On the other hand, \citet{Hutter2016} showed that VA-aggregation had one major advantage over $\eps$-QDP abstractions: when constructing the surrogate MDP $\b\mu$ using an $\eps$-QDP abstraction, the size of the MDP is bounded uniformly as
\beq
S \leq \left( \frac{3}{\eps (1-\gamma)^3}   \right)^{A}
\eeq
which is exponential in the size of the action-space $A$. Whereas, if we are able to learn the optimal policy using VA-aggregation, \citet{Hutter2016} showed that the size of the surrogate MDP would only have size linear in the action-space.  This is a significant difference, and thus we discuss potential ways towards using VA-aggregation in the following section.

\section{Conjecture on VA-aggregation}

\citet{Hutter2016} showed that there might exist environments where VA-aggregations may lead to arbitrarily worse policies.

\begin{theorem}[VA-aggregation Not Useful {\cite[Theorem 10]{Hutter2016}}]
    For any $\g$ and any (arbitrarily large) $C$, there exist $\mu$, $\psi$, and $B$ with $\pi^*(h) = \pi^*(\d h)$ and $\abs{V_\mu^*(h) - V_\mu^*(\d h)} \leq \eps$ for all $\psi(h) = \psi(\d h)$ such that for the corresponding surrogate MDP $\b\mu$ and $\u\pi(h) \coloneqq \pi^*(s)$ for $s = \psi(h)$, we have $V_\mu^*(h) - V_\mu^{\u\pi}(h) \geq C$.
\end{theorem}

However, a corollary of \autoref{theorem:v-agg} provides a possible way to salvage VA-aggregation.

\begin{corollary}
    \label{corollary: v-agg-bound}
    For any $\mu$, $\eps$-VDP abstraction $\psi$ and $B$ the following holds:
    \beqn
    Q_\mu^*(h\pi^*(s)) \leq V_\mu^*(h) \leq \sum_{\d h \in \H} B(\d h\| s \pi^*(s)) Q_\mu^*(\d h\pi^*(s)) + \frac{3\eps}{(1-\gamma)^2}
    \eeqn
    for any state $s$ and history $h$ such that $s = \psi(h)$.
\end{corollary}
\begin{proof}
    The proof can be found in the main text of \citet{Hutter2016} as the item $(4)$ under the paragraph titled ``Missing Bounds (ii) in Theorem 9''.
\end{proof}
This shows that $Q_\mu^*(h \pi^*(s))$ is a lower bound on its own expectation.  Holding $V_\mu^*(h)$ constant, as $Q_\mu^*(h\pi^*(s))$ gets more and more negative\footnote{To get the main message of this argument across, we (temporarily) allow the rewards to assume any real-value. A similar argument can be rendered for bounded rewards, but it would make the discussion unnecessarily involved.} $B(\d h \|s \pi^*(s))$ must approach zero, or else it would violate the upper bound on $V_\mu^*(h)$.  This means that as the Q-value of a given history gets smaller, the probability of being in that history (as captured by the $B$ distribution) approaches zero, see \citet{Johnston2015} for an example.
But, we also know that
\beq
V_\mu^{\u\pi}(h) = Q_\mu^{\u\pi}(h\pi^*(s))\leq Q_\mu^*(h \pi^*(s))
\eeq
for the uplifted policy $\u \pi$, which may not be optimal.  Thus, as $Q_\mu^*(h \pi^*(s))$ gets more and more negative, if the reward is allowed to be unbounded, the difference between the surrogate and the optimal values grows greater and greater. That is
\beq
\left( Q_\mu^*(h \pi^*(s)) \rightarrow - \infty \right) \implies \left( \abs{ V_\mu^*(h) - V_\mu^{\u\pi}(h) } \rightarrow \infty \right)
\eeq

Given this, then perhaps in general there is some relationship between the value difference and the dispersion probability.
\beq
\abs{V_\mu^*(h) - V_\mu^{\u\pi}(h) } \stackrel{?}{\propto} \frac{1}{B(h \| \psi(h) \u\pi(h))}
\eeq

Perhaps the value difference is inversely proportional to the probability of being in a given history.  If this is true for all histories, then large value differences only arise for the histories that occur with low probability.  We pose a conjecture based on the expected value of the value difference.

\begin{conjecture}[Expected Optimality of VA-aggregations]
    \label{conjecture:v-agg-expectation}
    For any environment $\mu$, $\eps$-VDP abstraction $\psi$ and $B$ the following holds:
    \beq
    \sum_{h} B(h\|s\pi^*(s)) \abs{V_\mu^*(h) - V_\mu^{\u\pi}(h)} \stackrel{?}{=} O \left( \frac{\eps}{(1-\gamma)^?} \right)
    \eeq
    for any state $s$, where $\u \pi(h) := \pi^*(s)$.
\end{conjecture}

This conjecture poses the idea that the difference between the optimal values and the values under the learned policy is bounded in expectation, rather than directly bounded as in the $\eps$-QDP case.  It could be the case that the value differences are bounded in the states that occur with high probability, and only fail in some small number of cases in those states that occur with very low probability.  If we take the expectation of the value difference, then larger value differences will be offset by the small probabilities in which those states occur, which would lead to some overall bound of the expected value.

Interestingly, if \Cref{conjecture:v-agg-expectation} were true, then we may get an \emph{almost surely} result. The agent using $\eps$-VDP abstraction can be sure that it has bounded performance loss on the histories which are actually being generated. In case the underlying process is an MDP, this result implies that the agent achieves the optimal value for all underlying states which have higher ``belief''. Moreover, if $B$ is being generated by an exploratory policy, which is usually the case, then a sufficiently exploratory policy makes sure that the ``belief'' is reflective of the true realization. That is, the agent will believe more in the underlying state which occur more. Therefore, the agent can not be ``cheated'' into believing a ``counter-factual'' distribution over the underlying state-space.

\Cref{conjecture:v-agg-expectation} proposed that value aggregation only failed in infrequently occurring micro states. However, maybe this would only work if instead we considered the macro states --- the aggregated states in the surrogate MDP. We conjecture (in case the underlying process is an ergodic MDP) that value aggregation succeeds in the majority of aggregated states, but only fails in low-probability aggregated states.

\begin{conjecture}[VA-aggregation Expectation in Macro States]
    \label{conjecture:v-agg-expectation_macro}
    For any ergodic finite-state MDP $\mu$, $\eps$-VDP abstraction $\psi$ and $B$ generated by a stationary distribution $\rho_\mu^{\u\pi} \in \Dist(\OR)$ under the uplifted policy $\u\pi$ the following holds:
    \beq
    \sum_{e \in \OR} \rho_\mu^{\u\pi}(e) \abs{V_\mu^*(e) - V_\mu^{\u\pi}(e)} \stackrel{?}{=} O \left( \frac{\eps}{(1-\gamma)^?} \right)
    \eeq
\end{conjecture}

Although the bound in the above conjecture may seem weak, if it holds then we can consider this to be the ``weakest'' form of optimality guarantee an abstraction map can provide.

Now, we experimentally investigate these conjectures in the following sections.

\section{Experimental Setup}
\label{chap:v-agg}

The setup is quite simple. We follow the approach of \citet{Johnston2015} to generate random ergodic MDPs with $X \coloneqq \abs{\OR}$ number of states and $A$ number of actions. It is important to mention that if we sample an MDP naively for the set of all possible MDPs with state-action space $\OR \times \A$, then with probability 1 we will get a state-action-value function that cannot be aggregated. On the other hand, \citet{Johnston2015} devised a simple way to generate aggregatable MDPs. He starts by first randomly generating an abstract action-values and later add non-optimal action-values to get the underlying MDP.

Once we have generated an MDP by the above method, we then aggregate it into $S_\psi \coloneqq \abs{\S_\psi}$ number of states which respect the VA-uniformity condition. We define the dispersion probability $B$ in terms of the stationary distribution $\rho_\mu^{\pi^*}$ under the optimal policy.  This stationary distribution describes the average time spent in a given state when following policy $\pi^*$ --- exactly a measure of the probability of being in a state.  To calculate this, we calculate the stationary distribution $\rho_\mu^a$ for each action $a$, and use it to construct the dispersion probability.
\beq
B(e \| \psi(e)a) := \frac{\rho_\mu^a(e)}{\sum_{\d e: \psi(e) = \psi(\d e)}\rho_\mu^a(\d e)}
\eeq
where the most recent observation $e$ is the state of the underlying MDP, and abstraction $\psi$ is assumed to be a map from $\OR$ to $\S_\psi$.
We then use this to learn $\u\pi(e):= \pi^*(\psi(e))$, and use this as an approximation to the optimal policy.  Specifically, we will test to see if there exists a linear relationship between  $\abs{V_\mu^* - V_\mu^{\u\pi}}$ and $\rho_\mu^{\u\pi}$. For numerical and presentational reasons, we use $-\log_2 \rho_\mu^{\u\pi}$ in the reported results instead of $\rho_\mu^{\u\pi}$.

Note that while we are interested in the general case concerning histories, we restrict ourselves to MDPs in this chapter.  This is for simplicity, as it is much simpler to randomly generate an MDP that is able to be value aggregated. The complete details about the setup can be found in \citet{McMahon2019}.

\section{Empirical Investigation of \Cref{conjecture:v-agg-expectation}}

We generated 1000 MDPs with different sets of parameters, e.g.\ aggregation factor, noise value and branching factor. Then, we performed VA-aggregation and learned the policy $\u\pi$.  As we are interested in looking for cases when $\abs{V_\mu^{*} - V_\mu^{\u\pi}} > 0$, if the learned policy was equal to the optimal policy the MDP was discarded and a new one was generated.  To test the linear correlation between $ \abs{V_\mu^{*} - V_\mu^{\u\pi}}$ and $- \log_2 \rho_\mu^{\u\pi}$, we calculated the Pearson correlation coefficient (PCC) \cite{Benesty2009},  and their associated $p$-values.  All of the MDPs generated had 64 states, 2 actions and a branching factor (i.e.\ maximum number of reachable states from any other state) of 4.

For the first tests, the aggregation value was set at 4, while the noise factor $\eps$ took values from $[1,5,10,15,20]$, see \autoref{tab:noise_tests}. For the second batch of experiments, the noise value was set at 5, and the aggregation factor $X/S_\psi$ took values in $[2,4,16,32]$, see \autoref{tab:agg_tests}.

\begin{table}
    \centering
    \begin{tabular}{ | c|| c | c | c | c | c |  }
        \hline
        \multicolumn{6}{|c|}{Noise Values} \\
        \hline
        $\eps$  		& 1 		& 5 		& 10		& 15		& 20 \\
        \hline
        \hline
        PCC  		& 0.00882 	& 0.00221 	& 0.00414	& -0.00202	& -0.00540 \\
        p-value 	& 0.02553 	& 0.57515 	& 0.29461	& 0.60883	& 0.17223  \\
        \hline
    \end{tabular}
    \caption[Pearson Correlation Coefficients for Varying Noise Values]{The Pearson correlation coefficients and p-values for varying noise values $\eps$.  All MDPs had 64 states, 2 actions, branching factor 4 and aggregation factor 4.}
    \label{tab:noise_tests}
\end{table}

\begin{table}
    \centering
    \begin{tabular}{ | c || c | c | c | c | }
        \hline
        \multicolumn{5}{|c|}{Aggregation Factors} \\
        \hline
        $X/S_\psi$  	& 2 		& 4 		& 16		& 32 \\
        \hline
        \hline
        PCC  					& 0.00577 	& 0.00554 	& 0.00461	& 0.00495 \\
        p-value 				& 0.14430 	& 0.16073 	& 0.24380	& 0.21020   \\
        \hline
    \end{tabular}
    \caption[Pearson Correlation Coefficients for Varying Aggregation Factors]{The Pearson correlation coefficients and p-values for aggregation factors $X/S_\psi$.  All MDPs had 64 states, 2 actions, branching factor 4 and noise value 1.}
    \label{tab:agg_tests}
\end{table}

Despite enforcing policy uniformity, there seems to be no linear relationship between $\abs{ V_\mu^{*} - V_\mu^{\u\pi}} $ and $- \log_2 \rho_\mu^{\u\pi}$.  The Pearson correlation coefficients are very close to 0, and the $p$-values indicate there is no reason to reject the null hypothesis that they are linearly independent. We observe in \autoref{fig:agg_16} two dense clusters of results.  The first cluster on the left are the high to mid probability events.  There is also a thin cluster of low probability events in the right section of the plot.


\begin{figure}[htb]
    \centering
    \begin{tikzpicture}
        \node (img)  {\includegraphics[width=0.8\linewidth]{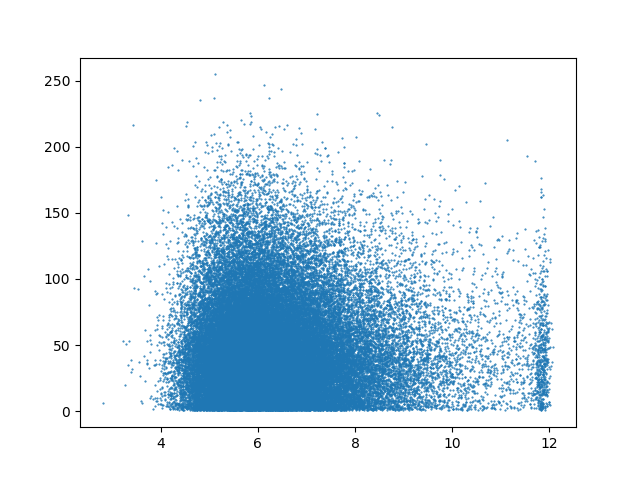}};
        \node[below=of img, node distance=0cm, yshift=1.5cm] {$-\log_2 \rho_\mu^{\u\pi}(e)$};
        \node[left=of img, node distance=0cm, rotate=90, anchor=center,yshift=-0.35cm] {$\abs{V^*_\mu(e) - V^{\u\pi}_\mu(e)}$};
    \end{tikzpicture}
    \caption[Value differences against $-\log_2 \rho^{\u\pi}_\mu$]{The results after generating 1000 MDPs and plotting the two variables.  Clearly they do not have a linear relationship.  The MDPs were generated with 2 actions, 64 states, aggregation factor 16, branching factor 4, noise 5 and $\delta = 5 \cdot 10^{-6}$.}
    \label{fig:agg_16}
\end{figure}

The cluster of results on the right of \autoref{fig:agg_16} is a by-product of how we ensure the MDP is quasi-positive  \cite{McMahon2019}. We add a small $\delta > 0$ to the transition probability of every state-action to get an ergodic transition matrix after normalization where every state is reachable under every policy. Therefore, a large number of (otherwise unreachable) states have very similar (and very small) values of $\rho^a_\mu$ for every action $a$.  This causes the clustering observed in the results.  Indeed, as we decrease  $\delta$, the cluster is pushed further and further to the right, as seen in \autoref{fig:small_epsilon}. If we remove these values and consider the results, see \autoref{fig:sanitised_epsilon}, we still see no evidence of a linear relationship between the two variables.


\begin{figure}[htb]
    \centering
    \begin{tikzpicture}
        \node (img)  {\includegraphics[width=0.8\linewidth]{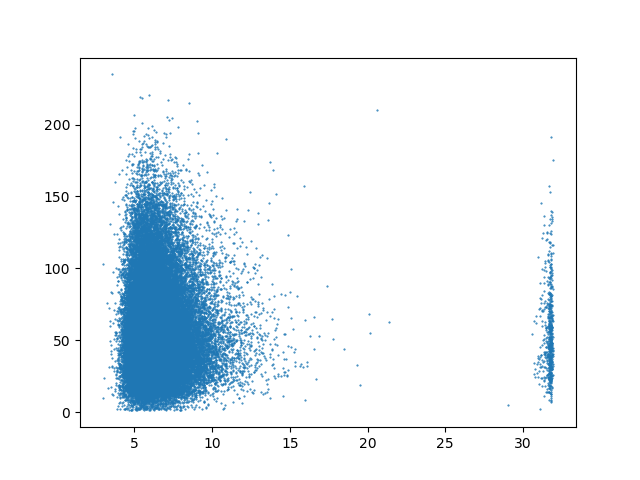}};
        \node[below=of img, node distance=0cm, yshift=1.5cm] {$-\log_2 \rho_\mu^{\u\pi}(e)$};
        \node[left=of img, node distance=0cm, rotate=90, anchor=center,yshift=-0.35cm] {$\abs{V^*_\mu(e) - V^{\u\pi}_\mu(e)}$};
    \end{tikzpicture}
    \caption[Results with small $\delta$ value.]{The results when $\delta$ is set to $5 \cdot 10^{-10}$.  The cluster of low probability events gets pushed to the right.}
    \label{fig:small_epsilon}
\end{figure}

\begin{figure}[htb]
    \centering
    \begin{tikzpicture}
        \node (img)  {\includegraphics[width=0.8\linewidth]{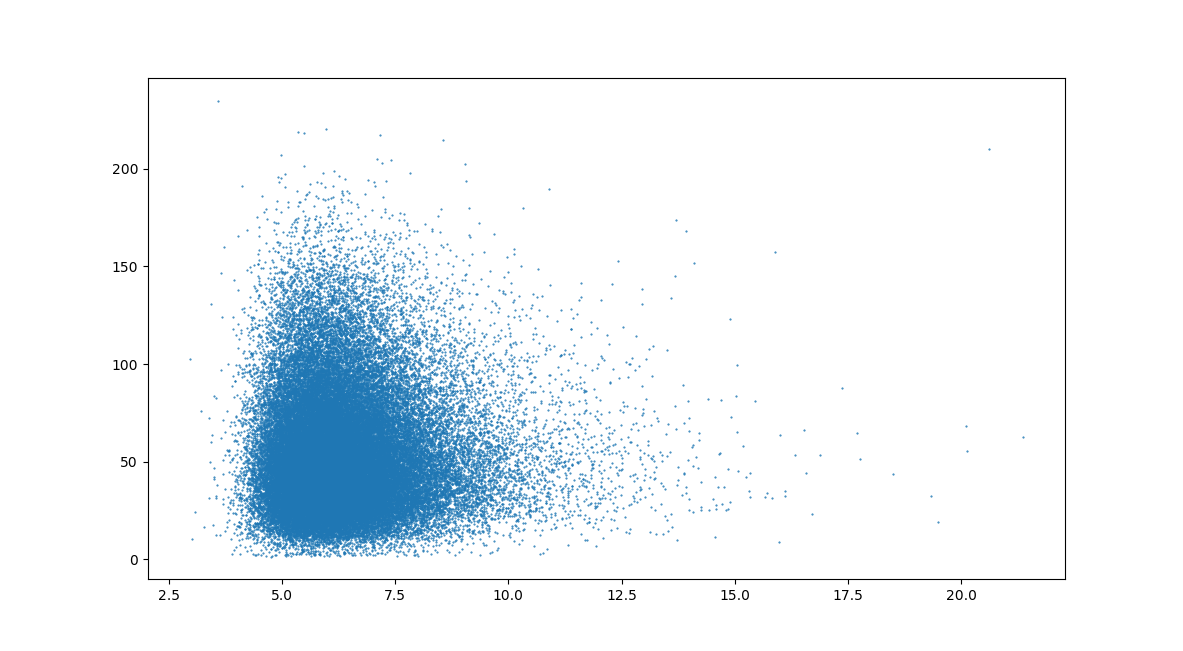}};
        \node[below=of img, node distance=0cm, yshift=1.4cm] {$-\log_2 \rho_\mu^{\u\pi}(e)$};
        \node[left=of img, node distance=0cm, rotate=90, anchor=center,yshift=-0.7cm] {$\abs{V^*_\mu(e) - V^{\u\pi}_\mu(e)}$};
    \end{tikzpicture}
    \caption[Results with the $\delta$ results filtered out]{The results when we exclude the low-probability events based $\delta$ is set to $5 \cdot 10^{-10}$.  There is no evidence of a linear relationship.}
    \label{fig:sanitised_epsilon}
\end{figure}


When generating results, we only considered those MDPs that failed to learn the optimal policy. The values for those MDPs that did learn the correct policy are, trivially, the same.  As the number of states increases, fewer MDPs learn the optimal policy.

If the bounds in Conjecture \ref{conjecture:v-agg-expectation} held, we would expect a linear relationship between the value difference and the stationary distribution.  The states occurring with high probability (on the left of the graph) would have a small value difference, and the states with low probability states (those tending towards the right of the graph) would generally have large value differences.  But we have no evidence of a linear relationship between the two variables, which makes it extremely unlikely that the conjecture is true.  It could be that the relationship between the two variables is more complex, or it could be that they are not correlated at all.

These results indicate that Conjecture \ref{conjecture:v-agg-expectation} is false: that the difference between the true and learned values is not bounded in expectation.  This negative experimental result indicates it is likely not possible to learn optimal behavior in general when using VA-aggregation.  In the next section we test to see if value difference is bounded in expectation over the macro states of the environment.

\section{Empirical Investigation of \Cref{conjecture:v-agg-expectation_macro}}

Using a similar framework as in the previous section, we empirically test this conjecture using random MDPs which can be value aggregated.
We calculate the expected value difference for a given micro state $s$
\beq
\sum_e B(e \| s\pi^*(s)) \abs{V_\mu^*(e) - V_\mu^{\u\pi}(e)}
\eeq
and then averaged with respect to the probability of being in a specific state.  This is represented by the stationary distribution $\rho^{\u\pi}_\mu(s) \coloneqq \sum_{e:\psi(e) = s} \rho_\mu^{\u\pi}(e)$ for each state $s$ of the surrogate MDP $\b\mu$, which is calculated by taking the left eigenvector of the transition matrix \cite{McMahon2019}.  For an MDP, we thus calculate

\bqan
\sum_{s} \rho^{\u\pi}_\mu(s) \sum_{e} B(e \| s\pi^*(s)) \abs{V_\mu^*(e) - V_\mu^{\u\pi}(e)}
&= \sum_{s} \rho^{\u\pi}_\mu(s) \sum_{e:\psi(e) = s} \fracp{\rho_\mu^{\u\pi}(e)}{\rho_\mu^{\u\pi}(s)} \abs{V_\mu^*(e) - V_\mu^{\u\pi}(e)} \\
&= \sum_{e} \rho_\mu^{\u\pi}(e)\abs{V_\mu^*(e) - V_\mu^{\u\pi}(e)} \numberthis
\eqan

We use the same experimental procedure as for the previous conjecture to generate MDPs and learn the values.  Then we calculate the stationary distribution, and finally calculate the normalized expected value difference for each MDP.

We generated 1000 random MDPs with 64 states, aggregation factor 4, 2 actions, branching factor 4 and noise value 1.  \autoref{fig:hist_expected} shows a histogram of the resulting scores for the 1000 MDPs. The noise having value 1 means that all of the micro states that are to be aggregated have maximum value difference of 2 --- i.e. $\abs{V_\mu^*(e) - V_\mu^*(\d e)} \leq 2$.  If our conjecture were true, we would expect the normalized expected value difference to be of the order $2/(1-\gamma)^c$ for some $c > 1$. It may feel like the case in \autoref{fig:hist_expected}, as for $\g=0.9$, any value of $c > 1$ leads to a bound greater than $200$. But, this (misleading) trend does not stand the next set of parameters.  The majority of MDPs have a normalized expected value difference exceeding 10. Even when we set the noise to 0, which means that all micro states that aggregate to the same macro state have exactly the same value,therefore, we do not observe any such bound, see \autoref{fig:hist_expected0}.

\begin{figure}[htb]
    \centering
    \begin{tikzpicture}
        \node (img)  {\includegraphics[width=0.8\linewidth]{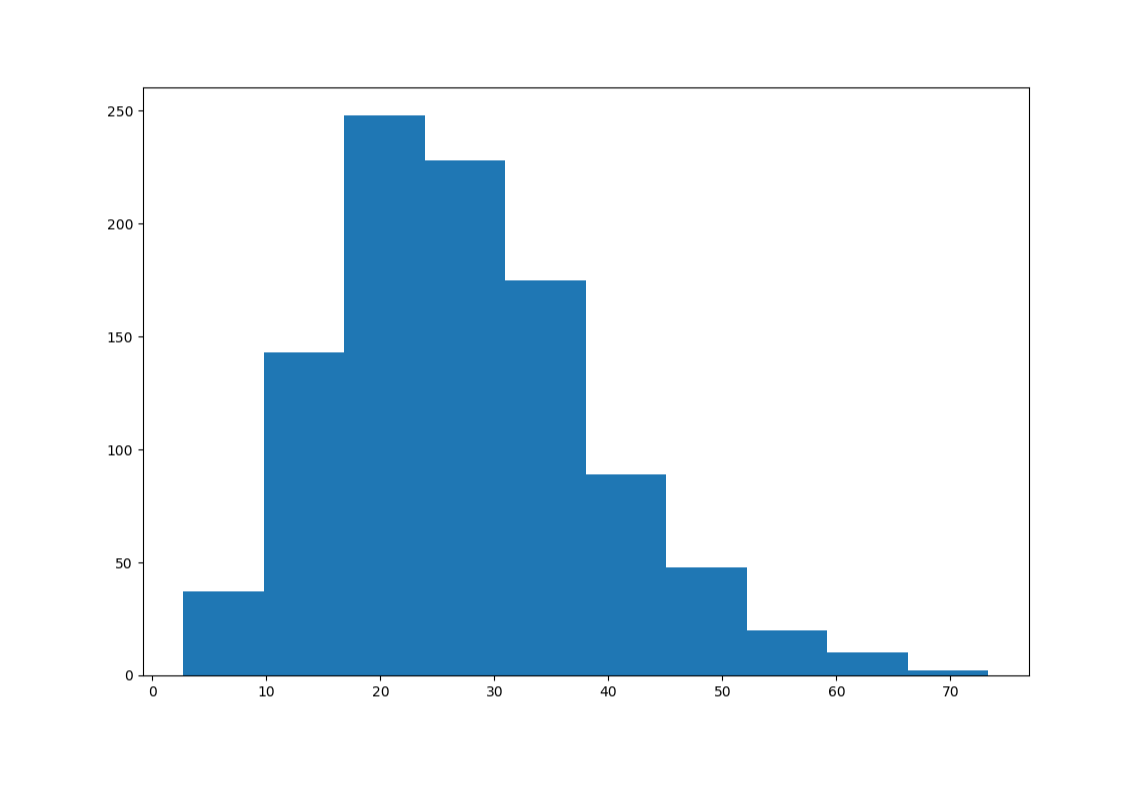}};
        \node[below=of img, node distance=0cm, yshift=2cm] {$\sum_{e} \rho^{\u\pi}_\mu(e)\abs{V^*_\mu(e) - V^{\u\pi}_\mu(e)}$};
        \node[left=of img, node distance=0cm, rotate=90, anchor=center,yshift=-0.7cm] {Frequency};
    \end{tikzpicture}
    \caption[Histogram of normalized expected value difference.]{A histogram of the normalized expected value difference calculated for 1000 randomly generated MDPs.  Each MDP was randomly generated with 64 states, aggregation factor 4, 2 actions, branching factor 4 and noise value 1.}
    \label{fig:hist_expected}
\end{figure}

\begin{figure}[htb]
    \centering
    \begin{tikzpicture}
        \node (img)  {\includegraphics[width=0.8\linewidth]{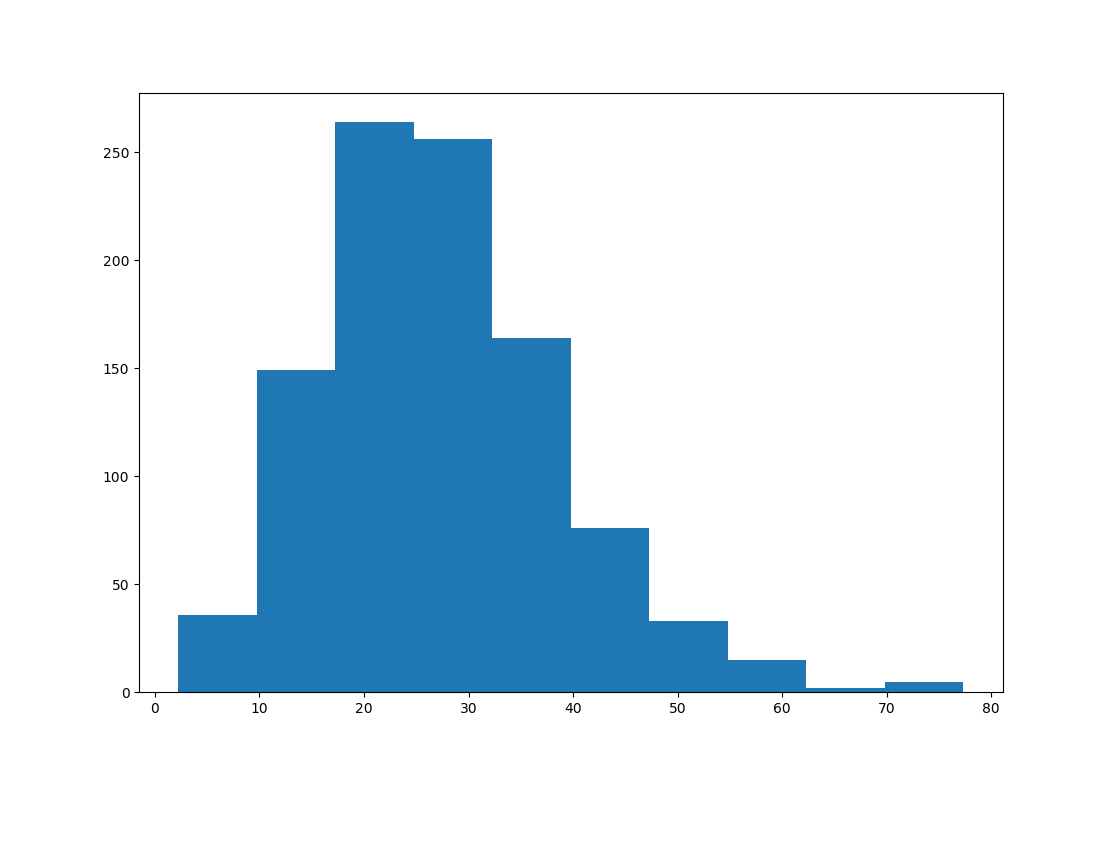}};
        \node[below=of img, node distance=0cm, yshift=2.3cm] {$\sum_{e} \rho^{\u\pi}_\mu(e)\abs{V^*_\mu(e) - V^{\u\pi}_\mu(e)}$};
        \node[left=of img, node distance=0cm, rotate=90, anchor=center,yshift=-0.7cm] {Frequency};
    \end{tikzpicture}
    \caption[Histogram of normalized expected value difference with zero noise.]{A histogram of the normalized expected value difference when the micro states all have the same values.  We randomly generated 1000 MDPs, each with 64 states, aggregation factor 4, 2 actions, branching factor 4 and noise value 0.}
    \label{fig:hist_expected0}
\end{figure}

%
%

As the expectations are not bounded, these empirical results indicate that Conjecture \ref{conjecture:v-agg-expectation_macro} does not hold.  Value aggregation still fails to learn optimal policies even when we weaken the learning condition to this extent.  This is a disappointing result.  If the difference between the optimal and learned values is not even bounded in expectation of the macro states, it seems very unlikely that we will be able to use value aggregation to extract any useful information or learn any optimal behavior.

\section{Empirical Investigation of \Cref{conj:vpdp}}

We repeat the same set of experiments to test \Cref{conj:vpdp}. Contrary to the above findings, this time we did not find any counter-example. Every surrogate MDP of $\eps$-VPDP abstraction was able to learn the optimal policy of the underlying MDP. This empirical backing increases our hope that this conjecture might be true. $\eps$-VPDP abstractions may turn out to be extremely ``useful'' after all. This asks for a formal inquiry.

\section{Conclusion}

We highlighted the missing theoretical results that stopped VA-aggregation from provably being able to learn optimal policies.  In an effort to fill that gap we introduced two conjectures that if true, would provide performance guarantees when learning in VA-aggregated environments.  But, we found no evidence that learning over VA-aggregated environments is bounded in expectation in the micro states.  We conjectured that perhaps a weaker result held instead --- that value aggregation is bounded in expectation over the macro states, but found no evidence that this was the case either. However, we did not find any counter-example to \Cref{conj:vpdp}.

VA-aggregation would be useful when using the extreme aggregation framework, as the surrogate MDP only requires size linear in the action-space of the original process, as opposed to the exponential required via $\eps$-QDP abstractions.  These results suggest that value aggregation will never be able to be used to solve such problems.  However, we do have performance guarantees for $\eps$-QDP abstractions.  Despite the larger  size required for construction of the surrogate MDP, these techniques are a promising way of reducing and solving domains with large state-spaces. Also with the binarization technique from \Cref{chap:action-seq} provides a significant improvement over the size of the state-space of the surrogate MDPs. However, the choice of how to \emph{meaningfully} binarize the action-space is not clear. $\eps$-VPDP has (and VA-aggregation had) potential to provide a better bound over the number of states (maybe) without binarization.

\chapter{Conclusion \& Outlook}\label{chap:conclusion}

\begin{outline}
    This chapter summarizes, hypothesizes and concludes the thesis. We start by providing a brief summary of the thesis. As with any research project, there are many questions left unanswered in the thesis. We go over some of the major future directions possible of this work. Especially, we consider the possibility of a ``mortal'' safe ARL agent, ARL with sparse reward, hierarchical ARL, and an ARL setup which finds only limit or average optimal policies. In the end, we conclude the thesis.
\end{outline}

\section{Thesis Summary}

The core of the thesis revolved around ARL with non-MDP abstractions. The $\eps$-QDP abstractions are put in the center stage in this work. As ARL is a GRL framework with abstraction, there is a potential that ARL setup can be used to design scalable generally intelligent agents. We proved a number of useful properties of $\eps$-QDP abstractions, and conjectured that similar properties might hold for $\eps$-VPDP abstractions. We provided convergence guarantees for TD-like algorithms in non-MDP domains, representation guarantees for non-MDP homomorphisms and significantly improved the upper bound on the number of states in ESA by sequentializing the action-space.
This thesis has mostly discussed the existence and usefulness of some non-MDP abstractions. We have surveyed a number of techniques which may lead to a sound (non-MDP) abstraction learning algorithm in \Cref{chap:abs-learning}.
Lastly, we empirically refuted the usefulness of VA-aggregations, which justifies the focus on $\eps$-QDP and $\eps$-VPDP abstractions in the thesis.

\section{Future Research Directions}

The problem of abstraction learning has been thoroughly discussed in \Cref{chap:abs-learning}. It is a high priority future research task to find a sound abstraction learning algorithm. Moreover, we have made a few assumptions in this thesis, which, if they do not hold in the domain of interest, may lead to some deep philosophical problems about decision-making under uncertainty.  In the following subsections, we discuss the consequences of the failure of some of these assumptions and point to future research directions where possible.

Especially, we consider the possibility of extension of ARL framework to design \emph{mortal} and {safe} AGI agent. Once extended, the ARL setup can help AGI agents to safely explore the environment. We also discuss the case of sparse rewards. This helps the designers, e.g.\ us, to naturally specify goals for ARL agents. We initiate the idea of a ``hierarchical'' ARL framework, where the abstraction could be a composition of many ``finer'' abstractions. Moreover, we speculate about a possible extension by relaxing the requirement of the uplifted policy from being a uniformly value-maximizing policy, i.e.\ $\u\pi \in \Pi^{\sup}_\eps$, to only being an asymptotically value-maximizing or pseudo-regret minimizing policy, i.e.\  $\u\pi \in \Pi^{\infty}_\eps$ or $\u\pi \in \b \Pi^{\infty}_\eps$, to further reduce the required number of states of a surrogate MDP. In the following subsections we go through each possible extension in details.

\subsection{Mortal Safe Agents}

We assumed that the agent-environment interaction never halts, i.e.\ the agents are \emph{eternal}. In reality, it may be more natural to model situations where this interaction may terminate. For example, the agent may be decommissioned or stopped for repairs/updates.  From the agent's perspective the world has stopped. The agent goes into a ``suspended-state''. It is a similar situation as a human goes through a coma, however, when woken up, the human can make sense of the updated situation based on the observations it (had) received before and after the coma. The human agent can conclude that they have ``gaps'' in their experience.
If an eternal agent has been ``suspended'' enough times in the past, just like a human recovering from coma, it can also argue about the experience ``gaps'' solely based on the interaction history.

However, the past experience is not sufficient for an eternal agent to reason about a ``death-state'' where the agent is permanently terminated and is never turned back on. ``Death'' does not exist for such agents. Formally, the agents only experience ``suspensions''. Therefore, it is not hard to imagine some real-life situations where the eternal agents may not behave ``safely'' to avert a catastrophic outcome. From the agent's perspective, it has simply a ``jump'' in the experience. Of course we (or the environment) can dispatch an extremely low reward afterwards to teach the agent to avoid such situations in the future. But, it could be an expensive exercise. Therefore, a framework of mortal agents which can argue about ``death-states'' is essential to design \emph{safe} AGI.\footnote{A safe AGI is a powerful general-purpose intelligence which is provably not an \emph{existential threat} to us.}

In a ``mortal'' RL setup, there may be a termination percept, say $e_\bot \in \OR$. which the environment (or model) dispatches forever under any policy. However, for the agents to argue about such a ``termination'' situations they must have a sufficient belief about the termination or have visited such situations in the past to ``remember'' it through the interaction history. For more discussion on this topic we recommend the reader to see \citet{Martin2016}.

It would be interesting to see whether we can formulate an analogous ARL framework for \emph{mortal} (and potentially safe) agents, and how they may behave under these abstractions. Our ARL setup can help such mortal agents to ``safely'' explore the environment by leveraging an abstraction which distinguishes between ``safe'' and ``unsafe'' states.

\subsection{Sparse Rewards}

In this thesis, we have assumed that the environment produces a reward at every instance of time. However, there are many natural situations where no meaningful reward signal can be produced for the current interaction cycle. How can one model such situations?
\begin{itemize}
    \item The environment (model) dispatches reward zero for such situations. However, this may deter the agent from getting into these portions of experience, which is not the ideal behavior under many circumstances. In some circumstances the agent can find a bug in specifications and loop around some rewarding situation \cite{Everitt2018}. So, it is sometimes more logical to not give any ``default'' reward if we, the designers of the environment models, do not know what should be an appropriate reward at that instance of time.
    \item The environment dispatches no reward, but the agent can input a proxy reward. This may be the current estimated ``value'' of the situation. Once the agent has an appropriate notion\footnote{Note that the definition of ``value'' in a sparse reward setting may not be the same as the history-value function defined in the thesis. In the sparse reward setting, it would require special care to handle ``no-reward'' time-steps.} of ``value'' in a sparse reward setting, it can work ``backwards'' from a rewarded time-step to ``no-reward'' time instances.
    \item A ``no-reward'' time-step might be an ideal position for the agent to add intrinsic motivation into the reward signal coming from the environment. Since the environment is not sure how it should respond to this ``no-reward'' situation, it seems like an intermediate situation for a ``good'' or ``bad'' upcoming state. It is critical to note that an always low reward situation, e.g. navigating through a grid-world with minimum reward at each step, is different than receiving no reward. In the former case the agent is driven to move away from these states as soon as possible, whereas, in the later case the agent is left clueless. It has to decide if it wants to revisit this state or not in the future. So maybe, we should update the reward space as $\R^\star = \R \cup \{ \square \}$, where $\square$ denotes no-reward. It is not a real-value. It is unclear how to define value in this case. The agent has to replace $\square$ (maybe) by its intrinsic real-valued reward in order to define value.
\end{itemize}

An ARL formulation based on a \emph{partial} reward function would be an interesting generalizations of the this thesis. An ARL setup, once formulated properly, can greatly help mitigate the difficulties of handling no-reward situations. An abstract state can help ``extrapolate'' the reward signal from one history to another.

\subsection{Hierarchical ARL}

The ARL setup defined in the thesis can be called a ``flat'' architecture. We assume a ``monolithic'' abstraction map. Sometimes, the environment may allow for a hierarchical structure where the level of ``information'' is different at different levels of control. For example, it is tedious to come up with an abstraction map straight from muscle movements to, say, performing a surgery. We do not think at this ``fine'' level of granularity to perform many ``abstract'' tasks. This observation asks for an extension of ARL to a hierarchical setup where the (overall) abstraction map is a composition of some ``finer'' (intermediate) maps.

\subsection{Limit Optimal Policies}

The core idea explored in the thesis is to analyze (and learn) abstractions which admit an optimal policy from $\Pi^{\sup}_\eps$. That is, the optimal policy of the surrogate MDP is  near-optimal at every history. We may gain a lot in terms of reduction of state-space if we relax this condition to limit optimal policies, i.e.\ we could try to find an abstraction whose surrogate MDP can help us find a limit or average optimal policy, i.e.\ $\u\pi \in \overline{\Pi}_\eps^\infty$ or $\u\pi \in \Pi_\eps^\infty$. This extension of the ARL framework should make it easier to find good abstraction maps which make early mistakes.

\section{Conclusion}

Current RL algorithms are often too brittle and demanding to be (directly) used in real-world applications. The theory presented in this thesis tackled the issue of scalability of AGI agents from the perspective of using a compact abstraction of the environment. The key point of the thesis is that we can leverage non-MDP abstractions to learn the optimal behavior via a surrogate MDP. Especially, we could do this for a broad range of environments. The typical requirement for the abstraction maps to be MDPs limits the scope of environments one can cater for with a fixed state-space. An MDP abstraction would eventually start behaving as a non-Markovian map as the complexity (e.g.\ the underlying state-space) of the environment increases. We showed the size of the state-space is upper bounded for \emph{any} environment if we allow for non-MDP abstractions (e.g.\ extreme $\eps$-QDP). On top of that, these non-MDP abstractions admit surrogate MDPs, so we can learn the optimal behavior without an extra overhead of planning with a non-MDP model, \emph{cf.} planning with POMDPs. Our results can help build resource bounded \emph{scalable} AGI agents. The ARL framework is a stepping stone in this direction!


\appendix



\chapter{Symbols \& Notation}\label{chap:notation-and-symbols}

This chapter provides a comprehensive list of symbols and notation used in the main text.

\section*{Number Sets}

\begin{notation}
    \item[$\SetN$] The set of natural numbers (starting from 1) $\{1, 2, \dots \}$%
    \item[$\SetB$] The set of binary symbols $\{0, 1\}$%
    \item[$\SetR$] The set of real numbers%
\end{notation}

\section*{Common Notation}

\begin{notation}
    \item[$\PowSet(X)$] The power set of $X$, i.e.\ $\PowSet(X) \coloneqq \{ A \mid A \subset X\}$%
    \item[$X \times Y$] Cartesian product of $X$ and $Y$, i.e.\ $X \times Y \coloneqq \{xy \mid x \in X, y \in Y\}$%
    \item[$\Dist(X)$] The set of probability distributions over $X$%
    \item[$\v x$] A finite vector of length $|\v x| < \infty$%
    \item[$\v x\trp$] Transpose of a vector (or a matrix) $\v x$%
    \item[$\d x$] A different member from a set $X$, i.e.\ $x, \d x \in X$%
    \item[$x'$] The next symbol in the sequence, i.e.\ if $x = x_n$ then $x' \coloneqq x_{n+1}$%
    \item[$\t x$] A local variable%
    \item[$\norm{\v x}$] A norm of $\v x$ whose nature (e.g.\ sup or weighted norm) is apparent from the context
    \item[{$\norm[\infty]{\v x}$}] The sup-norm of $\v x$
    \item[{$\norm[\v w]{\v x}$}] The $\v w$-weighted norm of $\v x$
    \item[{$f[\v \theta]$}] A $\v\theta$-parametrized function $f \equiv f(X \| \v \theta)$ over some space $X$, which is apparent from the context
    \item[{$\rm uniform(X)$}] The uniform distribution over $X$
\end{notation}

\section*{General Reinforcement Learning}

\begin{notation}
    \item[$\mu$] The true environment%
    \item[$\A$] The (finite) set of (original) actions%
    \item[$\O$] The (finite) set of observations%
    \item[$\OR$] The (finite) set of percepts%
    \item[$\R$] The (finite) set of rewards $\R \subseteq \SetR$%
    \item[$\H_n$] The set of all histories of length $n$, $\H_n \coloneqq (\OR \times \A)^{n-1} \times \OR$%
    \item[$\H$] The set of all finite histories, i.e.\ $\H \coloneqq \cup_{n \in \SetN} \H_n$%
    \item[$\H_\infty$] The set of all infinite histories%
    \item[$\g$] The discount factor for the true environment%
    \item[$V_\mu^\pi$] The history-value function of $\mu$ on policy $\pi$%
    \item[$Q_\mu^\pi$] The history-based action-value function of $\mu$ on policy $\pi$%
    \item[$V_\mu^*$] The optimal history-value function of $\mu$
    \item[$Q_\mu^*$] The optimal history-based action-value function of $\mu$%
    \item[$\pi_\mu$] The randomized optimal policy of $\mu$, i.e.\ $\pi_\mu[h] \coloneqq {\rm uniform}\left(\argmax_{a \in \A} Q^*_\mu(ha)\right)$ for any history $h \in \H$%
\end{notation}

\section*{Abstraction Reinforcement Learning}

\begin{notation}
    \item[$\nu$] A model of the true environment $\mu$%
    \item[$\U$] The universe of all causal models%
    \item[$\M$] A class of models $\M \subseteq \U$%
    \item[$\S$] The (finite) set of states%
    \item[$\B$] The (finite) set of abstract actions%
    \item[$\psi$] An abstraction of the true environment%
    \item[$\mu_\psi$] The abstract process induced by $\psi$%
    \item[$B$] A stochastic inverse of the abstraction%
    \item[$\H^\psi_n$] The set of all abstract histories of length $n$%
    \item[$\H^\psi$] The set of all finite abstract histories%
    \item[$\H^\psi_\infty$] The set of all infinite abstract histories%
    \item[$\lambda$] The discount factor for the abstract environment%
    \item[$v_\mu^\pi$] The state-value function of $\mu$ on policy $\pi$%
    \item[$q_\mu^\pi$] The state-action-value function of $\mu$ on policy $\pi$%
\end{notation}

\section*{Surrogate Markov Decision Process}

\begin{notation}
    \item[$\b \mu_{\rm MDP}$] A surrogate-MDP of $\mu$%
    \item[$\b q_\mu^\pi$] The action-value function of the surrogate-MDP on policy $\pi$%
    \item[$\b v_\mu^\pi$] The state-value function of the surrogate-MDP on policy $\pi$%
    \item[$\b q_\mu^*$] The optimal action-value function of the surrogate-MDP%
    \item[$\b v_\mu^*$] The optimal state-value function of the surrogate-MDP%
    \item[$\u \pi$] The uplifted randomized optimal policy of the surrogate-MDP%
\end{notation}

\section*{State-Action Homomorphism}

\begin{notation}
    \item[$\eps$] The small positive error constant%
    \item[$\DPi$] The maximum variation among abstracted policy members
    \item[$\DQ$] The maximum variation among abstracted action value members
    \item[$\psi^{-1}_b(s)$] The set of histories mapped to $(s,b)$ pair
    \item[$\psi_s(b)$] The history-dependent set of (original) actions mapped to an $(s,b)$ pair
    \item[$B^\pi$] $B$ and $\pi$ induced measure on the original action space
    \item[$\langle \cdot \rangle_B$] $B$ average
\end{notation}





\backmatter

\printbibliography



\end{document}